\pdfoutput=1

\documentclass[11pt]{article}

\usepackage[letterpaper,margin=1in]{geometry}

\usepackage[utf8]{inputenc} \usepackage[T1]{fontenc}    \usepackage{url}            \usepackage{booktabs}       \usepackage{amsfonts}       \usepackage{nicefrac}       \usepackage{microtype}      \usepackage{xcolor}         \usepackage{lipsum}
\usepackage{subfig}
\usepackage[shortlabels]{enumitem}
\usepackage{xspace}

\usepackage{tikz}
\usepackage{tabularx}
\usepackage{booktabs}  \usepackage{array}     

\usepackage{amsthm,amsmath,amssymb,amsfonts,amssymb,mathtools}
\usepackage{xcolor}
\usepackage{empheq}
\usepackage{dsfont}
\usepackage{mathrsfs}
\usepackage{graphicx}
\usepackage{enumitem}
\usepackage{subfig}
\usepackage{comment}
\usepackage{framed}
\usepackage{float}
\usepackage{algorithm2e}
\usepackage{needspace}
\usepackage[noend]{algpseudocode}
\usepackage{thm-restate}
\usetikzlibrary{backgrounds}
\pgfdeclarelayer{background}
\pgfdeclarelayer{foreground}
\pgfsetlayers{background,main,foreground}

\SetAlCapFnt{\normalfont\small}\SetAlCapNameFnt{\unskip\itshape\small}

\definecolor[named]{ACMBlue}{cmyk}{1,0.1,0,0.1}
\definecolor[named]{ACMYellow}{cmyk}{0,0.16,1,0}
\definecolor[named]{ACMOrange}{cmyk}{0,0.42,1,0.01}
\definecolor[named]{ACMRed}{cmyk}{0,0.90,0.86,0}
\definecolor[named]{ACMLightBlue}{cmyk}{0.49,0.01,0,0}
\definecolor[named]{ACMGreen}{cmyk}{0.20,0,1,0.19}
\definecolor[named]{ACMPurple}{cmyk}{0.55,1,0,0.15}
\definecolor[named]{ACMDarkBlue}{cmyk}{1,0.58,0,0.21}

\usepackage[colorlinks,citecolor=blue,linkcolor=magenta,bookmarks=true]{hyperref}
\usepackage[nameinlink]{cleveref}
\crefname{ineq}{Inequality}{Inequality}
\creflabelformat{ineq}{#2{\upshape(#1)}#3}
\crefname{sub}{Subsection}{Subsection}
\creflabelformat{Subsection}{#2{\upshape(#1)}#3}
\crefname{sdp}{SDP}{SDP}
\creflabelformat{sdp}{#2{\upshape(#1)}#3}
\crefname{lp}{LP}{LP}
\creflabelformat{lp}{#2{\upshape(#1)}#3}

\newtheorem{theorem}{Theorem}[section]
\newtheorem{lemma}[theorem]{Lemma}
\newtheorem{proposition}[theorem]{Proposition}

\newtheorem{claim}[theorem]{Claim}

\newtheorem{definition}[theorem]{Definition}

\newtheorem{fact}[theorem]{Fact}

\theoremstyle{definition}

\newtheorem{remark}[theorem]{Remark}

\crefname{claim}{claim}{claims}

\newcommand{\CS}{Cauchy-Schwarz\xspace}

\newcommand\norm[1]{\left\| #1 \right\|}

\renewcommand\vec[1]{\mathbf{#1}}
\DeclareMathOperator*{\pr}{\mathbf{Pr}}
\renewcommand{\Pr}{\mathbf{Pr}}
\DeclareMathOperator*{\E}{\mathbf{E}}
\newcommand{\proj}{\mathrm{proj}}

\def\d{\mathrm{d}}
\newcommand{\normal}{\mathcal{N}}

\DeclareMathOperator*{\argmin}{argmin}
\DeclareMathOperator*{\argmax}{argmax}

\newcommand{\tr}{\mathrm{tr}}
\newcommand{\bx}{\mathbf{x}}
\newcommand{\by}{\mathbf{y}}

\newcommand{\e}{\mathbf{e}}

\newcommand{\R}{\mathbb{R}}

\newcommand{\Z}{\mathbb{Z}}
\newcommand{\N}{\mathbb{N}}
\newcommand{\eps}{\epsilon}

\newcommand{\poly}{\mathrm{poly}}

\newcommand{\sign}{\mathrm{sign}}
\newcommand{\calN}{{\cal N}}

\newcommand{\opt}{\mathrm{OPT}}
\newcommand{\D}{\mathcal{D}}

\newcommand{\Ind}{\mathds{1}}
\newcommand{\1}{\Ind}

\newcommand{\littlesum}{\mathop{\textstyle \sum}}

\newcommand{\wt}{\widetilde}

\newcommand{\spaning}{\mathrm{span}}

\newcommand{\x}{\vec x}
\newcommand{\z}{\vec z}
\newcommand{\w}{\vec w}

\newcommand{\hide}[1]{}
\DeclareMathOperator{\supp}{supp}

\newcommand{\cB}{\mathcal{B}}

\newcommand{\cD}{\mathcal{D}}
\newcommand{\cE}{\mathcal{E}}
\newcommand{\cF}{\mathcal{F}}
\newcommand{\cG}{\mathcal{G}}

\newcommand{\cL}{\mathcal{L}}

\newcommand{\cN}{\mathcal{N}}

\newcommand{\cQ}{\mathcal{Q}}

\newcommand{\cY}{\mathcal{Y}}

\newcommand{\eqdef}{\coloneqq}

\newcommand{\abs}[1]{\lvert#1\rvert}
\newcommand{\Abs}[1]{\left\lvert#1\right\rvert}
\usepackage{multirow}

\def\colorful{0}

\ifnum\colorful=1

\else

\fi

\allowdisplaybreaks

\title{Robust Learning of Multi-index Models via \\ 
Iterative Subspace Approximation}

\author{
Ilias Diakonikolas\thanks{Supported by NSF Medium Award CCF-2107079 and an H.I. Romnes Faculty Fellowship.}\\
UW Madison\\
{\tt ilias@cs.wisc.edu}\\
\and
Giannis Iakovidis$^*$\\
UW Madison\\
{\tt iakovidis@wisc.edu}\\
\and
Daniel M. Kane\thanks{Supported by NSF Award CCF-1553288 (CAREER) and a Sloan
  Research Fellowship.}\\
UC San Diego\\
{\tt dakane@ucsd.edu }
 \and
Nikos Zarifis$^*$\\
UW Madison\\
{\tt zarifis@wisc.edu}\\
}

\date{}

\begin{document}

\maketitle

\setcounter{page}{0}

\thispagestyle{empty}

\vspace{-0.9cm}

\begin{abstract}
We study the task of learning Multi-Index Models (MIMs) in the presence of 
label noise under the Gaussian distribution. A $K$-MIM on $\mathbb{R}^d$ is any 
function $f$ that only depends on a $K$-dimensional subspace, 
i.e., $f(\mathbf{x}) = g(\mathbf{W} \mathbf{x})$ for a link function $g$ on $\mathbb{R}^K$ 
and a $K\times d$ matrix $\mathbf{W}$. We consider a class of 
well-behaved MIMs with finite ranges that satisfy certain regularity properties. Our main contribution is a general noise-tolerant 
learning algorithm for this class whose complexity is qualitatively optimal in the Statistical Query (SQ) model.  At a high-level, 
our algorithm attempts to iteratively construct 
better approximations to the defining subspace 
by computing low-degree moments of our function 
conditional on its projection to the subspace computed thus far, 
and adding directions with relatively large empirical moments.
For well-behaved MIMs, we show that this procedure 
efficiently finds a subspace $V$ so that $f(\mathbf{x})$ is close to a 
function of the projection of $\mathbf{x}$ onto $V$, which can then be found 
by brute-force. 
Conversely, for functions for which these conditional moments do not 
necessarily help in finding better subspaces, 
we prove an SQ lower bound providing evidence that no 
efficient algorithm exists.

As concrete applications of our general algorithm, 
we provide significantly faster noise-tolerant learners 
for two well-studied concept classes:
\begin{itemize}[leftmargin=*, nosep]
    \item {\bf Multiclass Linear Classifiers} 
    A multiclass linear classifier 
    is any function $f: \mathbb{R}^d \to [K]$ of the form 
    $f(\mathbf{x}) = \argmax_{i\in [K]}( \mathbf{w}^{(i)} \cdot \mathbf{x}  + t_i)\;,$
    where $\mathbf{w}^{(i)}  \in \mathbb{R}^d$ and $t_i \in \mathbb{R}$. We give a {\em constant-factor} 
    approximate agnostic learner for this class, i.e., an algorithm that 
    achieves 0-1 error $O(\mathrm{OPT})+\epsilon$. Our algorithm has sample complexity 
    $N = O(d) \, 2^{\mathrm{poly}(K/\epsilon)}$ and computational complexity 
    $\mathrm{poly}(N)$. This is the first constant-factor agnostic learner for 
    this class whose complexity is a fixed-degree polynomial in $d$. 
    In the agnostic model, it was previously known that 
    achieving error $\opt+\eps$ requires time $d^{\poly(1/\eps)}$, even for $K=2$. 
    Perhaps surprisingly, we prove an SQ lower bound showing that achieving error $\opt+\eps$, for $\eps = 1/\poly(K)$, 
    incurs complexity $d^{\Omega(K)}$
    even for the simpler case of Random Classification Noise.
 
    \item {\bf Intersections of Halfspaces} An intersection of $K$ 
    halfspaces is any function $f:\mathbb{R}^d\to\{\pm 1\}$ such that there exist 
    $K$ halfspaces $h_i(\mathbf{x})$ with $f(\mathbf{x}) = 1$ if and only if $h_i(\mathbf{x}) = 1$ 
    for all $i \in [K]$. We give an approximate agnostic learner for this 
    class achieving 0-1 error $K \, \tilde{O}(\mathrm{OPT}) + \epsilon$. 
    Our algorithm has  
    sample complexity $N=O(d^2) \, 2^{\mathrm{poly}(K/\epsilon)}$ and computational 
    complexity  $\mathrm{poly}(N)$. This is the first agnostic learner for 
    this class with {\em near-optimal} dependence on $\mathrm{OPT}$ in its error, whose complexity is a fixed-degree polynomial in $d$. 
    Previous algorithms 
    either achieved significantly worse error guarantees, 
    or incurred $d^{\poly(1/\eps)}$ time (even for $K=2$).
\end{itemize}
Furthermore, we show that in the presence of random classification noise, the complexity of our algorithm is significantly better, scaling polynomially with $1/\epsilon$. 
\end{abstract}

\setcounter{page}{0}
\thispagestyle{empty}
\newpage

\section{Introduction}

In this paper, we develop a new algorithm for learning a broad family of multi-index models (MIMs)
under the Gaussian distribution 
in the presence of label noise. 
Before we describe our general setting and algorithm, we start 
with a specific basic model that was the original motivation for this work.

\paragraph{Motivation: Multiclass Linear Classification} 
A multiclass linear classifier is a function 
$f: \R^d \to [K]$ of the form
$f(\x) = \argmax_{i\in [K]}( \w^{(i)} \cdot \x  + t_i)\;,$
where $\w^{(i)}  \in \R^d$ and $t_i \in \R$ 
(ties are resolved by selecting the smallest index).
Multiclass linear classification (MLC) is a textbook machine learning problem~\cite{SB-book}, 
that has also been extensively studied in the online learning literature; 
see, e.g.,~\cite{banditron} and references therein. 
In the realizable PAC setting (i.e., with clean labels), this task is
efficiently solvable to any desired error $\eps$ 
with sample size $O(dK/\eps)$ 
{via linear programming (or via the multiclass Perceptron algorithm for well-behaved marginal distributions, including the Gaussian)}. 
On the other hand, despite extensive effort, 
see, e.g.,~\cite{NDRT13,WLT18, Lipton18a, Song-etal-survey20, GWBL20}, 
our understanding of the computational complexity of learning 
is poor in the noisy setting---even for the simplest case of 
{\em random} label noise.

The special case of {\em binary} classification of Linear Threshold Functions (LTFs) 
or halfspaces---corresponding to $K=2$---is fairly well-understood. 
Specifically, sample and computationally efficient learning algorithms 
are known, even in the presence of partially corrupted 
labels~\cite{ABL17,DKS18-nasty,DGT19,DKTZ20, 
DKTZ20c, DiakonikolasKKT21, DKKTZ22}. Unfortunately, 
these prior approaches are specific to the binary case. {In particular, 
it is unclear
whether they can be generalized to give any non-trivial guarantees 
for the noisy multiclass setting.}

Perhaps the simplest model of random label noise is that of Random Classification Noise (RCN)~\cite{AL88}. 
In the binary classification setting, the RCN model posits that the label of each example 
is independently flipped with probability exactly $\eta$, for some noise parameter $0 < \eta < 1/2$. 
A slight generalization of the noise model, that is essentially equivalent, 
allows for different probabilities $\eta_+, \eta_- \in (0,1/2)$ 
of flipping the label of positive and negative examples. 
A classical result~\cite{Kearns:98} shows that, in the binary classification setting, 
any efficient Statistical Query (SQ) algorithm can be transformed into an efficient 
RCN-tolerant PAC learner (with a polynomial dependence in all relevant parameters). 
Unfortunately, no such result is known for the multiclass setting\footnote{{In fact, \Cref{thm-intro:SQ-RCN} implies that 
such a generalization is impossible 
without strong assumptions on the noise matrix.}}.

The standard generalization of the RCN model to the multiclass setting 
is the following \cite{PRMNQ17,zhang21k,VanRooyencorrupted}. 
Let $D$ be the distribution 
on noisy samples $(\x, y) \in \R^d \times [K]$, 
where in our setting 
the marginal distribution $D_{\x}$ is the standard Gaussian, 
and $f: \R^d \to [K]$ be the ground truth classifier. 
Then there exists 
a row-stochastic matrix $\vec{H} \in \R_+^{K \times K}$ 
such that 
$ \vec{H}_{i, j} = \pr_{(\x, y) \sim D}[y = i \mid f(\vec{x}) = j]$. 
Moreover, it is additionally assumed that there 
exists $\gamma > 0$ 
such that $\vec{H}_{i, i} \ge \vec{H}_{i, j} + \gamma$, 
for all $i \neq j \in [K]$\footnote{Roughly speaking, the 
parameter $\gamma>0$ plays the role of $1-2\eta$ in the binary 
case.}. 
Under these assumptions, MLC is information-theoretically 
solvable using $O(dK/(\gamma\eps))$ 
samples. On the other hand, no $\poly(d, K)$ time learning algorithm is known. 
Prior algorithmic work using gradient-based methods to achieve 
sample and time complexity that scales polynomially with $1/\sigma$, 
where $\sigma$ is the minimum singular value of $\vec{H}$ (see \cite{Lipton18a}). 
This quantity could be arbitrarily small or even zero, 
hence such approaches do not
in general lead to a computationally efficient algorithm. 
This discussion leads to the following question:
\begin{center}
{\em Is there a $\poly(d, K)$ time algorithm for multiclass linear classification \\
in the presence of Random Classification Noise?}
\end{center}

Perhaps surprisingly, we provide strong evidence that the answer to the above question is negative. 
Specifically, we establish a Statistical Query (SQ) lower bound showing that the complexity 
of any SQ algorithm for this task is $d^{\Omega(K)}$ (see \Cref{thm-intro:SQ-RCN}).

Interestingly, our negative result applies for learning algorithms that achieve optimal error guarantees 
(or, equivalently, are able to approximate the target classifier to any desired accuracy). On the other hand, 
more efficient algorithms (e.g., with a fixed $\poly(d)$ dependence on the runtime) 
may be possible if one is allowed to achieve {\em near-optimal} error, e.g., 
up to a constant factor. As an application of our new methodology, we show that this is indeed possible (see \Cref{thm-intro:AlgRCN-linear}). 

An {even more challenging goal} is to understand the algorithmic possibilities and limitations of learning MLCs
in the presence of {\em adversarial} label noise. In the adversarial label noise model, an adversary is
allowed to arbitrarily corrupt an unknown $\opt$ fraction of the labels, for some parameter $\opt <1/2$,
and the goal of the learner is to compute a hypothesis whose 0-1 error is competitive with $\opt$ 
(see \Cref{def:agnostic}). The information-theoretically optimal error guarantee in this model is $\opt+\eps$. 
Unfortunately, even for $K=2$, known computational lower bounds rule out the existence of such algorithms 
with complexity better than 
$d^{\poly(1/\eps)}$~\cite{DiakonikolasKPZ21,DKR23}. 

Consequently, to achieve a fixed polynomial runtime in $d$, one needs
to relax the error guarantee. For {the special case} $K=2$, 
{such} algorithms are known. 
{In particular,~\cite{DKS18-nasty} gave a $\poly(d/\eps)$ time 
algorithm with error guarantee $C \cdot \opt+\eps$, 
where $C$ is a universal constant. }

A natural question is whether qualitatively similar
algorithms exist for the multiclass setting. Specifically, a first 
goal is to establish a fixed-parameter tractability guarantee:
\begin{center}
{\em 
Is there a {\em constant-factor} approximate learner for MLC with complexity $\poly(d) f(\eps, K)$?}
\end{center}
As a corollary of our approach, 
we answer this question in the affirmative 
(see \Cref{thm-intro:SimplerAlgMulticlassClass}). 

\paragraph{Broader View: Learning Multi-Index Models} 
The class of multiclass linear classifiers discussed above is a (very) special case 
of a much broader class of functions, commonly known as Multi-Index Models (MIMs). 
A function is called a MIM if it depends only on {the projection onto} a ``low-dimensional'' subspace. 
Here we focus on MIMs whose range $\cal Y$ is a finite set. Formally, we have: 

\begin{definition}[Multi-Index Model] \label{def:mim}
A function $f: \R^d \to \cal Y$ is called a $K$-MIM 
if there exists a subspace $U \subseteq \R^d$ of dimension at most $K$ 
such that $f$ depends only on {the projection onto} $U$, i.e.,  $f(\x) = f ( \proj_U \x)$,  
where $\proj_U \x$ is the projection of $\x$ on $U$. Equivalently, 
$f$ is of the form $f(\x) = g(\mathbf{W} \x)$ for a link function $g: \R^K \to \cal{Y}$ 
and a $K \times d$ weight matrix $\mathbf{W}$.
\end{definition} 

{In typical applications}, 
the dimension $K$ of the hidden subspace is 
usually significantly smaller than the ambient dimension  $d$. 

{
In regards to the problem of learning MIMs from labeled examples, 
there two qualitatively different versions: 
the link function $g$ in \Cref{def:mim} 
could be a priori known or unknown to the learner. 
In the case of unknown $g$, 
it is typically assumed that it satisfies 
certain regularity conditions; 
as without any assumptions on $g$, 
learning is information-theoretically impossible. 
In this work, we focus on the setting where $g$ is unknown.}

Multi-index models~\cite{Friedman:1980tu, Huber85-pp, Li91, 
HL93,xia2002adaptive, Xia08} capture the common assumption that 
real-world data are inherently structured, in the sense that 
the relationship between the features is of a (hidden) lower-dimensional nature. 
MIMs can be viewed as a lens for studying a range of ubiquitous learning models, 
including neural networks, intersections of halfspaces, and multiclass classifiers.
Despite extensive prior investigation, there is still a vast 
gap in our understanding of the complexity of learning MIMs, 
in particular in the presence of noisy data. This discussion motivates the 
general direction addressed by this work:
\begin{center}
{\em  What is the complexity of learning general classes of MIMs 
in the presence of noisy data?   
}
\end{center}
Our main contribution makes progress towards this research direction.
Specifically, we propose a condition that defines a fairly general 
class of MIMs---that in particular includes multiclass linear classifiers and intersections of halfspaces---and qualitatively characterize its efficient learnability
in the presence of noise. That is, we give an efficient noise-tolerant learner 
for this class of MIMs and a qualitatively matching SQ lower bound. 
As a corollary, our general algorithm yields significantly faster 
noise-tolerant learners for both aforementioned concept classes.

\subsection{Our Results} \label{ssec:results}

In this section, we present our main results.
{For a concise overview of our algorithmic contributions, see Table~\ref{table:upperbounds}.}

\begin{table}[ht]
\centering
\setlength{\tabcolsep}{1.6pt} \begin{tabular}{|l|c|c|c|c|}
\hline
\textbf{Function Class} 
& \multicolumn{2}{c|}{\textbf{Prior Work}} 
& \multicolumn{2}{c|}{\textbf{Our Work}} \\
\hline
 & \textbf{Runtime} & \textbf{Error} 
 & \textbf{Runtime} & \textbf{Error} \\
\hline
Agnostic $K$-MLC 
& $d^{\poly(K/\epsilon)}$ 
& $\opt+\eps$
& \shortstack{(\Cref{thm-intro:SimplerAlgMulticlassClass})\\$\poly(d)\,2^{\poly(K/\epsilon)}$ 
}
& $O(\opt)+\eps$ \\
\hline
$K$-MLC with RCN 
& $d^{\poly(K/\epsilon)}$ 
& $\opt+\eps$
& \shortstack{(\Cref{thm-intro:AlgRCN-linear})\\$\poly(d) (1/\eps)^{\poly(K)}$ }
& $O(\opt)+\eps$ \\
\hline
\shortstack{Agnostic Intersections \\ of $K$ halfspaces} 
& \shortstack{$d^{O(\log(K)/\eps^4)}$ \\ 
$\poly(d)/\eps^{\poly(K)}$} 
& \shortstack{$\opt+\eps$ \\ 
$\poly(K)\,\tilde{O}(\opt^{1/11})$} 
& \shortstack{(\Cref{thm:SimpleLearningIntersections})\\$\poly(d)\,2^{\poly(K/\epsilon)}$ }
& $K\,\tilde{O}(\opt)+\eps$ \\  
\hline
\shortstack{Well-Behaved $K$-MIMs \\ 
(\Cref{def:SimpleGoodCondition})}
& $d^{\poly(\Gamma/\epsilon)}$ 
& $\opt+\eps$
& \shortstack{(\Cref{thm:SimpleMetaTheorem})\\$d^{O(m)}\,2^{\poly(mK\Gamma/\eps)}$}
& $\tau+\opt+\eps$ \\
\hline
\end{tabular}
\caption{Comparison of runtime and error guarantees between state-of-the-art efficient algorithms 
  (see \cite{KOS:08,DKS18-nasty} for intersections of $K$ halfspaces) and our subspace approximation procedure.}
\label{table:upperbounds}
\end{table}

\vspace{-0.3cm}

\subsubsection{Learning Well-behaved Multi-Index Models with Noise} 
Our general algorithmic result {(\Cref{thm:SimpleMetaTheorem})} 
applies to MIMs with finite range 
that satisfy a collection of regularity properties {(\Cref{def:SimpleGoodCondition})}. 

Throughout this section, we will use $d$ for the dimension of the feature space and $K$ for the dimension of 
the hidden low-dimensional subspace. 
A standard assumption in the 
literature on MIMs {(see, e.g.,~\cite{KOS:08, DKKTZfocs24})}, satisfied by a wide range of natural 
concept classes, 
is that the underlying function has bounded Gaussian 
surface area (GSA)---a complexity measure of the decision 
boundary (see \Cref{def:GSA} for the definition in the 
multiclass case). We denote by $\Gamma$ the upper bound 
on the GSA of the underlying concept class.

{The structural property that we will assume for every 
MIM function $f$ in our class is, roughly speaking, 
the following: 
it is not possible to add noise to $f$ to conceal crucial 
directions so that they remain invisible to moment-based 
detection. Specifically, for {\em any} subspace $V$, 
one of the following must hold:
\begin{itemize}[nosep]
\item[(a)] $V$ already captures all essential information of $f$, in 
which case there is a function depending solely on the projection onto 
$V$ that approximates $f$ well. 
\item[(b)] Otherwise, there is at least one point  
in $V$ such that when we condition on $\x^V$ near that point, the 
conditional distribution of a ``noisy'' version of $f$ 
exhibits a distinguishing moment of degree at most $m$. 
That is, there exists a {degree-$m$} polynomial whose conditional 
expectation, conditioned on $\x^V$ {(the projection of $\x$ on $V$)}, cannot remain hidden; and 
hence reveals a dependence on directions outside the subspace $V$.
\end{itemize}
{The parameter $m$ above turns out to essentially characterize the 
complexity of the learning problem. We show, qualitatively speaking, 
that there exists an algorithm with complexity $d^m$; and that 
this bound is essentially best possible in the SQ model.}

We now present an informal definition of this property and refer the reader to \Cref{def:GoodCondition} for a more precise formulation.}

\begin{definition}[Well-Behaved Multi-Index Model]\label{def:SimpleGoodCondition}
Fix $m\in \mathbb{Z}_+,\zeta, \tau\in (0,1)$, {$\Gamma>0$,} 
and $\mathcal{Y}\subseteq \Z$ a set of finite cardinality.  
We say that a $K$-MIM $f: \R^d \to \mathcal{Y}$ 
is $(m,\zeta,\tau, \Gamma)$-well-behaved 
if {$f$ has Gaussian surface area at most $\Gamma$, 
and} for any random variable $y(\x)$ supported on $\mathcal{Y}$ 
with $\pr_{\x\sim \mathcal N(\vec 0,\vec I)}[f(\x)\neq y(\x)]\leq \zeta$ and any subspace $V\subseteq \R^d$, one of the following conditions holds:
\begin{enumerate}[leftmargin=*]
    \item[a)]  there exists a function $g:V\to \mathcal{Y}$ such that $\pr_{\x\sim \mathcal N(\vec 0,\vec I)}[g(\x^{V})\neq f(\x) ]\le \tau$, or
    \item[b)] with non-negligible probability\footnote{By this we mean 
    that the probability is not-too-small, specifically it is independent 
    of the dimension $d$.} over $\vec x_0 \sim\cN(\vec 0,\vec I)$, 
there exists a mean-zero, variance-one polynomial 
    $p:U\to \R$  of degree at most $m$ and $z\in \mathcal{Y}$ 
    such that 
    $$\E_{\x\sim \mathcal N(\vec 0,\vec I)}[p(\x^U)\Ind(y(\x)=z)\mid \x^V=\x_0^V] \geq \poly(\zeta/m)\;,$$
 where $U=(V+W)\cap V^\perp$ {is the projection of $W$ onto $V^\perp$} 
 and $W$ is the subspace defining $f$. 
\end{enumerate}
We denote by $\mathcal{F}$ the class of $(m,\zeta,\tau, \Gamma)$-well-behaved $K$-MIMs over $\R^d$.
\end{definition}

\noindent {The aforementioned property enforces that no direction 
crucial to $f$ can be fully hidden when 
we look at a suitable low-degree moment. 
Intuitively, the parameter $\zeta$ in 
\Cref{def:SimpleGoodCondition} will determine the 
noise-tolerance of our learning algorithm, 
while the parameter $\tau$ will determine 
the final error guarantee of the algorithm. 
The parameter $m$ is the minimum degree of a non-zero 
moment on the subspace $U$. It is worth pointing out 
that the sample and computational complexity 
of our algorithm will scale with $d^m$. 

{
\begin{remark}\label{rem:gen-exp-intro}
\Cref{def:SimpleGoodCondition} implicitly introduces 
a complexity measure for learning a family of 
discrete MIMs. Let $\mathcal{C}$ 
be a family of MIMs with Gaussian Surface Area 
bounded by $\Gamma$, and let $\zeta,\tau>0$ 
be parameters. Define $m^\ast$ to be the {\em smallest integer} for which every 
$f\in \mathcal{C}$ is 
$(m^\ast,\zeta,\tau,\Gamma)$-well-behaved.
It turns out that the quantity $m^\ast$ 
is (essentially) a generalization of 
the generative 
exponent~\cite{damian2024generativeexponent}, 
a measure characterizing the SQ complexity of 
learning Single-Index Models (SIMs). For further 
details, see \Cref{app:compl-measures}. 
\end{remark}
}

As we will show in the next subsection, fundamental and well-
studied concept classes---specifically, multiclass linear 
classifiers and intersections of halfspaces---satisfy 
\Cref{def:SimpleGoodCondition} for reasonable values of the 
underlying parameters (in particular, with $m \leq 2$).}

Before we state our main algorithmic result, we recall the definition of 
PAC learning with adversarial label noise (aka agnostic learning) under 
the Gaussian distribution. 

\begin{definition}[Agnostic Multiclass Classification] 
\label{def:agnostic} 
Let $\mathcal Y$ be finite and $\cal{C}$ a class of functions 
from $\R^d$ to $\mathcal Y$.
Given i.i.d.\ samples $(\x,y)$ 
from a distribution $D$ on $\R^d\times\mathcal Y$ 
such that the marginal $D_{\x}$ is the standard normal 
and no assumptions are made on the labels, 
the goal is to output a hypothesis 
$h: \R^d \to \mathcal Y$ that, with high probability, 
has small 0-1 error 
$\mathrm{err}_{0-1}^D(h)\eqdef \pr_{(\x,y) \sim D}[h(\x) \neq y]$, 
as compared to the optimal 0-1 error of the class ${\cal C}$, i.e., 
$\opt\eqdef \inf_{c\in \mathcal C}\mathrm{err}_{0-1}^D(c)$.
\end{definition}

Our main result is a general algorithm for learning well-behaved MIMs 
in the presence of adversarial label noise (see \Cref{thm:MetaTheorem}).

\begin{theorem}[Learning Well-Behaved MIMs with Adversarial Label Noise]\label{thm:SimpleMetaTheorem}
Let $\mathcal{F}$ be the class of $(m,\zeta,\tau, \Gamma)$-well-behaved $K$-MIMs on $\R^d$ of \Cref{def:SimpleGoodCondition}.  
If $\zeta\geq \opt+\eps$, there exists an agnostic algorithm that draws 
$N =d^{m}\, 2^{\poly(mK\Gamma/\eps)}$ i.i.d.\ samples, it  
runs in $\poly(d,N)$ time, and computes a hypothesis $h$ such that 
with high probability $\mathrm{err}_{0-1}^D(h) \leq \tau+\opt+\eps$.
\end{theorem}

\noindent 
At a high-level, the main advantage of the above algorithm 
is that---as long as $m = O(1)$---its complexity 
is bounded by a {\em constant-degree 
polynomial in $d$}. 
In other words, we give a fixed-parameter
tractability result for the task of learning well-behaved MIMs. 
While the complexity guarantee we obtain is new even in the realizable setting, 
our algorithm is additionally robust to adversarial label noise---sometimes 
achieving near-optimal noise-tolerance (within constant factors). 
In addition to its own merits, the significance of our general 
algorithm is illustrated by its downstream 
applications (given in the following subsection).

An interesting feature of \Cref{thm:SimpleMetaTheorem} is that 
the complexity upper bound it establishes 
is qualitatively tight in the Statistical Query (SQ) model, 
as a function of the dimension $d$. Roughly speaking, 
if a function class does not satisfy the assumptions of 
\Cref{def:SimpleGoodCondition} for a specific value of $m$, 
then any SQ algorithm that achieves error competitive with 
\Cref{thm:SimpleMetaTheorem} 
must have SQ complexity $d^{\Omega(m)}$. 
In more detail, we show (see \Cref{thm:generalSQ} for a more detailed formal statement).

\begin{theorem}[SQ Lower Bound for Learning Not-Well-behaved MIMs, Informal]
\label{thm:simplegeneralSQ}
Suppose that a function class $\mathcal G$ is rotationally invariant and contains a bounded surface area MIM
not contained in the function class $\mathcal{F}$ of \Cref{def:SimpleGoodCondition} for parameters $m,\tau, \zeta$.
Then, any SQ algorithm that learns $\mathcal{G}$ 
{in the presence of $\zeta$ label noise} 
to error smaller than \(\tau -O(\zeta) \) 
has complexity at least $d^{\Omega(m)}$.
\end{theorem}

{
\begin{remark}[Characterization] \label{rem:char}
Theorems \ref{thm:SimpleMetaTheorem} and \ref{thm:simplegeneralSQ} together 
essentially characterize the complexity of learning a MIM family. In particular, let $\mathcal{C}$ 
be a rotationally invariant family 
of MIMs with Gaussian Surface Area 
bounded by $\Gamma$, and let $\zeta,\tau>0$ 
be parameters. Define $m^\ast$ to be the {\em smallest integer} for which every 
$f\in \mathcal{C}$ is 
$(m^\ast,\zeta,\tau,\Gamma)$-well-behaved.

Then there is an SQ algorithm that learns a (slightly 
less than) $\zeta$-noisy function from $\mathcal{C}$ 
to error $\tau+O(\zeta)$ with resources 
$d^{m^\ast}$ (times some function of the other 
parameters), but not an SQ algorithm that learns a 
$\zeta$-noisy function from $\mathcal{C}$ to error 
$\tau-O(\zeta)$ with resources $d^{o(m^\ast)}$. 
\end{remark}
}

Finally, we note that a variant of our general algorithm 
learns well-behaved MIMs in the presence of RCN with significantly 
improved complexity. Specifically, for the RCN setting, our algorithm  
has complexity scaling {\em polynomially} in $1/\eps$. See 
\Cref{thm:MetaAlgRCN} for the formal statement. By an 
application of the latter result, in the following subsection, 
we obtain the fastest known algorithm for learning multiclass 
linear classifiers in the presence of RCN.

\subsubsection{Concrete Learning Applications to Well-studied Concept Classes} 

In this subsection, we present algorithmic applications of 
our general MIM learner to two fundamental 
and well-studied concept classes.

\paragraph{Multiclass Linear Classifiers}   
We denote by $\mathcal{L}_{d,K}$ the family 
of $K$-class linear classifiers in $\R^d$. 
For the adversarial noise setting, 
we obtain the following result 
(see \Cref{thm:constantApprox}):

\begin{theorem}[Learning $\mathcal{L}_{d,K}$ with Adversarial Label Noise]
\label{thm-intro:SimplerAlgMulticlassClass}
There exists an agnostic learning algorithm that draws $N = d\,  2^{\poly(K/\eps)}$ i.i.d.\ labeled samples, 
runs in $\poly(d,N)$ time, and outputs 
a hypothesis $h$ such that with high probability 
$\mathrm{err}_{0-1}^D(h) \leq O(\opt)+\eps$, where $\opt$ is 
defined with respect to $\cL_{d,K}$.
\end{theorem}

\noindent \Cref{thm-intro:SimplerAlgMulticlassClass} {can be derived} from 
\Cref{thm:SimpleMetaTheorem} via a new structural result that we establish, 
roughly showing that the class $\mathcal{L}_{d,K}$ is $(1,\zeta,O(\zeta),O(K))$-well 
behaved for any $\zeta\in (0,1)$ (see \Cref{lem:progress}). 
{For the sake of exposition, 
in Section~\ref{sec:MulticlassAlgorithm} we give a direct algorithm establishing  \Cref{thm-intro:SimplerAlgMulticlassClass}.} 

\medskip

\noindent 
\Cref{thm-intro:SimplerAlgMulticlassClass} gives the first {\em constant-factor} approximate learner for multiclass linear classification 
with adversarially corrupted labels, 
whose complexity is a {\em fixed-degree polynomial in $d$}.
We provide some remarks to interpret this statement. First, even for $K=2$,
known hardness results imply that achieving error $\opt+\eps$ requires
$d^{\poly(1/\eps)}$ time~\cite{DiakonikolasKPZ21,DKR23}. Hence, to achieve
a fixed-degree $\poly(d)$ complexity guarantee, it is necessary to relax the final accuracy. 
The best thing to hope for here is a constant-factor approximation, i.e., 
a final error of $O(\opt)+\eps$ (like the one we achieve). For the $K=2$ 
case, prior work achieves such an error bound~\cite{DKS18-nasty}\footnote{Earlier work~\cite{ABL17} achieves such a constant-factor approximation for the special case
of homogeneous halfspaces.}; and in fact with complexity $\poly(d/\eps)$. 
On the other hand, {\em no non-trivial approximation algorithm was previously 
known for the multiclass case}. 
Importantly, no polynomial-time learner with error 
$o(d) \, \opt+\eps$ 
is known even for $K=3.$
The crucial difference is that the approximation ratio of our algorithm is an {\em absolute constant}, 
independent of the dimension and the number of classes.

We also investigate the complexity of learning multiclass linear classifiers 
in the presence of RCN. We start by establishing 
an SQ lower bound of $d^{\Omega(K)}$ for 
achieving $\opt +\eps$ error. 
Specifically, we show (see \Cref{thm:RCNLowerBoundK}):

\begin{theorem}[SQ Lower Bound for Learning $\mathcal{L}_{d,K}$ with RCN] \label{thm-intro:SQ-RCN}
For any $\eps\le 1/K^{C}$, where $C$ a sufficiently 
large universal constant, 
any SQ algorithm that learns $\mathcal{L}_{d,K}$ 
under the Gaussian distribution in the presence of RCN 
to error $\opt+\eps$, has complexity at least $d^{\Omega(K)}$.
\end{theorem}

\vspace{-0.4cm}

{
\paragraph{Digression: SQ Lower Bound for MLC with Partial Labels} 
As an additional implication of our SQ lower bound technique 
establishing \Cref{thm-intro:SQ-RCN}, we establish the first 
super-polynomial SQ lower bound for the well-studied task of learning 
multiclass linear classifiers with 
partial (or coarse) labels. This is a classical and widely studied model of misspecification in multi-class classification \cite{cour2011partiallabel, Liu2014supersetlearning, VanRooyencorrupted, Fotakis2021coarselabellearning}.
In this model, instead of observing instance-label pairs 
$(\mathbf{x}^{(i)}, f(\mathbf{x}^{(i)}))_{i=1}^m$, we observe instance-label subsets $(\mathbf{x}^{(i)}, S_i)_{i=1}^m$, where each 
$S_i \subseteq \mathcal{Y}$ and the true label always belongs to the 
observed subset (i.e., $f(\mathbf{x}^{(i)}) \in S_i$).
The goal is to find a hypothesis with error $\eps$ with respect to the 
original labels. A necessary assumption for identifiability in this 
model is that any incorrect label does not appear too frequently in 
the observed subset (as quantified by the parameter $\gamma$ in the 
theorem statement below); 
see \Cref{def:partial-label-distribution} for a formal definition. 
Under this assumption, the sample complexity of ERM for this task 
is $O(dK/(\gamma\epsilon))$ \cite{Liu2014supersetlearning}. 
On the computational front, we establish 
the following SQ lower bound 
(see \Cref{thm:SQ-partial-labels}):

\begin{theorem}[SQ Lower Bound for MLC with Partial Labels]\label{thm:SQ-partial-labels-intro}
Any SQ algorithm that learns $\mathcal{L}_{d,K}$ given  partial labels, with $\gamma= \poly(1/K)$, under the Gaussian distribution to error $O(1/K)$, has complexity $d^{\Omega(K)}$.
 \end{theorem}

\noindent The details are deferred to \Cref{app:SQ-partial}.

}

\medskip

{We now return to our algorithmic contributions.}
Interestingly, using a variant of our general MIM 
algorithm, we obtain an RCN-tolerant learner 
with $O(\opt)+\eps$ error, 
whose complexity is qualitatively better than our agnostic 
learner.  Specifically, we show the following (see \Cref{thm:MetaAlgRCN}):

\begin{theorem}[Learning $\mathcal{L}_{d,K}$ with RCN]\label{thm-intro:AlgRCN-linear}
Assume that the labels follow a distribution obtained by introducing RCN to a function within the class $\mathcal{L}_{d,K}$. 
There is an algorithm that  draws $N= O(d) \,(1 /\eps)^{\poly(K)}$ 
i.i.d.\ labeled samples, 
runs in time $\poly(d,N)$, and returns a hypothesis $h$ such that  
with high probability $\mathrm{err}_{0-1}^D(h) \leq O(\opt)+\eps$. 
\end{theorem}

{
\begin{remark}
Recent work~\cite{DMRT25} has shown that 
the above error guarantee is not possible 
for efficient SQ algorithms in the distribution-free setting. Specifically, 
any SQ algorithm for MLC with RCN that achieves {\em any} constant factor 
approximation requires super-polynomial in $d$ time. 
This hardness result further motivates the study of 
distribution-specific learning in the multiclass setting, 
even with random noise only.
\end{remark}
}

\vspace{-0.2cm}

\paragraph{Intersections of Halfspaces}  
An intersection of $K$ halfspaces is any  
function $f:\R^d\to\{\pm 1\}$ for which there exist $K$ halfspaces 
$h_i(\x) = \sign(\vec w^{(i)}\cdot\x+t_i)$ such that 
\( f(\x) = 1 \) if and only if \( h_i(\x) = 1 \) for all \( i \in [K]\). 
We denote by $\mathcal{F}_{K,d}$ the class of intersections of $K$ halfspaces on $\R^d$. 

The family of intersections of halfspaces is a fundamental concept 
class that has been extensively studied in many contexts 
for at least the past five decades. We show that our general MIM 
algorithm can be applied to $\mathcal{F}_{K,d}$, 
leading to a novel accuracy-complexity tradeoff in the agnostic setting. 
Specifically, we obtain the following result 
(see \Cref{thm:LearningIntersections} for a more detailed statement):

\begin{theorem}[Learning $\mathcal{F}_{K,d}$ with Adversarial Label Noise] \label{thm:SimpleLearningIntersections}
There exists an agnostic learning algorithm that draws 
$N =O(d^2) \, 2^{\poly(K/\eps)}$ i.i.d.\ labeled samples, 
runs in time $\poly(d,N)$, and computes a hypothesis $h$ such that with high probability  
$\mathrm{err}_{0-1}^D(h) \leq K\widetilde O(\opt)+\eps$, where $\opt$ 
is defined with respect to $\mathcal F_{K,d}$.
\end{theorem}

\noindent \Cref{thm:SimpleLearningIntersections} is obtained from 
\Cref{thm:SimpleMetaTheorem} via a new structural result,  
roughly showing that the class $\mathcal{F}_{d,K}$ is 
$(2,\zeta,K\tilde{O}(\zeta),O(\sqrt{\log(K)}))$-well 
behaved for any $\zeta\in (0,1)$ (see \Cref{prop:GoodCondIntersections}).

\medskip

\noindent {To highlight the 
significance of \Cref{thm:SimpleLearningIntersections}, some remarks are in order. 
Note that the information-theoretically optimal 
error in our setting is $\opt+\eps$. The only known algorithm 
achieving this guarantee is the $L_1$ polynomial regression~\cite{KKMS:05, KOS:08}, 
which has complexity $d^{O( {\log K}/{\epsilon^4})}$ for the 
class $\mathcal{F}_{d,K}$. (This complexity bound was recently shown to be near-optimal 
in the SQ model~\cite{DiakonikolasKPZ21, HsuSSV22}.)

A fundamental open question in the literature has been whether
one can achieve {\em near-optimal error} guarantees with 
qualitatively improved sample and computational complexity. 
\Cref{thm:SimpleLearningIntersections} answers this question in the 
affirmative. Specifically, for $K=O(1)$, it achieves 
near-optimal error guarantee (up to a $O(\log(1/\opt))$ factor), 
and its complexity has {\em fixed-polynomial dependence in the 
dimension} (namely, $d^2$ as opposed to $d^{f(k, 1/\eps)}$). 

In comparison, \cite{DKS18-nasty} developed 
a $\poly(d) (1/\eps)^{\poly(K)}$ time algorithm, achieving an 
error guarantee of $\poly(K) \tilde{O}(\opt^{1/11})$. The key 
distinction is that the error guarantee achieved by the 
\cite{DKS18-nasty} algorithm is highly suboptimal 
as a function of $\opt$ (by a large polynomial factor). Additionally,
the dependence on $K$ in our \Cref{thm:SimpleLearningIntersections} is significantly better by a polynomial factor. 
}

\vspace{-0.3cm}

\subsection{Overview of Techniques} \label{ssec:techniques}
\paragraph{Prior Techniques and Obstacles for Learning Multi-Index Models}  
Recall that our goal is, 
given access to noisy samples from a well-behaved MIM 
with finite label set, to learn the underlying function $f$ to high accuracy. 
We will use $W$ to denote the hidden $K$-dimensional subspace 
to whose projection $f$ depends.
The standard approach for solving such learning problems 
(see, e.g., \cite{Vempala10}) is to first learn some suitable approximation $V$ to $W$, and then use exhaustive search 
to learn a function of the projection onto $V$.
Computing such an approximation to $W$ 
is usually achieved by estimating the low-order moments 
of the level sets of $f$.
In particular, for each output value $i$ and 
each low-degree polynomial $p$, we want to approximate 
the expectation of $p(\x)\Ind(f(\x) = i)$. 
In the absence of label noise, these moments should depend 
solely on how $p$ interacts with the directions aligned with $W$.
Using this property, one can extract an approximation 
of a subspace of $W$, by choosing $V$ to be the space 
spanned by all directions with reasonably large moments.  
This approach yields an algorithm with runtime approximately 
$d^{O(m)}$, where $m$ is the number of moments being computed, 
plus some (often exponential) function of $\dim(V)$---which is 
typically regarded as ``constant'', since $d$ is much larger than $\dim(V)$.

Unfortunately, there may exist directions within $W$ 
where all low-degree moments vanish.
An innocuous such example occurs when $f$ depends only 
on the projection of $\x$ onto a subspace $W' \subseteq W$. 
In this case, all moments would vanish for directions in $W \cap 
W'^{\perp}$. However, if such moments vanish for different 
reasons---for example, due to label noise---the aforementioned 
approach could be problematic.
Therefore, in order to analyze this class of algorithms, 
one needs to prove that every direction $\vec v\in W$ 
for which the moments of $f$ vanish is a direction where $f$ approximately does not depend on---i.e., $f$ is close 
to a function of the projection onto $\vec v^{\perp}$.
Unfortunately, showing that the moments of $f$ are non-vanishing 
for all relevant directions can be challenging, which limits the applicability of this approach.

\vspace{-0.2cm}

\paragraph{New Dimensionality Reduction Technique} The main 
algorithmic contribution of this work is the design of a 
new, more general, approach 
that leads to significant improvements on the aforementioned dimensionality reduction technique. 
At a high-level, we design an iterative dimensionality reduction procedure (\Cref{alg:MetaAlg1,alg:MetaAlg2}) working as follows.
At each step{/iteration}, after learning an approximating subspace $V$, we partition $V$ into small rectangular regions $S$, each {of} probability mass at least $\alpha$, 
{where $\alpha$ is a carefully selected parameter.} 
We then compute the first $m$ moments of {$y(\x)$, where $y(\x)$ 
is the noisy version of $f(\x)$,}
conditioned on the projection of $\x$ to $V$ lying in each region $S$; 
this step requires roughly $d^{m}(1/\alpha)$ samples and runtime 
(see \Cref{alg:MetaAlg2}).
In the noiseless case (i.e., with clean labels), 
non-vanishing moments should still align with $W$, 
providing us with a better approximation.
We then find all directions with non-vanishing moments 
across all regions $S$, add them to $V$, and repeat the process.
We terminate this procedure once $y$ is sufficiently close 
to a function $g$ of the projection of $\x$ onto $V$.
Finally, using exhaustive search methods, 
we can recover an $\eps$-approximation to $g$, 
using roughly $(1/\eps)^{\dim(V)}$ samples and runtime.

The main advantage of our new method is that it 
succeeds under much weaker conditions, as compared to previous approaches.
As mentioned earlier, {prior} techniques require that 
every direction in $W$ to be either unimportant 
or have a non-trivial low-degree moment.
In contrast, our approach only requires that for each subspace $V$, 
either $f$ is approximately a function of the projection onto $V$, 
or that there exists some region $S$ with a non-trivial low-degree moment 
in a direction {in $V^{\perp}$} correlated with $W$ (see \Cref{def:GoodCondition}).
If we can {establish this structural property}, it follows 
that unless $f$ is already close to a function 
of the projection onto $V$ (in which case we can apply a brute-force method), 
one iteration of the algorithm will find a direction in $V^{\perp}$ 
that is at least slightly correlated with $W$.
As a result, in the next iteration, the distance between $V$ and $W$ decreases; 
hence, our algorithm will eventually terminate.

{
This generalization of the previously used one-shot dimension reduction technique has two main advantages.
Practically speaking, the more relaxed condition is 
often easier to prove. This allows, for example, our 
qualitative improvements in the error for robustly 
learning intersections of halfspaces. As far as we 
know, the one-shot dimension reduction technique may 
suffice to obtain error $\tilde O(\opt)$ in this 
case. However, the difficulty of proving such results 
means that the best analysis we know of gives 
substantially worse bounds (and only after a quite 
involved analysis). On the other hand, 
the condition needed for our iterative dimension 
reduction is somewhat easier to obtain, 
giving us a substantially better $\tilde O(K\opt)$ 
error with a simpler analysis.

Furthermore, our new technique allows us to 
efficiently learn a broader family of MIMs compared 
to what was previously possible 
with one-shot dimension reduction techniques, even in principle. 
In particular, as we show in \Cref{app:ss}, 
there exist Multi-Index Models that  
cannot be learned to better than constant error 
with any constant number of moments 
using the one-shot technique; and on the other hand, 
can be efficiently learned to arbitrary 
accuracy using only second moments via 
our iterative technique.
}

One disadvantage of our method {as described above} is that, 
while the runtime is typically polynomial 
in the dimension $d$, the dependence on other parameters can be problematic. 
Specifically, the number of regions $S$ is usually exponential in $\dim(V)$; 
hence, at each iteration, $\dim(V)$ can increase in the worst-case 
by an exponential amount. This can lead to power-tower dependencies 
on some parameters. We circumvent this issue by assuming that 
our function has a small fraction of regions $S$ with non-vanishing moments.
With this assumption, we can only add directions with large moments 
in a reasonable fraction of regions using a filtering step (see Line \ref{line:matrixMeta}), keeping the increase in $\dim(V)$ polynomial.
This bound can be further improved if {the noisy version of $f(\x)$ that we observe} 
is actually a function of the projection onto $W$, 
which holds in both the realizable and the RCN settings (see \Cref{sec:RCNAlgo}).
In this case, non-vanishing moments will exactly align with directions along $W$, 
up to a measurement error, ensuring that $\dim(V)$ never exceeds $\dim(W)$ 
(see \Cref{thm:MetaAlgRCN}).
{
For an illustration 
of our iterative procedure for the special case of multiclass linear classification, see \Cref{fig:algorithm-snapshot} in \Cref{sec:MulticlassAlgorithm}.}

\paragraph{SQ Lower Bounds} An appealing feature of our algorithmic 
technique is that it is provably qualitatively tight 
in the Statistical Query (SQ) model (see \Cref{sec:SQ-lower-bounds}).
In particular, suppose that a function $f$ has a subspace $V$ 
so that on the preimage of the projection onto any point $\x$ in $V$, 
the low-degree moments of $f$ are nearly vanishing. 
Then we prove an SQ lower bound showing that it is SQ-hard 
to distinguish between (i) the distribution $(\x,h(\x))$, 
where $h$ is a random rotation of the function $f$ about $V$, 
and (ii) a product distribution $(\x,y)$, where $\x$ and $y$ are conditionally
independent on $\x^V$ (see \Cref{thm:generalSQ}).
This means that unless $f$ is close to some function of the projection onto $V$, 
it is hard to learn $f$ in the SQ model.
In more detail, unless $f$ has the property that for each $V$, 
either $f$ is $\tau$-close to a function of the projection onto $V$, 
or there is some preimage with a non-trivial moment of degree at most $m$, 
then any SQ algorithm that learns $f$ to accuracy better than roughly $\tau$ 
either uses exponentially many queries or at least one query of accuracy 
smaller than $d^{-\Omega(m)}$. In summary, our algorithm learns $f$ 
if and only if there is an SQ algorithm that does so with comparable parameters 
as a function of $d$.

Finally, we give an overview of our SQ lower bound for learning multiclass linear classifiers with RCN (see \Cref{thm:RCNLowerBoundK} in \Cref{sec:RCNSQ}). 
This result leverages the fact that statistical queries correspond 
to evaluations of the confusion matrix (see \Cref{def:RCNdistribution}). 
If this confusion matrix is close to being non-invertible, 
then there is missing information between the noisy and noiseless queries.
To prove this, we find a low-dimensional concept, $f$, and a corresponding 
confusion matrix, $\vec{H}$, for which the first $\Omega(K)$ moments 
of the noisy label sets match the moments of the standard Gaussian distribution. 
As a result, we establish SQ-hardness in distinguishing between 
the distribution of random rotations of $f$ in the $d$-dimensional space 
and a product distribution.

\paragraph{Applications to Well-studied Concept Classes}   
We show that our iterative dimension-reduction technique 
yields significantly improved learners for 
multiclass linear classifiers and intersections of halfspaces in the presence of noise. 
In the regime of agnostically learning multiclass classifiers with $K$ classes 
(see \Cref{sec:MulticlassAlgorithm}), we show that our algorithm achieves 
an error of $O(\opt)+\eps$, using roughly $d2^{\poly(K/\eps)}$ samples and runtime 
(see \Cref{thm:constantApprox}). We show that if the noisy labels $y$ 
are at least $\Theta(\opt)+\eps$ far  from any function depending 
on the projection onto a subspace $V$ (otherwise, an exhaustive search in $V$ 
would yield a function $g$ with the desired error bound), 
then at least an $\epsilon$ fraction of our partition $S$ of $V$, 
will have the optimal $f$ more than $\Theta(\pr[f(\x) \neq y \mid S])$ 
far from any constant function. In this case, we establish a new 
structural result (see \Cref{prop:alg2}) showing 
that the first moments of the sub-level sets of $y$ 
align with $W$ on a non-trivial fraction of the partition sets $S$.
{By using this structural result, 
we conclude that the class of multiclass linear classifiers is  
$(1,\zeta,O(\zeta),O(K))$-well behaved, as in 
\Cref{def:SimpleGoodCondition}, for any $\zeta\in (0,1)$. Therefore, 
by applying our general algorithm we can obtain a learner with the 
aforementioned efficiency (\Cref{thm:SimpleMetaTheorem})}

Furthermore, we obtain an algorithm with improved time-accuracy tradeoff 
for agnostically learning intersections of $K$ halfspaces under the 
Gaussian distribution (see \Cref{sec:applications}). 
Specifically, we show that an error of $K\widetilde{O}(\opt)$ 
can be achieved with complexity $d^2 2^{\poly(K/\eps)}$ 
(see \Cref{thm:LearningIntersections}). An analogous structural result 
to the one we showed for multiclass linear classifiers can be established 
to prove that unless the optimal classifier is $K\widetilde{O}(\opt)$-close 
to a constant function, there exists a low-order moment 
(of degree at most two) that correlates non-trivially 
with the unknown subspace $W$ (see \Cref{lem:progressIntersections}). 
{Using this result, we deduce that the class of intersections 
of $K$ halfspaces is $(2,\zeta,K\tilde{O}(\zeta),O(\sqrt{\log(K)}))$-well behaved, 
as in \Cref{def:SimpleGoodCondition}, for any $\zeta\in (0,1)$. 
Hence, our learner again follows by applying our general 
iterative dimension-reduction technique. 
}

\subsection{Prior Work}
The task of learning MIMs is fundamental in machine learning, 
with a long line of research in both statistics and computer 
science~\cite{Friedman:1980tu, Huber85-pp, Li91, 
HL93,xia2002adaptive, Xia08}.
The majority of the literature focuses on learning various structured families of MIMs under natural 
distributional assumptions with clean labels~\cite{GLM18,JSA15,DH18,BJW18,GeKLW19,DKKZ20,DLS22,BBSS22,CDGJM23}.
A recent line of work has shown that 
achieving optimal error in the adversarial label setting 
under the Gaussian distribution cannot be achieved in fixed-degree
polynomial time, even for very basic MIM classes \cite{DiakonikolasKPZ21, DKMR22b, Tieg23}. These hardness results motivate the design of faster algorithms with relaxed error guarantees, which is the focus of this paper. Despite this being a natural direction, no general algorithmic framework was previously known in this setting.

The recent work~\cite{DKKTZfocs24} developed 
agnostic learners for MIMs under the Gaussian distribution with 
fixed-parameter tractability 
guarantees similar to ours. The major difference is that, 
while our framework relies on learning from random examples, 
the latter work~\cite{DKKTZfocs24} assumes query access to the target 
function---a much stronger access model.  

The task of efficiently learning intersections of halfspaces 
under Gaussian marginals is one of the most well-studied problems 
in the computational learning theory literature. A 
long line of work focused on achieving strong provable 
guarantees for this problem in the realizable setting \cite{Vempala10,ArriagaVempala:99,KlivansLT09,Kos:04}. 
However, the aforementioned work either fails 
in the presence of noise or scales superpolynomially with respect to the dimension (i.e., does not achieve fixed-degree polynomial runtime). 
An exception is the recent work~\cite{DKS18-nasty} that 
sacrifices the optimal error guarantee for a polynomial dependence 
in the input dimension.

Binary linear classification in the presence of noise 
under well-behaved distributions is a fairly well-understood problem, 
with a substantial body of work; see, 
e.g.,~\cite{KKMS:05,ABL17, DKS18-nasty, DKTZ20, DKTZ20c, 
DiakonikolasKKT21, DKKTZ22} and references therein. However, in the 
multiclass setting, most known results focus on online or bandit 
classification \cite{banditron,BPSBDC19}. That is, 
the subject of designing efficient algorithms in the multiclass setting 
that work in the presence of noise is rather poorly understood.
It is worth pointing out that the multiclass regime 
offers several new interesting models of noisy information that are 
challenging, even if the noise is random~\cite{PRMNQ17,VanRooyencorrupted,
Fotakis2021coarselabellearning,cour2011partiallabel}.

\subsection{Notation and Preliminaries}
\label{sec:prelims}

\medskip

\noindent {\bf Basic Notation}
For $n \in \Z_+$, let $[n] := \{1, \ldots, n\}$ {and $\overline \R\eqdef \R\cup \{\pm \infty\}$}. {We denote by $2^S=\{T\mid T\subseteq S\}$ the power set of the set $S$.}  We use small boldface characters for vectors
and capital bold characters for matrices.  For $\bx \in \R^d$ and $i \in [d]$, $\bx_i$ denotes the
$i$-th coordinate of $\bx$, and $\|\bx\|_2 := (\littlesum_{i=1}^d \bx_i^2)^{1/2}$ denotes the
$\ell_2$-norm of $\bx$. {Throughout this text, we will often omit the subscript and simply write $\|\x\|$ for the $\ell_2$-norm of $\bx$.}
We will use $\bx \cdot \by $ for the inner product of $\bx, \by \in \R^d$
and $ \theta(\bx, \by)$ for the angle between $\bx, \by$. 
{For vectors $\x, \vec v\in \R^d$ and a subspace $V\subseteq \R^d$ denote by $\x^{V}$ the projection of $\x$ onto $V$, {$\x^{\perp V}$ the projection of $\x$ onto the orthogonal complement of $V$} and by $\x^{\vec v}$ the projection of $\x$ onto the line spanned by $\vec v$.
For two subspaces $V,W\subseteq \R^d$, we denote by $W^{V}=\{\w^{V}:\w\in W\}$ and by $V+W=\{\w+\vec v: \w\in W, \vec v\in V\}$, note that both  $W^{V}$ and $V+W$  are subspaces.
Furthermore, for a set of vectors  $L\subseteq \R^d$, we denote by $\spaning(L)$ the subspace of $\R^d$ defined by their span.}
We slightly abuse notation and denote {by}
$\vec e_i$ the $i$-th standard basis vector in $\R^d$.  
{For a matrix $A \in \mathbb{R}^{m \times n}$, the Frobenius norm of $A$ is denoted by 
$\norm{A}_{F}=(\sum_{i=1}^m\sum_{j=1}^nA_{i,j}^2)^{1/2}$. 
For a set $A\subseteq \R^d$ and a point $\x\in \R^d$, we define the distance from $\x$ to $A$ as $\mathrm{dist}(\x,A)=\min_{\vec y\in A}\norm{\x-\vec y}$.
}

We use the standard asymptotic notation, where 
$\wt{O}(\cdot)$ is used to omit polylogarithmic factors.
{Furthermore, we use $a\lesssim b$ to denote that there exists an absolute universal constant $C>0$ (independent of the variables or parameters on which $a$ and $b$ depend) such that $a\le Cb$, $\gtrsim$ is defined similarly.
We use the notation $g(t)\le \poly(t)$ for a quantity $t\ge1$ to indicate   that there exists constants $c,C>0$ such that $g(t)\le Ct^c$. Similarly we use $g(t)\ge \poly(t)$ for a quantity $t<1$ to denote  that there exists constants $c,C>0$ such that $g(t)\ge Ct^c$.
}

\medskip

\noindent {\bf Probability Notation}
We use $\E_{x\sim D}[x]$ for the expectation of the random variable $x$ according to the
distribution $D$ and $\pr[\mathcal{E}]$ for the probability of event $\mathcal{E}$. For simplicity
of notation, we may omit the distribution when it is clear from the context.  For $(\x,y)$
distributed according to $D$, we denote {by} $D_\x$ to be the distribution of $\x$ and {by} $D_y$ to be the
distribution of $y$.
Let $\normal( \boldsymbol\mu, \vec \Sigma)$ denote the $d$-dimensional Gaussian distribution with mean $\boldsymbol\mu\in  \R^d$ and covariance $\vec \Sigma\in \R^{d\times d}$. We denote by $\phi_d(\cdot)$ the pdf of the $d$-dimensional {standard normal} and we use  $\phi(\cdot)$ for the pdf of the {$1$-dimensional} standard normal.

For additional preliminaries, we refer the reader to \Cref{sec:Addprelims}.

\subsection{Organization}
{The paper is structured as follows: 
In \Cref{sec:MulticlassAlgorithm}, we present a special case of our 
subspace approximation procedure for learning multiclass linear 
classifiers under agnostic noise, 
establishing \Cref{thm-intro:SimplerAlgMulticlassClass}.
In \Cref{sec:generalAlgorithm}, we generalize  this approach, stating 
the necessary conditions (see \Cref{def:SimpleGoodCondition}) for our 
algorithm to work for general classes of MIMs, establishing \Cref{thm:SimpleMetaTheorem}. In \Cref{sec:applications}, we present 
additional  applications of our algorithmic approach. In \Cref{sec:intersections},
we apply our general theorem to the class of intersections of halfspaces, establishing \Cref{thm:SimpleLearningIntersections}. 
In \Cref{sec:RCNAlgo}, we adapt our general algorithm 
for the case of random noise, and demonstrate that it yields 
qualitative stronger complexity guarantees 
(\Cref{thm-intro:AlgRCN-linear}). 
Finally, \Cref{sec:SQ-lower-bounds} gives our SQ lower bounds.
First,  we demonstrate that the complexity of our algorithm is qualitatively optimal as a function of the dimension, establishing \Cref{thm:simplegeneralSQ}. Second, we give an SQ lower bound 
(\Cref{thm-intro:SQ-RCN}) 
for learning multiclass linear classifiers under random classification noise (and in the presence of coarse labels).

\section{Learning Multiclass Linear Classifiers with Adversarial Label Noise}
\label{sec:MulticlassAlgorithm}

{In this section, we establish \Cref{thm-intro:SimplerAlgMulticlassClass}.
Note that we introduce this special case as a preliminary step to build intuition for our general algorithm in \Cref{sec:generalAlgorithm}.
First, we state the formal version of the theorem and  provide the algorithm pseudocode along with the  necessary definitions.}

\begin{theorem}[Learning Multiclass Linear Classifiers with Agnostic Noise]\label{thm:constantApprox} Let $\mathcal{L}_{d,K}$ be the class of $K$-class linear classifiers in $\R^d$.
   Let $D$ be a distribution over $\mathbb{R}^d \times [K]$ whose $\x$-marginal is $\mathcal{N}(\vec0, \vec I)$,  let $\opt=\inf_{f\in \cL_{d,K}}\pr_{(\x,y)\sim D}[f(\x)\neq y]$ and let $\epsilon, \delta \in (0,1)$. {There exists a sufficiently  large universal constant $C>0$ such that} \Cref{alg:constantApprox} draws $N = d 2^{\poly(1/\eps,K)}\log(1/\delta)$ i.i.d.\ samples from $D$, runs in $\poly(N)$ time, and returns a  hypothesis $h$ such that, with probability at least $1 - \delta$, it holds that $\pr_{(\x,y)\sim D}[h(\x)\neq y] \leq C\cdot\opt +\eps$.
\end{theorem}
Before we present the algorithm, we provide the following definitions, which characterize how we discretize a low-dimensional subspace and approximate {distributions} over it.

To effectively approximate the distribution over a specific subspace \( V \), we aim to preserve a substantial portion of its mass and partition it into fine regions of non-negligible mass. By selecting an orthonormal basis for \( V \) and excluding regions where any coordinate exceeds \( \sqrt{\log(k/\epsilon)} \), we ensure that at least a \( 1 - \epsilon \) fraction of the mass remains. The remaining space is then partitioned into cubes aligned with \( V \), each with edges of length \( \epsilon \) and containing approximately \( \epsilon^{\dim(V)} \) mass.

{\begin{definition}[$\eps$-Approximating Partition]\label{def:approximatingPartition}
{Let $V$ be a $k$-dimensional subspace of $\mathbb{R}^d$ with an orthonormal basis $\vec v^{(1)}, \dots, \vec v^{(k)}$, and let $\eps \in (0,1)$.} An $\eps$-approximating partition with respect to $V$ is a collection of sets 
$\mathcal{S}$ defined as follows: For each multi-index $\vec{j}=(\vec j_1,\dots,\vec j_k)\in [M_\eps]^k$, define $S_{\vec j} = \{ \x \in \mathbb{R}^d : z_{\vec j_{i}-1} \leq \vec{v}^{(i)} \cdot \x \leq z_{\vec j_{i}}, i\in [k] \}$ where  $z_i$'s are defined as \( z_i = -\sqrt{2 \log(k/\epsilon)} + i \epsilon+t\), for $i\in \{0,\dots, M_{\eps}\},t\in(0,\eps/2)$ and \( M_{\eps} = \left\lceil (2\sqrt{2 \log(k/\epsilon)})/\epsilon \right\rceil \). Furthermore, to each set $S_{\vec j}\in \mathcal{S}$, we associate  a set of linear equations  $\cG_{S_{\vec j}}=\{\vec v^{(1)}\cdot \x=t_1,\dots,\vec v^{(k)}\cdot \x=t_k\}$, where for each $i\in[k]$, $t_i$ is chosen so that $z_{\vec j_{i}-1} \leq t_i \leq z_{\vec j_{i}}$.
\end{definition}}

    {Furthermore, we define the piecewise constant approximation of a distribution over labeled examples as the function that minimizes the 0-1 loss with respect to the labels in each subset of a partition. Our algorithm will first estimate the hidden subspace of the multiclass linear classifier and then return a piecewise constant approximation over it.}
    \begin{definition}[Piecewise Constant Approximation]\label{def:h}
    {Let $D$ be a distribution over $\R^d\times [K]$, let  $V$ be a subspace of $\mathbb{R}^d$ and let $\eps\in (0,1)$. Let $\mathcal{S}$ be an $\eps$-approximating partition with respect to $V$ (see \Cref{def:approximatingPartition}).
    A piecewise constant approximation of the distribution $D$, with respect to $\mathcal{S}$, is any classifier $h_{\mathcal{S}}:\mathbb{R}^d\to [K]$ such that for each $S\in \mathcal{S}$ and $\x\in S$, 
    $h_{\mathcal{S}}$ is defined as follows $h_{\mathcal{S}}(\x)\in\argmin_{i\in [K]} \pr_{(\x,y)\sim D}[i\neq y\mid \x\in S]$.}
    \end{definition}
    
We now present our {subspace approximation} method for learning  multiclass linear classifiers under agnostic noise (see \Cref{alg:constantApprox}).
\begin{algorithm}
    \centering
    \fbox{\parbox{5.1in}{
            {\bf Input:} Accuracy $\eps > 0$, failure probability $\delta>0$  sample access to a distribution $D$ over $\mathbb{R}^d\times [K]$ { whose $\x$-marginal is $\cN(\vec 0,\vec I)$}.
            \\
            {\bf Output:} A hypothesis $h$ such that with probability at least $1-\delta$ it holds that $\pr_{\x\sim \cN(\vec 0,\vec I)}[h(\x)\neq y]\le C\cdot\opt{+}\eps$, where $C$ is a universal constant {and $\opt\eqdef \inf_{f\in \mathcal{L}_{d,K}} \pr_{(\x,y)\sim D}[f(\x)\neq y]$}.
            \begin{enumerate}
\item  Let $c$ be  a sufficiently large constant, and {let $T$ and $p$ be sufficiently large, constant-degree polynomials in $1/\eps$ and $K$.} \label{line:init}
             \item {Let $N\gets d2^{T^c}\log(1/\delta)$, $\eps'\gets 1/p$, $L_1\gets\emptyset$ and $t\gets 1$.}\label{line:init_vars}
                \item  While $t\le T$
            \label{line:loop}
             \begin{enumerate}
             \item {Draw a set {$S_t$} of $N$ i.i.d.\ samples from $D$.}
             \item $\mathcal{E}_t\gets$ \Cref{alg:correlating-vectors}($\eps, \eps',\delta, \spaning(L_t), S_t$).
          
                \item $L_{t+1}\gets L_{t}\cup \mathcal{E}_t $.\label{line:update}
                \item $t\gets t+1$.
            \end{enumerate}
             \item
             Construct $\mathcal{S}$, an $\eps'$-approximating partition with respect to $\spaning(L_t)$ (see \Cref{def:approximatingPartition}).
             \item { {Draw $N$ i.i.d.\ samples from $D$ and construct the piecewise constant classifier $h_{\mathcal{S}}$ as follows: For each $S \in \mathcal{S}$, assign a label that appears most frequently among the samples falling in $S$. Return $h_{\mathcal{S}}$.}}\label{line:return-statement}

            \end{enumerate}
    }}
    \medskip
        \caption{Learning Multiclass Linear Classifiers under Agnostic Noise }
 \label{alg:constantApprox}
\end{algorithm}

\begin{algorithm}
    \centering
    \fbox{\parbox{5.1in}{
            {\bf Input:}  $\eps,\eps',\delta>0$, a subspace $V$ {of $\R^d$, and a set of samples $\{(\x^{(1)},y_1),\dots, (\x^{(N)},y_N)\}$ from a distribution $D$ over  $\R^d\times [K]$ whose $\x$-marginal is $\cN(\vec 0,\vec I)$}.\\
            {\bf Output:} A set of unit vectors $\mathcal E$.
            \begin{enumerate}
\item {Let $\sigma$ be a sufficiently small, constant-degree polynomial in $\eps, 1/K$.} 
            \item {Construct the empirical distribution $\widehat{D}$ of $\{(\x^{(1)},y_1),\dots, (\x^{(N)},y_N)\}$.} 
\item {Construct} an $\eps'$-approximating partition $\mathcal{S}$ with respect to $V$ (see \Cref{def:approximatingPartition}).

            \item For each $S\in \mathcal{S}$ and for each label $i\in[K]$ compute $\widehat{\vec u}^{(S,i)}=\E_{(\x,y)\sim \widehat{D}}[\x \Ind(y=i)\mid \x\in S]$.\label{line:moments} 
            \item Let $\widehat{\vec U}= \sum_{(S,i) \in \mathcal{S}\times [K]} (\widehat{\vec u}^{(S,i)})^{\perp V}((\widehat{\vec u}^{(S,i)})^{
                \perp V
                })^{\top}\pr_{(\x,y)\sim {\widehat{D}}}[S]$.  \label{line:matrix}
            \item Return the set $\mathcal{E}$ of {unit} eigenvectors of $\widehat{\vec U}$  {corresponding to eigenvalues at least $\sigma$.      }
            \end{enumerate}
    }}
    \medskip
        \caption{{Estimating a relevant direction}}
 \label{alg:correlating-vectors}
\end{algorithm}

\newpage 

{Let $f^*$ be a $K$-multiclass linear classifier with error $\opt$, i.e., $\Pr_{(\x,y)\sim D}[f(\x)\neq y]=\opt$.} \Cref{alg:constantApprox} proceeds in a sequence of iterations: In each iteration $t$, it maintains and updates a set of unit vectors $L_t$ so that its span $V_t$ provides a better approximation to the span of the weight vectors of $f^*$, denoted by $W^*$. To this end, \Cref{alg:constantApprox} uses \Cref{alg:correlating-vectors} to construct a new list  $\mathcal{E}_t$ that contains  vectors with non-trivial correlation to $(W^*)^{\perp V_t}$.

{First, in \Cref{sec:struct-multiclass}, we show that if  $f^*$ is not $C\cdot\opt$-close to any constant classifier, then there exists at least one label \(i\in[K]\) for which the first moment $\E[\x \1(y=i)]$ correlates non-trivially (by at least $\poly(\eps/K)$) with a vector in $W^*$.} This, in fact, suffices to initialize a list of vectors that correlate with a vector in $W^*$. In order to expand this list and at the same time improve the approximation of  $W^*$, we need to show that,  we can find vectors that not only correlate with $W^*$ but are also orthogonal to the current list of vectors $L_t$.

To this end, \Cref{alg:correlating-vectors} uses the current set $L_t$ of weight vectors to partition {$\R^d$} into sufficiently small regions $S$.
In each such region $S$, the  classifier $f^*$ has a limited dependence on the vectors in $L_t$.
Specifically, we show that $f^*$ is $\eps$-close to a function $f_S^*$ that does not depend on weight vectors contained in $L_t$.  
Thus, within each region $S$, we obtain a multiclass classifier that depends only on vectors in $(W^*)^{\perp V_t}$.
Therefore, in every region \(S\), either the classifier can be well-approximated by a constant function or the first moments will correlate non-trivially with a vector in \((W^*)^{\perp V_t}\), thus enabling further improvement of the approximation.

{Furthermore, for \(f^*\) to depend mostly on directions in 
\((W^*)^{\perp V_t}\), the cubes \(S\) in the partition 
must have width on the order of \(\poly(\eps/K)\). 
Such a fine partitioning would typically yield a number of cubes 
that is exponential in \(\dim(V_t)\) and, hence 
could lead to a power-tower update in the dimension of \(V_t\) 
when updating \(L_t\). In order to avoid this power-tower dependence, 
we leverage the fact that the total probability mass of cubes 
with non-trivial correlation is at least \(\poly(\eps/K)\) 
(this will be proved in the analysis of 
\Cref{alg:correlating-vectors} in \Cref{prop:alg2}). 
By weighting each moment vector by the probability mass 
of its corresponding cube 
(see Line~\ref{line:matrix} of \Cref{alg:correlating-vectors}), we are guaranteed 
to find a vector with non-trivial correlation 
by examining only a \(\poly(K/\eps)\)-sized list of candidate vectors.
}

As follows from the above, our algorithm is not proper; that is, it 
does not return a multiclass linear classifier. Instead, it 
produces a piecewise constant function over a sufficiently 
fine partition of $\R^d$. 

{See \Cref{fig:algorithm-snapshot} for an illustration.} 

\paragraph{Organization of \Cref{sec:MulticlassAlgorithm}}
{The body of this section is structured as follows:
In \Cref{sec:struct-multiclass}, we state and prove the key structural result about the class of multiclass linear classifiers, \Cref{lem:progress}, which will be exploited by our algorithm. In \Cref{sec:project-multiclass}, we provide the analysis of \Cref{alg:correlating-vectors} in \Cref{prop:alg2}. In \Cref{sec:proofMulticlass},  we analyze \Cref{alg:constantApprox}, providing the full proof of \Cref{thm:constantApprox}. Finally, in \Cref{sec:2.7}, we provide the proof of \Cref{lem:tails}, a supporting lemma that is used in the  proof of \Cref{lem:progress}. }

\begin{figure}[H]
  \centering
  \begin{minipage}[b]{0.48\textwidth}
    \centering
    \begin{tikzpicture}[scale=0.7]
\fill[red!60, opacity=0.6] 
    (-5,5) -- (5,5) -- (5,2.5) -- (0,0) -- (-5,5) -- cycle;
  
\fill[green!60, opacity=0.6]
    (0,0) -- (5,2.5) -- (5,-5) -- (0,0) -- cycle;
  
\fill[orange!60, opacity=0.6]
    (0,0) -- (-5,-2.5) -- (-5,5) -- (0,0) -- cycle;
  
\fill[blue!60, opacity=0.6]
    (-5,-5) -- (5,-5) -- (0,0) -- (-5,-2.5) -- cycle;

\draw[white,thick] (-5,5)   -- (5,-5);   \draw[white,thick] (-5,-2.5) -- (5,2.5);  

\fill[blue!30, opacity=0.4] (-5,-2) rectangle (5,0);
\fill[blue!30, opacity=0.4] (-5,0) rectangle (5,2);
  
\draw[dashed,white,thick] (-5,-4) -- (5,-4);
  \draw[dashed,white,thick] (-5,-2) -- (5,-2);
  \draw[dashed,white,thick] (-5,0)  -- (5,0);
  \draw[dashed,white,thick] (-5,2)  -- (5,2);
  \draw[dashed,white,thick] (-5,4)  -- (5,4);
\draw[cyan, thick, ->, >=stealth, line width=2pt] (0,0) -- (0,1)
  node[left, color=cyan] {$\vec{v}^{(1)}$};

\end{tikzpicture}
   \end{minipage}
  \hfill
  \begin{minipage}[b]{0.48\textwidth}
    \centering
    \begin{tikzpicture}[scale=0.7]
\fill[red!60, opacity=0.6] 
    (-5,5) -- (5,5) -- (5,2.5) -- (0,0) -- (-5,5) -- cycle;
  
\fill[green!60, opacity=0.6]
    (0,0) -- (5,2.5) -- (5,-5) -- (0,0) -- cycle;
  
\fill[orange!60, opacity=0.6]
    (0,0) -- (-5,-2.5) -- (-5,5) -- (0,0) -- cycle;
  
\fill[blue!60, opacity=0.6]
    (-5,-5) -- (5,-5) -- (0,0) -- (-5,-2.5) -- cycle;

\draw[white,thick] (-5,5)   -- (5,-5);   \draw[white,thick] (-5,-2.5) -- (5,2.5);  

\draw[dashed,white,thick] (-4,-5) -- (-4,5);
  \draw[dashed,white,thick] (-2,-5) -- (-2,5);
  \draw[dashed,white,thick] (0,-5)  -- (0,5);
  \draw[dashed,white,thick] (2,-5)  -- (2,5);
  \draw[dashed,white,thick] (4,-5)  -- (4,5);
\draw[dashed,white,thick] (-5,-4) -- (5,-4);
  \draw[dashed,white,thick] (-5,-2) -- (5,-2);
  \draw[dashed,white,thick] (-5,0)  -- (5,0);
  \draw[dashed,white,thick] (-5,2)  -- (5,2);
  \draw[dashed,white,thick] (-5,4)  -- (5,4);
  
\draw[cyan, thick, ->, >=stealth, line width=2pt] (0,0) -- (1,0)  node[above, color=cyan] {$\vec{v}^{(2)}$};
\draw[cyan, thick, ->, >=stealth, line width=2pt] (0,0) -- (0,1) node[left, color=cyan] {$\vec{v}^{(1)}$};
\end{tikzpicture}   \end{minipage}
  \caption{Illustration of two consecutive steps of our iterative subspace approximation procedure for 4-class Multiclass Linear Classification. In the left figure, 
  the space $\R^d$ is partitioned according to one already discovered direction. The highlighted rectangular regions correspond to areas where a constant function over the current subspace $V$ yields an insufficient approximation, indicating that additional relevant directions can be found.
  In contrast, the non-highlighted regions already permit a sufficiently good constant approximation.
  The right figure illustrates that we discovered the second hidden direction so that partitioning the space to sufficiently 
  small cubes leads to a good approximation.
  }
  \label{fig:algorithm-snapshot}
\end{figure}
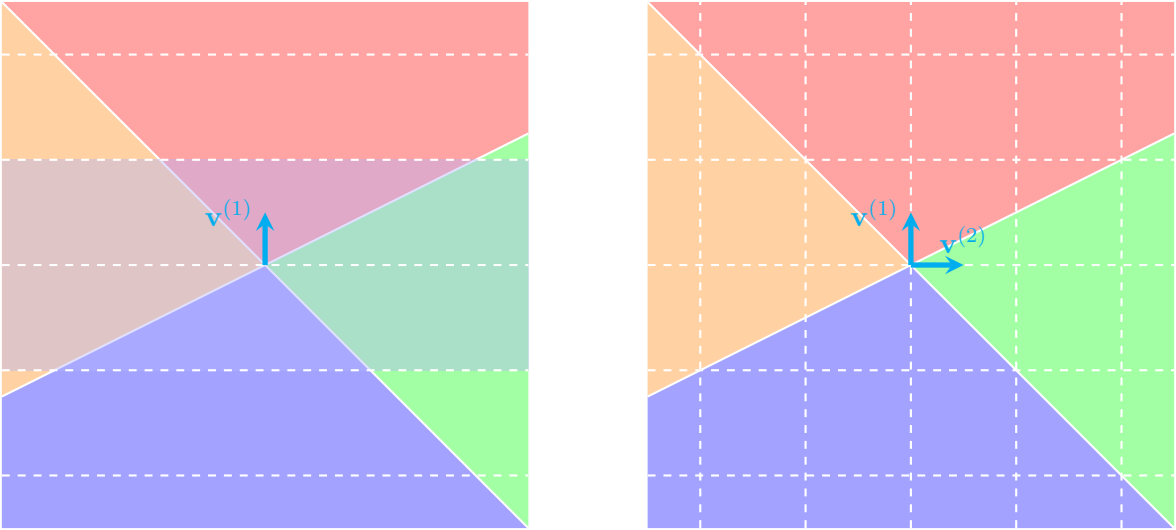

\subsection{Structural Result for Multiclass Linear Classification}  \label{sec:struct-multiclass}
In this section, we prove our main structural result, \Cref{lem:progress}, which enables us to efficiently learn multiclass linear classifiers under adversarial noise. In particular, we show that for a multiclass linear classifier $f$, if $f$ is far from any constant function by more than a constant multiple of its error with respect to the noisy label $y$, then the first moment of at least one of the noisy label regions has a {non-trivial} correlation with a vector in the subspace of weights {defining $f$}. 

{We prove this by finding carefully chosen vectors $\vec v^{(i)}{\in} \R^d$, $i{\in} [K]$, 
such that $\sum_{i=1}^K\E[\vec v^{(i)}\cdot \x\Ind(y=i)]$ 
is guaranteed to be non-trivially large.
Note that this sum is equal to $\sum_{i=1}^K\E[\vec v^{(i)}\cdot \x\Ind(f(\x)=i)]$ 
plus an error term for instances when $y$ is not equal to $f(\x)$.
By carefully selecting these vectors, we ensure that even a small error probability does not significantly reduce the overall sum.
}
Formally, we show the following:

\begin{proposition}[Structural Result]\label{lem:progress}Let $d\in \mathbb{Z}_+$, $\eps\in (0,1)$, $D$ be a distribution on labeled examples $(\vec x,y)\in \mathbb{R}^d\times [K]$ whose $\x$-marginal is $\cN(\vec0,\vec I)$. Also, let $f$ be a multiclass linear classifier, $f(\vec x)=\argmax_{i\in [K]}(\vec{w}^{(i)} \cdot \vec{x}+t_i)$ with $\vec{w}^{(1)},\ldots,\vec{w}^{(K)}\in \mathbb{R}^d$, and let  $C$ be a sufficiently large positive universal constant. If $\eps\le 1/C$ and $f(\x)$ is not close to any constant classifier, i.e., $\min_{i\in [K]} \pr_{\vec x\sim \cN(\vec 0,\vec I)}[f(\x)\neq i]\ge C\cdot \pr_{(\x,y)\sim D}[f(\x)\neq y]+\eps$, then for $j =\argmax_{i\in [K]} t_i$ 
there exists $i\in [K]$ with $\vec{w}^{(i)}-\vec{w}^{(j)}\neq \vec 0$ such that \[\frac{\vec{w}^{(i)}-\vec{w}^{(j)}}{\norm{\vec{w}^{(i)}-\vec{w}^{(j)}}}\cdot \E_{(\vec x,y)\sim D}[ \vec{x} \Ind[y=i] ]\ge \poly(1/K,\eps)\;.\]
\end{proposition}
\begin{proof}
Define $\opt\eqdef\pr_{(\x,y)\sim D}[f(\x)\neq y]$.
{Without loss of generality, we may assume that $\vec w^{(1)}=\vec{0}$ and $t_1=0$. This follows from the observation that subtracting the same function from every classifier does not change the output label.
In particular, if we subtract $\vec{w}^{(j)}\cdot \vec{x}+t_j$, where $j\in\argmax_{i\in [K]} t_i$, from all affine functions and then rename the labels accordingly, one of the resulting affine functions becomes identically $\vec{0}$ and we obtain $t_i\le 0$ for all $i\in [K]$.
}

We define  $L_i:\mathbb{R}^d \to \R, L_{i}(\x)=\vec{w}^{(i)} \cdot \vec{x}+t_i$ the corresponding affine functions such that $f(\x)=\argmax_{i\in [K]} L_i(\x)$ and we define the function $s(\x)\eqdef \max_{i\in [K]} L_i(\x)$.

For some parameter $\gamma>0$ which we will {chose appropriately below}, let $\vec v^{(i)}\eqdef \vec w^{(i)}/(\gamma-t_i)$ (this is well-defined as $\gamma> 0$ and $t_i\le 0$) for all $i\in[K]$. {Let $g(\vec x)\eqdef\max_{i\in [K]}\vec v^{(i)}\cdot \vec{x}$. Note that for all nonzero vectors $\vec v^{(i)}$, we have that $L_i(\x)\ge \gamma$ if and only if $\vec v^{(i)}\cdot \x \ge 1$ which means that $g(\x)\geq 1$. Furthermore, because $\vec w^{(1)}=\vec 0$ and $t_1=0$, we have that $g(\x)\geq 0$ for all $\x\in\R^d$.} We can decompose the first moment {of label $i$ as follows:}
\begin{align}\E_{(\vec x,y)\sim D}[\vec x\Ind(y=i)]&=  \E_{(\vec x,y)\sim D}[ \vec x\Ind(f(\vec x)=i)]-\E_{(\vec x,y)\sim D}[ \vec x\Ind(y\neq i,f(\vec x)=i)]+\E_{(\vec x,y)\sim D}[ \vec x\Ind(y=i,f(\vec x)\neq i )]\notag\\
 &=\E_{\vec x\sim \cN(\vec 0,\vec I)}[ \vec x\Ind(f(\vec x)=i)]-\E_{(\vec x,y)\sim D}[ \vec xS_i(\vec x,y)]+\E_{(\vec x,y)\sim D}[ \vec xT_i(\vec x,y)]\;,\label{eq:decomposition}
\end{align}
where $S_i(\vec x,y)=\Ind(y\neq i,f(\vec x)=i)$ and $T_i(\vec x,y)=\Ind(y=i,f(\vec x)\neq i )$ for  $i\in [K]$. 
Taking the inner product {of $\E[x\Ind(y=i)]$} with $\vec{v}^{(i)}$ and summing \Cref{eq:decomposition} {over} $i\in[K]$, we have that
\begin{align}\label{eq:2.4-2}
    \E_{(\vec x,y)\sim D}[\vec v^{(y)} \cdot \vec{x}]&=\underbrace{\E_{\vec x\sim \cN(\vec 0,\vec I)}[\vec v^{(f(\vec{x}))}\cdot \vec{x}]}_{I_1}\underbrace{\; -\sum_{i=1}^K\E_{(\vec x,y)\sim D}[ \vec{v}^{(i)}\cdot \vec{x}S_i(\vec x,y)]+\sum_{i=1}^K\E_{(x,y)\sim D}[ \vec{v}^{(i)} \cdot \vec{x}T_i(\vec x,y)]}_{I_2}\;.
\end{align}
Note that $I_1$ is the 
{the contribution of each $\vec v^{(i)}$ over $\x$ in the region where $f(\x)=i$},
and $I_2$ is the contribution that corresponds to the noise.
 We first bound $I_1$ from below.
\begin{claim}\label{clm:good-region}
    It holds that $I_1\geq \pr_{\x\sim \mathcal{N}(\vec 0,\vec I)}[g(\x)\geq 1]$.
\end{claim}
\begin{proof}[Proof of \Cref{clm:good-region}]
    
{Note that the random variable  $\vec v^{(f(\x))}\cdot \x$ is non-negative since $L_{f(\x)}(\x)=\max_i L_i(\x)\geq L_1(\x)= 0$ which implies that  $\vec w^{(f(\x))}\cdot\x\geq |t_i|$ and consequently, $\vec v^{(f(\x))}\cdot\x\ge 0$.}
Hence, by applying Markov's inequality, we have  \[ I_1=
\E_{\vec{x}\sim \cN(\vec 0, \vec{ I})}[\vec v^{(f(\vec x))} \cdot \vec x ]\ge  \pr_{\x\sim \cN(\vec 0, \vec I)}[\vec v^{(f(\vec x))} \cdot \vec x\ge 1]=\pr_{\x\sim \cN(\vec 0, \vec I)}[L_{f(\vec x)} ( \vec x)\ge \gamma]\;,\]
where the last equality holds by the definition of $\vec v^{(i)}$. {It remains to show that $\pr_{\x\sim \cN(\vec 0, \vec I)}[L_{f(\vec x)} ( \vec x)\ge \gamma]\geq \pr_{\x\sim \cN(\vec 0, \vec I)}[g(\vec x)\ge 1]$.
To see this, note that the event $L_{f(\vec x)} ( \vec x)\ge \gamma$ 
contains the event $g(\x)\ge 1$. If $g(\x)\ge 1$ then by definition there exists some $i\in [K]$ such that $\vec v^{(i)} \cdot \vec x\ge 1$, 
which implies that $L_{i} ( \vec x)\ge \gamma$ (recall that $\vec v^{(i)}= \vec w^{(i)}/(\gamma-t_i)$), 
and consequently, $L_{f(\vec x)} ( \vec x)\ge \gamma$. 
This gives that 
$\pr_{\x\sim \cN(\vec 0, \vec I)}[L_{f(\vec x)} ( \vec x)\ge \gamma]\geq \pr_{\x\sim \cN(\vec 0, \vec I)}[g(\vec x)\ge 1]$ 
and completes the proof of \Cref{clm:good-region}.}
\end{proof}

It remains to bound the contribution from the noise, i.e., the quantity $I_2$.
\begin{claim}\label{clm:bad-region}
    {There exists an absolute constant $C'>0$ and $\gamma>0$ such that \[I_2\geq - C' \left(\sqrt{\opt \pr_{x\sim \cN(\vec 0, \vec I)}[g(\vec x)\ge 1]}+\opt\right)\;,\] and $1/48\ge  \pr_{\x\sim \cN(\vec 0,\vec I)}[s(\vec x)\ge \gamma]\ge \pr_{\x\sim \cN(\vec 0,\vec I)}[s(\x)\neq 0]/96$.}
\end{claim}
\begin{proof}
We show how to bound the contribution of the terms {involving} $S_i$; similar arguments can be used to bound the contribution of the terms involving $T_i$.
    In order to bound  the contribution of the $S_i$ terms {we can write}
\begin{align*}
     \sum_{i=1}^K\E_{(\vec x,y)\sim D}[ \vec{v}^{(i)}\cdot \vec{x}S_i(\vec x,y)] 
     &\leq \sum_{i=1}^K\E_{(\vec x,y)\sim D}\left[  \max_{j\in [K]}(\vec v^{(j)}\cdot \vec{x})S_i(\vec x,y)\right]\\ 
     &=\E_{(\vec x,y)\sim D}\left[  \max_{j\in [K]}(\vec v^{(j)}\cdot \vec{x})S(\vec x,y)\right]\\
     &= \E_{\vec x \sim \cN(\vec 0,\vec I)}\left[ \max_{j\in [K]}(\vec v^{(j)}\cdot \vec{x})\E_{(\x,y)\sim D}[ S(\vec x,y)\mid \vec x]\right]\\&=\E_{\vec x \sim \cN(\vec 0,\vec I)}\left[ g(\x)\E_{(\x,y)\sim D}[ S(\vec x,y)\mid \vec x]\right]\;,
\end{align*}
where $S= \sum_{i=1}^K S_i$. 
{Since $f$ has misclassification error $\opt$, we have $\E_{(\vec x,y)\sim D}[S(\vec x,y)]\le \opt$.
Moreover, for each point $(\vec x,y)$, at most one among $\{S_i, T_i : i\in[K]\}$ can be nonzero which implies that $\E_{(\x,y)\sim D }[ S(\vec x,y)\mid \vec x]\le 1$. Define $R(\vec x)\eqdef \E_{(\x,y)\sim D}[ S(\vec x,y)\mid \vec  x]$. Then, by the law of total expectation, we have that $\E_{\vec x\sim \cN(\vec 0,\vec I)}[R(\vec x)]=\E_{(\vec x,y)\sim D}[S(\vec x,y)]\le \opt$ and $R(\vec x)\in [0,1]$.
}

{Recall that $g(\x)\geq 0$ by construction, hence for any $t>0$}, 
\begin{align*}
 \E_{\vec x \sim \cN(\vec 0,\vec I)}[ g(\vec x)R(\vec x)]&\leq  t\E_{\vec x \sim \cN(\vec 0,\vec I)}[R(\vec x)\1(g(\x)\leq t)]+\E_{\vec x \sim \cN(\vec 0,\vec I)}[ g(\vec x)R(\vec x)\1(g(\x)> t)]\\&\leq t \opt+ \E_{\vec x \sim \cN(\vec 0,\vec I)}[ g(\vec x)\1(g(\x)> t)]\;,    
\end{align*}
where we use the fact that $\E_{\vec x \sim \cN(\vec 0,\vec I)}[R(\x)]\leq \opt$ and $R(\x)\leq 1$. We next bound the term $\E_{\vec x \sim \cN(\vec 0,\vec I)}[ g(\vec x)\1(g(\x)> t)]$. Using that for any positive function $G(\x)$ it holds that $G(\x)=\int_{0}^\infty \1(G(\x)\geq z)\d z$, we have that
\begin{align}
    \E_{\vec x \sim \cN(\vec 0,\vec I)}[ g(\vec x)\1(g(\x)> t)]=\int_{t}^\infty \E_{\vec x \sim \cN(\vec 0,\vec I)}[\1(g(\x)\geq y)]\d y
    &\leq \int_{t}^\infty \E_{\vec x \sim \cN(\vec 0,\vec I)}[\1(g(\x)\geq \lfloor y\rfloor)]\d y\nonumber
      \\ &= \sum_{n=\lfloor t\rfloor}^\infty\pr_{\vec x \sim \cN(\vec 0,\vec I)}[ n\le g(\vec x)] \label{eq:one-eq}\;,
\end{align}
where in the first equality, we exchange the order of the integrals using Tonelli's theorem and in the  inequality, we used that $\1(g(\x)\geq y)$ is a decreasing function of $y$.
In order to bound the terms $\pr_{\vec x \sim \cN(\vec 0,\vec I)}[ n\le g(\vec x)]$, we show that {the tails of $g$ decay at least quadratically}. We show the following:
\begin{restatable}{lemma}{BRUTEFORCE}\label{lem:tails-gen}  Let $\vec w^{(1)},\dots, \vec w^{(K)}\in \mathbb{R}^d$, $t_1,\dots, t_K\in \mathbb{R}_-$ and 
$s:\mathbb{R}^d\to \mathbb{R}$ such that 
$s(\vec x)= \max_{i\in [K]} \max(\vec{w}^{(i)} \cdot \vec{x} +t_i,0)$. There exists  $\gamma>0$ such that  $$\pr_{\x\sim \cN(\vec 0,\vec I)}[s(\x)\neq 0]/96\le \pr_{\x\sim \cN(\vec 0,\vec I)}[s(\vec x)\ge \gamma]\le 1/48$$ and 
$$\pr_{\vec x\sim \cN(\vec 0,\vec I)}[s(\vec x)\ge n\gamma]\lesssim (1/n^2)\pr_{\vec x\sim \cN(\vec 0,\vec I)}[s(\vec x)\ge \gamma]$$ 
for all $n\in \mathbb{Z}_+$.
\end{restatable}
{For the proof of \Cref{lem:tails-gen}, we refer the reader to \Cref{sec:2.7}.}
 Now, we leverage \Cref{lem:tails-gen} to complete the proof. {Recall that the function $s(\x)$ of \Cref{lem:tails-gen} is defined as $s(\x)=\max_iL_i(\x)$. We need to show that $\pr_{\x\sim \mathcal{N}(\vec{0},\vec{I})}[g(\x)\geq n] \leq \pr_{\x\sim \mathcal{N}(\vec{0},\vec{I})}[s(\x)\geq n\gamma]$. First, note that \(
\pr_{\x\sim \mathcal{N}(\vec{0},\vec{I})}[g(\x)\geq 1] = \pr_{\x\sim \mathcal{N}(\vec{0},\vec{I})}[s(\x)\geq \gamma],
\) by construction.
{Furthermore, recall that $g(\x)=\max_{i\in [K]}\vec v^{(i)}\cdot\x$ where $\vec v^{(i)}=\vec w^{(i)}/(\gamma-t_i)$ and that $L_i(\x)=\vec w^{(i)}\cdot\x +t_i$. Since \( t_i \leq 0 \) for all \( i \in [K] \), it follows that if \( \w^{(i)}\cdot\x \geq n(\gamma - t_i) \), then \( \w^{(i)} \cdot\x\geq n\gamma - t_i \).} This, in turn, implies that
\(
\pr_{\x\sim \mathcal{N}(\vec{0},\vec{I})}[g(\x)\geq n] \leq \pr_{\x\sim \mathcal{N}(\vec{0},\vec{I})}[s(\x)\geq n\gamma].
\)
Therefore, using \Cref{lem:tails-gen} into \Cref{eq:one-eq}, we have that} 
\begin{align*}
    \E_{\vec x \sim \cN(\vec 0,\vec I)}[ g(\vec x)\1(g(\x)> t)] 
    \lesssim \sum_{n=\lfloor t\rfloor}^\infty\pr_{\vec x \sim \cN(\vec 0,\vec I)}[ 1\le g(\vec x)]/n^2\lesssim \frac{\pr_{\vec x \sim \cN(\vec 0,\vec I)}[ 1\le g(\vec x)]}{t-2}\;,
\end{align*}
where in the last inequality we used the approximation of the sum by an integral.
Therefore, by choosing $t= 2+ ( \pr_{\vec x \sim \cN(\vec 0,\vec I)}[1\le g(\vec x)]/\opt)^{1/2}$, we have  that 
\begin{align*}
    \E_{\vec x \sim \cN(\vec 0,\vec I)}[ g(\vec x)R(\vec x)]
    \lesssim t\opt+ \pr_{\vec x \sim \cN(\vec 0,\vec I)}[1\le g(\vec x)]/(t-2)\lesssim \sqrt{\opt\pr_{\vec x \sim \cN(\vec 0,\vec I)}[1\le g(\vec x)]}+\opt.
\end{align*}
For the noisy terms defined by $T_i$, we have that 
\begin{align*}
\sum_{i=1}^K\E_{(x,y)\sim D}[ \vec{v}^{(i)} \cdot \vec{x}T_i(\vec x,y)]&=-\sum_{i=1}^K\E_{(x,y)\sim D}[ \vec{v}^{(i)} \cdot (-\vec{x})T_i(\vec x,y)]=-\sum_{i=1}^K\E_{(x,y)\sim D'}[ \vec{v}^{(i)} \cdot \vec{x}T_i(\vec x,y)]\;,
\end{align*}
where $D'$ is a reflection of distribution $D$ with 
respect to $\vec 0$. As a result, we can arrive at the 
same bound as for the term defined by the $S_i$'s. 
Therefore, we conclude that the total contribution of the noise is at most 
$O(\sqrt{\opt\pr_{\vec x \sim \cN(\vec 0,\vec I)}[1\le g(\vec x)]}+\opt)${, which completes the proof of \Cref{clm:bad-region}}.
\end{proof}

{Now we show that $I_2\geq - \pr_{x\sim \cN(\vec 0, \vec I)}[g(\vec x)\ge 1]/2$.
Recall that by the assumptions of \Cref{lem:progress} we have
\[
\pr_{\x\sim \cN(\vec 0,\vec I)}\big[f(\x)\neq i\big]\ge C\,\opt+\eps \quad\text{for all } i\in [K].
\]
We show that $\pr_{\x\sim \cN(\vec 0,\vec I)}\big[s(\x)\neq 0\big]\ge \pr_{\x\sim \cN(\vec 0,\vec I)}\big[f(\x)\neq 1\big]$. To see this, observe that for any $i\in [K]$ for which $L_i$ is not identically zero, we have that $L_i(\x)\neq 0$ almost surely (with respect to the Gaussian measure). Moreover, if for some $i\in [K]$ the function $L_i$ is identically zero and $i\neq 1$, then without loss of generality we may assume that $\pr_{(\x,y)\sim D}[y=i]=0$, since we could set $1$ to be the label with minimum index in the set $\argmax_{i\in[K]} t_i$ (for the original biases $t_i$). Hence, in every case for any $i\neq 1$ we have $L_i(\x)\neq 0$ almost surely.
Therefore, if $f(\x)\neq 1$ then, except for a measure zero set it holds that $L_i(\x)> 0$ for some $i\neq 1$, which implies that $s(\x)\neq 0$. As a result, it holds that $\pr_{\x\sim \cN(\vec 0,\vec I)}\big[s(\x)\neq 0\big]\ge \pr_{\x\sim \cN(\vec 0,\vec I)}\big[f(\x)\neq 1\big]$.
Hence, it follows that
\[
\pr_{\x\sim \cN(\vec 0,\vec I)}\big[s(\x)\neq 0\big]\ge \pr_{\x\sim \cN(\vec 0,\vec I)}\big[f(\x)\neq 1\big]\ge C\,\opt+\eps.
\]
By the choice of $\gamma$ in \Cref{clm:bad-region}, we also have
\[
\pr_{\x\sim \cN(\vec 0,\vec I)}\big[s(\x)\ge \gamma\big]\ge \frac{1}{96}\,\pr_{\x\sim \cN(\vec 0,\vec I)}\big[s(\x)\neq 0\big].
\]
Since
$
\pr_{\x\sim \cN(\vec 0,\vec I)}\big[g(\x)\ge 1\big]=\pr_{\x\sim \cN(\vec 0,\vec I)}\big[s(\x)\ge \gamma\big]$,
we get that that $
\pr_{\x\sim \cN(\vec 0,\vec I)}[g(\x)\ge 1]\geq ({C\,\opt+\eps})/{96}$.
Finally, by choosing $C$ sufficiently large (in particular, ensuring that $C\geq 256/((C')^2+1)$), from \Cref{clm:bad-region}, we conclude that 
\[
I_2\ge -\frac{1}{2}\,\pr_{\x\sim \cN(\vec 0,\vec I)}\big[g(\x)\ge 1\big]\;.
\]
and that $\pr_{\x\sim \cN(\vec 0,\vec I)}\big[g(\x)\ge 1\big]\gtrsim \eps$.}
{Combining the above with \Cref{eq:2.4-2} and \Cref{clm:good-region}, we obtain}
\begin{equation}
    \sum_{i:\vec w^{(i)}\neq \vec 0} \vec w^{(i)}/(\gamma-t_i)\cdot\E_{(\x,y)\sim D}[\x\Ind(y=i)]\gtrsim \pr_{\x \sim \cN(\vec 0,\vec I)}[g(\x)\ge 1]\gtrsim\eps\;.\label{eq:almost-final-statement}\notag
\end{equation}
{This implies that there exists some $j\in [K]$ such that
\[
\frac{\vec w^{(j)}}{\gamma-t_j}\cdot \E[\vec x\,\Ind(y=j)]\gtrsim \eps/K.
\]}
{
To conclude the proof, we must show that $\gamma-t_j\ge \eps\,\norm{\vec{w}^{(j)}}$. Suppose, for the sake of contradiction, that for some $i\in [K]$ we have $\gamma-t_i\le \eps\norm{\vec{w}^{(i)}}$.
Then, by the definition of $L_i$, we have
\[
\pr_{\x \sim \cN(\vec 0,\vec I)}\Big[L_i(\x)\ge \gamma\Big]
=\pr_{x \sim \cN(0,1)}\Bigg[x\ge \frac{\gamma-t_i}{\norm{\vec{w}^{(i)}}}\Bigg]
\ge \pr_{x \sim \cN(0,1)}[x\ge \eps]\;.
\]
However, for $\eps<2\sqrt{\log(48)}$, \Cref{fact:gaussianfacts} implies that
\[
\pr_{\vec x \sim \cN(\vec 0,\vec I)}\Big[\max_{i\in [K]} L_i(\vec x)\ge \gamma\Big]> \frac{1}{48},
\]
which leads to a contradiction from \Cref{lem:tails-gen} (and the choice of $\gamma$). Thus, we must have $
\gamma-t_i\ge \eps\,\norm{\vec{w}^{(i)}}$, for all  $i\in [K]$.
}
{
As a result, we obtain
\[
\left(\frac{\vec w^{(j)}}{\norm{\vec w^{(j)}}}\right)\cdot\E_{(\vec x,y)\sim D}[\vec x\,\Ind(y=j)]
\gtrsim \eps^2/K.
\]
Finally, using the fact that we centered the affine functions defining $f$ by subtracting $\vec w^{(1)}\cdot \vec x+t_1$ from each of them, we obtain the desired statement.
This completes the proof of \Cref{lem:progress}.}

\end{proof}

\subsection{Finding a Relevant Direction}\label{sec:project-multiclass}

This section contains the  proof of \Cref{prop:alg2}, which states that if our current subspace approximation does not lead to a good enough error, then \Cref{alg:correlating-vectors} would return {at least one} vector with non-trivial projection onto the {hidden} subspace.

{\begin{proposition}\label{prop:alg2} Let $d,k, K\in \Z_+$, $\eps,\delta\in(0,1)$ and $\eps'\le c\eps^2/(K^2\sqrt{k})$, for some sufficiently small absolute constant $c>0$.
Let $D$ be a distribution supported on $ \mathbb{R}^d \times [K]$ whose $\x$-marginal is $\mathcal{N}(\vec 0, \vec I)$ and $V$ be a $k$-dimensional subspace of $\mathbb{R}^d$. Let $f$ be a multiclass linear classifier with weight vectors $\vec w^{(1)},\dots,\vec w^{(K)}\in \R^d$ and define $\opt=\pr_{(\x, y)\sim D}[f(\x)\neq y]$. 
 Furthermore, let $h_{\mathcal{S}}$ be a piecewise constant approximation of $D$, with respect to an $\eps'$-approximating partition $\mathcal{S}$ defined relative to the subspace $V$ (as in \Cref{def:approximatingPartition,def:h}).

    There exists a sample size $N=d( k/\eps')^{O(k)}(K/\eps)^{O(1)}\log(1/\delta)$ and a universal constant $C>0$, such that if $\pr_{(\x, y)\sim D}[h_{\mathcal{S}}(\x)\neq y]> C\cdot\opt+\eps$,  then \Cref{alg:correlating-vectors} given $N$ samples, runs in $\poly(N)$ time and 
    with probability at least $1-\delta$ returns a set of unit vectors $\mathcal E$ of size $|\mathcal E|=\poly(1/\eps,K)$ with the following property: 
    there exists a $\vec v\in \mathcal E$ and $i,j\in [K]$, $i\neq j$, 
    such that 
    $$\frac{(\vec{w}^{(i)}-\vec{w}^{(j)})^{\perp V}}{\norm{(\vec{w}^{(i)}-\vec{w}^{(j)})^{\perp V}}} \cdot \vec v\ge   \poly(1/K,\eps)\;.$$
\end{proposition}
}

Before proving \Cref{prop:alg2}, we  develop the necessary tools for working with  multiclass linear classifiers restricted to cubes  extending along the orthogonal complement of a subspace.
First, we define the notion of a projected classifier onto a set of linear equations  (\Cref{def:projectedClassifier}).

The intuition behind this definition is that when we restrict our 
classifier to sufficiently small cubes on \(V\), 
its behavior mainly depends  on its components in the orthogonal 
complement of \(V\). In other words, the classifier 
is nearly equivalent to one that fixes its component on \(V\) 
to a specific point. We denote by the \emph{projected classifier} 
of \(f\) with respect to a set of linear equations \(\mathcal{G}\) 
the classifier obtained by fixing the coordinates 
of \(\mathbf{x}^V\) according to \(\mathcal{G}\).

\begin{definition}[Projected Classifier]\label{def:projectedClassifier}
Let  $f(\x)=\argmax_{i\in [K]}(\vec w^{(i)}\cdot\x +t_i)$ with $\vec w^{(i)}\in \R^d,t_i\in \R$ for $i\in [K]$, let $\mathcal G=\{\vec v^{(1)}\cdot\x =z_1, \ldots, \vec v^{(k)}\cdot\x =z_k\}$ be a set of $k$ linear equations and let $V$ be the subspace spanned by orthonormal vectors $\vec v^{(1)},\ldots, \vec v^{(k)}\in \R^d$. We denote by $f_{\mathcal{G}}(\x)$ the linear classifier $f$ projected onto  $\mathcal G$, i.e., $f_{\mathcal{G}}(\x)=\argmax_{i\in [K]}((\vec w^{(i)})^{\perp_{V}}\cdot\x +t_i+t_{i,\cG}')$, where $t_{i,\cG}'=\sum_{j=1}^k z_j a_{ij}$ where $a_{ij}=\vec w^{(i)}\cdot \vec v^{(j)}$. 
\end{definition}
{
Now we prove \Cref{lem:projection}, which states that 
for any partition of a subspace into sufficiently small cubes, 
the target function is well-approximated by a piecewise  
function constructed using a projected classifier for each region.
This will be our main supporting result for proving \Cref{prop:alg2}, 
as it allows us to reduce the problem to instances on the orthogonal 
complement and use \Cref{lem:progress}.}

\begin{lemma}[Projection Lemma]\label{lem:projection}
     Let $d, K,k\in \Z_+,\eps, \eps' >0$, $f:\mathbb{R}^d\to [K]$ be a multiclass linear classifier,  $V$ be a $k$-dimensional subspace of $\mathbb{R}^d, k\ge 1$. 
     Let  $\mathcal{S}$ be an $\eps'$-approximating partition of $\mathbb{R}^d$ with respect to $V$ as in \Cref{def:approximatingPartition}{, with $\eps'\leq \eps^2/(K^2\sqrt{k})$}.  Then there exists a set  $\mathcal{T}\subseteq \mathcal{S}$ such that $\sum_{S\in \mathcal{T}} \pr_{\cN(\vec 0,\vec I)}[S]\ge 1-\eps$ and 
     for any $S\in \mathcal{T}$, it holds that
     $$ \pr_{\x\sim \cN(\vec 0,\vec I)}[f(\x)\neq f_{\cG_{S}}(\x)\mid \x\in S] \le \eps \;,$$
where $f_{\mathcal G_{S}}$ {denotes} the projected classifier onto a set of linear equations $\mathcal G_{S}$ that {is associated with}  $S$, as defined in \Cref{def:projectedClassifier}.
\end{lemma}
\begin{proof}
Let $\vec w^{(1)},\dots,\vec w^{(K)}\in \mathbb{R}^d$ and $t_1,\dots,t_K\in \mathbb{R}$ such that $f(\x)=\argmax_{i\in [K]}(\vec w^{(i)}\cdot \x+t_i)$.
Let $\vec v^{(1)},\dots, \vec v^{(k)}$ be an orthonormal basis of $V$.
Fix a set of the partition $S\in \mathcal{S}$ and denote by $f_{\mathcal G_{S}}=\argmax((\vec w^{(i)})^{\perp {V}}\cdot\x +t_i+t_{i}')$ ({see \Cref{def:projectedClassifier} for the definition of $t_i'$}) the projected classifier that corresponds to the set of linear equations $\mathcal G_{S}=\{\vec v^{(1)}\cdot\x =z_1, \ldots,\vec v^{(k)}\cdot\x =z_k\}$ ({see \Cref{def:approximatingPartition} for the definition of $\mathcal G_{S}$}). 
Let $\vec w^{(i,j)}=\vec w^{(i)}-\vec w^{(j)}$, $t_{i,j}=t_i-t_j$ and $t_{i,j}'=t_i'-t_j'$. 
By the union bound, we have that   
\begin{eqnarray*}  
&\pr_{\x\sim \cN(\vec 0,\vec I)}[f(\x)\neq f_{\cG_{S}}(\x)\mid \x\in S] \le\\ 
& \sum_{i\neq j} \pr_{\x\sim \cN(\vec 0,\vec I)}[\sign(\vec w^{(i,j)}\cdot \x+t_{i,j}) \neq \sign((\vec w^{(i,j)})^{\perp V}\cdot \x+t_{i,j}+t'_{i,j})\mid \x\in S]\;.
\end{eqnarray*}    
{ Fix $i,j\in[K]$ with $i\neq j$. 
Notice that for a sign change to occur, i.e., $\sign(\vec w^{(i,j)}\cdot \x+t_{i,j}) \neq \sign((\vec w^{(i,j)})^{\perp V}\cdot \x+t_{i,j}+t'_{i,j})$ it is necessary that
\[
\Bigl|\vec w^{(i,j)}\cdot \x+t_{i,j}\Bigr|\le \Bigl|\w^{(i,j)}\cdot \x^{V}-t_{i,j}'\Bigr|.
\]
To see that, note that we can add and subtract $\w^{(i,j)}\cdot \x^{V}$, and we get that if for some $\x\in S$ we have $\sign(\vec w^{(i,j)}\cdot \x+t_{i,j}) \neq \sign((\vec w^{(i,j)})^{\perp V}\cdot \x+t_{i,j}+t'_{i,j})$ we also have,
$\sign(\vec w^{(i,j)}\cdot \x+t_{i,j}) \neq \sign((\vec w^{(i,j)})\cdot \x+t_{i,j}+t'_{i,j}-\w^{(i,j)}\cdot \x^{V})$, hence it should hold that  $\Bigl|\vec w^{(i,j)}\cdot \x+t_{i,j}\Bigr|\le \Bigl|\w^{(i,j)}\cdot \x^{V}-t_{i,j}'\Bigr|$. Therefore, it suffices to bound the probability that $\pr[|\vec w^{(i,j)}\cdot \x+t_{i,j}|\le |\w^{(i,j)}\cdot \x^{V}-t_{i,j}'|\mid \x\in S]$.
}
{By the definition of $f_{\cG_{S}}$, if $\x\in S$ then for any $i,j\in [K]$ we have
\[
\left|\w^{(i,j)}\cdot \x^{V}-t_{i,j}'\right|
=\left|\sum_{l=1}^{k}(\vec w^{(i,j)}\cdot \vec v^{(l)})(\x\cdot \vec v^{(l)}-z_{l})\right|.
\]
This equality holds because by definition the projection of $\x$ onto the subspace $V$ with orthonormal basis $\{\vec v^{(1)},\ldots,\vec v^{(k)}\}$ is given by $
\x^{V}=\sum_{l=1}^{k}(\x\cdot \vec v^{(l)})\vec v^{(l)}$,
and the term $t_{i,j}'$ is defined as
$
t_{i,j}'=\sum_{l=1}^{k}z_l\,(\vec w^{(i,j)}\cdot \vec v^{(l)})$.
Next, by applying the triangle inequality we have
\[
\left|\sum_{l=1}^{k}(\vec w^{(i,j)}\cdot \vec v^{(l)})(\x\cdot \vec v^{(l)}-z_{l})\right|
\le \sum_{l=1}^{k}\Bigl|\vec w^{(i,j)}\cdot \vec v^{(l)}\Bigr|\Bigl|\x\cdot \vec v^{(l)}-z_{l}\Bigr|.
\]
By the definition of the partition $S$, if $\x\in S$ then for every $l\in \{1,\ldots,k\}$ the projection of $\x$ onto the direction $\vec v^{(l)}$ is $\eps'$-close to the corresponding value $z_l$, i.e., $
\Big|\x\cdot \vec v^{(l)}-z_{l}\Big|\le \eps'$.
Therefore,
\[
\sum_{l=1}^{k}\Bigl|\vec w^{(i,j)}\cdot \vec v^{(l)}\Bigr|\Bigl|\x\cdot \vec v^{(l)}-z_{l}\Bigr|
\le \sum_{l=1}^{k}\Bigl|\vec w^{(i,j)}\cdot \vec v^{(l)}\Bigr|\eps'.
\]
Finally, by applying the \CS inequality, we obtain
\[
\sum_{l=1}^{k}\Bigl|\vec w^{(i,j)}\cdot \vec v^{(l)}\Bigr|
\le \sqrt{k}\,\sqrt{\sum_{l=1}^{k}\Bigl(\vec w^{(i,j)}\cdot \vec v^{(l)}\Bigr)^2}
=\sqrt{k}\,\|(\vec w^{(i,j)})^{V}\|\;.
\]
 Combining the above, we have $
\left|\w^{(i,j)}\cdot \x^{V}-t_{i,j}'\right|
\le \|(\vec w^{(i,j)})^{V}\|\sqrt{k}\,\eps'$.
}
Averaging over $S\in \mathcal{S}$ gives us 
\begin{align*}
    \sum_{S\in \mathcal{S}} \pr[f(\x)\neq f_{\cG_{S}}(\x)\mid \x\in S] \pr[S]&\le \sum_{i\neq j\in [K]}\sum_{S\in \mathcal{S}} \pr[\abs{\vec w^{(i,j)}\cdot \x+t_{i,j}}\le  \|(\vec w^{(i,j)})^{V}\|\sqrt{k} \eps'\mid \x\in S] \pr[S]\\
    &\le \sum_{i\neq j\in [K]}\pr[\abs{\vec w^{(i,j)}\cdot \x+t_{i,j}}\le  \|(\vec w^{(i,j)})^{V}\|\sqrt{k} \eps']
    \\
    &\le \sum_{i\neq j\in [K]}\pr[\abs{\vec w^{(i,j)}\cdot \x+t_{i,j}}\le  \|\vec w^{(i,j)}\|\sqrt{k} \eps']
    \\&\le \sum_{i\neq j\in [K]} \sqrt{k}\eps'
    \le K^{2}\sqrt{k}\eps' \;,
\end{align*}
where we used the anti-concentration of the Gaussian distribution (see \Cref{fact:gaussianfacts}).
{Finally, since $\eps'\leq \eps^2/(K^2\sqrt{k})$ it holds that $\sum_{S\in \mathcal{S}} \pr[f(\x)\neq f_{\cG_{S}}(\x)\mid \x\in S] \pr[S]\leq \eps^2$. To conclude the proof, let ${\mathcal S}_{bad}$ be the set of $S\in\mathcal S$ such that if $S\in {\mathcal S}_{bad}$ then $\pr[f(\x)\neq f_{\cG_{S}}(\x)\mid \x\in S]>\eps$. Then, we have that
\[
 \eps^2\geq \sum_{S\in \mathcal{S}} \pr[f(\x)\neq f_{\cG_{S}}(\x)\mid \x\in S] \pr[S]\geq \sum_{S\in {\mathcal S}_{bad}} \pr[f(\x)\neq f_{\cG_{S}}(\x)\mid \x\in S] \pr[S]>\eps \sum_{S\in {\mathcal S}_{bad}}\pr[S]\;.
\]
Hence, $\sum_{S\in {\mathcal S}_{bad}}\pr[S]\leq \eps$. By taking $\mathcal T= \mathcal S/{\mathcal S}_{bad}$ we conclude the proof.}
\end{proof}

We now proceed with the proof of \Cref{prop:alg2}.
\begin{proof}[Proof of \Cref{prop:alg2}]
Let $\vec w^{(1)},\ldots,\vec w^{(K)} \in \mathbb{R}^d$ and $t_1,\ldots,t_K\in \mathbb{R}$ be the parameters defining $f$, i.e., $f(\x)=\argmax_{i\in [K]}( \vec w^{(i)}\cdot \x+t_i)$. 
{For brevity, define the difference vectors $\vec w^{(i,j)}=\vec w^{(i)}-\vec w^{(j)}$, and let $(\widehat{\vec w}^{(i,j)})^{\perp V}=(\vec w^{(i,j)})^{\perp V}/\norm{(\vec w^{(i,j)})^{\perp V}}$ denote the normalized projection of $\w^{(i,j)}$ onto the orthogonal complement of a given subspace $V$. Also, define $W=\spaning(\w_1,\dots,\w_K)$, and let $C$ be a constant as specified in \Cref{lem:progress}.
For each cube $S\in \mathcal{S}$ (as defined in \Cref{def:approximatingPartition}), let $\cG_S$ be the corresponding set of linear equations, and let $f_{\cG_S}$ be the projection of $f$ onto $\cG_S$ (see \Cref{def:projectedClassifier}).}
{Furthermore, for every $S\in\mathcal S$ and $i\in [K]$, define the population moment $\vec u^{(S,i)}=\E_{(\x,y)\sim D}[\x\Ind(y=i)\mid \x\in S ]$,and let $\widehat{\vec u}^{(S,i)}=\E_{(\x,y)\sim \widehat{D}}[\x\Ind(y=i)\mid \x\in S ]$ its empirical counterpart (see Line \ref{line:moments}).} Also, denote by $\widehat{\vec U}$ the matrix computed at Line \ref{line:matrix} of \Cref{alg:correlating-vectors}.

{We first show that if the current subspace $V$ does not lead to a good approximation of $f$, meaning that any function defined only on $V$ gets error at least $\Omega(\opt+\eps)$ with the labels $y$, then when we partition the space into small cubes with respect to $V$, on a nontrivial fraction of them $f$ is far from being constant.}
\begin{claim}[{Large} Fraction of Good {Cubes}]\label{it:LargeError2FractionofStrips}
    If $\pr_{(\x, y)\sim D}[h_{\mathcal{S}}(\x)\neq y]> (C+1)\opt+5\eps$, then there exists a subcollection $\mathcal{T}\subseteq \mathcal{S}$ satisfying $\sum_{S\in \mathcal{T}}\pr_{\x\sim \cN(\vec 0,\vec I)}[S]\ge \eps$; and for every $S\in \mathcal{T}$ we have $\min_{i\in [K]}\pr_{(\x,y)\sim D}[f_{\cG_{S}}(\x)\neq i\mid \x\in S]> C\pr_{(\x,y)\sim D}[f_{\cG_{S}}(\x)\neq y\mid \x\in S]+\eps$.
\end{claim}

\begin{proof}[Proof of \Cref{it:LargeError2FractionofStrips}]

{Define $\mathcal{T}$ to be the collection of $S\subseteq\mathcal{S}$ such that $\min_{i\in[K]}\pr[i\neq y\mid \x\in S]> (C+1)\pr[f(\x)\neq y\mid \x\in S]+2\eps$.}
Assume for the sake of contradiction that $\sum_{S\in \mathcal{T}}\pr_{\x\sim \cN(\vec 0,\vec I)}[S]\leq 2\eps$. We have that 
\begin{align}
    \pr_{(\x,y)\sim D}[h_{\mathcal{S}}(\x)\neq y] &\le \sum_{S\in \mathcal{S}}\min_{\widehat{y}\in [K]} \pr_{(\x,y)\sim D}[\widehat{y}\neq y\mid \x\in S]\pr_{\x\sim \cN(\vec 0,\vec I)}[S]+\eps \notag\\
    &\le \sum_{S\not\in \mathcal{T}}\left((C+1)\pr_{(\x,y)\sim D}[f(\x)\neq y\mid \x\in S]+2\eps\right)\pr_{\x\sim \cN(\vec 0,\vec I)}[S]+3\eps \notag\\
    &\le (C+1)\opt+5\eps \notag\;,
\end{align}
{where in the first line we used the fact (see \Cref{fact:gaussianfacts}) that the mass of points outside the discretization is at most $\eps$, in the second line we applied our definition of $\mathcal{T}$ and in the last inequality we used that the total error of $f$ is $\opt$. This contradicts our hypothesis that $\pr[h_{\mathcal{S}}(\x)\neq y] > (C+1)\opt+5\eps$. Hence, it must be that $\sum_{S\in \mathcal{T}} \pr_{\x\sim \cN(\vec 0,\vec I)}[S] > 2\eps$.}

{Next, {recall that the cube width $\eps'$ is chosen to be less than $c\eps^2/(K^2\sqrt{k})$ for a sufficiently small constant $c$.}
{Taking $c<1/(C+1)$, we can apply}  \Cref{lem:projection} and get that there exists a collection $\mathcal T'\subseteq\mathcal S$ with $\sum_{S\in \mathcal T'}\pr[S]\geq 1-\eps$, such that $\pr_{\x\sim \cN(\vec 0,\vec I)}[f(\x)\neq f_{\cG_{S}}(\x)\mid \x\in S] \le \eps/(C+1) $. Note that for $\mathcal T''=\mathcal T'\cap \mathcal T$, it must hold that $\sum_{S\in \mathcal T''}\pr[s]\geq \eps$. For all $ S\in \mathcal{T}''$, it holds that $\min_{i\in [K]}\pr_{\x\sim \cN(\vec 0,\vec I)}[i\neq y\mid \x\in S]> (C+1)\pr_{(\x,y)\sim D}[f_{\cG_{S}}(x)\neq y\mid \x\in S]+\eps$. Since,
\[
\pr_{\x\sim \cN(\vec 0,\vec I)}\Bigl[i\neq y \mid \x\in S\Bigr]
\le \pr_{\x\sim \cN(\vec 0,\vec I)}\Bigl[f_{\cG_{S}}(\x)\neq i \mid \x\in S\Bigr]
+\pr_{\x\sim \cN(\vec 0,\vec I)}\Bigl[f_{\cG_{S}}(\x)\neq y \mid \x\in S\Bigr]\;,
\]
we get that  $\pr_{\x\sim \cN(\vec 0,\vec I)}[f_{\cG_{S}}(\x)\neq i\mid \x\in S]> C\pr_{(\x,y)\sim D}[f_{\cG_{S}}(\x)\neq y\mid \x\in S]+\eps$ for all $S\in \mathcal{T}''$, which  completes the proof.}
\end{proof}

In the next claim, we show that if there exists a {cube} where the target function is far  from being constant, then we can find a  vector that correlates with the unknown parameters.
\begin{claim}[Correlating Moments]
\label{it:GoodStript2GoodCorrelation}
    Let $S\in \mathcal{S}$  such that $\min_{i\in [K]}\pr_{(\x,y)\sim D}[f_{\cG_{S}}(\x)\neq i\mid \x\in S]> C\pr_{(\x,y)\sim D}[f_{\cG_{S}}(\x)\neq y\mid \x\in S]+\eps$. Then there exists {$i,j,l\in [K]$ with $i\neq j$  such that  $(\widehat{\vec w}^{(i,j)})^{\perp V}\cdot \vec u^{(S,l)}\ge \poly(1/K,\eps)$.}
\end{claim}
\begin{proof}[Proof of \Cref{it:GoodStript2GoodCorrelation}]
 First, notice that  the output label of $f_{\cG_{S}}$ is only determined by the projection onto the orthogonal complement of $V$, i.e.,  
\begin{align*}
    f_{\cG_{S}}(\x)= \argmax_{i \in [K]}((\vec w^{(i)})^{\perp V}\cdot\x +t_i+t_{\cG_{S},i}')=\argmax_{i \in [K]}((\vec w^{(i)})^{\perp V}\cdot\x^{\perp V} +t_i+t_{\cG_{S},i}')= f_{\cG_{S}}(\x^{\perp {V}}).
\end{align*}
Furthermore, {observe that whether $\x$ belongs to $S$ depends only on its projection onto $V$}, hence without loss of generality it can written as $\Ind(\x^{V}\in S)$.
 Thus, we have that
\begin{align*}
    \pr_{(\x,y)\sim D}[f_{\cG_{S}}(\x)\neq y\mid \x\in S]&=\pr_{(\x,y)\sim D}[f_{\cG_{S}}(\x^{\perp V})\neq y\mid \x\in S]\\
    &=\sum_{y=1}^K\int_{\mathbb{R}^d}\Ind(f_{\cG_{S}}(\x^{\perp V})\neq y)\Ind(\x^{V}\in S)
    D(\x,y)d\x /\pr[\x^{V}\in S]\\
    &=\sum_{y=1}^K\int_{V^{\perp}}\Ind(f_{\cG_{S}}(\x^{\perp V})\neq y)\left(\frac{1}{\pr[\x^{V}\in S]}\int_{V} 
    D(\x,y)\Ind(\x^{V}\in S)d\x^{V}\right) d\x^{\perp V}\\
    &=\pr_{(\x,y)\sim D_{V^\perp}^S}[f_{\cG_{S}}(\x)\neq y]\;,
\end{align*}
where {$D_{V^\perp}^S$ denotes the marginal distribution on $V^\perp$ obtained by averaging $D$ over $V$ while conditioning on the event $\x^{V}\in S$}.
Similarly, we have $\pr_{(\x,y)\sim D_{V^\perp}^S}[f_{\cG_{S}}(\x)\neq i]> C\pr_{(\x,y)\sim D_{V^\perp}^S}[f_{\cG_{S}}(\x)\neq y]+\eps$ for all $i\in [K]$.
 Moreover, by the properties of the standard Gaussian distribution, the $\x$-marginal of $D_{V^\perp}^S$ is a standard Gaussian.
  
 As a result, by applying \Cref{lem:progress} for $f_{\cG_{S}}$, we have that there exist $i,j\in [K], i\neq j$ such that $(\widehat{\vec w}^{(i,j)})^{\perp V}\cdot \E_{(\vec x,y)\sim D^S_{V^{\perp}}}[ \vec{x} \Ind(y=i) ]\ge \poly(1/K,\eps)$. Furthermore, note  that by averaging over $V$ conditioning on $S$ we have that $\E_{(\vec x,y)\sim D}[ \vec{x}^{\perp V} \Ind(y=i)\mid \x\in S ]=\E_{(\vec x,y)\sim D_{V^\perp}^S}[ \vec{x} \Ind(y=i) ]$. Therefore, there exist $i, j\in [K], i\neq j$,  such that $(\widehat{\vec w}^{(i,j)})^{\perp V}\cdot \vec u^{(S,l)}\ge \poly(1/K,\eps)$ for some $l\in [K]$.

\end{proof}

In the next claim, we show that by taking enough samples, the empirical moments are close to the population ones.
\begin{claim}
    [Concentration of Empirical Distribution]\label{it:MommentConcentration}
    {Let $p(\eps, 1/K)$ be a sufficiently small polynomial in \( \eps \) and \( 1/K \).}
    If {$N= d(k/\eps')^{Ck}(K/\eps)^{C}\log(1/\delta)$ for a sufficiently large universal constant $C>0$}, then  
    with probability $1-\delta$ we have that $\norm{(\widehat{\vec u}^{(S,i)}-{\vec u}^{(S,i)})^{\perp V} }\le p(\eps,1/K)$ and $\pr_{\widehat{D}}[S]\ge\pr_{D}[S]/2$  for all $S\in \mathcal{S}, i\in [K]$.
\end{claim}
\begin{proof}[Proof of \Cref{it:MommentConcentration}]
First note that by the definition of an approximating partition (see \Cref{def:approximatingPartition}) and \Cref{fact:gaussianfacts}, for all $S\in \mathcal{S}$ we have that   $\pr_{D}[S ]= (\eps'/k)^{\Omega(k)}$ and $\abs{\mathcal{S}}= ( k/\eps')^{O(k)}$.
Therefore, by the union bound and Hoeffding's inequality, if $N\ge (k/\eps')^{Ck}\log(\abs{\mathcal{S}}/\delta)= (k/\eps')^{O(k)}\log(1/\delta) $ for a sufficiently large absolute constant $C>0$, with probability at least $1-\delta$  we have that  $\abs{\pr_{\widehat{D}}[S]-\pr_{D}[S]}\le \pr_{D}[S]/2$ for all $S\in \mathcal{S}$.
Hence, $\pr_{\widehat{D}}[S]\ge \pr_{D}[S]/2$, i.e., the number of samples that fall in each set $S\in \mathcal{S}$ is at least $N\pr_{D}[S]/2= N(\eps'/k)^{\Omega(k)}$.

{Furthermore, note that by the definitions of ${\vec u}^{(S,i)} $ and $\widehat{\vec u}^{(S,i)}$ it holds that
\begin{align*}
(\widehat{\vec u}^{(S,i)}-{\vec u}^{(S,i)})^{\perp V}&= \E_{(\x,y)\sim \widehat{D}}[\x^{\perp V}\Ind(y=i)\mid \x\in S ]- \E_{(\x,y)\sim D}[\x^{\perp V}\Ind(y=i)\mid \x\in S]    
\end{align*}
Therefore, to prove our concentration result, it suffices to show that when conditioning on $\x \in S$ the random vector $\x^{\perp V} \Ind(y=i)$ with $(\x,y)\sim D$ is  subgaussian in every unit direction with subgaussian norm bounded by a constant.
To see this, note that whether $\x$ belongs in $S$ depends only on the projection of $\x$ onto $V$, and $\x^{\perp V}$ is independent of $\x^{V}$. Now fix a unit vector $\vec v\in \R^d$, we have that for any $p\geq 1$ 
\begin{align*}
 \E_{(\x,y)\sim D}\left[\left(\vec v\cdot \x^{\perp V} \Ind(y=i)-\E_{(\x,y)\sim D}[\vec v\cdot \x^{\perp V}\Ind(y=i)\mid \x\in S]\right)^p\mid S \right]^{1/p}&\lesssim
 \E_{(\x,y)\sim D}\left[\left(\vec v \cdot\x^{\perp V} \Ind(y=i)\right)^p \mid S\right]^{1/p}\\
 &\leq \E_{(\x,y)\sim D}\left[\left(\vec v\cdot\x^{\perp V} \right)^p \right]^{1/p}\\
 &\leq \E_{z\sim \cN(0,1)}\left[(z)^p \right]^{1/p}\lesssim \sqrt{p} \;,
\end{align*}
where in the first inequality we used the triangle and Jensen's inequalities, while in the second we exploited the independence of $\mathbf{x}^{V}$ and $\mathbf{x}^{\perp V}$ and the fact that the indicator function is bounded by 1.
Therefore $\x^{\perp V} \Ind(y=i)$ is subgaussian.}  Hence, by the use of a standard covering argument it follows that if the number of samples that falls in each discretization set $S$ is at least $d\poly(K,1/\eps)\log(K/\delta)$ {for a sufficiently large polynomial,} then $\norm{(\widehat{\vec u}^{(S,i)}-{\vec u}^{(S,i)})^{\perp V}}\le p(\eps,1/K)$ for all $i\in [K]$.

Combining the  above results along with the union bound gives us that {there exists an $N= d(k/\eps')^{O(k)}(K/\eps)^{O(1)}\log(1/\delta)$ that} suffices.
\end{proof}

Applying \Cref{it:LargeError2FractionofStrips,it:GoodStript2GoodCorrelation}, we have that if $\pr_{(\x, y)\sim D}[h_{\mathcal{S}}(\x)\neq y]> (C+1)\opt+\eps$ then there exists a subset $\mathcal{T}\subseteq \mathcal{S}$ of the $\eps'$-approximating partition $\mathcal S$ with $\sum_{S\in \mathcal{T}}\pr_{D}[S]= \Omega(\eps)$ such that for all $S\in \mathcal{T}$ we have $(\widehat{\vec w}^{(i,j)})^{\perp V}\cdot \vec u^{(S,l)}\ge \poly(1/K,\eps)$ for some $i,j,l\in [K], i\neq j$.
Moreover, if {$N= d(k/\eps')^{C'k}(K/\eps)^{C'}\log(1/\delta)$ for a sufficiently large universal constant $C'>0$}, using the concentration result of \Cref{it:MommentConcentration}, we have that the empirical moments of each set satisfy the same guarantees, i.e., $(\widehat{\vec w}^{(i,j)})^{\perp V}\cdot\widehat{\vec u}^{(S,l)}\ge \poly(1/K,\eps)$, for some $i,j,l\in [K], i\neq j,$ for all $ S\in \mathcal{T}$.  Furthermore, note that by \Cref{it:MommentConcentration} we also have that $\sum_{S\in \mathcal{T}}\pr_{\widehat{D}}[S]= \Omega(\eps)$.

Therefore, there exists a subset $\mathcal{T}'\subseteq \mathcal{T}$ and $i,j\in [K]$ with $\sum_{S\in \mathcal{T}'}\pr_{\widehat{D}}[S]\ge \Omega(\eps/K^2)$ such that for all $S\in \mathcal{T}'$ we have that  $(\widehat{\vec w}^{(i,j)})^{\perp V}\cdot \widehat{\vec u}^{(S,l)}\ge \poly(1/K,\eps)$ for some $l\in [K]$.
Hence, for some $i\neq j\in [K]$ we have that 
$$\left((\widehat{\vec w}^{(i,j)})^{\perp V}\right)^{\top}\widehat{\vec U}(\widehat{\vec w}^{(i,j)})^{\perp V}\ge \frac{\eps}{K^2}\poly(\eps,1/K)=\poly(1/K,\eps)\;,$$
where we used the fact that $\x\x^\top$ is PSD for any $\x\in \R^d$.

{Since the set $S$ depends only on the projection $\x^V$, we have $(\vec u^{(S,i)})^{\perp V}= \E_{\x^{\perp V}}[\x^{\perp V}\E_{\x^{V}}[\Ind(y=i)\mid \x\in S]]$. Therefore, we can bound the norm of $(\vec u^{(S,i)})^{\perp V}$ as follows:}
\begin{align*}
\norm{(\vec u^{(S,i)})^{\perp V}}&\leq \max_{\norm{\vec u}=1}\E_{\x^{\perp V}}[\vec u\cdot\x^{\perp V}\E_{\x^{V}}[\Ind(y=i)\mid \x\in S]]\\&\leq \max_{\norm{\vec u}=1}(\E_{\x^{\perp V}}[(\vec u\cdot\x^{\perp V})^2]\E_{(\x,y)\sim D}[\Ind(y=i)\mid \x\in S]])^{1/2}\le 1\;,
\end{align*}
where we used the \CS inequality and the dual characterization of the $\ell_2$ norm. {In the last inequality we used that $\x^{\perp V}$ is a standard Gaussian, and henceforth the first expectation is at most $1$, and the second expectation is at most $1$ as well because $\Ind(y=i)$ is bounded by $1$.}
By the triangle inequality, it follows that 
$$\norm{\widehat{\vec U}}_{F}\le \sum_{ S\in \mathcal{S},i\in [K]}\|(\widehat{\vec u}^{(S,i)})^{\perp V}\|\pr_{\widehat{D}}[S] \le \sum_{ S\in \mathcal{S},i\in [K]}(1+\poly(\eps,1/K))\pr_{\widehat{D}}[S]\le 2K\;.$$

{Hence, by \Cref{cl:coorelationofeigenvalues}, there exists a unit eigenvector $\vec v$ of $\widehat{\vec U}$ with eigenvalue at least $\poly(\eps/K)$ (for a sufficiently small polynomial)}, such that  
$(\widehat{\vec w}^{(i,j)})^{\perp V}\cdot \vec v\ge  \poly(1/K,\eps)$,
for some $i\neq j\in [K]$. Moreover, the number of such eigenvectors is at most $2K/\poly(1/K,\eps)=\poly(K,1/\eps)$. This completes the proof of \Cref{prop:alg2}.
\end{proof}

{Before proving our main result of this section, \Cref{thm:constantApprox}, we first show that if we have learned one direction of our target model to sufficiently high accuracy, then we may, without loss of generality, restrict the model to the remaining directions. Specifically, we establish the following:}

\begin{lemma}\label{lemma:projectIfSmall}
     Let $K\ge 2,\eps\in (0,1)$ and $f(\x)=\argmax_{i\in [K]}(\vec w^{(i)}\cdot\x +t_i)$ with $\vec w^{(i)}\in \R^d, t_i\in \R$ for $i\in [K]$ be a multiclass linear classifier. Let $V$ be a {non-empty} subspace of $\R^d$ such that for all $i,j\in[K]$ with $i\neq j$ it holds that $\|(\vec w^{(i)}-\vec w^{(j)})^{V}\|\leq (\eps^2/K^4) \|\vec w^{(i)}-\vec w^{(j)}\|$. Then, it holds that
    \[
    \pr[f(\x)\neq f^{\perp V}(\x)]\lesssim \eps\;,
    \]
    where $f^{\perp V}(\x)=\argmax((\vec w^{(i)})^{\perp V}\cdot\x +t_i)$.
\end{lemma}
\begin{proof}
For simplicity, let $\vec w^{(i,j)}=\vec w^{(i)}-\vec w^{(j)}$ and $t_{i,j}=t_i-t_j$.
    Using the union bound inequality, we have that
    \[
    \pr[f(\x)\neq f^{\perp V}(\x)]\leq \sum_{i,j\in [K]}\pr[\sign(\vec w^{(i,j)}\cdot\x +t_{i,j})\neq \sign((\vec w^{(i,j)})^{\perp V}\cdot\x +t_{i,j})]\;.
    \]
    {We can write $\vec w^{(i,j)}=(\vec w^{(i,j)})^{V}+(\vec w^{(i,j)})^{\perp V}$. According to the assumptions of \Cref{lemma:projectIfSmall}, it holds that} $\|(\vec w^{(i,j)})^{V}\|\le (\eps^2/K^4)  \|\vec w^{(i,j)}\|$ and $\|(\vec w^{(i,j)})^{\perp V}\|\ge \sqrt{1- (\eps^4/K^8)}  \|\vec w^{(i,j)}\|$. 
    Therefore, using the independence of any two orthogonal directions of a standard Gaussian, along with \Cref{fact:gaussianfacts}, we have that
    \begin{align*}
        \pr[\sign(\vec w^{(i,j)}\cdot\x +t_{i,j})\neq \sign((\vec w^{(i,j)})^{\perp V}\cdot\x +t_{i,j})]&\leq \pr[|(\vec w^{(i,j)})^{V} \cdot \x|\geq 
        \abs{(\vec w^{(i,j)})^{\perp V}\cdot\x +t_{i,j}}]
        \\&\le \pr_{z\sim \cN(0,1)}[|z|\geq \frac{\|(\vec w^{(i,j)})^{\perp V}\|(\eps/K^2)}{\|(\vec w^{(i,j)})^{V}\|}]+\eps/K^2
        \\&\le \pr_{z\sim \cN(0,1)}[|z|\geq \frac{\sqrt{1-\eps^{4}/K^8}}{\eps/K^2}]+\eps/K^2\le 2\eps/K^2\;.
    \end{align*} 
    Therefore, we have that 
    \[
    \pr[f(\x)\neq f^{\perp V}(\x)]\le 2\eps\;.
    \]
\end{proof}

\subsection{Proof of \Cref{thm:constantApprox}}
\label{sec:proofMulticlass}

In this section, given \Cref{prop:alg2}, we proceed to the proof of \Cref{thm:constantApprox}. {In particular, we show that \Cref{alg:constantApprox} returns, with probability at least $1-\delta$, a hypothesis $h$ whose 0-1 error is $O(\opt)+\eps$.}
\begin{proof}[Proof of \Cref{thm:constantApprox}]
Let $\vec w^{*(1)},\ldots,\vec w^{*(K)} \in \mathbb{R}^d$ and $t_1^*,\ldots,t_K^*\in \mathbb{R}$ be the parameters of the affine functions of an optimal classifier $f^*$, i.e., $f^*(\x)=\argmax_{i\in [K]}( \vec w^{*(i)}\cdot \x+t_i^*)$ and denote by $W^*=\spaning(\vec w^{*(1)},\ldots,\vec w^{*(K)})$. 
Let $C>0$ be {the} constant that satisfies the statement of \Cref{prop:alg2}.
{The algorithm maintains a list of vectors $L_t$ (updated at Line~\ref{line:update}), and define
$V_t=\mathrm{span}(L_t)$, $\dim(V_t)=k_t$.}
{We set $\eps'=\poly(\eps,1/K)$, for a sufficiently small polynomial (Line \ref{line:init_vars}) which will be quantified later, and for each $t\in [T]$ let $\mathcal{S}_t$ be an} arbitrary $\eps'$-approximating partitions with respect to $V_t$ (see \Cref{def:approximatingPartition}).
Let $h_t:\mathbb{R}^d\to [K]$ be {the} piecewise constant classifier defined {by $h_t= h_{\mathcal{S}_{t}}$ as in \Cref{def:h},} for the distribution $D$. 

{To prove the correctness of \Cref{alg:constantApprox}, we need to show that if $h_t$ has significant error, then the algorithm improves its current approximation $V_t$ to $W^*$.
To quantify this improvement, we define the potential function}
\[
\Phi_t = \sum_{i=1}^{K}\norm{(\vec b^{(i)})^{\perp V_t}}^2\;,
\]
where $\vec b^{(1)},\dots,\vec b^{(K)}$ is an orthonormal basis of $W^*$.

Our first step is to show that as long as the algorithm has not terminated, then with high probability we can find a weight vector that correlates non-trivially with the unknown subspace $W^*$.

{\begin{claim}[Large Error Implies Good Correlation]\label{it:constantCorrelation}
   Suppose that at iteration $t$ the classifier $h_t$ satisfies $\pr_{(x,y)\sim D}[h_{t}(\x)\neq y]> C\cdot\opt+2\eps$, and that $\eps'$ is bounded as $\eps'\leq c\eps^2/(K^2\sqrt{k_t})$, for a sufficiently small absolute constant $c>0$. Then there exists  sample size,  $N\ge d(k_t/\eps')^{O(k_t)}(K/\eps)^{O(1)}\log(1/\delta)$, such that with probability at least $1-\delta$ there exists unit vectors $\vec v^{(t)}\in V_{t+1}$  and  $\w\in W^*$ such that $\w\cdot (\vec v^{(t)})^{\perp V_{t}}\ge \poly(1/K,\eps)$.
\end{claim}}

\begin{proof}[Proof of \Cref{it:constantCorrelation}]

{Let $\cE_{t+1} = \{\vec u^{(1)},\dots,\vec u^{(k')}\}$ be the set of unit vectors returned by  \Cref{alg:correlating-vectors} at iteration $t+1$.
For the sake of contradiction assume that for every $\vec u\in \cE_{t+1}$ and every unit vector $\vec w\in W^*$, it holds that $\abs{\vec u\cdot \w}\le \eps^2/(C^2K^4)$.}

{By \Cref{lemma:projectIfSmall}, there exists a classifier $f(\x)=\argmax_{i\in [K]}(\w^{(i)}\cdot \x+t_i^*) $ such that $\pr_{\x\sim \cN(\vec 0,\vec I)}[f(\x)\neq f^*(\x)]\le \eps/C$, and for every $i\in[K]$, the vector $\w^{(i)}$  is defined as $ \w^{(i)}=(\vec w^{*(i)})^{\perp U},$ where $U=\mathrm{span}( \vec u^{(1)},\dots, \vec u^{(k')})$.}

Since {$f$ approximates $f^*$ up to error $\eps/C$, we have  $\pr[f(\x)\neq y]\le \opt +\eps/C$. Thus, the assumption of \Cref{it:constantCorrelation} that $\pr_{(x,y)\sim D}[h_{t}(\x)\neq y]> C\cdot\opt+
    2\eps$ implies that $\pr_{(\x,y)\sim D}[h_{t}(\x)\neq y]> C\pr_{(\x,y)\sim D}[f(\x)\neq y]+
    \eps$.
Therefore, by applying \Cref{prop:alg2} for $f$, we deduce that the returned set $\cE_{t+1}$ should contain a vector $\vec u^{(l)}\in \cE_{t+1}$ so that for some indices $i,j\in[K]$, $i\neq j$, it holds $\vec u^{(l)}\cdot \vec w^{(i,j)} \ge \poly(1/K,\eps) \norm{(\w^{(i,j)})^{\perp V_{t}}}$, where $\w^{(i,j)}\eqdef\w^{(i)}-\w^{(j)}$ with $(\w^{(i,j)})^{\perp V_{t}}\neq \vec 0$. 
However, since $\vec w^{(i,j)}\in  U^{\perp}$ it should hold that $ \vec w^{(i,j)}\cdot \vec u^{(l)}=0$, which is a contradiction. Therefore, there exists $\vec u\in \cE_{t+1}$ and unit vector $\vec w\in W^*$, such that $\abs{\vec u\cdot \w}\ge \eps^2/(C^2K^4)$, which concludes the proof.}
\end{proof}
Our next claim quantifies the decrease of the potential function $\Phi_t$ between the iterations $t$ and $t+1$.
\begin{claim}
    [Potential Decrease]\label{it:potentialDecrease}Let $\beta \ge 0$. If there exists unit vector $\vec v^{(t)}\in   V_{t+1}$ and unit vector $\vec w\in W^*$ such that
    $$\vec w\cdot (\vec v^{(t)})^{\perp V_t}\ge   \beta\;,$$
    then $\Phi_{t+1}\le \Phi_t-\beta^2$.
\end{claim}
\begin{proof}[Proof of \Cref{it:potentialDecrease}]
{First, note that since $\{\vec b^{(i)}\}_{i=1}^K$ is an orthonormal basis for $W^*$, any vector $\vec w\in W^*$ can be written as
$\sum_{i=1}^K (\vec b^{(i)}\cdot \vec w)\vec b^{(i)}$. Thus, we can write $\vec w\cdot (\vec v^{(t)})^{\perp V_t}\ge   \beta$ as $\sum_{i=1}^K (\vec b^{(i)}\cdot \vec w)(\vec b^{(i)}\cdot (\vec v^{(t)})^{\perp V_t}) \ge \beta$. 
 {Recall that our potential function is defined as $\Phi_t = \sum_{i=1}^{K}\norm{(\vec b^{(i)})^{\perp V_t}}^2$.}
Hence, by the Pythagorean theorem, we have that
\begin{align*}
    \Phi_{t+1}=\sum_{i=1}^{K}\|(\vec b^{(i)})^{\perp V_t}\|^2 - \|(\vec b^{(i)})^{V_{t+1}\cap \perp V_t}\|^2=\Phi_t-\sum_{i=1}^K\|(\vec b^{(i)})^{V_{t+1}\cap \perp V_t}\|^2\;.
    \end{align*}
   Using the dual characterization of the norm, we can further bound
\begin{align*}
    \sum_{i=1}^K\|(\vec b^{(i)})^{V_{t+1}\cap \perp V_t}\|^2\geq \sum_{i=1}^K\Abs{\vec b^{(i)}\cdot (\vec v^{(t)})^{\perp V_t}}^2 
\end{align*}
Note that from \CS inequality we have that
\[
\sum_{i=1}^K \left|\vec b^{(i)}\cdot (\vec v^{(t)})^{\perp V_t}\right|^2
\ge \left|\sum_{i=1}^K (\vec b^{(i)}\cdot \vec w)(\vec b^{(i)}\cdot (\vec v^{(t)})^{\perp V_t})\right|^2\;,
\]
where we used that $\sum_{i=1}^K \left|\vec b^{(i)}\cdot \w\right|^2=1$ from Pythagorean theorem. Therefore, we have
    \begin{align*}
    \Phi_{t+1}&\le \Phi_t- \Abs{\sum_{i=1}^K(\vec b^{(i)}\cdot \vec w)(\vec b^{(i)}\cdot (\vec v^{(t)})^{\perp V_t})}^2
   \\ &\le   \Phi_t- \beta^2\;.
\end{align*}
This completes the proof of \Cref{it:potentialDecrease}.}
\end{proof}
{Note that by \Cref{it:constantCorrelation,it:potentialDecrease}, if $h_t$ has large misclassification error, we can decrease the potential $\Phi_t$ by a non-trivial amount. However, since the potential function is nonnegative, this implies that at some iteration $t$, $h_t$ must have a small error. We aim to show that once this happens, the subsequent hypotheses $h_t$ also maintain small error.}
In the following claim, we show that by adding new vectors to the subspace $V_t$, we decrease the error of our piecewise constant approximation.\begin{claim}[Error Decrease]\label{it:ErrorDecrease} For each $t\in[T]$, it holds  $\pr_{(\x,y)\sim D}[h_{t+1}(\x)\neq y]\le  \pr_{(\x,y)\sim D}[h_{t}(\x)\neq y]$. 
\end{claim}
\begin{proof}[Proof of \Cref{it:ErrorDecrease}]
{Since \Cref{prop:alg2} holds for any $\eps'$-partitions $\mathcal{S}_t, t\in [T]$, independent of the choice of basis and threshold points},
{we can assume that all the classifiers $h_t$} are computed with respect to extensions of a common orthonormal basis and that all threshold points are aligned.
Hence, $\mathcal{S}_{t+1}$ is a subdivision of $\mathcal{S}_t$.
Thus,  each set $S\in \mathcal{S}_t$ is equal to a union of sets $S_1,\dots,S_l\in \mathcal{S}_{t+1}$, i.e., $S=\cup_{i=1}^l S_i $. Therefore, by the definition of $h_t$, we have that
\begin{align*}
 \pr_{(x,y)\sim D}[h_{t+1}(\x) \neq y, \x \in S] &= \sum_{i=1}^l \pr_{(x,y)\sim D}[h_{t+1}(\x) \neq y \mid \x\in S_i] \pr[S_i]= \sum_{i=1}^l\min_{\widehat{y} \in [K]} \pr_{(x,y)\sim D}[y\neq \widehat{y}  \mid \x\in S_i]\pr[S_i]\\
 &\le \min_{\widehat{y} \in [K]}\sum_{i=1}^l \pr_{(x,y)\sim D}[y\neq \widehat{y} \mid \x\in S_i]\pr[S_i]=\min_{\widehat{y} \in [K]} \pr_{(x,y)\sim D}[\widehat{y} \neq y, \x \in \cup_{i=1}^l S_i]\\
 &\le \pr_{(x,y)\sim D}[h_{t}(\x) \neq y, \x \in S]\;.
\end{align*}
Summing over all $S\in \mathcal{S}$ we conclude the proof.
\end{proof}
{In the following claim, we show that if we use a sufficiently large number of samples, then with high probability we obtain an accurate estimate of the population piecewise constant approximation function.}
{\begin{claim}[Concentration of the Piecewise Constant Approximation]\label{cl:hConcetration}
Let $\eps, \eps', \delta\in (0,1)$ and $k,K\in \Z_+$ with $\eps' \le \eps/2$. Let $V\subseteq \R^d$ be a $k$-dimensional subspace, and consider a piecewise constant {approximation}  $h:\R^d\to [K]$ {of $D$,}
defined with respect to an $\eps'$-approximating partition of $V$ (see \Cref{def:h}). Let $\widehat{D}$ be the empirical distribution obtained from $N$ i.i.d.\ samples drawn from $D$, and let $\widehat{h}$ be a piecewise constant {approximation of $\widehat{D}$} defined with respect to the same partition. Then there exists $N= (k/\eps')^{O(k)}(K/\eps)^{O(1)}\log(1/\delta)$
such that, with probability at least $1-\delta$, we have
$
\pr_{(\x,y)\sim D}[\widehat{h}(\x)\neq y] \le \pr_{(\x,y)\sim D}[h(\x)\neq y] + \eps\;.$
\end{claim}}
\begin{proof}
{Let $\mathcal{S}$ denote the approximating partition used to define both $h$ and $\widehat{h}$. 
By the definition of an approximating partition (\Cref{def:approximatingPartition}) and using \Cref{fact:gaussianfacts}, we have that for any $S\in \mathcal S$ it holds $\pr_{D}[S ]= (\eps'/k)^{\Omega(k)}$ which implies that $\abs{\mathcal{S}}=(k/\eps')^{O(k)}$.}
Therefore, by the union bound and Hoeffding's inequality it holds that if $N\ge C\log(\abs{\mathcal{S}}K/\delta)/(\eps/\abs{S})^2=  (k/\eps')^{O(k)}(K/\eps)^{O(1)}\log(1/\delta) $ {for a sufficiently large universal constant $C>0$,} then with probability at least $1-\delta$, it holds that for every $S\in\mathcal{S}$ and every $i\in[K]$, we have that  $\abs{\pr_{\widehat{D}}[\x\in S,y=i]-\pr_{D}[\x\in S,y=i]}\le \eps/(4\abs{\mathcal{S}})$. Hence, with probability $1-\delta$ we have that for all $S\in \mathcal{S}$ it holds  $\pr_{(\x,y)\sim D}[\widehat{h}(\x)\neq y]\le \pr_{(\x,y)\sim D}[h(\x)\neq y]+\eps/(2\abs{\mathcal{S}})$. The claim follows, as the total error is simply the summation of the error over {cubes} and the  error outside the {partition} which is at most $\eps'$ from \Cref{def:approximatingPartition}.
\end{proof}

Note that from Lines \ref{line:init}, \ref{line:loop} of \Cref{alg:constantApprox} we perform at most ${T=\poly(K,1/\eps)}$ iterations.
{Moreover, in each iteration we add at most $\poly(K,1/\eps)$ vectors to our current vector set $L_t$, hence the dimension $k_t$ of $V_t$ satisfies 
$k_t= \poly(K,1/\eps)$ for all $t=1,\dots, T$.

Assume that for every $t=1,\dots,T$ the error satisfies $\pr_{(\x, y)\sim D}[h_t(\x)\neq y]> C\cdot\opt+\eps$.
Note that since \( k_t = \poly(K,1/\eps) \) for all \( t \in [T] \), if we choose \( \eps' = \poly(\eps/K) \) for a sufficiently small polynomial, then for all \( t \in [T] \) we have \( \eps' \le c\eps^2/(K^2\sqrt{k_t}) \) for a sufficiently small absolute constant \( c > 0 \).
Hence,} using that $N=d2^{\poly(K,1/\eps)}\log(1/\delta)$ for a sufficiently large polynomial (Line \ref{line:init_vars} of \Cref{alg:constantApprox}), we can apply \Cref{it:constantCorrelation} and conclude that with probability $1-\delta$ there exist unit vectors $\vec v^{(t)}\in  V_{t+1}$ {and unit vectors $\w^{(t)}\in W^*$ for $ t\in[T]$ such that $\vec w^{(t)}\cdot (\vec v^{(t)})^{\perp V_t}\ge \poly(1/K,\eps)$.}
 {By \Cref{it:potentialDecrease}, it follows that with probability at least $1-\delta$, the potential function decreases by at least $\poly(\eps,1/K)$ at each iteration, that is  $\Phi_t\le \Phi_{t-1}- \poly(\eps,1/K)$. }
 After $T$ iterations, it obtain that $\Phi_T\le \Phi_{0}- T\poly(\eps,1/K)=1-T\poly(\eps,1/K)$.
 {However, since $T$ is chosen to be a sufficiently large polynomial of $\eps$ and $K$, we would arrive at a contradiction, since $\Phi_T\ge 0$.}
 Hence, we have that $\pr_{(\x, y)\sim D}[h_t(\x)\neq y]\le C\cdot\opt+\eps$ for some $t\in [T]$. 
 {Since the error of $h_t$ can only decrease over the iterations (see \Cref{it:ErrorDecrease}) and $h_t$ is close to its empirical version (\Cref{cl:hConcetration}), we have that $\pr_{(\x, y)\sim D}[h(\x)\neq y]< C\cdot\opt+\eps$.}

\item \textbf{Sample and Computational Complexity:} From the analysis above{,} we have that the algorithm terminates in $T=\poly(1/\eps,K)$ iterations and at each iteration we draw $d2^{\poly(K,1/\eps)}\log(1/\delta)$ samples. Hence{,} we have that the total number of samples is of the order of $d2^{\poly(K,1/\eps)}\log(1/\delta)$. Moreover{, the algorithm runs in} $\poly(N)$ time as all the computational operations can be implemented in polynomial time.
\end{proof}

\subsection{Proof of \Cref{lem:tails-gen}}
\label{sec:2.7}

This section proves \Cref{lem:tails-gen}, 
one of the supporting technical lemmas used in the proof of \Cref{lem:progress}. 
For completeness we restate the lemma below.
\BRUTEFORCE*

\begin{proof}

The proof of the lemma immediately follows from \Cref{lem:tails} and \Cref{cl:rightgamma}. First, we establish \Cref{lem:tails}, which provides a bound on the tails of a piecewise affine function. Then, using \Cref{cl:rightgamma}, we show that there is indeed a suitable $\gamma$ to which this tail bound can be applied.

\begin{claim}[Tail Bounds]\label{lem:tails}
    Let $\vec w^{(1)},\dots, \vec w^{(K)}\in \mathbb{R}^d$, $t_1,\dots, t_K\in \mathbb{R}$ and 
$s:\mathbb{R}^d\to \mathbb{R}$ such that $s(\vec x)= \max_{i\in [K]} \max(\vec{w}^{(i)} \cdot \vec{x} +t_i,0)$. For any $\gamma>0$ such that  $\pr_{\x\sim \cN(\vec 0,\vec I)}[s(\vec x)\ge \gamma]\le 1/48$,  we have that $\pr_{\vec x\sim \cN(\vec 0,\vec I)}[s(\vec x)\ge n\gamma]\lesssim (1/n^2)\pr_{\vec x\sim \cN(\vec 0,\vec I)}[s(\vec x)\ge \gamma]$ for all $n\in \mathbb{Z}_+$.
\end{claim}
\begin{proof}
By the properties of addition and scalar multiplication of independent Gaussian random variables, we have that for any $n\in \Z_+$ a standard Gaussian random variable $\vec x\sim \cN(\vec  0,\vec I)$ can be written as $\vec x=\frac{1}{n}\sum_{i=1}^{n^2}\vec{z}^{(i)}$ where $\vec{z}^{(i)}$ are drawn i.i.d.\ from $\cN(\vec0,\vec I)$.
Now if for some $i\in [n^2]$ we have that $s(\vec{z}^{(i)})\ge n\gamma$, then there exists $j\in [K]$ such that $\vec{w}^{(j)}\cdot \vec{z}^{(i)}+t_j\ge n\gamma$ (by definition of $s(\x)$) and, furthermore, we have that
\begin{align*}
 \vec w^{(j)} \cdot \vec x+t_j&=\frac{1}{n}\vec w^{(j)} \cdot \vec{z}^{(i)}+t_j+ \frac{1}{n}\sum_{k\neq i,k\in[n^2]}\vec w^{(j)}\cdot \vec{z}^{(k)}\geq \gamma + \vec w^{(j)}\cdot\bigg(\frac{1}{n}\sum_{k\neq i,k\in[n^2]} \vec{z}^{(k)}\bigg). 
\end{align*}
Note that $\frac{1}{n}\sum_{k\neq i} \vec{z}^{(k)}$ is distributed according to $\cN(\vec0,\frac{n^2-1}{n^2}\vec I)$. Hence, by symmetry we have that $\pr[\vec w^{(j)}\cdot(\frac{1}{n}\sum_{k\neq i} \vec{z}^{(k)})\ge 0 ]=1/2$. As a result we have that
$\pr[s(\vec x)\ge \gamma\mid   s(\vec{z}^{(i)})\ge n\gamma ]\ge 1/2$, for all $i\in [n^2]$.

{Denote by $p_n=\pr_{\x\sim \cN(\vec{0},\vec I)}[s(\vec x)\ge n\gamma]$.
To prove \Cref{lem:tails}, it suffices to show that $p_n \le 4p_1/n^2$, for all $n \ge 1$. We will prove it by induction. Note that for $n=1$ the hypothesis holds trivially. Assume as an inductive hypothesis that $p_{n-1}\le 4 p_1/(n-1)^2$ for some $n\ge 2$. We will prove that $p_{n}\le 4 p_1/n^2$ for $n\ge 2$.

Fix $i\in [n^2]$ and denote by $A$ the event that $s(\vec x)\ge \gamma$ and  by $B_i^{(n)}$ the event that $s(\vec{z}^{(i)})\ge n\gamma$. Note that
\begin{align*}
    \pr\left[A, B_i^{(n)}, \bigcap_{j\neq i} B_j^{(n)} \right] &\ge \pr[A,B_i^{(n)} ]- \sum_{\substack{j=1 \\ j \neq i}}^{n^2} \pr[B_i^{(n)},B_j^{(n)}]\\
    &\ge p_{n}/2 -(n^2-1)p_n^2\;,
\end{align*}
where in the first inequality we used the union bound, and in the second inequality we used  the independence of the $\z^{(i)}$'s and that $\pr[s(\vec x)\ge \gamma\mid   s(\vec{z}^{(i)})\ge n\gamma ]\ge 1/2$, for all $i\in [n^2]$.

Note that it holds $p_n \le p_{n-1}$, since $\pr_{\x\sim \cN(\vec 0,\vec I)}[s(\vec x)\ge n\gamma]\le \pr_{\x\sim \cN(\vec 0,\vec I)}[s(\vec x)\ge (n-1)\gamma]$. Consequently, by the inductive hypothesis we have that $(n^2-1)p_n^2\le 4p_1(n^2-1)/(n-1)^2 p_n$. Hence,  since $(n^2-1)/(n-1)^2\le 3$ for all $n\ge 2$ and $p_1\le 1/48$, we have that $(n^2-1)p_n^2\le p_n/4$. 

Summing over all $i\in [n^2]$ we have 
\begin{align*}
    p_1=\pr[A]\ge \pr\left[A, \bigcup_{i\in [n^2]} \left( B_i^{(n)},  \bigcap_{j\neq i} B_j^{(n)}\right)\right]\ge n^2p_n/4\;.
\end{align*}
Therefore, we have that $p_n\le 4p_1/n^2$, which completes the proof of \Cref{lem:tails}.}
\end{proof}

{To apply \Cref{lem:tails}, we must choose a threshold $\gamma$ such that
$
\pr_{\x\sim \cN(0,I)}[s(\x)\ge \gamma]\le 1/48.$
We now show the following result.}
\begin{claim} \label{cl:rightgamma}
   {For any sufficiently small} $c\in (0,1)$, there exists $\gamma>0$ such that  
    $$c\ge \pr_{\x\sim \cN(\vec 0,\vec I)}[s(\x)\ge  \gamma]\ge \frac{c}{2}\pr_{\x\sim \cN(\vec 0,\vec I)}[s(\x)\neq 0]\;.$$
\end{claim}
\begin{proof}[Proof of \Cref{cl:rightgamma}]
Notice that the quantity $\pr_{\x\sim \cN(\vec 0,\vec I)}[s(\x)\ge  \gamma]$ is continuous with respect to $\gamma$ for any $\gamma>0$. Indeed, $\pr_{\x\sim \cN(\vec 0,\vec I)}[\gamma \le s(\x)\le  \gamma+\eps]\le \sum_{i=1}^K\pr_{\x\sim \cN(\vec 0,\vec I)}[-t_i+\gamma\le \vec w^{(i)}\cdot \x\le -t_i+\gamma+\eps]$ where each term such that $\vec w^{(i)}\neq \vec 0$  tends to $0$ as $\eps$ tends to $0$ and each term such that $\vec w^{(i)}=\vec 0$ is equal to $0$ from the fact that $-t_i+\gamma>0$. 
Also note that $\pr_{\x\sim \cN(\vec 0,\vec I)}[s(\x)\ge  \gamma]$ tends to $\pr_{\x\sim \cN(\vec 0,\vec I)}[s(\x)\neq 0]$ as $\gamma$ approaches $0$; hence,  there exists  $\gamma'$ such that $\pr_{\x\sim \cN(\vec 0,\vec I)}[s(\x)\neq 0]-\eps/2<\pr_{\x\sim \cN(\vec 0,\vec I)}[s(\x)\ge  \gamma']\le \pr_{\x\sim \cN(\vec 0,\vec I)}[s(\x)\neq 0] $. Therefore, if $c\ge \pr_{\x\sim \cN(\vec 0,\vec I)}[s(\x)\neq 0]$ we have that $\gamma'$ satisfies the claim, because $\pr_{\x \sim \cN(\vec 0,\vec I)}[s(\x)\neq 0]\ge \eps$. 

If $c<\pr_{\x\sim \cN(\vec 0,\vec I)}[s(\x)\neq 0] $, assume that for all $\gamma$ such that $c\ge \pr_{\x\sim \cN(\vec 0,\vec I)}[s(\x)\ge  \gamma]$ we have that $c/2\ge \pr_{\x\sim \cN(\vec 0,\vec I)}[s(\x)\ge  \gamma]$. Let $ \gamma_0=\sup\{\gamma: \pr_{\x\sim \cN(\vec 0,\vec I)}[s(\x)\ge  \gamma]\ge c\}$. Notice that this set is non-empty as $\gamma'$ belongs in this set; hence $\gamma_0\ge \gamma'>0$. We have that for all  $\delta>0$ sufficiently small $\pr_{\x\sim \cN(\vec 0,\vec I)}[s(\x)\ge  \gamma_0-\delta]\ge c$ while $\pr_{\x\sim \cN(\vec 0,\vec I)}[s(\x)\ge  \gamma_0+\delta]\le c/2$, which is a contradiction as $\pr_{\x \sim \cN(\vec 0,\vec I)}[s(\x)\ge \gamma]$ is continuous for all $\gamma>0$. Thus, there exists $\gamma$ such that $c\ge \pr_{\x\sim \cN(\vec 0,\vec I)}[s(\x)\ge  \gamma]\ge c/2\ge (c/2)(\pr_{\x\sim \cN(\vec 0,\vec I)}[s(\x)\neq 0])$.
\end{proof}

Combining the the above two statements completes the proof of \Cref{lem:tails-gen}.
\end{proof}

\section{Learning Multi-Index Models with Agnostic Noise}
\label{sec:generalAlgorithm}
 In this section, we extend the algorithm presented in \Cref{sec:MulticlassAlgorithm} for a  general class of MIMs that satisfy specific regularity conditions (\Cref{def:SimpleGoodCondition}), establishing \Cref{thm:SimpleMetaTheorem}. 
We begin by stating the formal versions of the regularity conditions and the theorem. Following this, we provide the pseudocode for the extended algorithm along with the necessary definitions.

As we will discuss later, we need a continuity condition on the target function in order to be able to accurately approximate it over small regions. To this end, we define the Gaussian Surface Area of a multiclass classifier. A bound on this quantity will be crucial for the analysis of our algorithm.
\vspace{-0.3em}
 \begin{definition}[Gaussian Surface Area]
\label{def:GSA}
  For a Borel set $A \subseteq \R^d$, its Gaussian surface area is defined by
  $
  \Gamma(A) := \liminf_{\delta \to 0} \frac{\pr_{\x\sim \cN(\vec 0,\vec I) }[\x\in A_\delta \setminus A]}{\delta},
  $  where $A_\delta = \{\x : \mathrm{dist}(\x, A) \leq \delta\}$. Let $\mathcal{Y}$ be a set of finite cardinality.
  For a multiclass function $f:\R^d \mapsto \mathcal Y$, we  define its Gaussian surface area to be the total surface area of each of the label sets
  $K_z = \{ \x \in \R^d: f(\x) = z\}$, i.e., $\Gamma(f) = \sum_{z\in \mathcal Y}\Gamma(K_z)$.
  For a class of multiclass concepts $\mathcal{C}$, we define $\Gamma(\mathcal{C}) := \sup_{f \in \mathcal{C}} \Gamma(f)$.
\end{definition}
Now we proceed to define the regularity conditions that we assume in order to obtain efficient algorithms. 
{\begin{definition}[Well-Behaved Multi-Index Model]\label{def:GoodCondition}
Let $d,K,m\in \mathbb{Z}_+, \alpha,\zeta,\sigma,\tau\in (0,1)$ and $ \mathcal{Y}\subseteq \Z$ of finite 
cardinality.  We say that a class of functions $\mathcal{F}$ from  $\R^d\mapsto \mathcal Y$ 
is a family of concept classes of $(m,\zeta,\alpha,K,\tau,\sigma,\Gamma)$-well-behaved Multi-Index Models 
if, for every $f \in \mathcal{F}$, the following conditions hold:
\begin{enumerate}[leftmargin=*]
\item  There exists a subspace $W\subseteq\R^d$ of dimension at most $K$ such that $f$ depends only {on the projection onto } $W$, i.e., 
for every $\x \in \R^d$ it holds $f(\x) = f( \x^{W})$.
\item The Gaussian surface area of $f$ is bounded by $\Gamma$, i.e., $\Gamma(f)\le \Gamma$.

\item Let $\x\sim \cN(\vec 0,\vec I)$. For any random variable $y$ supported on $\mathcal{Y}$ with $\pr[f(\x)\neq y]\leq \zeta$ and any subspace $V\subseteq \R^d$,
\begin{enumerate}[leftmargin=*, nosep]
    \item either there exists a function $g:V\to \mathcal{Y}$ such that $\pr_{\x,y}[g(\x^{V})\neq f(\x) ]\le \tau$,
    
  \item or with probability at least $\alpha$ over $\vec x_0 \sim\cN(\vec 0,\vec I)$, where $\x_0$ is independent of $\x$, there exists a polynomial $p:U\to \R$ of degree at most $m$ and a label $z\in \mathcal{Y}$ such that $\E[p(\x^{U})]=0$ and  $\E[p(\x^U)\Ind(y=z)\mid \x^V=\x_0^V]>\sigma \|p(\x^{U})\|_2 $, where $U=(V+W)\cap V^\perp$.\end{enumerate}

\end{enumerate}
\end{definition}
}
We can now state the formal version of \Cref{thm:SimpleMetaTheorem}. 

\begin{theorem}[Approximately Learning Multi-Index Models with Agnostic Noise]\label{thm:MetaTheorem} Fix $d,K,m\in \Z_+$, $\zeta,\alpha,\tau,\sigma\in (0,1)$ and $\Gamma\in \R_+$.
   Let $D$ be a distribution over $\mathbb{R}^d \times \mathcal{Y}$ whose $\x$-marginal is $\mathcal{N}(\vec0, \vec I)$ and let $\epsilon, \delta \in (0,1)$. Let $\mathcal{F}$ a  class of  $(m,\zeta,\alpha,K,\tau,\sigma,\Gamma)$-well-behaved Multi-Index Models according to \Cref{def:GoodCondition} and let $\opt =\inf_{f\in \mathcal{F}}\pr_{(\x,y)\sim D}[f(\x)\neq y]$. If $\zeta\ge \opt +\eps$, then, \Cref{alg:MetaAlg1} draws $N ={d}^{O(m)}2^{\poly(mK\Gamma\abs{\mathcal{Y}}/(\eps\alpha\sigma ))}\log(1/\delta)$ i.i.d.\ samples from $D$, runs in time $\poly(N)$, and returns a  hypothesis $h$ such that, with probability at least $1 - \delta$, it holds 
   $\pr_{(\x,y)\sim D}[h(\x)\neq y] \leq\tau+\opt+\eps\;.$
\end{theorem}

\Cref{def:GoodCondition} outlines the crucial properties necessary for our algorithm to produce a satisfactory solution. We first assume that our concept class consists of Multi-Index Models with finite Gaussian surface area. This condition bounds the complexity of the decision boundary of the multi-index models. We leverage this condition in order to show that if we partition our space to sufficiently small cubes then the function should be almost constant within each cube.
This is a condition that holds for many classes of interest, such as multiclass linear classifiers and intersections of halfspaces.

In addition, we impose a structural condition on the target class to guarantee progress at each step of the algorithm.
{Intuitively, if $V$ is the subspace we have currently identified, we require that either it already provides a good approximation (Condition 3(a)) or we can use moments to uncover further directions in the hidden subspace $W$ Condition 3(b).}
Specifically, we assume that for any  perturbation $y$ of a function $f$ from the class that we are attempting to learn, and for any given subspace $V$, one of the following must hold:
(i) there exists a good approximation to $y$ using a function defined over $V$ (Condition 3(a)), or
(ii) with some probability over orthogonal cylinders to $V$, there exists a low-degree, zero-mean, variance-one polynomial along directions that correlate with the ground-truth subspace $W$, which non-trivially correlates with the level sets of $y$ over that orthogonal cylinder (Condition 3(b).
In this manner, extracting these directions and adding them to $V$ enables an improved approximation of the hidden-subspace $W$.
Although Condition 3(b) is defined over specific points with a measure of zero, we demonstrate in \Cref{lem:Point2FiniteWidthCylinders} that, by leveraging our assumptions, there exists a $\alpha$ fraction of finite width cylinders that exhibit this property.

Condition 3(a) of \Cref{def:GoodCondition} further provides an approximation factor that the optimal function within the subspace $V$ can achieve. This function can be approximated using a non-parametric regression algorithm.
In fact, in \Cref{lem:PieceWiseConstantApproxSuffices}, we formally show that a piece-wise constant approximation is enough to approximate the unknown function $g$. 

We prove this by leveraging the Gaussian noise sensitivity of $f$ to show that a noisy version of $g$ is nearly constant on sufficiently thin cubes defined over $V$. This way using the noisy version of $g$ we can construct a good enough piecewise constant approximation. Consequently, the function that minimizes the disagreement over these cubes will also exhibit low error.

\begin{algorithm}[h]
    \centering
    \fbox{\parbox{5.1in}{
            {\bf Input:}  Accuracy $\eps \in(0,1)$, failure probability $\delta\in(0,1)$, and sample access to a distribution  $D$ over $\mathbb{R}^d\times \mathcal{Y}$ { whose $\x$-marginal is $\cN(\vec 0,\vec I)$}.\\
            {\bf Output:} A hypothesis $h$ such that, with probability at least $1-\delta$ it holds that $\pr_{(\x,y)\sim D}[h(\x)\neq y]\leq \tau+\opt+\eps${, where $\opt\eqdef \inf_{f\in \mathcal{F}} \pr_{(\x,y)\sim D}[f(\x)\neq y]$}.
            \begin{enumerate}
            \item  Let $T,p$ be  sufficiently large polynomials in $1/\alpha, 1/\eps,1/\sigma, K, \Gamma, \abs{\mathcal{Y}},m$, and let $C$ be a sufficiently large universal constant.\label{line:initMeta}  
             \item  Initialize $L_1=\emptyset$, $t\gets 1$, $\eps'\gets 1/p$, $N\gets {d}^{C m}2^{T^C}\log(1/\delta)$.\label{line:initMetaParams}
                \item  While $t\le T$
            \label{line:loopMeta}
             \begin{enumerate}
             \item Draw a set $S_t$ of $N$ i.i.d.\ samples from $D$.
             \item $\mathcal{E}_t\gets$ \Cref{alg:MetaAlg2}($\eps, \eps', \delta, \spaning(L_t), S_t$).
          
                \item $L_{t+1}\gets L_{t}\cup \mathcal{E}_t $.\label{line:updateMeta}
                \item $t\gets t+1$.
            \end{enumerate}
            \item {
             Construct an $\eps'$-approximating partition, $\mathcal{S}$,  with respect to $\spaning(L_t)$, as defined in \Cref{def:approximatingPartition}.}
             \item {Draw $N$ i.i.d.\ samples from $D$ and construct the piecewise constant classifier $h_{\mathcal{S}}$ as follows: For each $S \in \mathcal{S}$, assign a label that appears most frequently among the samples falling in $S$. Return $h_{\mathcal{S}}$.}
            \end{enumerate}
    }}

    \vspace{0.2cm}
        \caption{Learning Multi-Index Models  with Agnostic Noise }
 \label{alg:MetaAlg1}
\end{algorithm}
\begin{algorithm}[h]
    \centering
    \fbox{\parbox{5.5in}{
            {\bf Input:}  $\eps > 0$, $\eps'>0$, $\delta>0$, a subspace $V$ of $\R^d$ and a set of samples $\{(\x^{(1)},y_1),\dots,(\x^{(N)},y_N)\}$ {from a distribution $D$ over  $\R^d\times \mathcal{Y}$ whose $\x$-marginal is $\cN(\vec 0,\vec I)$}.\\
            {\bf Output:} A set of unit vectors $\mathcal E$.
            \begin{enumerate}
            \item Let $t,\eta$ be  sufficiently small polynomials in $\eps,  \alpha,\sigma, 1/K, 1/\Gamma,1/m, 1/\abs{\mathcal{Y}}$.
            \item Construct the empirical distribution $\widehat{D}$ of $\{(\x^{(1)},y_1),\dots,(\x^{(N)},y_N)\}$.
            \item Construct $\mathcal{S}$, an $\eps'$-approximating partition with respect to $V$ (\Cref{def:approximatingPartition}).
            \item \label{line:regression}For each $S\in \mathcal{S}$ and each $i \in \mathcal{Y}$, compute a polynomial $p_{S,i}(\x)$ such that \begin{align*}
                \E_{(\x, y)\sim D}[&(\Ind(y=i) -p_{S,i}(\x^{\perp V}))^2\mid \x \in S]\\&\le \min_{p'\in \mathcal{P}_m} \E_{(\x, y)\sim D}[(\Ind(y=i) -p'(\x^{\perp V}))^2\mid \x \in S]+ \eta^2 \;.
            \end{align*}
            \item For each $S\in\mathcal{S}$, compute $U_{S,i}$ as the set of unit eigenvectors of $\E_{\x\sim D_{\x}}[\nabla p_{S,i}(\x^{\perp V}) \nabla p_{S,i}(\x^{\perp V})^\top\mid \x \in S]$ with eigenvalues at least $\sigma^2/(4K)$ and let $U_S=\cup_{i\in \mathcal{Y}} U_{S,i} $.\label{line:setUs}
            
            \item Let $\widehat{\vec U}= \sum_{S\in \mathcal{S},\vec u\in U_S} (\vec u)^{\perp V}((\vec u)^{
                \perp V
                })^{\top}\pr_{(\x,y)\sim {\widehat{D}}}[S]$.  \label{line:matrixMeta}
            \item Return the set $\mathcal{E}$ of {unit} eigenvectors of $\widehat{\vec U}$  {with  corresponding eigenvalues at least $t$.      }
            \end{enumerate}
    }}
    \medskip
        \caption{{Estimating a relevant direction}}
 \label{alg:MetaAlg2}
\end{algorithm}

\begin{lemma}[Piecewise Constant Approximation Suffices]
    \label{lem:PieceWiseConstantApproxSuffices}
    Let $d,k\in \Z_+, \eps, \eps'\in (0,1)$ and let $\mathcal{Y}$ be a set of finite cardinality.
    Let $D$ be a distribution supported on $\R^d\times \mathcal{Y}$, whose $\x$-marginal is $\cN(\vec 0,\vec I)$, and let $V\subseteq \R^d$ be a $k$-dimensional subspace, with $k\geq 1$.
    Suppose that for a function $f\in\mathcal F$, it holds that $\pr_{(\x, y) \sim D}[f(\x) \neq y] \leq \epsilon$. Additionally, assume that there exists a function $g: V \to \mathcal{Y}$  such that $\pr_{(\x, y) \sim D}[g(\x^V) \neq f(\x)] \leq \tau$. Let $\eta=\eps'^2/(C\Gamma(f)\abs{\mathcal{Y}}\sqrt{k})$, for $C>0$ a sufficiently large universal constant, and $\mathcal{S}$ be an $\eta$-piecewise approximating partition with respect to $V$ (see $\Cref{def:approximatingPartition}$). Then, it holds that $\pr_{(\x, y) \sim D}[h_{\mathcal{S}}(\x) \neq y] \leq \tau + \epsilon +\eps'$, where $h_{\mathcal{S}}$ is an $\eta$-piecewise constant approximation function (see \Cref{def:h}).
\end{lemma}
\begin{proof}
Fix $i\in \mathcal{Y}$ and $\rho\in(0,1)$ and consider the boolean functions $f_i(\x)\eqdef \Ind(f(\x)=i)$, $g_i(\x)\eqdef \Ind(g(\x^V)=i)$. We denote the Gaussian noise operator by \(T_{\rho}\), as defined in \Cref{def:gaussian-noise}.
By \Cref{fct:semi-group}, we have that $\norm{f_i-g_i}_{L_1}\ge\norm{T_{\rho}f_i-T_{\rho}g_i}_{L_1}$. Hence, from the triangle inequality we also have that $\norm{f_i-g_i}_{L_1}+\norm{T_{\rho}f_i-f_i}_{L_1}\ge\norm{f_i-T_{\rho}g_i}_{L_1} $. Furthermore, by {the} Ledoux-Pisier inequality (\Cref{fact:ledouxpissier}), we have that $\norm{T_{\rho}f_i-f_i}_{L_1}\leq \sqrt{\pi}\Gamma(f_i)\rho$. Therefore, we have that
\[
\norm{f_i-T_{\rho}g_i}_{L_1} \leq \norm{f_i-g_i}_{L_1}+ \sqrt{\pi}\Gamma(f_i)\rho\;.
\]

Moreover, note that for any continuously differentiable function $F:\R^d\mapsto \R$, it holds that $F(\vec a)-F(\vec b)=\int_{0}^1 \nabla F(\vec b+t (\vec a-\vec b)) \cdot (\vec a-\vec b)\d t$. By {the} \CS inequality, we have that  $|F(\vec a)-F(\vec b)|\le \max_{t\in (0,1)}\|\nabla F(\vec b+t (\vec a-\vec b))\| \|\vec a-\vec b\|$. By taking $F=T_{\rho}g_i$ and using \Cref{fct:semi-group}, we have that for any $\x,\vec y\in \R^d$, it holds
\[
|T_{\rho}g_i(\x)-T_{\rho}g_i(\vec y)|\leq \frac{\|\vec x-\vec y\|}{\rho}\;.
\]
{For each $S\in \mathcal{S}$, pick a point $\x_S\in S$. Then for any \(\x \in S\), the distance between \(\x\) and \(\x_S\) is at most the diameter of \(S\). In our setting, since the partition is constructed along a \(k\)-dimensional subspace \(V\) with edge length roughly \(\eta\) (i.e., each coordinate is in an interval of length \(\eta\)), the distance between any \(\x \in S\) and \(\x_S\) is at most
$
\|\x - \x_S\| \le \sqrt{k} \eta.
$
Thus, for every \(\x \in S\)
\[
|T_\rho g_i(\x) - T_\rho g_i(\x_S)| \le \frac{\|\x - \x_S\|}{\rho} \le \frac{\sqrt{k} \cdot \eta}{\rho}.
\]
By averaging over the partition $\mathcal S$, we have that
\[
\sum_{S\in\mathcal S } \norm{(T_{\rho}g_i(\x_S)-T_{\rho}g_i)\1(\x \in S)}_{L_1}\leq \sum_{S\in\mathcal S }\frac{\sqrt{k}\eta}{\rho}\pr[S]\leq \frac{\sqrt{k}\eta}{\rho}\;.
\]
This in turn gives that 
\[
\sum_{S\in\mathcal S } \norm{(T_{\rho}g_i(\x_S)-f_i)\1(\x\in S)}_{L_1}\leq  \norm{f_i-g_i}_{L_1}+ \sqrt{\pi}\Gamma(f_i)\rho+\frac{\sqrt{k}\eta}{\rho}\;.
\]}
Furthermore, by linearity of expectation and the fact that $\sum_{i\in \mathcal{Y}}g_i(\x)=1 $ we have that $\sum_{i\in \mathcal{Y}}T_{\rho}g_i(\x)=1$ for $\x \in \R^d$. Moreover, $T_{\rho}g_i(\x)\in [0,1]$ for all $\x\in \R^d,i\in \mathcal{Y}$. 

{Now, we define a multiclass function \(G: \mathbb{R}^d \to \mathcal{Y}\) in the following way: for each point \(\x \in \mathbb{R}^d\), \(G(\x)\) is a random label chosen according to the probability distribution induced by \(T_{\rho} g_i(\x_S)\), where \(S \in \mathcal{S}\) is the region such that \(\x \in S\). 
Furthermore, for every probability vector \(\vec{p} \in \mathbb{R}^{|\mathcal{Y}|}\) (where each element of \(\vec{p}\) represents the probability of the corresponding label) and the unit vector \(\e_z\), which represents the one-hot encoding of the label \(z \in \mathcal{Y}\). We have that}
\begin{align*}
    \|\vec p-\e_z\|_1= \abs{\vec p_z-1}+ \sum_{i\neq z,i\in \mathcal{Y}}\abs{\vec p_i}= 2\pr_{i\sim \vec p}[i\neq z]\;. 
\end{align*}
Denote by $A=\bigcup_{S\in \mathcal{S}}S$ the set of points that lie inside the regions of $\mathcal{S}$.
Considering the disagreement between $G$ and $f$ in $A$, we obtain
\begin{align*}
    \pr[G(\x)\neq f(\x), \x \in A]&\le \sum_{S\in \mathcal{S}}\pr[G(\x)\neq f(\x), \x\in S]  =  \sum_{i\in \mathcal{Y}}\sum_{S\in\mathcal S } \norm{(T_{\rho}g_i(\x_S)-f_i)\1(\x\in S)}_{L_1}/2 \\
    &\le \sum_{i\in \mathcal{Y}}\|f_i-g_i\|_{L_1}/2+\sqrt{\pi}\Gamma(f)\rho+\abs{\mathcal{Y}}\sqrt{k}\eta/\rho\;.
\end{align*}
Noting that $\sum_{i \in \mathcal{Y}} \|f_i - g_i\| = 2\tau$, set $\rho = \eps'/(C{\Gamma(f)})$ and $\eta = {\eps' \rho}/{(C|\mathcal{Y}| \sqrt{k}})$, for a sufficiently large constant $C>0$. 
With these choices, we obtain $\pr[G(\x)\neq f(\x),\x\in A]\le \tau +\eps'/2$.
By the triangle inequality we have that $\pr[G(\x)\neq y,\x\in A]]\le \tau +\eps+\eps'/2$.
Since  $h_{\mathcal{S}}$  chooses the label that minimizes the error at each $S\in \mathcal{S}$, we have that $\pr[h_{\mathcal{S}}(\x)\neq y,\x \in A]\le \pr[G(\x)\neq y,\x\in A] $.
Finally, considering that $\pr[\x\not \in A]\le \eta$ (recall that $\eta\leq \eps'^2$) completes the proof of \Cref{lem:PieceWiseConstantApproxSuffices}.
\end{proof}
In the following lemma, we show that, due to the bound on the Gaussian surface area of \(f\), 
{if there is a non-negligible probability that the conditional moments of \(f\)'s level sets along a subspace \(V\) are large, the same property holds when \(V\) is partitioned into sufficiently small cubes.
That is, if we partition \(V\) into thin cubes, then with non-trivial probability the conditional moments computed over each cube remain large.}
The intuition behind the proof is that, since the function is smooth (by the GSA bound), re-randomizing over sufficiently small cubes barely affects its output. However, by re-randomizing over cubes, the moments of the levelsets become independent of which point in the cube we condition on. As a result, if we observe distinguishing moments when conditioning on a point in the cube, then these moments persist when conditioning on the entire cube.  

\begin{lemma}[Existence of Large Moments on Small Cubes]
\label{lem:Point2FiniteWidthCylinders}
    Let $D$ be a distribution over $\R^d\times \mathcal{Y}$ whose $\x$-marginal is $\cN(\vec 0,\vec I)$ and let
    $\mathcal{F}$ be a class of $(m,\zeta+O(\eps),\alpha,K,\tau,\sigma,\Gamma)$-well-behaved Multi-Index Models according to \Cref{def:GoodCondition}.
    Let $V$ be a $k$-dimensional subspace of $\R^d$, with $k\ge 1$, and let $\mathcal{S}$ be an $\eta$-approximating partition of  $V$ according to \Cref{def:approximatingPartition}, with $\eta=\eps^3/(C\abs{\mathcal{Y}}\Gamma^2 k^2)$, for $C>0$ a sufficiently large universal constant.
    
    Suppose that for $f\in \mathcal{F}$ with $\pr_{(\x,y)\sim D}[f(\x)\neq y]\le \zeta$ there exists no function $g:V\to \cY$ such that $\pr_{(\x,y)\sim D}[g(\x^V)\neq f(\x)]\le \tau$, then there exists $\mathcal{T}\subseteq \mathcal{S}$ with $\sum_{S\in \mathcal{T}}\pr[S]\ge \alpha$ such that for each $S\in \mathcal{T}$ there exists a zero-mean polynomial $p:U\to \R$ of degree at most $m$ {and $z\in \mathcal Y$} such that $\E_{(\x,y)\sim D}[p(\x^U)\Ind(y=z)\mid \x \in S]>\sigma \|p(\x^{U})\|_2 $, where $U=(V+W)\cap V^{\perp}$. 
\end{lemma}
\begin{proof}
Since $\mathcal{F}$ satisfies \Cref{def:GoodCondition}, with probability at least $\alpha$ over 
$\x_0 \sim \mathcal{N}(\mathbf{0}, \mathbf{I})$, independent of $\x$, there exists a zero-mean polynomial 
$p: U \to \mathbb{R}$ of degree at most $m$ such that 
\[
   \E_{(\x,y)\sim \D}[p(\x^U)\Ind(y=z)\mid \x^V=\x_0^V]>\sigma \|p(\x^{U})\|_2 \;,
\]
where $U = (V + W) \cap V^{\perp}$.

In order to prove the statement, we need to show that the above condition can be extended from a single point to encompass an entire cube $S$.
 To achieve this, we introduce a random variable $y'$, which is $(\zeta+O(\eps))$-close to $f$, such that, when conditioned on any $S \in \mathcal{S}$, $y'$ is distributed uniformly within $S$ along the directions defined by $V$, while the other directions remain consistent with those of $y$.

{For any $\x\in \R^d$, denote by $\mathcal{S}_\x$ the set $S\in \mathcal{S}$ such $\x\in S$.} 
We define a random vector $\x'$ such that $\x'^{\perp V} = \x^{\perp V}$ and $\x'^V$ {is the sampled from the standard Gaussian over $V$ conditioned on $\mathcal{S}_{\x}$, i.e., $\mathcal{N}(\vec 0,\vec I)\mid \mathcal{S}_\x$.
Note that \(\x'\) also follows the standard normal distribution over \(\R^d\) because its \(V\)-component is resampled from \(\mathcal{N}(\vec{0}, \vec I)\).}

In the following claim we will prove that $f(\x)$ and $f(\x')$ are approximately close.
\begin{claim} \label{cl:unif_claim}
It holds that    $\pr[f(\x) \neq f(\x')] =O( \eps)$.
\end{claim}
\begin{proof}
{For each $i\in\mathcal Y$, define the boolean functions $f_i(\x)\eqdef \Ind(f(\x)=i)$ and let
$g_{\x_0}^{(i)}(\x)\eqdef T_{\rho}f_i(\x_0^{\perp V}+\x^{V})$, where \(T_{\rho}\) is the Gaussian noise operator as given in \Cref{def:gaussian-noise}.
 Fix a point $\x_0\in \R^d$ and a cube $S\in \mathcal{S}$. We first show that for any two points $\x,\z\in S$, the difference $\abs{g_{\x_0}^{(i)}(\x)-g_{\x_0}^{(i)}(\z) }$ is small.}
 
{By using that the gradient of $g_{\x_0}^{(i)}(\x)$ is bounded by $1/\rho$ and that the distance between any \(\x,\vec z \in S\)  is at most
$
\|\x - \x_S\| \le \sqrt{k} \eta
$, we get that

\[
|g_{\x_0}^{(i)}(\x)-g_{\x_0}^{(i)}(\z)|\leq \frac{\eta\sqrt{k}}{\rho}\;,
\]
for any $\x,\vec \z\in S$.} Consequently by the triangle inequality and \Cref{cl:gsaV}, we have
\begin{align*}
   &\sum_{S\in \mathcal{S}} \E_{\x_0,\x^V}[\abs{(f_i(\x_0^{\perp V}+\x^V)- g_{\x_0}^{(i)}(\x))\Ind(\x\in S)}]\\&\le
   \sum_{S\in \mathcal{S}} \E_{\x_0,\x^V}[\abs{(f_i(\x_0^{\perp V}+\x^V) -g_{\x_0}^{(i)}(\x_S))\Ind(\x\in S)}] +\eta\sqrt{k}/\rho\lesssim \sqrt{\rho}\Gamma(f_i)+\eta\sqrt{k}/\rho\;,
\end{align*}
for any collection of points $\x_S\in S$.
Furthermore, by linearity of expectation and the fact that $\sum_{i\in \mathcal{Y}}f_i(\x)=1$, we have that $\sum_{i\in \mathcal{Y}}g_{\x_0}^{(i)}(\x)=1$ and $g_{\x_0}^{(i)}(\x)\in [0,1]$ for all $\x,\x_0\in \R^d,i\in \mathcal{Y}$. 
Therefore, we can consider the multiclass function $G:\R^d\to \mathcal{Y}$ such that $G(\x)$ is a random label sampled from the probability distribution induced by $g_{\x}^{(i)}(\x_S)$, where $S \in \mathcal{S} $ such that $ \x \in S$. Note that  for every probability vector $\vec p\in \R^{|\mathcal{Y}|}$ and label $z\in \mathcal{Y}$ we have that $
    \|\vec p-\e_z\|_1=  2\pr_{i\sim \vec p}[i\neq z]$.
Therefore, we have that the $\pr[f(\x)\neq G(\x)]\lesssim \sqrt{\rho}\Gamma(f)+\eta\sqrt{k}\abs{\mathcal{Y}}/\rho\lesssim \eps$. 

{Notice that since $\x'$ is distributed as $\cN(\vec 0,\vec I)$, we have that $\pr[f(\x')\neq G(\x')]=O(\eps)$. However, as $G(\x)$ is independent $\x^V$, we have that $\pr[f(\x')\neq G(\x')]=\pr[f(\x')\neq G(\x)]=O(\eps)$. Therefore, by the triangle inequality we have that $\pr[f(\x)\neq f(\x')]=O(\eps) $, which completes the proof of \Cref{cl:unif_claim}.}\qedhere

\end{proof}

{Define $y'$ such that or all $\z\in \R^d$, the distribution of $y'$ given $\x'=\z$ is the same as the distribution of $y$ given $\x=\z$.}
By the definition of $y'$, it follows that 
$\pr_{\x', y'}[f(\x') \neq y'] = \pr_{\x'}\left[\E_{y'}[f(\x') \neq y' \mid \x']\right] \leq \zeta$.
Therefore, it holds that $\pr[f(\x) \neq y'] \leq \zeta + O(\eps)$.

Therefore, with probability at least $\alpha$ over $\x_0 \sim \mathcal{N}(\mathbf{0}, \mathbf{I})$ independent of $\x$, there exists a zero-mean polynomial $p: U \to \mathbb{R}$ of degree at most $m$ and a label $y\in\mathcal Y$ such that 
\[
\E_{\x,y'}[p(\x^U)\Ind(y'=z)\mid \x^V=\x_0^V]>\sigma \|p(\x^{U})\|_2\;.
\]
However, since the moments of $y'$ are identical when conditioned on each $\x_0^V$ that belongs to the same region in $S$, we deduce that if the moment condition holds for some polynomial at a particular point $\x_0^V$, then it must hold for the entire cube $S$, such that $\x_0^V\in S$.
{Therefore, there exists a subset $\mathcal{T}\subseteq\mathcal{S}$ of the partition with $\sum_{S\in \mathcal{T}}\pr[S]\ge \alpha$ satisfying the following: 
for all $S\in \mathcal{T}$ there exists 
a zero-mean polynomial $p: U \to \mathbb{R}$ 
of degree at most $m$ such that 
\(
\E_{\x,y'}[p(\x^U)\Ind(y'=z)\mid \x\in  S]>\sigma \|p(\x^{U})\|_2\;.
\)

Noticing that the moments with respect to $U$ on each $S \in \mathcal{S}$ are the same for $y$ and $y'$, completes the proof of \Cref{lem:Point2FiniteWidthCylinders}.}
\end{proof}

 The remaining part of the section is structured as follows: In \Cref{sec:rel-dir-find}, we provide the analysis of \Cref{alg:MetaAlg2} with \Cref{prop:MetaAlg2}. 
 In \Cref{sec:meta-thm-proof},  we analyze \Cref{alg:MetaAlg1} providing the full proof of \Cref{thm:MetaTheorem}.

\subsection{Finding a Relevant Direction}
\label{sec:rel-dir-find}

In this section, we prove \Cref{prop:MetaAlg2}, which states that if the target function class satisfies \Cref{def:GoodCondition} and the current approximation of the ground-truth subspace  provides a suboptimal piecewise approximation, then \Cref{alg:MetaAlg2} outputs a small set of vectors, at least one of which has a non-trivial projection onto the ground-truth subspace.

\begin{proposition}\label{prop:MetaAlg2}
Let $d,k\in \Z_+$, $\eps,\delta\in(0,1)$ and $\eps'\leq c\eps^3/(\abs{\mathcal{Y}}\Gamma^2 k^2))$, for $c$ a sufficiently small universal constant. 
Let $D$ be a distribution supported on $ \mathbb{R}^d \times \mathcal{Y}$ whose $\x$-marginal is $\mathcal{N}(\vec 0, \vec I)$, let $V\subseteq \R^d$ be a $k$-dimensional subspace, and
let $\mathcal{F}$ be a class of $(m,\zeta+O(\eps),\alpha,K,\tau,\sigma,\Gamma)$-well-behaved Multi-Index Models according to \Cref{def:GoodCondition}.
Let $f\in \mathcal{F}$ with $\pr_{(\x,y)\sim D}[f(\x)\neq y]=\opt$ and let $W$ be a subspace such that $f(\x)=f(\x^{W})$ for all $\x\in \R^d$, with $K=\dim(W)$.  Additionally, let $h_{\mathcal{S}}$ be a piecewise constant approximation of $D$, with respect to an $\eps'$-approximating partition $\mathcal{S}$, according to $V$, as defined in \Cref{def:approximatingPartition,def:h}.

There exists $N=(dm)^{O(m)}(k/\eps')^{O(k)}\log(1/\delta)(K\Gamma\abs{\cY} /(\sigma\eps\alpha))^{O(1)}$ such that: If $\zeta\ge \opt$ and $\pr_{(\x,y)\sim D}[h_{\mathcal{S}}(\x)\neq y]> \tau +\opt+3\eps $, then \Cref{alg:MetaAlg2} {given $N$} i.i.d.\ samples from $D$, runs in $\poly(N)$ time, and with probability at least $1-\delta$, returns a list of unit vectors $\mathcal E$ of size $|\mathcal E|=\poly(mK\Gamma\abs{\mathcal{Y}}/(\eps\alpha\sigma))$ such that: For some $\vec v\in \mathcal E$ and unit vector $\vec w\in W$ it holds that
    \begin{align*}
        \vec w^{\perp V} \cdot \vec v\ge  \poly(\eps\alpha\sigma/(m K\Gamma\abs{\mathcal{Y}}))\;.
    \end{align*}
\end{proposition}

{Before proving \Cref{prop:MetaAlg2}, we show that once a direction of our target model is learned to sufficiently small accuracy, the observed moments depend only on the remaining directions. 

The intuition is as follows. Re-randomizing over directions that have negligible projection onto \(W\) barely affects the output of the function. Thus, the function is close to a version \(y'\) of the noisy label \(y\) in which these components have been averaged. Consequently, if the function is far from being a function of the current subspace \(V\) (i.e., it still depends significantly on directions outside \(V\)), then the level sets of \(y'\) will exhibit large distinguishing moments. In this way, the influence of the re-randomized (negligible) directions is eliminated, leaving only the significant, distinguishing moments on the remaining directions.

The proof follows a structure similar to that of \Cref{lem:Point2FiniteWidthCylinders}, with the main difference being the set of directions over which re-randomization occurs.
}

\begin{lemma}\label{cl:projectIfSmallMeta}{
Let $D$ be a distribution over $\mathbb{R}^d \times \mathcal{Y}$  whose $\x$-marginal is $\cN(\vec 0,\vec I)$. Let $\mathcal{F}$ be a class of $(m, \zeta + O(\eps), \alpha, K, \tau, \sigma, \Gamma)$-well-behaved Multi-Index Models  as defined in \Cref{def:GoodCondition}. Let $V$ be a $k$-dimensional subspace of $\mathbb{R}^d$, with $k\ge 1$, and let $\mathcal{S}$ be an $\eta$-approximating partition of $V$ (see \Cref{def:approximatingPartition}), where $\eta=\eps^3/(C\abs{\mathcal{Y}}\Gamma^2 k^2)$ and $C>0$ is a sufficiently large universal constant. 
Let  $f \in \mathcal{F}$ be a function that depends only on a $K$-dimensional subspace $W$, and let $y$ be random variable  over $\mathcal{Y}$ such that $\pr_{(\x, y) \sim D}[f(\x) \neq y] \leq \zeta$. 

Suppose that no function $g: V \to \mathcal{Y}$ exists with $\pr_{(\x, y) \sim D}[g(\x^V) \neq f(\x)] \leq \tau$. 
Then, there exists $\mathcal{T}\subseteq \mathcal{S}$ with $\sum_{S\in \mathcal{T}} \pr[S]\ge\alpha$ such that for each $S\in \mathcal{T}$ the following hold:
\begin{enumerate}
    \item  There exists a label $z\in\mathcal Y$ and a zero-mean, variance-one polynomial $p:U\to \R$ of degree at most $m$ such that $\E_{(\x,y)\sim D}[p(\x^{U})\Ind(y=z)\mid \x \in S]>\sigma $, where $U{=}(V{+}W)\cap V^{\perp}$. \label{item1}
    \item For each such polynomial $p$ satisfying (1), it holds that $\nabla p(\x^U)\cdot\vec  v=0$ for every $\vec v\in E$ and all $\x\in \R^d$, where $E\subseteq{V^{\perp}}$ is a subspace satisfying  $\|\vec v^{W}\|\leq \eps/\Gamma(f)$ for every unit vector $\vec v\in E$.
\end{enumerate} }
    \end{lemma}
    \begin{proof}
   Let $\x,\x'$ be random variables defined as $\x= \x^{\perp E}+\x^{E}$ and $\x'=\x^{\perp E}+{\vec z}^{E}$, where $\z$ is an independent standard Gaussian vector.  Denote by $\mathcal{S}$ an $\eta$-approximating partition with respect to $V$.
Notice that for any unit direction $\w\in W$, both $\w \cdot \x$ and $\w \cdot \x'$ are distributed according to a standard Gaussian. Furthermore, by independence we have that
\begin{align*}
  \E[(\w \cdot \x)(\w \cdot \x')]&= \E[(\w \cdot \x^{\perp E})^2]=\|\w^{\perp E}\|^2\;.
\end{align*}
Given that $E$ is the subspace of small projections,  (i.e., for any unit vector \(\w\in W\), we have \(\|\w^E\| \le \eps/\Gamma(f)\)), it follows that $\E[(\w \cdot \x^{\perp E})^2] \ge 1-(\eps/\Gamma(f))^2$ {for any unit vector $\w\in W$. 
This means that $\x$ and $\x'$ are at least $1-(\eps/\Gamma(f))^2$ correlated in $W$.
By the Gaussian noise sensitivity bound (see \Cref{fact:ledouxpissier}), we have that
\begin{align*}
    \pr_{\x,\x'}[f(\x)\neq f(\x')]\le O(\eps)\;.
\end{align*}
Next, define a new random variable $y'$ over $\mathcal{Y}$ which is distributed like $y$ but with respect to \(\x'\); that is, for every \(\z\in \R^d\) and \(i\in \mathcal{Y}\), $D(y'=i\mid \x'=\z)=D(y=i\mid \x=\z)$.} We have that
\begin{align*}
    \pr[f(\x)\neq y']\le\pr[f(\x')\neq y']+O(\eps) \le \pr[f(\x)\neq y]+O(\eps)\le \zeta +O(\eps)  \;.
\end{align*}
As a result, by \Cref{lem:Point2FiniteWidthCylinders}, we have that there exists $\mathcal{T}\subseteq \mathcal{S}$ with $\sum_{S\in \mathcal{T}} \pr[S]\ge\alpha$ such that for each $S\in \mathcal{T}$ there exists a zero-mean, variance-one polynomial $p:U\to \R$ of degree at most $m$ and a label $z\in \mathcal{Y}$ such that $\E_{\x,y'}[p(\x^U)\Ind(y'=z)\mid \x \in S]>\sigma  $, where $U=(V+W)\cap V^\perp$. However we have that for all $S\in \mathcal{T}$ 
\begin{align*}
\E_{\x,y'}[p(\x^U)\Ind(y'=z)\mid \x \in S]= \E\left[\E_{\x^{E} }[p(\x^U)\mid \x^{\perp E}]\E_{\x^{E}}[\Ind(y=i)\mid \x^{\perp E}] \mid \x \in S\right]\;, 
\end{align*}
{where we used that $y'$ is independent of $\x^E$}.
Notice that $p'(\x^U)\eqdef \E_{\x^{E} }[p(\x^U)\mid \x^{\perp E}]$ is a zero-mean polynomial of degree at most $m$ with variance at most one by Jensen's inequality. 
Furthermore, as $p'$ is independent of $\x^{E}$ we have that $\nabla p'(\x^U)\cdot\vec  v=0$ for any $\vec v\in E$ and $\x\in \R^d$.
\end{proof}
We now proceed with the proof of \Cref{prop:MetaAlg2}.
\begin{proof}[Proof of \Cref{prop:MetaAlg2}]

Denote by $\vec w^{(1)}, \dots, \vec w^{(K)}$ an orthonormal basis for the subspace $W$.
Let $\widehat{\vec U}$ be the matrix computed in Line \ref{line:matrixMeta} of \Cref{alg:MetaAlg2} and let $\eta^2$ be the $L_2^2$ error chosen in the polynomial-regression step (see Line \ref{line:regression} of \Cref{alg:MetaAlg2}).

{First, observe that since $\pr_{(\x,y)\sim D}[h_{\mathcal{S}}(\x)\neq y]> \tau +\opt+3\eps$, \Cref{lem:PieceWiseConstantApproxSuffices} for every implies that for every function $g:V\to \cY$  it holds that $\pr_{(\x,y)\sim D}[g(\x^V)\neq f(\x)]\ge \tau$. Then, by 
\Cref{cl:projectIfSmallMeta} there exists a $\mathcal{T}\subseteq \mathcal{S}$ with $\sum_{S\in \mathcal{T}}\pr[S]\ge \alpha$, such that for each $S\in \mathcal{T}$ there exists a label $z\in \mathcal{Y}$ and a zero-mean variance-one  polynomial $q_S:U\to \R$ of degree at most $m$ satisfying $\E_{(\x,y)\sim D}[q_S(\x^U)\Ind(y=z)\mid \x \in S]>\sigma $}, where $U=(V+W)\cap(V^{\perp})$. 
{Furthermore, {for each $S\in\mathcal S$}, define $E_S\eqdef \mathrm{span}(\{\vec v\in U_S:\norm{\vec v^{W}}\le \eps/(\sqrt{\abs{U_S}}\Gamma)\})$.
Then, for every unit vector $\vec v\in E_S$, we have  $\nabla q_S(\x^{U})\cdot \vec v=0$  since $\norm{\vec v^{W}}\le \eps/\Gamma$ holds for every unit vector $\vec v\in E_S$.}

In the following claim, we leverage the existence of non-trivial  moments over $U$ for regions  $S\in \mathcal{T}$. We show that if the sample size is sufficiently large then for each region $S\in \mathcal{T}$ the associated set $U_S$ (see Line \ref{line:setUs} of \Cref{alg:MetaAlg2}) contains a vector that correlates with $W$.

\begin{claim}[Existence of Correlating Vectors]
\label{it:corr-vec-existance} Let $C>0$ be  a sufficiently large universal constant. Fix $S\in \mathcal{T}$. 
If the number of samples that fall in $S $ is $N_S\ge (dm)^{Cm}\log(\abs{\mathcal{Y}}/\delta)/\eta^{C}$, then with probability at least $1-\delta$ there exists a unit vector $\vec v\in U_S$ and $i\in[K]$ such that $$(\vec w^{(i)})^{\perp V} \cdot \vec v\ge \poly(\eps\sigma/(mK\abs{\mathcal{Y}}\Gamma)\;.$$ Moreover, it holds that $\abs{U_S}=\poly(mK\abs{\mathcal{Y}}/\sigma)$.
\end{claim}
\begin{proof}[Proof of \Cref{it:corr-vec-existance}]
{Denote by $p_{S,i}$ (for $S\in \mathcal{S}$ and $i\in \mathcal{Y}$) the degree at most $m$ polynomial computed at Line \ref{line:regression} of \Cref{alg:MetaAlg2}. 
Notice that membership of a point $\x$ in a cube $S$ depends solely on its projection onto $V$, i.e., on $\x^V$. Therefore, the error guarantee obtained in the regression step} $\E_{(\x, y)\sim D}[(\Ind(y=i) -p_{S,i}(\x^{\perp V}))^2\mid \x \in S]\le \min_{p'\in \mathcal{P}_m} \E_{(\x, y)\sim D}[(\Ind(y=i) -p'(\x^{\perp V}))^2\mid \x \in S]+\eta^2$ is equivalent to  
$\E_{(\x, y)\sim D^S_{V^{\perp}}}[(\Ind(y=i) -p_{S,i}(\x))^2]\le \min_{p'\in \mathcal{P}_m} \E_{(\x, y)\sim D^S_{V^{\perp}}}[(\Ind(y=i) -p'(\x))^2 ]+\eta^2$, where $D_{V^\perp}^S$ is  the marginal obtained by averaging $D$ over $V$ conditioned on  $S$, i.e., $D^S_{V^{\perp}}(\x^{\perp V},y)=\E_{\x^V}[D(\x,y)\mid \x \in S]$.

{Moreover, by the properties of the Gaussian distribution, we have that the $\x$-marginal of $D_{V^\perp}$ is a standard Gaussian.
Therefore, if $N_S\ge (dm)^{Cm}\log(\abs{\mathcal{Y}}/\delta)/\eta^{C}$ for a sufficiently large universal constant $C>0$, then by applying the union bound and \Cref{fact:regressionAlg}, we have that with probability at least $1-\delta$:
\[\E_{(\x, y)\sim D^S_{V^{\perp}}}[(\Ind(y=i) -p_{S,i}(\x))^2]\le \min_{p'\in \mathcal{P}_m} \E_{(\x, y)\sim D^S_{V^{\perp}}}[(\Ind(y=i) -p'(\x))^2 ]+\eta^2\;.\]
Furthermore, by the orthogonality of Hermite polynomials, we have that \[\E_{(\x, y)\sim D^S_{V^{\perp}}}[(\Ind(y=i) -p_{S,i}(\x))^2]=\sum_{\beta \subseteq \N^d} ( \E_{(\x,y)\sim D^S_{V^{\perp}}}[H_{\beta}( \x)\Ind(y=i)] - \E_{(\x,y)\sim D^S_{V^{\perp}}}[H_{\beta}( \x)p_{S,i}(\x)])^2\;.\] In particular, if we decompose the error into its Hermite polynomial components of degree $t$, we have that $\sum_{t=1}^m\E_{(\x, y)\sim D^S_{V^{\perp}}}[(\Ind(y=i)^{[t]} -p_{S,i}^{[t]}(\x))^2]  \le \eta^2$.}

{Next, by the rotational invariance of the Gaussian distribution, we may, without loss of generality, choose an orthonormal basis $\vec e_1,\dots, \vec e_{k'}$ for the subspace  $U$ projected onto $E_S^{\perp}$. Note that \(k'\ge 1\), otherwise, the polynomial \(q_S\) (which is defined over this space) could not exist.} 
Furthermore, since $U$ is a subset of $V^\perp$, we similarly have that  $\E_{(\x,y)\sim D^S_{V^{\perp}}}[q_S(\x^U)\Ind(y=z)]>\sigma$ for some $z\in \mathcal{Y}$.
Decomposing $q_S(\x^{U})$ in the basis of Hermite polynomials we have that $q_S(\x^{U})= \sum_{\beta \in \N^{d'}, 1\le \norm{\beta}_1\le m}\widehat{q}_S(\beta)H_\beta(\x^U)$, where $d'=\dim(U)$.
Moreover, note that since $q_S$ has no component outside $\e_1,\dots, \e_{k'}$, we have that $q_S(\x^{U})=\sum_{\beta \in J}\widehat{q}_S(\beta)H_\beta(\x^U)$, where $J$ denotes the set of $\beta \in \N^{d'}$ such that $\beta_i\ge 1$ for some $i\in [k']$.
Considering the correlation of $q_S$ with the region of the label $z$, we have that 
\begin{align*}
    \E_{(\x,y)\sim D^S_{V^{\perp}}}[q_S(\x^U)\Ind(y=z)]&= \sum_{\beta\in J}\widehat{q}_S(\beta)\E_{(\x,y)\sim D^S_{V^{\perp}}}[H_\beta(\x^U)\Ind(y=z)]\\&\le \left(\sum_{\beta\in J}\E_{(\x,y)\sim D^S_{V^{\perp}}}[H_\beta(\x^U)\Ind(y=z)]^2\right)^{1/2}\;,
\end{align*}
where we used the fact that $\|q_S\|_2=1$ and the \CS inequality. {Hence, we deduce that the sum of the squares of the Hermite coefficients corresponding to degrees up to $m$ in the expansion of the random variable $\Ind(y=z)$ is at least $\sigma^2$.}

Now evaluating the quadratic form of $\vec M=\E_{\x\sim D_{\x}}[\nabla p_{S,z}(\x^{\perp V}) \nabla p_{S,z}(\x^{\perp V})^\top\mid \x \in S]$ (Line \ref{line:setUs} of \Cref{alg:MetaAlg2}) using \Cref{fact:gradientNorm} gives us
\begin{align*}
\sum_{i=1}^{k'}\e_i^\top\vec M \e_i&= \sum_{i=1}^{k'}\E_{(\x,y)\sim D^S_{V^{\perp}}}[(\nabla p_{S,z}(\x)\cdot\e_i )^2]\ge \sum_{\beta \in \N^d} \left(\sum_{i=1}^{k'} \beta_i) (\widehat{p}_{S,z}(\beta)\right)^2\\
&\ge \sum_{\beta \in J} \E_{(\x,y)\sim D^S_{V^{\perp}}}[H_\beta(\x^U)\Ind(y=z)]^2- 2\eta  \\
&\ge \sigma^2- 2\eta \ge \sigma^2/2
\end{align*}
for $\eta\le \sigma/2$. Thus, since $\dim(U)\le K$ we have that $\e_j^{\top} \vec M\e_j\ge \sigma^2/(2K)$ for some $j \in [k']$. 

Since $\vec{M}$ is symmetric and therefore diagonalizable, we can denote by $\vec{u}^{(i)}$ the unit eigenvectors of $\vec{M}$ and by $\lambda_i$ the corresponding eigenvalues. Additionally, let $U_{S,z}$ denote the set of eigenvectors of $\vec{M}$ with eigenvalues greater than ${\sigma^2}/{4K}$ (see Line \ref{line:setUs} of \Cref{alg:MetaAlg2}).
{Decomposing $\vec e_j$ in the basis of eigenvectors gives us   
$$
\vec e_j^\top  \vec M\,\vec e_j = \sum^{d-k}_{i=1} \lambda_i\, (\vec e_j\cdot \vec u^{(i)})^2.
$$
Using that $U_{S,z}$ contains eigenvectors of $\vec{M}$ with eigenvalues greater than ${\sigma^2}/{4K}$, we get that 
\[\frac{\sigma^2}{4K}+\sum_{\vec u^{(i)} \in U_{S,z}} \lambda_i(\vec e_j\cdot \vec  u^{(i)})^2\ge \vec e_j^\top  \vec M\,\vec e_j\ge  \frac{\sigma^2}{2K}\;.
\]}
Hence, there exists an eigenvector $\vec u^{(i)}\in U_{S,z}$ such that  $\lambda_i(\vec e_j\cdot \vec  u^{(i)})^2  \ge  \sigma^2/(4K \abs{U_{S,z}})$. This implies that $(\vec e_j\cdot \vec  u^{(i)})^2  \ge  \sigma^2/(4 \abs{U_{S,z}}K \text{tr}(\vec M))$ as $\vec M$ is positive definite. Therefore, by \Cref{fact:regressionSubspaceBound}, we have that $\abs{U_{S,z}}=\poly(mK/\sigma), \text{tr}(\vec M)= O(m)$ which implies that  $\abs{\vec e_j\cdot \vec  u^{(i)} } \ge \poly(\sigma/(mK))$.
{Finally, note that since $\vec e_j\cdot \vec  u^{(i)}\neq 0 $,  
we have that $\vec u^{(i)}$ cannot belong in the space $E_S$.
Thus, there exists a unit vector $\w\in W$ such that 
$\abs{\vec u^{(i)}\cdot \vec w}\ge \eps/(\sqrt{\abs{U_S}}\Gamma)\ge \poly(\eps\sigma/(mK\abs{\mathcal{Y}}\Gamma) $, 
which concludes the proof of \Cref{it:corr-vec-existance}.
}
\end{proof}
 
By the definition of an approximating partition (see \Cref{def:approximatingPartition}) and \Cref{fact:gaussianfacts}, for all $S\in \mathcal{S}$ we have that   $\pr_{D}[S ]= (\eps'/k)^{\Omega(k)}$ and $\abs{\mathcal{S}}= ( k/\eps')^{O(k)}$.
Therefore, by the union bound and Hoeffding's inequality, it holds that if the sample size $N\ge (k/\eps')^{Ck}\log(\abs{\mathcal{S}}/\delta)= (k/\eps')^{O(k)}\log(1/\delta) $ for a sufficiently large constant $C>0$, then with probability at least $1-\delta$  we have that  $\abs{\pr_{\widehat{D}}[S]-\pr_{D}[S]}\le \pr_{D}[S]/2$ for all $S\in \mathcal{S}$.
Hence, $\pr_{\widehat{D}}[S]\ge \pr_{D}[S]/2$, i.e., the number of samples that fall in each set $S\in \mathcal{S}$ is at least $N\pr_{D}[S]/2= N(\eps'/k)^{\Omega(k)}$.

If we choose $N$ so that $N\geq (dm)^{Cm}(k/\eps')^{Ck}\log(1/\delta)(K\Gamma \abs{\cY} /(\sigma\eps\alpha))^{C}$ for a sufficiently large universal constant $C>0$, by \Cref{it:corr-vec-existance}, we have that  for all $S\in \mathcal{T}$ it holds that $\vec v\cdot (\w^{(i)})^{\perp V}\ge \poly(\eps\sigma/(mK\abs{\mathcal{Y}}\Gamma)$ for some $\vec v\in U_{S}, i\in [K]$.
Hence, we have that there exists a subset $\mathcal{T}'\subseteq \mathcal{T}$ and $i\in [K]$ with $\sum_{S\in \mathcal{T'}}\pr_{\widehat{D}}[S]\ge \sum_{S\in \mathcal{T'}}\pr_{D}[S]/2= \Omega(\alpha/K)$ such that for all $S\in \mathcal{T'}$ we have that  $\vec v\cdot (\w^{(i)})^{\perp V}\ge \poly(\eps\sigma/(mK\abs{\mathcal{Y}}\Gamma)$ for some $\vec v \in U_{S}$.

Denote by $\widehat{\vec U}$ the matrix computed in Line \ref{line:matrixMeta} of \Cref{alg:MetaAlg2}. For some $i \in [K]$, we have that 
$$\left((\vec w^{(i)})^{\perp V}\right)^{\top}\widehat{\vec U}(\vec w^{(i)})^{\perp V} \gtrsim \frac{\alpha}{K}\poly(\eps\sigma/(mK\abs{\mathcal{Y}}\Gamma)\ge\poly(\eps\alpha\sigma/(mK\abs{\mathcal{Y}}\Gamma)\;,$$
where we used the fact that $\widehat{\vec U}$ is a PSD matrix.

Moreover, as $\norm{\vec v}=1$ for all $\vec v\in U_S, S\in \mathcal{S}$, by the triangle inequality we can see that 
$$\norm{\widehat{\vec U}}_{F}\le \sum_{ S\in \mathcal{S},\vec v\in U_S}\norm{\vec v}\pr_{\widehat{D}}[S] \le \sum_{ S\in \mathcal{S},\vec v\in U_S}\pr_{\widehat{D}}[S]\le \poly(mK\abs{\mathcal{Y}}/\sigma)\;,$$
where we used the fact that  $\abs{U_S}=\poly(mK\abs{\mathcal{Y}}/\sigma)$ by \Cref{it:corr-vec-existance}. 

Hence, by \Cref{cl:coorelationofeigenvalues}, 
we have that there exists a unit eigenvector $\vec v$ of $\widehat{\vec U}$ with eigenvalue at least $\poly(\eps\alpha\sigma/(mK\abs{\mathcal{Y}}\Gamma)$ for a sufficiently small polynomial, such that  
$(\vec w^{(i)})^{\perp V}\cdot \vec v\ge \poly(\eps\alpha\sigma/(mK\Gamma\abs{\mathcal{Y}}))$
for some $i\in [K]$. Moreover, the number of such eigenvectors is at most $\poly(mK\abs{\mathcal{Y}}\Gamma/(\eps\alpha\sigma))$, which completes the proof of \Cref{prop:MetaAlg2}.
\end{proof}

\subsection{Proof of \Cref{thm:MetaTheorem}} \label{sec:meta-thm-proof}
In this section, given \Cref{prop:MetaAlg2}, we prove \Cref{thm:MetaTheorem}. {Recall that, in \Cref{prop:MetaAlg2}, we have shown that if our current approximation to the hidden subspace is not accurate enough to produce a classifier that has sufficiently small error, then \Cref{alg:MetaAlg2} efficiently finds a direction that has non-trivial correlation to the hidden subspace. In the proof that follows, we iteratively apply this argument to show that, after a moderate number of iterations, \Cref{alg:MetaAlg1} outputs a classifier with sufficiently small error. }
\begin{proof}[Proof of \Cref{thm:MetaTheorem}]
{
Denote by $f^*\in \mathcal{F}$ a classifier such that $\pr[f^*(\x)\neq y]\le  \opt+\eps/4$. }
We show that \Cref{alg:MetaAlg1}, with high probability, returns a hypothesis $h$ with 0-1 error at most $\tau+\opt+\eps$. Let $W^*$ be a subspace of dimension $K$ such that $f^*(\x)=f^*(\x^{W^*})$ and denote by $\vec w^{*(1)},\ldots,\vec w^{*(K)} \in \mathbb{R}^d$ an orthonormal basis of $W^*$.
Let $L_t$ be the list of vectors maintained by the algorithm (Line \ref{line:updateMeta} of \Cref{alg:MetaAlg1}) and $V_t=\mathrm{span}(L_t)$, $\dim(V_t)=k_t$.
Let $\eps'$ be the partition width parameter set at Line \ref{line:initMetaParams} of \Cref{alg:MetaAlg1}. For $t\in [T]$, let $\mathcal{S}_t$ be arbitrary $\eps'$-approximating partitions with respect to $V_t$ (see \Cref{def:approximatingPartition}).
Let $h_t:\mathbb{R}^d\to [K]$ be piecewise constant classifiers, defined as $h_t= h_{\mathcal{S}_{t}}$ according to \Cref{def:h} for the distribution $D$.

To prove the correctness of \Cref{alg:MetaAlg1}, we need to show that if $h_t$ has significant error, then the algorithm improves its approximation $V_t$ to $W^*$.
For quantifying the improvement at every step, we consider the following potential function
\[
\Phi_t = \sum_{i=1}^{K}\norm{(\vec w^{*(i)})^{\perp V_t}}^2\;.
\]

Note that from Lines \ref{line:initMeta} and \ref{line:loopMeta} of \Cref{alg:MetaAlg1}, we perform at most $\poly(mK\Gamma\abs{\cY}/(\eps\alpha\sigma))$ iterations.
{Furthermore, in each iteration, we update the vector set with at most  $\poly(mK\Gamma\abs{\cY}/(\eps\alpha\sigma))$ vectors (see \Cref{prop:alg2})}. Hence, it follows that $k_t\le \poly(mK\Gamma\abs{\cY}/(\eps\alpha\sigma))$, for all $t=1,\dots, T$.

 Assume that $\pr_{(\x, y)\sim D}[h_t(\x)\neq y]> \tau+\opt+\eps/2$ for all $t=1,\dots, T$.  
 Using the fact that $N={d}^{Cm}2^{(mK\Gamma\abs{\mathcal{Y}}/(\eps\alpha\sigma ))^C}\log(1/\delta)$ for a sufficiently large universal constant $C>0$ (Line \ref{line:initMetaParams} of \Cref{alg:MetaAlg1}), we can apply \Cref{prop:MetaAlg2} and conclude that, with probability $1-\delta$, there exist unit vectors $\vec v^{(t)}\in  V_{t+1}$ for $ t=[T]$ such that 
 $\vec w\cdot (\vec v^{(t)})^{\perp V_t}\ge \poly(\eps\alpha\sigma/(mK\Gamma\abs{\mathcal{Y}}))$.
 Thus, by \Cref{it:potentialDecrease}, we have that with probability $1-\delta$, for all $t\in [T]$,  \[\Phi_t\le \Phi_{t-1}- \poly(\eps\alpha\sigma/(mK\Gamma\abs{\mathcal{Y}}))\;.\]
 After $T$ iterations, it follows that $\Phi_T\le \Phi_{0}- T\poly(\eps\alpha\sigma/(mK\Gamma\abs{\mathcal{Y}}))=
 K-T\poly(\eps\alpha\sigma/(mK\Gamma\abs{\mathcal{Y}}))$.
 However, since $T$ is set to be a sufficiently large polynomial of $m,\Gamma,K,\abs{\cY},1/\eps, 1/\alpha$ and $1/\sigma$ we would arrive at a contradiction, since $\Phi_T\ge 0$.
 Hence, we have that $\pr_{(\x, y)\sim D}[h_t(\x)\neq y]\le \tau+\opt+\eps/2$, for some $t\in \{1,\dots, T\}$. 
 Since the error of $h_t$ can only be decreasing by \Cref{it:ErrorDecrease} and  $h_t$ is close to its sample variant by \Cref{cl:hConcetration}, we have that $\pr_{(\x, y)\sim D}[h(\x)\neq y]\le \tau+\opt+\eps$.

\item \textbf{Sample and Computational Complexity:} From the analysis above we have that the algorithm terminates in $\poly(mK\Gamma\abs{\mathcal{Y}}/(\eps\alpha\sigma))$ iterations and at each iteration we draw of order ${d}^{O(m)}2^{\poly(mK\Gamma\abs{\mathcal{Y}}/(\eps\alpha\sigma ))}\log(1/\delta)$ samples. Hence, we have that the total number of samples is  ${d}^{O(m)}2^{\poly(mK\Gamma\abs{\mathcal{Y}}/(\eps\alpha\sigma ))}\log(1/\delta)$. Moreover, we use at most $\poly(N)$ time as all operations can be implemented in sample polynomial time.
\end{proof}

\section{{Algorithmic Applications of General MIM Learner}}
\label{sec:applications}
In this section, we {highlight} {selected} applications of our main result.
First, in \Cref{sec:intersections} we present an application of our {algorithmic approach} to the class of intersections of halfspaces, establishing \Cref{thm:SimpleLearningIntersections}. Second, in \Cref{sec:RCNAlgo}, we show that a slight modification of our {general} algorithm can achieve improved {complexity} in the RCN setting, establishing \Cref{thm-intro:AlgRCN-linear}.
 
Additionally, we note that our result on learning Multiclass Linear Classifiers, as presented in \Cref{thm:constantApprox}, can also be derived as a corollary of  \Cref{thm:MetaTheorem}, using similar arguments. 

\subsection{Learning Intersections of Halfspaces with Adversarial Label Noise}\label{sec:intersections}
In this section, we present \Cref{thm:LearningIntersections}, on  learning intersections of halfspaces under agnostic noise, which is derived as a corollary of \Cref{thm:MetaTheorem}. The proof of this result follows from \Cref{prop:GoodCondIntersections}, which relies on a new structural result about the class of intersections of $K$-halfspaces, \Cref{lem:progressIntersections}.
{First, we state the formal version of the theorem.}

\begin{theorem}[Learning Intersections of Halfspaces with Adversarial Label Noise]\label{thm:LearningIntersections}
Fix the class $\mathcal{F}_{K,d}$ of intersections of $K$ halfspaces in $\R^d$.
    Let $D$ be a distribution over $\mathbb{R}^d \times \{0,1\}$ whose $\x$-marginal is $\mathcal{N}(\vec0, \vec I)$ and let $\epsilon, \delta \in (0,1),K,d\in \Z_+$. Then, \Cref{alg:MetaAlg1} draws $N =d^22^{\poly(K/\eps)}\log( 1/\delta)$ i.i.d.\ samples from $D$, runs in time $\poly(N)$, and returns a  hypothesis $h$ such that, with probability at least $1 - \delta$, it holds that $\pr_{(\x,y)\sim D}[h(\x)\neq y] \leq K \, O(\opt\log(1/\opt))+\eps$, where $\opt=\inf_{f\in \mathcal{F}_{K,d}}\pr_{(\x,y)\sim D}[f(\x)\neq y]$.
\end{theorem}
\begin{proof}
First, note that without loss of generality we can assume that $\opt\geq \eps^2$ by randomly flipping an $\eps^2$ fraction of the labels with probability $1/2$. Therefore, by the mean value theorem we have that 
\begin{align*}
(\opt+\eps^2)\log(1/(\opt+\eps^2) +1)&\leq (\opt)\log(1/(\opt) +1)+\eps^2\log(1/\opt +1)\\&\leq (\opt)\log(1/(\opt) +1)+\eps\;.    
\end{align*}
 Hence,  by applying \Cref{prop:GoodCondIntersections} we have that the class of intersections of $K$-halfspaces is a class of $(2,\opt+\eps,\eps,K,KO(\opt\log(1/\opt+1)+\eps),\poly(\eps/K),O(\sqrt{\log K}))$-well-behaved Multi-Index Models. 
 
 Thus, by \Cref{thm:MetaTheorem} we have that \Cref{alg:MetaAlg1} draws $N=d^22^{\poly(K/\eps)}\log( 1/\delta)$ and in $\poly(N)$ time returns a hypothesis $h$ with error $KO(\opt\log(1/\opt+1))+O(\eps)$. 
Finally, note that without loss of generality $\opt\leq 1/2$, since one of the constant classifiers that predict $0$ and $1$ can achieve error at most $1/2$. Consequently, we have that $\pr_{(\x,y)\sim D}[h(\x)\neq y] \leq K \, O(\opt\log(1/\opt))+O(\eps)$, since $x\log(1/x+1)=\Theta(x\log(1/x))$ for all $x\in (0,1/2]$.
\end{proof}

{For the class of intersections of $K$ halfspaces, we adapt the general algorithm \Cref{alg:MetaAlg2} to perform polynomial regression with polynomials of degree at most $2$ in Line \ref{line:regression} of \Cref{alg:MetaAlg2}.
This degree suffices for learning intersections of halfspaces, as demonstrated in \Cref{prop:GoodCondIntersections}.} 
{Moreover, the regression task can be efficiently reduced to covariance estimation with respect to the Frobenius norm because the entries of the covariance matrix correspond directly to the Hermite coefficients. Consequently, the regression can be executed in polynomial time using $O(d^2)$ samples, so that the overall sample complexity of the algorithm scales as $d^2$ rather than $d^{O(1)}$.}

The following proposition shows that the class of intersections of halfspaces satisfies the regularity conditions in \Cref{def:GoodCondition}.

\begin{proposition}\label{prop:GoodCondIntersections}
Let $d,K\in \Z_+, \zeta, \eps \in (0,1)$.  $\mathcal{F}_{K,d}$ is a class of $(2,\zeta,\eps,K,KO(\zeta\log(1/\zeta+1) +\eps),\poly(\eps/K),O(\sqrt{\log K}))$-well-behaved Multi-Index Models according to \Cref{def:GoodCondition}.

\end{proposition}
\begin{proof}
Let $f\in \mathcal{F}_{K,d}$ and denote by $W$ the subspace spanned by the normal vectors of the halfspaces that define $f$, which has dimension at most $K$. From \cite{KOS:08,chernozhukov2017detailed} we have that the Gaussian surface area of  $ \mathcal{F}_{K,d}$ is bounded by $O(\sqrt{\log(K)})$. 
It remains to show that  Condition 3 of \Cref{def:GoodCondition} is satisfied for the class $\mathcal{F}_{K,d}$.

{Let $D$ be a distribution over $\R^d\times \{0,1\}$ ,where the $\x$-marginals of $D$ are standard normal, such that $\pr_{(\x,y)\sim D}[f(\x)\neq y]\le \zeta$.
Let $V$ be a subspace of $\R^d$, and define $\tau(x)=x\log(1/x+1)$.}
{Assume that for every function \( g : V \to \{0,1\} \) it holds that \( \pr[f(\vec{x}) \neq g(\vec{x}^V)] > CK \tau(\zeta) + 2 \epsilon \), for a sufficiently large constant $C>0$. In particular, consider the function  $h(\x)=\argmin_{i\in \{0,1\}}\Pr_{(\x,y)\sim D}[y\neq i\mid \x^{V}]$ which depends only on $\x^V$ (i.e., $h(\x)=h(\x^V)$). Then the same lower bound applies to $h$, hence \( \pr_{(\x,y)\sim D}[h(\x) \neq y] > (C-4)K \tau(\zeta) + 2\epsilon \), since $\tau(x) \ge x/4$ for all $x\in (0,1].$}

{We first show that if the subspace $V$ is far from optimal, then when we partition the space into cylinders determined by points in $V$, the target function is far from constant on a nontrivial fraction of these cylinders. }
Specifically, we show the following

{\begin{claim}[Fraction of Good Points]\label{it:LargeError2FractionofStripsInter}
With probability at least $\eps$ over a random draw of $\x_0 \sim \cN(\vec{0},\vec{I})$ (independent of $\x$), the following holds: for every $i\in \{0,1\}$,
$$
\Pr_{(\x,y)\sim D}[y \neq i \mid \x^V=\x_0^V] > (C-4)K\,\tau\left(\Pr_{(\x,y)\sim D}[f(\x) \neq y \mid \x^V=\x_0^V]\right) + \eps\;,
$$
and $
\Pr_{(\x,y)\sim D}[f(\x)\neq y\mid \x^V=\x_0^V] \le 4/(C-4)
$.
\end{claim}
\begin{proof}[Proof of \Cref{it:LargeError2FractionofStripsInter}]
Assume for contradiction that the set $U$ that contains $\x_0\in \R^d$ such that for all $i\in\{0,1\}$ it holds $\pr_{(\x,y)\sim D}[y\neq i\mid \x^{V}=\x_0^V]> (C-4)K\tau\left(\pr_{(\x,y)\sim D}[f(\x)\neq y\mid \x^{V}=\x_0^V]\right)+\eps$
has probability mass less than $\eps$, i.e., $\pr[x\in U]< \eps$. 
Then, by the definition of the piecewise constant classifier
we can write
$$
\Pr_{(\x,y)\sim D}[h(\x)\neq y] = \E_{\x_0^V}\left[\min_{i\in\{0,1\}} \Pr_{(\x,y)\sim D}\left[i\neq y \mid \x^V=\x_0^V\right]\right].
$$

Note that for every $\x_0^V$ not in $U$, it holds that
$$
\min_{i\in\{0,1\}} \Pr_{(\x,y)\sim D}\bigl[i\neq y \mid \x^V=\x_0^V\bigr] \le (C-4)K\,\tau\Bigl(\Pr_{(\x,y)\sim D}\bigl[f(\x)\neq y \mid \x^V=\x_0^V\bigr]\Bigr) + \eps.
$$
Therefore, by taking the expectation over $\x_0^V$ and applying Jensen's inequality (noting that $\tau(x)=x\log(1/x+1)$ is concave for all $x>0$), we obtain
\begin{align*}
\Pr_{(\x,y)\sim D}[h(\x)\neq y] 
&\le (C-4)K\,\E_{\x_0^V}\Bigl[\tau\left(\Pr_{(\x,y)\sim D}\left[f(\x)\neq y \mid \x^V=\x_0^V\bigr]\right)\right] + \eps+\pr[\x^V\in U]\\
&\le (C-4)K\,\tau\left(\E_{\x_0^V}\left[\Pr_{(\x,y)\sim D}\left[f(\x)\neq y \mid \x^V=\x_0^V\right]\right]\right) + 2\eps\\
&= (C-4)K\,\tau\left(\Pr_{(\x,y)\sim D}\left[f(\x)\neq y\right]\right) + 2\eps\;,
\end{align*}
 which contradicts the assumed error lower bound for $h$. Therefore, it holds that $\Pr_{\x\sim \cN(\vec{0},\vec{I})}[\x \in U] \ge \eps$.
Finally, note that for every \(\x_0\) in \(U\) we have, by the definition of \(U\),
$$
(C-4)K\,\tau\left(\Pr_{(\x,y)\sim D}\left[f(\x)\neq y \mid \x^V=\x_0^V\right]\right) + \eps \le \Pr_{(\x,y)\sim D}\left[y\neq i \mid \x^V=\x_0^V\right] \le 1.
$$
By using that $\tau(x) \ge x/4$ for all $x\in (0,1]$ we deduce the desired claim.
\end{proof}}
Fix a point $\x_0\in\R^d$ that satisfies the statement of \Cref{it:LargeError2FractionofStripsInter}. Let $D_{V^{\perp}}$ be a  distribution supported on $V^{\perp}\times \{0,1\}$  defined as $D_{V^{\perp}}(\x^{\perp V},y)=D(\x,y\mid \x^V=\x_0^V)$. 
By the properties of the Gaussian distribution, we have that the $\x$-marginal of $D_{V^\perp}$ is a standard Gaussian distribution.
Moreover, by the definition of $D_{V^\perp}$, we have that for $f'(\x)=f(\x^{\perp V}+\x_0^{V})$ and all $i\in \{0,1\}$ it holds that  
\begin{align*}
\pr_{(\x,y)\sim D_{V^{\perp}}}[y\neq i]&> (C-4)K\tau(\pr_{(\x,y)\sim D_{V^{\perp}}}[f'(\x)\neq y])+\eps\\
&\geq ((C-4)K/4)\pr_{(\x,y)\sim D_{V^{\perp}}}[f'(\x)\neq y]\log(1/\pr_{(\x,y)\sim D_{V^{\perp}}}[f'(\x)\neq y])+\eps    
\end{align*}
 and $\pr_{(\x,y)\sim D}[f'(\x)\neq y]\le 4/(C-4)$.
Finally, applying \Cref{lem:progressIntersections}  for every intersection of halfspaces $f'(\x)=f(\x^{\perp V}+\x_0^{V})$ for $\x_0^V$ that satisfies the statement of \Cref{it:LargeError2FractionofStripsInter} completes the proof of \Cref{prop:GoodCondIntersections}.
\end{proof}
Below, we state and prove a structural result for intersections of halfspaces. 
 In particular, we show that if an intersection of halfspaces is far from being a constant function, then either its first moment or its second moment exhibits a non-trivial correlation with the weight vectors that define it. This structural property is crucial for our algorithm, as it guarantees that at every iteration we can identify a direction along which to improve our approximation.
{{
\begin{proposition}
    [Structural Result for Intersections of Halfspaces]\label{lem:progressIntersections}Let $d\in \mathbb{Z}_+$, $\eps,\opt\in (0,1)$, and let $D$ be a distribution over $ \mathbb{R}^d\times \{0,1\}$ whose $\x$-marginal is $\cN(\vec0,\vec I)$. Let $f\in \mathcal{F}_{K,d}$ be a classifier with error $\pr_{(\vec x,y)\sim D}[f(\vec x)\neq y]=\opt$, and assume that $f$ can be expressed in the form $f(\vec x)=\prod_{i\in [K]}\Ind(\vec{w}^{(i)} \cdot \vec{x}+t_i\ge 0)$ with $\w^{(1)},\dots, \w^{(K)}\in \mathbb{R}^d, t_1,\dots, t_K\in \R$. Let $C>0$ be a sufficiently large universal constant.
    If it holds that $\min(\pr_{(\vec x,y)\sim D}[y\neq 0],\pr_{(\vec x,y)\sim D}[y\neq 1])\ge C K \opt\log(1/\opt)+\eps$ and $\opt\le 1/C$, then there exists some $i\in [K]$ such that
 \[
 \max\left(\left|\E_{(\vec x,y)\sim D}\left[\vec{w}^{(i)}\cdot \x y\right]\right|,\left|\E_{(\x, y) \sim D}\left[\left((\vec{w}^{(i)}\cdot \x)^2 -1\right) y\right]\right|\right) \ge \poly(1/K,\eps)\;.\]
\end{proposition}
}
\begin{proof}
Let $W=\spaning(\vec w^{(1)},\dots,\vec w^{(K)})$ and without loss of generality assume that $\norm{\vec w^{(i)}}=1$ for all $i\in [K]$, since normalizing the weight vectors does not change the output of $f$.
{Assume that for all \(i\in[K]\) the first moments are small, i.e., $\abs{\E_{(\vec x,y)\sim D}[\vec{w}^{(i)}\cdot \x y]}\le \eps p(\eps/K)$, where $p$ is a sufficiently small universal polynomial  (to be quantified later). We  show that there exists  $i\in [K]$ such that the second moment is large, i.e., }$\abs{\E_{(\x, y) \sim D}[((\vec{w}^{(i)}\cdot \x)^2 -1) y]}\ge \poly(\eps/K)$ for a sufficiently small universal polynomial.
 We can decompose the second-moment matrix as follows:
\begin{align}\E_{(\vec x,y)\sim D}[(\vec{x}\vec{x}^\top -\vec I) y]&=  \E_{(\vec x,y)\sim D}[ (\vec{x}\vec{x}^\top -\vec I) f(\x)]+\E_{(\vec x,y)\sim D}[ (\vec{x}\vec{x}^\top -\vec I)(\Ind( f(\vec x)\neq y,y=1 )-\Ind(y\neq f(\vec x),f(\x)=1))]\notag\\
 &=\underbrace{\E_{\vec x\sim \cN(\vec 0,\vec I)}[({\vec{x}}{\vec{x}}^\top-\vec I) f(\x)]}_{\vec I_1}\underbrace{+\E_{(\vec x,y)\sim D}[ ({\vec{x}}{\vec{x}}^\top-\vec I) (T(\vec x,y)-S(\vec x,y) )]}_{\vec I_2}\;,\notag
\end{align}
{where $S(\x,y)\eqdef \Ind(y\neq f(\vec x),f(\x)=1), T(\x,y)\eqdef\Ind( f(\vec x)\neq y,y=1 )$.}
Note that $\vec I_1$ is the second moment of $f(\x)$, i.e., the noiseless classifier, and $\vec I_2$ is the contribution corresponding to the noise.

First, we show the following  tail bounds, which will help us bound the contribution corresponding to the noise.
\begin{claim}\label{cl:truncatedmeans}
   Let $\eps\in (0,1)$ less than a sufficiently small universal constant and let $R:\R^d\to [0,1]$ be a function such that $\E_{\x\sim \cN(\vec 0,\vec I)}[R(\x)]=\eps$. Then, for every unit vector $\vec v\in \R^d$, the following bounds hold:
   \begin{itemize}
       \item[i)] $\E_{\x\sim \cN(\vec 0,\vec I)}[R(\x)(\vec v\cdot \vec x)^2]\lesssim \eps \log(1/\eps)$.
       \item[ii)] $\E_{\x\sim \cN(\vec 0,\vec I)}[R(\x)(\vec v\cdot \vec x)]\lesssim \eps \sqrt{\log(1/\eps)}$.
   \end{itemize}
\end{claim}
\begin{proof}
First, we will prove (i) and then we will apply it to prove (ii). Note that the expectation $\E[(\vec v\cdot \vec x)^2R(\x)]$ is minimized if $R$ puts all of its mass at the tails of the function $(\vec v\cdot \x)^2$, i.e., $R(\x)=\Ind((\vec v\cdot \vec x)^2\ge y_{\eps})$, where $y_{\eps}$ is the smallest number such that $\E_{\x\sim \cN(\vec 0,\vec I)}[\Ind((\vec v\cdot \vec x)^2\ge y_{\eps})]\le \eps$.
Hence, we have that
\begin{align}
    \E_{\x\sim \cN(\vec 0,\vec I)}[(\vec v\cdot \vec x)^2R(\x)]&\le  \E_{\x\sim \cN(\vec 0,\vec I)}[(\vec v\cdot \vec x)^2\Ind((\vec v\cdot \vec x)^2\ge y_{\eps})]\notag\\
    &=\int_{\mathbb{R}^d}(\vec v\cdot \vec x)^2 \Ind((\vec v\cdot \vec x)^2\ge y_{\eps})\phi(\x)d\x\notag\\
     &= \frac{1}{\sqrt{2\pi}}\int_{ \sqrt{y_{\eps}} }^{\infty} x^2e^{-x^{2}/2}dx\notag\\
    &=  \frac{1}{\sqrt{2\pi}} y_{\eps}e^{-{y_{\eps}}/2}+ \eps\label{partialIntegration}\\
     &\le   \eps\log(1/\eps) + \eps \lesssim  \eps\log(1/\eps)\;,\label{MonotonicitySubstitution} 
\end{align}
where in \eqref{partialIntegration} we simply used integration by parts. In \eqref{MonotonicitySubstitution} we used the fact that the function $xe^{-x/2}$ is decreasing for $x\ge 2$. Note that, by \Cref{fact:gaussianfacts}, we have $y_{\eps}\gtrsim\log(1/\eps)\ge 2$, and since $\eps$ is assumed to be sufficiently small. 

Finally, we can prove (ii) by simply using \CS
\begin{align*}
    \E_{\x\sim \cN(\vec 0,\vec I)}[\vec v\cdot \vec xR(\x)]^2\le \E_{\x\sim \cN(\vec 0,\vec I)}[(\vec v\cdot \vec x)^2R(\x)]\E_{\x\sim \cN(\vec 0,\vec I)}[R(\x)]\lesssim\eps^2\log(1/\eps)\;,
\end{align*}
which concludes the proof of \Cref{cl:truncatedmeans}.
\end{proof}
We next show that there exists a direction within $W$ along which the second moment of $f(\x)$ is small. More precisely, we have the following:
\begin{claim}\label{cl:inliers_intersections}
There exists  an  $i\in [K]$ such that $$(\w^{(i)})^{\top}\vec I_1\w^{(i)}\lesssim -\frac{\pr[f(\x)=1]\pr[f(\x)\neq 1]}{K}+p\bigg(\frac{\eps}{K}\bigg)+\opt\log(1/\opt)\;.$$ 
\end{claim}
\begin{proof}

 {{Denote by  $\mu\eqdef \E[\x f(\x)]/\pr[f(\x)=1]$ and define $\bar{\x}\eqdef\x-\vec \mu$. Let $\vec M=\E_{\vec x\sim \cN(\vec 0,\vec I)}[\bar{\vec{x}}\bar{\vec{x}}^\top f(x)]$. For any defining vector $\vec w^{(i)}$,  by \Cref{fact:VarianceMonotonicity} we have 
 \[(\vec w^{(i)})^{\top}\vec M\vec w^{(i)}\le \pr[f(\x)=1]\big(1-ce^{-(\max(0,t_i))^2/2}\big)\;,\] for a sufficiently small universal constant $c>0$.
Next, we claim that there exists some $i\in[K]$ for which the bias $t_i$ is at most $ 2\sqrt{\ln({K}/{\pr[f(\x)\neq 1]})}$.
To see this, assume for contradiction that for every $i\in[K]$ we have $t_i> 2\sqrt{\ln(\frac{K}{\pr[f(\x)\neq 1]})}$ and note that if $t_i> 2\sqrt{\ln(\frac{K}{\pr[f(\x)\neq 1]})}$ then $\pr_{\x \sim \cN(\vec 0,\vec I)}[\vec w^{(i)}\cdot \x +t_i\le 0]< \pr[f(\x) \neq 1]/K$. By applying the union bound, we obtain
$$\pr_{\x \sim \cN(\vec 0,\vec I)}[f(\x)\neq 1]\le \sum_{i=1}^K \pr_{\x \sim \cN(\vec 0,\vec I)}[\vec w^{(i)}\cdot \x +t_i\le 0]< \pr[f(\x) \neq 1]\;,$$
which is a contradiction. Hence, there must exist some index $i\in[K]$ with $(\vec w^{(i)})^{\top}\vec M\vec w^{(i)} \le \pr[f(\x)=1](1-c\pr[f(\x)\neq 1]/K)$. Fix such an index $i\in[K]$. 
Then, by definition,
$$
(\vec w^{(i)})^\top \vec I_1\, \vec w^{(i)}
=\; (\vec w^{(i)})^\top \vec M\, \vec w^{(i)} - \Pr_{\x\sim \cN(\vec 0,\vec I)}[f(\x)=1] 
+ \frac{\left(\E_{\x\sim \cN(\vec 0,\vec I)}\Big[\vec w^{(i)}\cdot \x\, f(\x)\Big]\right)^2}{\Pr_{\x\sim \cN(\vec 0,\vec I)}[f(\x)=1]}.
$$
Substituting the bound on $(\vec w^{(i)})^\top \vec M\, \vec w^{(i)}$, we obtain
$$
(\vec w^{(i)})^\top \vec I_1\, \vec w^{(i)}
\le - c\,\frac{\Pr[f(\x)=1]\,\Pr[f(\x)\neq 1]}{K} 
+ \frac{\left(\E_{\x\sim \cN(\vec 0,\vec I)}\Big[\vec w^{(i)}\cdot \x\, f(\x)\Big]\right)^2}{\Pr_{\x\sim \cN(\vec 0,\vec I)}[f(\x)=1]}.
$$
}

{Note that  $\E_{\x\sim \cN(\vec 0,\vec I)}[\vec w^{(i)}\cdot \x f(\x)]= \E[\w^{(i)}\cdot \x y]+\E[\w^{(i)}\cdot \x(S(\x,y)-T(\x,y))]$.
Recall that by assumption we have that $\abs{\E_{(\vec x,y)\sim D}[\vec{w}^{(i)}\cdot \x y]}\le \eps p(\eps/K)$ for all $i\in [K]$.
Moreover, using the triangle inequality  $\Pr[f(\x)=1] \ge \Pr[y=1] - \Pr[f(\x)\neq y]$ together with the assumption in \Cref{lem:progressIntersections} that $\Pr[y=1]\geq CK\opt\log(1/\opt) +\eps$, it follows that $\Pr[f(\x)=1] \geq \opt+\eps$. Thus, \[\E[\w^{(i)}\cdot \x y]^2/\pr[f(\x)=1]\le p(\eps/K)\;.\]
}

{Next, observe that  $\abs{\E[\w^{(i)}\cdot \x S(\x,y)]}= \abs{\E[\w^{(i)}\cdot \x R(\x)]}$, where we define $R(\vec x)=\E_{(\x,y)\sim D}[ S(\vec x,y)\mid \vec  x]$. Since, for all $\vec x\in \mathbb{R}^d$ it holds that $R(\vec x)\in [0,1]$, and $\E_{\vec x\sim \cN(\vec 0,\vec I)}[R(\vec x)]\le \opt$, by \Cref{cl:truncatedmeans} we have that $\abs{\E[(\w^{(i)}\cdot \x) S(\x,y)]}\lesssim \opt\sqrt{\log(1/\opt)}$.
Consequently, since $\pr[f(\x)=1]\ge \opt$, it follows that $\E[\w^{(i)}\cdot \x S(\x,y)]^2/\pr[f(\x)=1]\lesssim\opt \log(1/\opt)$. A similar argument shows that $\E[\w^{(i)}\cdot \x T(\x,y)]^2/\pr[f(\x)=1]\lesssim\opt \log(1/\opt)$.
Combining these bounds, we obtain
\begin{align*}
    (\w^{(i)})^{\top}\vec I_1\w^{(i)}\lesssim -\frac{\pr[f(\x)=1]\pr[f(\x)\neq 1]}{K}+p\bigg(\frac{\eps}{K}\bigg)+\opt\log(1/\opt)\;,
\end{align*}
which concludes the proof of \Cref{cl:inliers_intersections}.}}
\end{proof}
 Now we bound the contribution corresponding to the noise, using \Cref{cl:truncatedmeans}.
\begin{claim}\label{cl:outliers_intersections}
 For every unit vector $\vec v\in \R^d$, it holds that $\abs{\vec v^{\top }\vec I_2\vec v}=O(\opt \log(1/\opt))$.   
\end{claim}
\begin{proof}   
 Let $\vec v\in \mathbb{R}^d$ be a unit vector, and denote by $R(\vec x)=\E_{(\x,y)\sim D}[ S(\vec x,y)\mid \vec  x]$. Note that $R(\vec x)\in [0,1]$ for all $\vec x\in \mathbb{R}^d$, and moreover $\E_{\vec x\sim \cN(\vec 0,\vec I)}[R(\vec x)]\le \opt$. Hence, by \Cref{cl:truncatedmeans} we have that $\abs{\E_{(\vec x,y)\sim D}[ (\vec{v}\cdot \vec{x})^2 S(\vec x,y)]}=O(\opt\log(1/\opt))$.
 By using similar arguments we have that $\abs{\E_{(\vec x,y)\sim D}[ (\vec{v}\cdot \vec{x})^2  T(\vec x,y)]}=O( \log(1/\opt) \opt)$.
 Therefore, by triangle inequality, we have that
 \begin{align*}
   \abs{\vec v^{\top }\vec I_2\vec v}&\le\abs{\E_{(\vec x,y)\sim D}[ (\vec{v}\cdot \vec{x})^2 S(\vec x,y)]}+\abs{\E_{(\vec x,y)\sim D}[ (\vec{v}\cdot \vec{x})^2 T(\vec x,y)]}+\E[S(\x,y)+T(\x,y)]\\
   &=O( \opt\log(1/\opt))\;,
 \end{align*}
 {which concludes the proof of \Cref{cl:outliers_intersections}.}
\end{proof}
Denote by $t$ the quantity $t\coloneqq C'\opt\log(1/\opt) $, for sufficiently large constant $C'$. Combining \Cref{cl:inliers_intersections,cl:outliers_intersections}, {we have that there exists $i\in[K]$ such that
\begin{align*}
    (\w^{(i)})^\top\E_{(\vec x,y)\sim D}[(\vec{x}\vec{x}^\top -\vec I) y] (\w^{(i)})&\lesssim -\pr[f(\x)=1]\pr[f(\x)\neq 1]/K-t+p(\eps/K)\\
   &\le  -\left(\min_{i\in\{0,1\}}(\pr[y=i] )-\opt\right)/(2K)-t+p(\eps/K)\\&\le -\eps/(2K)+p(\eps/K) \;,
\end{align*}
where we used the fact that one of $\pr[f(\x)=1]$ or $\pr[f(\x)\neq 1]$ must be at least $1/2$, and the assumption of \Cref{lem:progressIntersections} that for every label $i\in\{0,1\}$ it holds that $\pr[y=i]\ge CK\opt \log(1/\opt)+\eps$, for a sufficiently large constant $C$. Setting $p(\eps/K)\le c\eps/K$ for a sufficiently small universal constant $c>0$, we ensure that the error term $p(\eps/K)$ is negligible compared to the main term. This concludes the proof of  \Cref{lem:progressIntersections}.}
\end{proof}

}
 \subsection{Learning Multi-Index Models with RCN}
\label{sec:RCNAlgo}
In this section, we demonstrate that our algorithmic approach becomes more efficient when the labeled examples are perturbed by uniform random noise.
In particular, we study the performance of our approach in the RCN setting (see \Cref{def:RCNdistribution}). 
We show that  \Cref{alg:MetaAlg1} can be modified into \Cref{alg:MetaAlgRCN}, which achieves improved complexity in the RCN regime compared to the agnostic noise setting. The main result of this section is \Cref{thm:MetaAlgRCN} (\Cref{thm-intro:AlgRCN-linear}).
\begin{definition}[RCN Distribution]\label{def:RCNdistribution}
Let $D$ be a distribution over $\R^d\times \mathcal{Y}$. We say that $D$ is an RCN distribution for the class $\mathcal{F}\subseteq\{f:\R^d\to \mathcal{Y}\}$ if there exists a function $f\in \mathcal{F}$ and a row-stochastic matrix $\vec H\in \R^{\abs{\cY}\times \abs{\cY}}$ with $\vec H_{i,i}>\vec H_{i,j}+\gamma$ for all $i,j\in \cY$, for some $\gamma>0$, such that $\pr_{(\x,y)\sim D}[y=i\mid \x ]=\vec H_{f(\x),i}$.
We refer to $\vec H$ as the confusion matrix of the distribution.
\end{definition}

{The key distinction that enables us to have an improved algorithm (in terms of complexity) in the RCN setting is that the output label is determined solely by the ground truth label.
Consequently, the moments of the observed labels (noisy labels) depend only on the subspace defining the classifier. This implies that any nonzero moment yields vectors that lie approximately within that subspace (see \Cref{lem:improvedCorrelation}). The main result of this section is stated in \Cref{thm:MetaAlgRCN}.}

\begin{algorithm}[h!]
    \centering
    \fbox{\parbox{5.1in}{
            {\bf Input:} $\eps, \delta>0$ and  sample access to a distribution $D$ over $\mathbb{R}^d\times \mathcal{Y}$ {that satisfies \Cref{def:RCNdistribution}} { and whose $\x$-marginal is $\cN(\vec 0,\vec I)$}.\\
            {\bf Output:} A hypothesis $h$ such that with probability at least $1-\delta$ it holds that $\pr_{\x\sim \cN(\vec 0,\vec I)}[h(\x)\neq y]\le \tau+\opt+\eps${, where $\opt\eqdef \inf_{f\in \mathcal{F}} \pr_{(\x,y)\sim D}[f(\x){\neq} y]$}.
            \begin{enumerate}
            \item {Let $C$ be a sufficiently large universal constant.}
            \item { Let $N\gets (dm)^{Cm}2^{K^C}(\Gamma \abs{\cY} /\eps)^{CK}\log(1/\delta)/(\sigma\alpha)^{C}$,
              $\eps'\gets \eps^3/(C\abs{\mathcal{Y}}\Gamma^2 K^2))$, $L_1\gets\emptyset$,$t\gets 1$.}\label{line:initRCN}
                \item  While $t\le 2K$
            \label{line:loopRCN}
             \begin{enumerate}
             \item {Draw a set $S_t$ of $N$ i.i.d.\ samples from $D$.}

             \item $\mathcal{E}_t\gets$ \Cref{alg:MetaAlg2}($\eps, \eps',\delta,\spaning(L_t),S_t$).
          \item Choose any vector $\vec v\in \mathcal{E}_t$.
          \item $L_{t+1}\gets L_{t}\cup \{\vec v\} $.\label{line:updateRCN}
                \item $t\gets t+1$.
            \end{enumerate}
             \item  {
             Construct $\mathcal{S}$ an $\eps'$-approximating partition with respect to $\spaning(L_t)$, as defined in \Cref{def:approximatingPartition}.}
             \item {Draw $N$ i.i.d.\ samples from $D$ and construct the piecewise constant classifier $h_{\mathcal{S}}$ as follows: For each $S \in \mathcal{S}$, assign a label that appears most frequently among the samples falling in $S$. Return $h_{\mathcal{S}}$.}
             
            \end{enumerate}
    }}
    \medskip
        \caption{Learning {Multi-Index Models} under RCN }
 \label{alg:MetaAlgRCN}
\end{algorithm}
\begin{theorem}[Learning MIMs with  RCN]\label{thm:MetaAlgRCN}
Let $D$ be a distribution over $\mathbb{R}^d \times \mathcal{Y}$ whose $\x$-marginal is $\mathcal{N}(\vec 0, \vec I)$, and let $\epsilon, \delta \in (0,1/2)$. Let $\opt=\inf_{f\in\mathcal{F}}\pr_{(\x,y)\sim D}[f(\x)\neq y]$, where $\mathcal{F}$ is a class of $(m,\zeta+\eps,\alpha,K,\tau,\sigma,\Gamma)$-well-behaved Multi-Index Models, as  in \Cref{def:GoodCondition}. If $\zeta\ge \opt$ and $D$ satisfies \Cref{def:RCNdistribution}, then \Cref{alg:MetaAlgRCN} draws $N= (dm)^{O(m)}2^{\poly(K)}(\Gamma \abs{\cY} /\eps)^{O(K)}\log(1/\delta)/(\sigma\alpha)^{O(1)} $ i.i.d.\ samples from $D$, runs in time $\poly(N)$, and returns a  hypothesis $h$ such that, with probability at least $1 - \delta$, it holds $\pr_{(\x,y)\sim D}[h(\x)\neq y] \le\tau+\opt+\eps$.
\end{theorem}

{We remark that for the class of multiclass linear classifiers using \Cref{lem:progress}, we can show that this class is $(1,\zeta,\eps,K,O(\zeta+\eps),\poly(\eps/K),O(K\sqrt{\log(K)}))$-well-behaved. As a result, by applying \Cref{thm:MetaAlgRCN}, we obtain an efficient algorithm for learning under Random Classification Noise (RCN).
Moreover, for the class of multiclass linear classifiers \Cref{alg:MetaAlg2} performs linear regression under Gaussian marginals (see Line \ref{line:regression}). This task that can be efficiently reduced to mean estimation, specially to estimating $\E[\Ind(y=i)]$ and $\E[\x\Ind(y=i)]$. 
Since these mean estimates can be computed in polynomial time using $O(d)$ samples, the overall sample complexity of the algorithm scales as $d$, rather than $d^{O(1)}$. This establishes \Cref{thm-intro:AlgRCN-linear}.}

Our proof of the \Cref{thm:MetaAlgRCN} relies on the following proposition, which states that, given that our distribution is an RCN distribution with respect to a function in the class,  any vector returned by \Cref{alg:MetaAlg2} has an improved correlation with the ground truth subspace compared to the agnostic case (see \Cref{prop:MetaAlg2}).

\begin{proposition}[Structural Result for RCN]\label{lem:improvedCorrelation}
Let $d,k,K\in \Z_+$, $\eps,\delta\in(0,1)$  and $\eps'\le c\eps^3/(\abs{\mathcal{Y}}\Gamma^2 k^2))$, for $c$ a sufficiently small universal constant. 
Let $D$ be a distribution supported on $ \mathbb{R}^d \times \mathcal{Y}$ whose $\x$-marginal is $\mathcal{N}(\vec 0, \vec I)$, let $V$ be a $k$-dimensional subspace of $\R^d$ and let $\mathcal{F}\subseteq \{f:\R^d\to \mathcal{Y}\}$ be a class of $(m,\zeta+O(\eps),\alpha,K,\tau,\sigma,\Gamma)$-well-behaved Multi-Index Models according to \Cref{def:GoodCondition}.
{Consider a function $f\in\mathcal{F}$ with $\pr_{(\x, y)\sim D}[f(\x)\neq y]=\opt$ and suppose that $f$ depends only on a $K$-dimensional subspace $W$, i.e., $f(\x)=f(\x^{W})$ for all $\x\in \R^d$. 
In addition, let $h_{\mathcal{S}}$ be a piecewise constant approximation of $D$ with respect to an $\eps'$-approximating partition $\mathcal{S}$ of $V$ (as defined in \Cref{def:approximatingPartition,def:h}). 

Then, there exists a sample size $N=(dm)^{O(m)}(k/\eps')^{O(k)}\log(1/\delta)(K\Gamma\abs{\cY} /(\sigma\eps\alpha))^{O(1)}$ such that if $D$ satisfies the RCN condition (see \Cref{def:RCNdistribution}) for $f$,
and if $\pr_{(\x, y)\sim D}[h_{\mathcal{S}}(\x)\neq y]> \tau+\opt+\eps$ and $\zeta\ge \opt$, then \Cref{alg:MetaAlg2}, when given $N$ i.i.d.\ samples from $D$, runs in time
$\poly(N)$ and, with probability at least $1-\delta$, returns a list of unit vectors $\mathcal E$ of size $|\mathcal E|=\poly(mK\Gamma/(\eps\sigma\alpha))$, such that
for every vector $\vec v\in \mathcal E$, there exists a unit vector $\w\in W$ with}
$$\w \cdot \vec v\ge 1- \eps\;.$$
\end{proposition}
\begin{proof}
Denote by $\vec H$ the stochastic matrix that perturbs the labels of $f$.
Also let $\widehat{\vec U}$ be the matrix computed in Line \ref{line:matrixMeta} of \Cref{alg:MetaAlgRCN} and let $\eta^2=\poly(\eps\sigma\alpha/(mK\Gamma))$ be the $L_2^2$ error chosen in the polynomial-regression step, i.e., Line \ref{line:regression} of \Cref{alg:MetaAlgRCN}.

 {We show that if the sample size is sufficiently large then for each region $S\in \mathcal{S}$ the associated set $U_S$ (see Line \ref{line:setUs} of \Cref{alg:MetaAlg2}) contains only vectors that have only a small component in the subspace orthogonal to $W$. This in turn implies that any large vector must have a significant component in $W$.}

\begin{claim}\label{cl:NobadvectorRCN}
Fix $S\in \mathcal{S}$. Suppose that the number of samples falling in $S$ satisfies $N_S\ge(dm)^{Cm}\log(\abs{\mathcal{Y}/\delta})/\eta^{C}$ for a sufficiently large universal constant $C>0$. Then, for any 
$\vec u\in U_S$ and any unit vector $\vec v\in  W^{\perp}$ we have that $\abs{\vec v \cdot \vec u} \le \sqrt{2Km}\eta/\sigma $.\end{claim}
\begin{proof}
Let $N_S, S\in \mathcal{S}$, be the number of samples that land in the set $S$. From the rotational invariance of the standard Gaussian, without loss of generality let $\vec e_1,\dots, \vec e_t$ be a basis of $ W^{\perp}$, $t\in\Z_+$.
By \Cref{fact:regressionAlg} and the union bound, we have that with $N_S=(dm)^{O(m)}\log(\abs{\cY}/\delta)/\eta^{O(1)}$ i.i.d.\ samples and runtime $\poly(N_S,d)$, we can compute polynomials $p_{S,i}$ such that \[\E_{(\x, y)\sim D}[(\Ind(y=i) -p_{S,i}(\x^{\perp V}))^2]\le \min_{p'\in \mathcal{P}_m} \E_{(\x, y)\sim D}[(\Ind(y=i) -p'(\x^{\perp V}))^2]+\eta^2\;.\]
Hence, by averaging over $V$ conditioned on the cube $S$, we obtain that \[\E_{(\x, y)\sim D_{V^{\perp}}}[(\Ind(y=i) -p_{S,i}(\x))^2]\le \min_{p'\in \mathcal{P}_m} \E_{(\x, y)\sim D_{V^{\perp}}}[(\Ind(y=i) -p'(\x))^2]+\eta^2\;,\] where $D_{V^\perp}$ is defined to be the marginal obtained by averaging $D$ over $V$ conditioned on  $S$, i.e., $D_{V^{\perp}}(\x^{\perp V},y)=\E_{\x^V}[D(\x,y)\mid \x \in S]$.
Define, for each multi-index $\beta\in\N^d$, the Hermite coefficients of $p_{S,i}$ and $\Ind(y=i)$ by $\widehat{p}_{S,i}(\beta)\eqdef \E_{(\x,y)\sim D_{V^{\perp}}}[H_{\beta}( \x)p_{S,i}(\x)]$ and $ \widehat{g}_i(\beta)\eqdef \E_{(\x,y)\sim D_{V^{\perp}}}[H_{\beta}( \x)\Ind(y=i)]$.
By the orthogonality of Hermite polynomials, we have \[\E_{(\x, y)\sim D_{V^{\perp}}}[(\Ind(y=i) -p_{S,i}(\x))^2]=\sum_{\beta \in \N^d} (\widehat{p}_{S,i}(\beta)- \widehat{g}_i(\beta) )^2\;.\] Thus, restricting the sum to multi-indices $\beta$ with $1\le \|\beta\|_1\le m$, it follows that $\sum_{\beta\in \N^d,1\le \|\beta\|_1\le m }(\widehat{p}_{S,i}(\beta)- \widehat{g}_i(\beta))^2 \le \eta^2$.
Moreover,  if $\beta_i\neq 0$, then for some $i\in [t]$, it holds
\begin{align*}
    \widehat{g}_i(\beta)&=
    \E_{(\x,y)\sim D}[H_{\beta}( \x^{\perp V})\Ind(y=i)\mid S]=\sum_{j=1}^K  \vec H_{j,i}\E_{\x \sim \cN(\vec 0,\vec I)}[H_{\beta}( \x^{\perp V})\Ind(f(\x)=j)\mid S] \\ 
    &=\sum_{j=1}^K  \vec H_{j,i}\E_{\x^{V}}[\E_{\x^{\perp V\cap W}}[ H_{\beta}( \x^{\perp V\cap W})\Ind(f(\x^{W})=j)]\E_{\x^{\perp V\cap\perp W}}[H_{\beta}( \x^{\perp V\cap \perp  W}) ] \mid S]=0\;,
\end{align*}
where we used the fact that each $\x \in S$ depends only on $\x^{V}$, and  any non-constant Hermite polynomial has zero mean under the standard Gaussian distribution. 
{Thus, let $E_{error}$ denote the set of multi-indices $\beta$ for which the corresponding Hermite coefficient of $\Ind(y=i)$ is zero, i.e., $\widehat{g}_i(\beta)=0$. Then, we have that the sum of the Hermite coefficients, restricted to multi-indexes that in set $E_{error}$ is bounded by
$$
\sum_{\beta\in E_{error}}\widehat{p}_{S,i}(\beta)^2 \le \eta^2\;.
$$}By \Cref{fact:gradientNorm}, for any $S\in \mathcal{S},i\in \cY$, we have that 
\begin{align*}
\vec e_1^{\top}\E_{(\x,y)\sim D_{V^{\perp}}}[\nabla p_{S,i}(\x)\nabla p_{S,i}(\x)^{\top}]\vec e_1=  \E_{(\x,y)\sim D_{V^{\perp}}}[(\nabla p_{S,i}(\x)\cdot\vec e_1)^2]&= \sum_{\beta \in \N^d} \beta_1 (\widehat{p}_{S,i}(\beta))^2\\
&\le  m\eta^2\;. 
\end{align*}
{Hence, since $U_{S,i}$ contains the eigenvectors} of $\E_{(\x,y)\sim D_{V^{\perp}}}[\nabla p_{S,i}(\x)\nabla p_{S,i}(\x)^{\top}]$ with eigenvalue greater than $\sigma^2/(2K)$, using \Cref{cl:smallQuadraticFormNotLargeEigProj} we complete the proof.
\end{proof}
\begin{claim}
    \label{cl:smallQuadraticFormNotLargeEigProj}
Let $\vec M\in \R^{d\times d}$ be a symmetric, PSD matrix and let $\vec v\in \R^d$ be a unit vector such that $\vec v^{\top} \vec M \vec v\le \eps$. Let $U$ denote  the set of eigenvectors of $\vec M$ with eigenvalue {at least} $\lambda$. Then, for every $\vec u\in U$, it holds that $\abs{\vec u\cdot \vec v}\le \sqrt{\eps/\lambda}$.
\end{claim}
\begin{proof}
As $\vec M$ is symmetric, by the Spectral Decomposition Theorem, we can decompose it to the basis of  its eigenvectors $\vec u^{(1)},\dots \vec u^{(d)}$. Denoting by $\lambda_1,\dots,\lambda_d$ the corresponding eigenvalues, we have that 
\begin{align*}
   \eps \ge \vec v^{\top} \vec M \vec v= \sum_{i=1}^d\lambda_i (\vec u^{(i)}\cdot \vec v)^2\ge \sum_{\vec u\in U}\lambda  (\vec u\cdot \vec v)^2 \;,
\end{align*}
where we used the fact that $\vec M$ is a PSD matrix. This completes the proof of \Cref{cl:smallQuadraticFormNotLargeEigProj}.
\end{proof}

First note that by the definition of an approximating partition (\Cref{def:approximatingPartition}) and \Cref{fact:gaussianfacts}, for all $S\in \mathcal{S}$ we have that   $\pr_{D}[S ]= (\eps'/k)^{\Omega(k)}$ and $\abs{\mathcal{S}}= ( k/\eps')^{O(k)}$.
Therefore, by the union bound and Hoeffding's inequality, it holds that if $N\ge (k/\eps')^{Ck}\log(\abs{\mathcal{S}}/\delta)= (k/\eps')^{O(k)}\log(1/\delta) $ for a sufficiently large constant $C>0$, then with probability at least $1-\delta$  we have that  $\abs{\pr_{\widehat{D}}[S]-\pr_{D}[S]}\le \pr_{D}[S]/2$ for all $S\in \mathcal{S}$.
Hence,  $\pr_{\widehat{D}}[S]\ge \pr_{D}[S]/2$, i.e., the number of samples that fall in each set $S\in \mathcal{S}$ is at least $N\pr_{D}[S]/2= N(\eps'/k)^{\Omega(k)}$.

Therefore, {if $N\ge (dm)^{Cm}(k/\eps')^{Ck}\log(1/\delta)(K\Gamma \abs{\cY} /(\sigma\eps\alpha))^{C}$ for a sufficiently large universal constant $C>0$, then} by applying \Cref{cl:NobadvectorRCN} we have that the following: with probability $1-\delta$, for any unit vector $\vec v\in W^{\perp}$ 
\begin{align*}
    \vec v^{\top} \widehat{\vec U} \vec v\le \sum_{S\in \mathcal{S}, \vec u\in U_S}(\vec v \cdot \vec u )^2\pr_{\widehat{D}}[S] \le \eta^2 \poly(mK\abs{\cY}/\sigma)\;,
\end{align*}
where we used the fact that  $\abs{U_S}=\poly(mK\abs{\mathcal{Y}}/\sigma)$ by \Cref{it:corr-vec-existance}.
Hence, since $\widehat{\vec U}$ is a symmetric PSD matrix, applying \Cref{cl:smallQuadraticFormNotLargeEigProj} we have that for all $\vec u\in \mathcal{E}$ and  unit vectors $\vec v\in W^{\perp}$ it holds that $\abs{\vec u\cdot \vec v}\le \eta/\sqrt{\poly(\eps\sigma\alpha/(mK\Gamma\abs{\cY})}$. As a result, we have that if $\eta$ is chosen to be a sufficiently small polynomial of $\eps, \sigma,\alpha, 1/\Gamma, 1/m,1/K,1/\abs{\cY}$, then for all $\vec u^{(i)}\in \mathcal{E}$ and any $\vec v\in W^{\perp}$ we have that $\abs{\vec u^{(i)}\cdot \vec v}\le \eps$.

Moreover, by \Cref{prop:MetaAlg2} we have that if $N\ge (dm)^{Cm}(k/\eps')^{Ck}\log(1/\delta)(K\Gamma \abs{\cY} /(\sigma\eps\alpha))^{C}$ { for a sufficiently large universal constant $C>0$}, then there exists a $\vec v\in \mathcal{E}$. Moreover, as $\vec v$ is a unit vector and $\norm{\vec v^{\perp W}}\le \eps$, we have that $\norm{\vec v^{W}}\ge \sqrt{1-\eps^2}\ge 1-\eps$ which proves \Cref{lem:improvedCorrelation}.
\end{proof}
Now we are ready to prove \Cref{thm:MetaAlgRCN}. Our proof is similar to the proof of \Cref{thm:MetaTheorem}. The difference is that we are going to use \Cref{lem:improvedCorrelation} which provides improved correlation when compared to \Cref{prop:MetaAlg2}.

\begin{proof}[Proof of \Cref{thm:MetaAlgRCN}]
We show that \Cref{alg:MetaAlgRCN}, with high probability, returns a hypothesis $h$ with 0-1 error at most $\tau+\opt+\eps$.
Let $f^*$ be an optimal classifier and denote by $W^*$ the subspace of its weight vectors. 
Let $L_t$ be the set of vectors maintained by the algorithm (Line \ref{line:updateRCN}) and $V_t=\mathrm{span}(L_t)$, $\dim(V_t)=k_t$. {Denote by $T=2K$ the total number of iterations.}
{Also let $\eps'$ be the parameter set at Line \ref{line:initRCN} of \Cref{alg:MetaAlgRCN}, and for $t\in [T]$ let $\mathcal{S}_t$ be arbitrary $\eps'$-approximating partitions with respect to $V_t$ (see \Cref{def:approximatingPartition}).
Let $h_t:\mathbb{R}^d\to [K]$ be  piecewise constant classifiers, defined as $h_t= h_{\mathcal{S}_{t}}$ according to \Cref{def:h} for the distribution $D$.}

Note that from Lines \ref{line:initRCN}, \ref{line:loopRCN} of \Cref{alg:MetaAlgRCN}, we perform  $O(K)$ iterations. Furthermore, in each iteration, we update the vector set by adding one vector. Hence, 
$k_t=O(K)$ for all $t=1,\dots, T$.

 Assume that $\pr_{(\x, y)\sim D}[h_t(\x)\neq y]>\tau+\opt+\eps$ for all $t=1,\dots, T$.  
 Using the fact that $N=(dm)^{Cm}2^{K^C}(\Gamma \abs{\cY} /\eps)^{CK}\log(1/\delta)/(\sigma\alpha)^{C}$ for a sufficiently large universal constant $C>0$ (Line \ref{line:initRCN} of \Cref{alg:MetaAlgRCN}), we can apply \Cref{lem:improvedCorrelation} and conclude that, with probability $1-\delta$, there exist unit vectors $\vec v^{(t)}\in  V_{t+1}$ and unit vectors $\w^{(t)}\in W^*$ for $ t\in[T]$ such that $\w^{(t)}\cdot (\vec v^{(t)})^{\perp V_t}\ge 1-\eps\ge 1/2$.
 Thus, from \Cref{it:potentialDecrease}, we have that with probability $1-\delta$, for all $t\in [T]$,  $\Phi_t\le \Phi_{t-1}- 1/2$. 
 After $T$ iterations, it follows that $\Phi_T\le \Phi_{0}- T/2=K-T/2$.
 However, if $T$ is set to be $2K$ we would arrive at a contradiction, since $\Phi_T\ge 0$.
 Hence, we have that $\pr_{(\x, y)\sim D}[h_t(\x)\neq y]\le \tau+\opt+\eps$ for some $t\in \{1,\dots, T\}$. 
 Since the error of $h_t$ can only be decreasing  (see \Cref{it:ErrorDecrease}), and  $h_t$ is close to its sample variant by \Cref{cl:hConcetration}, we have that $\pr_{(\x, y)\sim D}[h(\x)\neq y]\le \tau+\opt+2\eps$.

\vspace{-0.1cm}

\item \textbf{Sample and Computational Complexity:} Note that the algorithm terminates in $O(K)$ iterations 
and at each iteration we draw $N=(dm)^{O(m)}2^{\poly(K)}(\Gamma \abs{\cY} /\eps)^{O(K)}\log(1/\delta)/(\sigma\alpha)^{O(1)}$ samples. Hence, the total sample size is  $(dm)^{O(m)}2^{\poly(K)}(\Gamma \abs{\cY} /\eps)^{O(K)}\log(1/\delta)/(\sigma\alpha)^{O(1)}$. Moreover we use at most $\poly(N)$ time as all operations can be implemented in polynomial time in $N,d$ and $\abs{\mathcal{Y}}$.
\end{proof}
\vspace{-1em}
\section{Statistical Query Lower Bounds}
\label{sec:SQ-lower-bounds}
\vspace{-0.2em}
In this section we prove \Cref{thm:simplegeneralSQ,thm-intro:SQ-RCN}. First, in \Cref{sec:generalSQ} we present SQ lower bounds for general classes of Multi-Index Models that do not satisfy \Cref{def:SimpleGoodCondition}. Essentially showing that the algorithm of \Cref{sec:generalAlgorithm} is quantitatively optimal in the SQ model. Second, in \Cref{sec:RCNSQ}, we present an SQ lower bound for the problem of learning multiclass linear classifiers under RCN. 

\medskip

{Before presenting our hardness results, we first provide a short introduction to the SQ model.

\paragraph{Basics on SQ Model} 
SQ algorithms are a broad class of algorithms
that, instead of having direct access to samples, 
are allowed to query expectations of bounded functions of the distribution.

\begin{definition}[SQ algorithms] \label{def:vstat}
Let $D$ be a distribution on $\mathcal X$. 
A statistical query is a bounded function $q: \mathcal X \to [-1,1]$. 
For $u>0$, the $\mathrm{VSTAT}(u)$ oracle responds to the query $q$
with a value $v$ such that $|v - \E_{\bx \sim D}[q(\bx)]| \leq \tau$, where 
$\tau=\max(1/u,\sqrt{\mathrm{Var}_{\bx \sim D}[q(\x)]/u})$. 
We call $\tau$ the \emph{tolerance} of the statistical query.
A \emph{Statistical Query algorithm} is an algorithm 
whose objective is to learn some information about an unknown 
distribution $D$ by making adaptive calls to the corresponding oracle.
\end{definition}

\noindent The SQ model was introduced in~\cite{Kearns:98}
as a natural restriction of the PAC model~\cite{Valiant:84}. 
Subsequently, the model has been extensively studied 
in a range of contexts, see, e.g.,~\cite{Feldman16b}.
The class of SQ algorithms is broad and captures a range of known 
supervised learning algorithms. 
More broadly, several known algorithmic techniques in machine learning
are known to be implementable using SQs 
(see, e.g.,~\cite{FGR+13, FeldmanGV17}).

Throughout this section, we leverage a number of previously known facts about the SQ model that allow us to establish hardness. We refer the reader to \Cref{app:SQ-basics} for their formal statements.
}

\vspace{-0.4em}

\subsection{SQ Lower Bounds for Learning Multi-Index Models}\label{sec:generalSQ}
We show that if a Multi-Index Model class $\mathcal F$ does not satisfy \Cref{def:GoodCondition}, any SQ algorithm that 
learns $\mathcal F$ to an accuracy comparable to {that of} \Cref{thm:MetaTheorem} requires super-polynomial complexity. {Specifically, we establish the following theorem:}

\begin{theorem}[SQ Optimality of General Algorithm]
    \label{thm:generalSQ}
    Let \( d, K, m \in \mathbb{Z}_+ \) with $d^c\geq \max(m,K)$ for a sufficiently small absolute constant $c>0$, and let  \( \mathcal{Y} \) be a set of finite cardinality. Additionally, let $\tau,\opt\in (0,1)$ with $\tau,\opt\geq d^{-O(m)}$.
    Assume the existence of (i) a function $f:\R^d\mapsto \mathcal Y$ (ii) {a random variable $(\x, y)$ over $\R^d \times \mathcal Y$ 
    with $\x$-marginal the standard Gaussian} 
    such that $\pr[f(\x)\neq y]=\opt$ (iii) a subspace $W$ of dimension at most $K$, and iv) a subspace $V$ such that:
    \begin{enumerate}[leftmargin=*, nosep]
        \item The function $f$ depends only on the projection onto $W$, i.e., $f(\x)=f(\x^W)$.
        \item For every function $g:V\to \mathcal{Y}$ we have that $\pr_{\x\sim \mathcal{N}(\vec{0}, \vec{I})}[g(\x^V)\neq f(\x)]\ge \tau$.
        \item  For any $\x_0\in \R^d, i\in [K]$, and for any zero-mean polynomial $p:U\to \R$ of degree at most $m$, it holds that $\E_{\x\sim \mathcal{N}(\vec{0}, \vec{I})}[p(\x^{U})\Ind(y=i)\mid \x^V=\x_0^V]=0$, where $U=(V+W)\cap V^{\perp}$. \label{item:zero-moments}
    \end{enumerate}
    Let $\mathcal F$ denote the class of MIMs defined as $\mathcal F=\{f(\x^{\perp U}, \vec U\vec x^{U}): \mathrm{for\ all\ } \vec U\in \R^{k\times d}: \vec U \vec U^\top=\vec I_k\}$, where $k=\dim(U)$.
Then, any SQ algorithm that learns \( \mathcal{F} \) 
{in the presence of $\opt$ adversarial label noise} 
up to an error of smaller than \( \tau -2\opt \) with probability at least $2/3$  requires either $2^{d^{\Omega(1)}}$ queries or at least one query to \( \text{VSTAT}\left(d^{\Omega(m)}/|\mathcal{Y}|\right) \).
\end{theorem}
\begin{proof}

Let $U=(V+W)\cap V^{\perp}$ with $K'=\dim(U)$ and let $U^{\perp}$ be the orthogonal complement of $U$.

Let $\mathcal Q \subseteq \R^{d\times K'}$ be the set of matrices from \Cref{fact:NearOrthogonality} for which it holds that $\abs{\mathcal{Q}}=2^{\Omega(d^{c})}, \vec Q\vec Q^{\top }=\vec I_{K'}$ and $\norm{\vec Q\vec P^{\top }}_{F}\le O(d^{2c-1+2\alpha})$ for all $\vec Q,\vec P\in \mathcal{Q}$, where $\alpha=\log K'/\log d$.  For each $\vec Q\in \mathcal Q$, we define the function $f_{\vec Q}(\x):= f(\x^{\perp U},\vec Q \x^{U})$ along with the coupled random variables $y_{\vec Q}$ satisfying $\pr[y_{\vec Q}=i | \mathbf{x}=\mathbf{z}] = \pr[y=i | \mathbf{x}=\mathbf{z}^{\perp U}+\vec Q\,\mathbf{z}^{U}]$, for all $i\in\mathcal{Y},\z\in \R^d$.
We consider the class of functions $\mathcal F_{\mathcal Q}=\{f_{\vec Q}(\x): \vec Q\in \mathcal Q \}$ and note that by definition $\mathcal F_{\mathcal Q}\subseteq \mathcal F$.

For each $\vec{Q} \in \mathcal{Q}$, let $D_{\vec{Q}}$ be the joint distribution of $(\mathbf{x}, y_{\vec{Q}})$, where $\mathbf{x} \sim \mathcal{N}(\mathbf{0}, \mathbf{I})$.
Denote  by $\mathcal{D}=\{D_{\vec Q}:\vec Q\in \mathcal Q\}$ the class of these distributions. 
Denote by $D_0$ the distribution of pairs $(\x,y_0)$, where $\x\sim \cN(\vec 0,\vec I)$ and $y_0$ a random variable supported on $\mathcal{Y}$ that depends only on $W^V$ such that $\pr_{(\x,y_0)\sim D_0}[y_0=i\mid \x^{\perp U}=\z^{\perp U}]= \pr_{\x\sim \cN(\vec 0, \vec I )}[y=i\mid \x^{\perp U}=\z^{\perp U}]$, for all $i\in \mathcal{Y},\z\in \R^d$. Denote by $y_i(\x^{\perp U})=\pr_{(\x,y)\sim D_0}[y_0=i\mid \x^{\perp U}]$.

{We show that is SQ-Hard to determinate whether we have access to distribution $D_0$ or to a distribution $D\in \mathcal{D}$.
To this end,} we show that for all $D_{\vec P}, D_{\vec Q}\in \mathcal{D}$ with $\vec P\neq\vec Q$, {the correlation}  $\chi_{D_0}(D_{\vec P}, D_{\vec Q})$ (see \Cref{def:bccorrelated}) is small. 
Using the law of total expectation, we have that
\begin{align*}
    \chi_{D_0}(D_{\vec P},D_{\vec Q})&=
    \sum_{i\in \mathcal{Y}}\E_{\x\sim \mathcal N(\vec 0,\vec I)}\left[\frac{1}{y_i(\x^{\perp U})}(\Ind(y_{\vec P}=i)-y_i(\x^{\perp U}))(\Ind(y_{\vec Q}=i)-y_i(\x^{\perp U}))\right]
    \\&=\sum_{i\in \mathcal{Y}}\E_{\x^{\perp U}\sim \mathcal N(\vec 0,\vec I)}\bigg[\frac{1}{y_i(\x^{\perp U})}\E_{\x^{U}\sim \mathcal N(\vec 0,\vec I)}[\Ind(y_{\vec P}=i)-y_i(\x^{\perp U}))(\Ind(y_{\vec Q}=i)-y_i(\x^{\perp U}))\mid \x^{\perp U}]\bigg]\;.
\end{align*}

{Considering each term separately, by \Cref{fact:Correlation} and Assumption \eqref{item:zero-moments} of the theorem statement, we have that
\begin{align*}
    \E_{\x^{U}\sim \mathcal N(\vec 0,\vec I)}[(\Ind(y_{\vec P}=i)-&y_i(\x^{\perp U}))(\Ind(y_{\vec Q}=i)-y_i(\x^{\perp U}))\mid \x^{\perp U}]\\&\leq \sum_{t=m}^{\infty} \|\vec P \vec Q^\top\|_2^t (y^{[t]})^2\leq \|\vec P \vec Q^\top\|_2^m\sum_{t=m}^{\infty} (y^{[t]})^2\;,
\end{align*}
where $y^{[t]}$ is the degree-$t$ Hermite part of the random variable $\Ind(y=i)\mid \x^{\perp U}$ and that $ \|\vec P \vec Q^\top\|_2\leq 1$. Furthermore, using Parseval's identity we get that $\sum_{t=m}^{\infty} (y^{[t]})^2\leq \E[\Ind(y=i)^2\mid \x^{\perp U}]=\E[\Ind(y=i)\mid \x^{\perp U}]$.
Therefore, we have that}
\begin{align*}
    \chi_{D_0}(D_{\vec P},D_{\vec Q})
    &\leq \|\vec P \vec Q^\top\|_2^m\sum_{i\in \mathcal{Y}}\E_{\x^{\perp U}\sim \mathcal N(\vec 0,\vec I)}\left[\frac{\E[\Ind(y=i)\mid \x^{\perp U}]}{y_i(\x^{\perp U})}\right]= \|\vec P \vec Q^\top\|_2^m|\mathcal Y|\;,
\end{align*}
where we used that $\E_{\x^{\perp U}\sim \mathcal N(\vec 0,\vec I)}\left[({\E[\Ind(y=i)\mid \x^{\perp U}]})/{y_i(\x^{\perp U})}\right]= 1$, by the definition of $y_i(\x^{\perp U})$.
Hence, $\chi_{D_0}(D_{\vec P},D_{\vec Q})= d^{-\Omega(m)}\abs{\mathcal{Y}}$. Similarly, we can show that $\chi_{D_0}(D_{\vec P},D_{\vec P})\le \abs{\mathcal{Y}}$ for all $D_{\vec P}\in \mathcal{D}$.
Therefore, we have that $\mathcal{D}$ is $(d^{-\Omega(m)}\abs{\mathcal{Y}},\abs{\mathcal{Y}})$-correlated relative to $D_0$ (see \Cref{def:bccorrelated}), and thus from \Cref{lem:SQdim} we have that any SQ algorithm that solves the decision problem $\mathcal{B}(\mathcal{D},D_0)$ with probability at least $2/3$ requires $2^{d^{\Omega(1)}}$ queries to the $\text{VSTAT}(d^{\Omega(m)}/\abs{\mathcal{Y}})$ oracle.

Hence, in order to complete the statement of the theorem, we need to show the following standard reduction from the testing problem to the learning problem (see, e.g. Chapter $8$ of \cite{diakonikolas2023algorithmic}).
\begin{claim}
    Any SQ algorithm that learns $\mathcal{F}$  with agnostic noise with error better than $\tau-2\opt$ can be used to solve the decision problem  $\mathcal{B}(\mathcal{D},D_0)$ using one more query of accuracy $\opt$.
\end{claim}
\begin{proof}
{Assume there exists an SQ algorithm $\mathcal{A}$ that, when given SQ access to any distribution $D'$ over $\R^d\times \mathcal{Y}$ with standard Gaussian $\x$-marginals, outputs a hypothesis $h$ satisfying $\pr_{(\x,y)\sim D'}[h(\x)\neq y]< \tau -2\opt$, where $\opt=\inf_{f\in \mathcal{F}}\pr_{(\x,y)\sim D'}[f(\x)\neq y]$. Using such an algorithm $\mathcal A$, we show that we can solve the testing problem.

Suppose we run $\mathcal{A}$ on a distribution $D_{\vec P}\in\mathcal{D}$. Then, by definition, $\mathcal{A}$ returns a hypothesis $h$ with error $\pr_{(\x,y)\sim D_{\vec P}}[h(\x)\neq y]< \tau -2\opt$. In contrast, when we run $\mathcal{A}$ with SQ oracle access to $D_0$ no hypothesis can achieve error lower than $\tau - \opt$, as $D_0$ is a product distribution for each fixed projection of $\x$ in $V$. Since $\tau - \opt > \tau - 2\opt$, an additional query of tolerance less than $\opt$, would distinguish between $D_{\vec P}$ and $D_0$, thus solving  $\mathcal{B}(\mathcal{D},D_0)$.}
\end{proof}
\end{proof}

\subsection{SQ Lower Bounds for Learning {Multiclass Linear Classifiers} with RCN}\label{sec:RCNSQ}
{In this section, we  present an SQ hardness result for learning  Multiclass Linear Classifiers in the presence of RCN under the Gaussian distribution.
In particular, we show that achieving $0-1$ error of $\opt+\eps$ in the SQ model must either make a super-polynomial (with respect the dimension $d$) number of queries or queries with accuracy that depends on the spectral properties of the  confusion matrix (see \Cref{def:RCNdistribution}), rather than  on the amount of noise. }
The main result of this section is \Cref{thm-intro:SQ-RCN}.
\begin{theorem}[SQ Lower Bound for Learning under MLC RCN]
    \label{thm:RCNLowerBoundK}
Any SQ algorithm that learns $K$-Multiclass Linear Classifiers under the $d$-dimensional standard normal in the presence of RCN to error $\opt+\poly(1/K)$, requires either  $2^{d^{\Omega(1)}}$ queries or at least one query to  $\text{VSTAT}\left(d^{\Omega(K)}\right)$.
\end{theorem}
\paragraph{{SQ-Hard Instances}}To construct our family of hard instances, we first define a two-dimensional classifier  $f:\mathbb{R}^2\to [K]$ as follows: 
\begin{align}\label{eq: 2Dmodel}
    \vec w^{(k)}=\begin{bmatrix}
    \cos \left(\frac{2\pi k}{K}\right)  \\
    \sin \left(\frac{2\pi k}{K}\right) \\
\end{bmatrix}, k\in [K], \text{ and } 
f(\x)=\argmax_{k\in [K]}( \vec{w}^{(k)} \cdot \vec{x})\;.
\end{align}
{Geometrically, the classifier $f$ partitions the plane into $K$ equal angular sectors, assigning each sector to a different class.}
Given a set of $2\times d $ matrices $\cQ$, we can define a class of functions from $\mathbb{R}^d$ to $[K]$ as follows
\begin{align}\label{eq:family}
    \cF_{\cQ}= \{f(\vec Q \vec x): \vec Q\in \cQ \}\;.
\end{align}
We show that it is hard to distinguish a distribution labeled by a classifier in $\cF_{\cQ}$ corrupted with RCN from a distribution with random labels independent of $\x$.
Before proving \Cref{thm:RCNLowerBoundK}, we first need to show that we can match the appropriate number of moments with the standard Gaussian.
{\begin{lemma}[{Moment matching under RCN}]\label{lem:existance}
Let $K\in\Z_+$ be such that $K\ge 8$ and $K$ is divisible by $4$.
Then there exists a $K\times K$ doubly stochastic matrix $\vec{H}$ and a distribution $D$ over $\R^2\times [K]$, whose $\x$-marginal is $\cN(\vec{0},\vec{I})$, satisfying \Cref{def:RCNdistribution} for the confusion matrix $\vec{H}$ and the function $f$ defined in \Cref{eq: 2Dmodel}, such that:
    \begin{enumerate}
        \item[i)] $\min_{i\neq j} \vec H_{i,i}- \vec H_{i,j}=\Omega(1/K^3).$
        \item[ii)] For every polynomial $p:\mathbb{R}^2\to \mathbb{R}$ of degree at most $K/2-2$, it holds that \\$\E_{\x\sim \cN^2}\left[\left(\pr_{(\x,y)\sim D}\left[y=i\mid \x\right]-{1}/{K}\right)p(\x)\right]=0$ for all $i\in [K]$\;.
    \end{enumerate}
\end{lemma}}
\begin{proof}
Let $\vec H\in \R^{K\times K}$ be a doubly stochastic matrix and  define the functions  $\tilde{f}_i$ for $i\in [K]$ as  $\tilde{f}_i(\x)=\sum_{j=1}^K\vec H_{j,i}\Ind(f(\x)=j)$.
For any function $p:\R^2\to \R$, we denote by $\bar{p}$ the zero mean centering of $p$, i.e., $\bar{p}(\x)=p(\x)-\E_{\x\sim \cN^2}\left[p(\x)\right]$, and by $\vec v^{(p)}\in \R^K$ the vector defined as $\vec v^{(p)}_i=\E_{\x\sim \cN^2}\left[ \bar{p}(\x)\Ind(f(\x)=i)\right]$. {We will refer to $\vec v^{{(p)}}$ as the vector of expectations of $p$.}
Also denote by $\boldsymbol{\omega}^{(k)}$ the root of unity vector of rate $k$, i.e. ${\boldsymbol{\omega}}^{(k)}\in \mathbb{C}^K$ such that $\boldsymbol{\omega}_{j}^{(k)}= e^{\frac{2\pi kj i}{K}}$, where $i$ is the imaginary unit.

We start by showing that a function $p$ has zero correlation with $\tilde{f}_i-1/K$ for all $i\in [K]$ under the standard Gaussian if and only if the vector of expectations of $p$ over the level sets belongs in the  kernel of $\vec{H}^\top$.

\begin{claim}\label{cl:polynomial2kerequiv}
Let $p:\R^2\to \R$. We have that $\vec v^{(p)}\in \ker(\vec{H}^{\top})$  if and only if for all $i\in [K]$ it holds  $\E_{\x\sim \cN^2}[(\tilde{f}_i(\x)-1/K)p(\x)]=0$.
\end{claim}
\begin{proof}[Proof of \Cref{cl:polynomial2kerequiv}]
By the rotational invariance of the standard Gaussian we have that $\pr_{\x\sim \cN^2}[f(\x)=i]=1/K$, for all $i\in [K]$. From the assumption that $\vec{H}$ is column stochastic, we have that $\E_{\x\sim \cN^2}[\tilde{f}_i(\x)]=\sum_{j=1}^K\vec H_{ji}\E[\Ind(f(\x)=i)]=1/K$.
Thus, for all $i\in [K]$ we have that
    \begin{align*}
    \E_{\x\sim \cN^2} \left[\left(\tilde{f}_i(\x)-\frac{1}{K}\right)p(\x)\right]&=\E_{\x\sim \cN^2} \left[\tilde{f}_i(\x)p(\x)\right]-\frac{1}{K}\E_{\x\sim \cN^2} \left[p(\x)\right]  \\
 &=\E_{\x\sim \cN^2} \left[\tilde{f}_i(\x)p(\x)\right]-\E_{\x\sim \cN^2} \left[p(\x)\right]\E_{\x\sim \cN^2} \left[\tilde{f}_i(\x)\right]\\
&= \E_{\x\sim \cN^2} \left[\tilde{f}_i(\x)\bar{p}(\x)\right]\;.
\end{align*}
Therefore, substituting the definition of $\tilde{f}_i$ we have that
    \begin{align*}
    \E_{\x\sim \cN^2} \left[\tilde{f}_i(\x)\bar{p}(\x)\right]=
         \sum_{j=1}^K \vec H_{j,i} \E_{\x\sim \cN^2}\left[ \bar{p}(\x)\Ind(f(\x)=j)\right]= (\vec{H}^\top)_i\vec v^{(p)}\;,
    \end{align*}
which concludes the proof.
\end{proof}

We next observe that all polynomials $\R^2\mapsto\R$ of degree at most $m$ can be expressed as a linear combination of polynomials of the form $(x+iy)^a(x-iy)^b$, where $a,b>0, a+b\le m$ and $i$ is the imaginary unit. Specifically, we show that the vectors associated with each polynomial of this basis are precisely the root of unity vectors. 
\begin{claim}\label{cl:fourier}
Let $p(x,y)=(x+iy)^a(x-iy)^b, a,b\in \Z_+$. If $a\neq b$, then there exists a constant $c\in \mathbb{C}$ such that  $\vec v^{(p)}=c\boldsymbol{\omega}^{(a-b)}$; and if $a=b$, then $\vec v=\vec 0$. 
\end{claim}
\begin{proof}[Proof of \Cref{cl:fourier}]
By the definition of $f$, the  region $\{(x,y)\in\R^2:f(x,y)=j\}$ is equivalent to $\theta\in \left[{2\pi j}/{K}-{\pi}/{K},{2\pi j}/{K}+{\pi}/{K}\right]$, where $\theta=\arctan\left(
{y}/{x}\right)$. 
Fix $a,b\in \Z_+, a\neq b$, and let $\theta_{1,j}={2\pi j}/{K}-{\pi}/{K}$, $\theta_{2,j}={2\pi j}/{K}+{\pi}/{K}$, $j\in [K]$. We have that
    \begin{align*}
       \vec v_j^{(p)}= \E_{(x,y)\sim \cN^2}[p(x,y)\Ind(f(x,y)=j)]&= \E_{(x,y)\sim \cN^2}[\left(x^2+y^2\right)^{\frac{a+b}{2}}e^{i\theta(a-b)} \Ind(f(x,y)=j)]\\
         &= \int_0^\infty\int_{\theta_{1,j}}^{\theta_{2,j}} r^{a+b+1}  e^{i\theta(a-b)}g(r)d\theta dr\\
        &= \left(\int_0^\infty r^{a+b+1} g(r) dr \right)\left(\int_{\theta_{1,j}}^{\theta_{2,j}}   e^{i\theta(a-b)}d\theta\right)\\
        &=c\frac{1}{i(a-b)}\left(e^{i\theta_{2,j}(a-b)}-e^{i\theta_{1,j}(a-b)}\right)\;,
        \end{align*}
where $g:\mathbb{R}\to \mathbb{R}_+$ denotes the pdf of the Gaussian distribution in polar coordinates and $c=\left(\int_0^\infty r^{a+b+1} g(r) dr \right)$.
Notice that $c$  depends only on $a+b$, and hence is the same for all coordinates of the vector $\vec v^{(p)}$. 
Also note that for $a=b$ we have that the vector $\vec v^{(p)}$ is parallel to the all ones vector, as the expectation is constant.

Therefore, it holds that $\vec v_{(j+1)\mod K}^{(p)}=\vec v_j^{(p)}e^{\frac{2\pi i}{K}(a-b)}$ for all $j\in \{0,\dots, K-1\}$,
 since  $\theta_{1,(j+1)}=\theta_{1,j}+{2\pi}/{K}$ and $\theta_{2,(j+1)}=\theta_{2,j}+{2\pi}/{K}$.
{Moreover, if $a\neq b$, we observe that the integral of $p(x,y)$ {over $\R^2$} is zero, which concludes the proof of the claim.}
\end{proof}

By \Cref{cl:fourier},  in order to have $0$ correlation with all polynomials of degree at most $m$, we need to have that $\boldsymbol{\omega}^{(k)} \in \ker(\vec{H}^\top)$ for all $k\in \{1,\dots , m,K-m,\dots, K-1\}$, as the roots of unity are symmetric, i.e., $\boldsymbol{\omega}^{(-k)}=\boldsymbol{\omega}^{(K-k)}$.

We first construct a stochastic vector $\vec h\in \mathbb{R}^K$ with $\vec  h_{0}- \vec h_{j}=\Omega(1/K^3)$, for all $ j\in \{1,\dots, K-1 \}$ that is also orthogonal to the aforementioned root of unity vectors for $m=\Omega(K)$.
Next, we will extend this vector to a matrix.
Define the vectors $\vec h^{(1)}=\boldsymbol{\omega}^{(0)}/K $, $ \vec h^{(2)}=(\boldsymbol{\omega}^{(K/2-1)}+\boldsymbol{\omega}^{(K/2+1)})/K$ and $\vec h'=2\vec h^{(1)}+\vec h^{(2)}$. Note that $\vec h^{(1)}_j=1/K$ and $\vec h_j^{(2)}=({2}/{K})(-1)^j\cos\left({2\pi j}/{K}\right)$ for all $j\in \{0,\dots,K-1\}$.
Hence, we have that all  the coordinates of $\vec h'$ are nonnegative.
Moreover, $\vec h'$ attains its maximum value at $\vec h_0'={4}/{K}$, since $\vec  h_{K/2}'=0$ (because $(-1)^j\cos\left({2\pi j}/{K}\right)=-1$ when $j=K/2$ and {$K$ is divisible by $4$}).
Furthermore, it holds  that $\max_{j\in[K-1]} \vec h_j'\le 2/K+\max_{j\in[K-1]} \abs{ \vec h_j^{(2)}}=2/K+2/K\cos\left({2\pi}/{K}\right) $, and hence we have
\begin{align*}
    \vec h'_0-\max_{j\in[K-1]} \vec h_j'\ge \frac{2}{K}\left(1-\cos\left(\frac{2\pi}{K}\right)\right)\ge \frac{2}{K}\left(\left(\frac{2\pi}{K}\right)^2/2-\left(\frac{2\pi}{K}\right)^4/24\right)\gtrsim \frac{1}{K^3}\;,
\end{align*}
where the second inequality follows from the Taylor series of  $\cos(x) = 1 - \frac{x^2}{2!} + \frac{x^4}{4!} - \frac{x^6}{6!} + \cdots$, which is an alternating decreasing series when $x={2\pi}/{K}\le 1$ (thus it can be upper bounded by the first three terms).
Let $\vec h$ be  the normalization of $\vec h'$ such that the sum of its coordinates equals $1$.
Given that each element of $\vec h'$ is at most $4/K$, we have that the normalization factor is at most $4$. Hence, $\vec h_0-\max_{j\in[K-1]} \vec h_j\gtrsim {1}/{K^3}$.
Moreover, from the well-known fact that $\boldsymbol{\omega}^{(0)},\dots,\boldsymbol{\omega}^{(K-1)}$ form an orthogonal basis of $\mathbb{C}^K$, we have that 
$\vec h$ is orthogonal to $\boldsymbol{\omega}^{(k)}$ for all $k\in \{1,\dots,K/2-2, K/2, K/2+2,\dots, K-1\}$, since it is a linear combination of $\boldsymbol{\omega}^{(0)}, \boldsymbol{\omega}^{(K/2-1)}, \boldsymbol{\omega}^{(K/2+1)}$.

Now consider $\vec{H}$, the circulant matrix generated by row shifts of the vector $h$. 
By definition, we have that $\vec{H}$ is a row stochastic matrix and that $\min_{i\neq j} \vec H_{ii}- \vec H_{i,j}=\Omega(1/K^3)$. 
Moreover, $\vec{H}$ is symmetric, because it holds that $\vec H_{i,j}=\vec h_{j-i\mod K}=\vec h_{i-j\mod K}=\vec H_{j,i}$ and $\vec h_j= \vec h_{K-j}, j\in [K-1]$. {This equality is justified by the identity $(-1)^j\cos\left({2\pi j}/{K}\right)=(-1)^{K-j}\cos\left({2\pi(K-j)}/{K}\right)$, which holds since $K$ is an even integer.} Hence $\vec{H}$ is also column stochastic.

By \Cref{fact:circulant}, the eigenvectors  of any circulant matrix are exactly $\boldsymbol{\omega}^{(k)}$ with eigenvalues $ \vec h\cdot \boldsymbol{\omega}^{(k)}, k\in \{0,\dots, K-1\}$.
Hence,  $\boldsymbol{\omega}^{(k)} \in \ker(\vec{H})=\ker(\vec{H}^\top)$ for all $k\in \{1,\dots,K/2-2, K/2, K/2+2,\dots, K-1\}$. 
Therefore, by applying \Cref{cl:polynomial2kerequiv}, we conclude the proof of \Cref{lem:existance}. 
\end{proof}
\subsubsection{Proof of \Cref{thm:RCNLowerBoundK}}

\begin{proof}[Proof of \Cref{thm:RCNLowerBoundK}]
Consider the case where $K$ is divisible by $4$ and let  $\vec H$ be a matrix that satisfies the statement of \Cref{lem:existance}
Let $f:\R^2\to [K]$ be the linear classifier defined in \Cref{eq: 2Dmodel}, and let $\mathcal{F}_{\mathcal{Q}}$ be the class of functions defined in \Cref{eq:family}, where $\mathcal{Q}\subseteq \R^{d\times 2}$ is the family that satisfies the statement of \Cref{fact:NearOrthogonality}.
We define the family of distributions $\mathcal{D}$ such that $D_{\vec Q}\in \mathcal{D}$ if and only if $D_{\vec Q}$ is a joint distribution $(\x,y_{\vec Q})$, where $\x \sim \cN(\vec 0,\vec I)$ and $y_{\vec Q}$ is a random variable supported on $[K]$ defined as $\pr[y_{\vec Q}=i\mid \x]=\vec H_{f(\vec Q\x), i}$ for some $\vec Q\in \vec \mathcal{Q}$. Note that all distributions in $\mathcal{D}$ satisfy \Cref{def:RCNdistribution} for confusion matrix $\vec H$ and some function in $\mathcal{F}_{\cQ}$.
Denote by $D$ the product distribution with $\x$-marginals distributed as standard normal and the labels $y$ are distributed uniformly on $[K]$.
For $D_{\vec Q}, D_{\vec P}\in \mathcal{D}$ with $\vec Q\neq \vec P$, we have that 
\begin{align*}
    \chi_{D}(D_{\vec Q}, D_{\vec P})&= K\sum_{i=1}^K\E_{\x\sim \cN^d}\left[(\pr[y_{\vec Q}=i\mid \x]-\frac{1}{K})(\pr[y_{\vec P}=i\mid \x]-\frac{1}{K})\right]\\
    &= K\sum_{i=1}^K\E_{\x\sim \cN^d}\left[(\vec H_{f(\vec Q\x),i}-\frac{1}{K})(\vec H_{f(\vec P\x),i}-\frac{1}{K})\right]\\
    &= K\sum_{i=1}^K\sum_{j=1}^{\infty}\norm{\vec P\vec Q^{\top}}_2^j \E_{\x\sim \cN^2}\left[((\vec H_{f(\x),i}-\frac{1}{K})^{[j]})^2\right]\\
    &= \norm{\vec P\vec Q^{\top}}_2^{K/2-2}K\sum_{i=1}^K \E_{\x\sim \cN^2}\left[(\vec H_{f(\x),i} -\frac{1}{K})^2\right]\le  d^{\Omega(K)(c-1/2)}K\;,
\end{align*}
where we used \Cref{fact:Correlation}, Parseval's identity and the fact that $\vec H$ satisfies the statement of \Cref{lem:existance}. Similarly, we have that $\chi_{D}(D_{\vec Q},D_{\vec Q})\le K$ for all $D_{\vec Q}\in \mathcal{D}$.

Consequently, we have that there exists a family of distributions of size $2^{\Omega(d^c)}$ that is $(d^{\Omega(K)(c-1/2)}K, K)$ correlated. Hence, by \Cref{lem:SQdim} we have that any SQ algorithm that solves the decision problem $\cB(\cD, D)$ requires $2^{\Omega(d^c)}d^{\Omega(K)(c-1/2)}=2^{d^{\Omega(1)}}$ (where we used the assumption that $d/\log(d)$ is much larger that $K$) queries to the  $\text{VSTAT}(d^{\Omega(K)}$.

In order to complete the proof of the theorem, we need to show the following {standard} reduction from the testing problem to the learning problem.
\begin{claim}
Let $\gamma=\min_{i,j}\vec H_{i,i}-\vec H_{i,j}$.
    Any SQ algorithm that learns multiclass linear classifiers under the standard normal in the presence of RCN with error better than $\opt +\gamma/2$ can be used to solve the decision problem  $\mathcal{B}(\mathcal{D},D)$ using one more query of accuracy $\gamma/4$.
\end{claim}
\begin{proof}
Assume that there exists an SQ algorithm $\mathcal{A}$ that if given access to a distribution $D'$ over $\R^d\times \mathcal{Y}$ whose $\x$-marginal is the standard normal returns a hypothesis $h$ such that $\pr_{(\x,y)\sim D'}[h(\x)\neq y]\le \opt +\gamma/4$. Using $\mathcal A$, we can solve the testing problem.

Note that by running  $\mathcal{A}$ to distribution  $D_{\vec P}\in \mathcal{D}$, $\mathcal A$ returns a hypothesis $h$ such that $\pr_{(\x,y)\sim D_{\vec P}}[h(\x)\neq y]\le \opt +\gamma/4$. Moreover, by running  $\mathcal{A}$ with SQ oracle access to $D$ results in a hypothesis $h$ that cannot perform better than $1-1/K,$ as $D$ is a product distribution with the uniform distribution over the labels. Since $\opt<1 -1/K-\gamma/2$, because $\vec H_{i,i}\ge 1/K+\gamma/2$ as $\sum_{j} \vec H_{i,j}=1, i\in [K]$, we have that with an additional query of tolerance less than $\gamma/4$, we can solve  $\mathcal{B}(\mathcal{D},D)$.
\end{proof}
For $K$ not divisible by $4$, we add $r = 4 \mod K$ additional classes whose weight vectors are identical to $\vec{w}^{(1)}$, and with no noise, i.e.,
for these added classes $j$, we have that $\vec H_{j,j} = 1$ and $\vec H_{j,i} = \vec H_{i,j} = 0$ for all $j \neq i$. 
For the remaining classes, i.e., for $i,j \in [K-r]$, we consider a matrix $\vec H_{i,j}$ that satisfies the conditions of \Cref{lem:existance}. 
Note that since these additional classes have zero probability (as ties are broken by the smallest index), our {SQ hardness result} extends to the general case of any $K$. This completes the proof.
\end{proof}

\bibliographystyle{alphaabbr}

\bibliography{allrefs}

\newcommand{\etalchar}[1]{$^{#1}$}
\begin{thebibliography}{PRKM{\etalchar{+}}17}

\bibitem[AAM22]{abbe2022merged}
E.~Abbe, E.~B. Adsera, and T.~Misiakiewicz.
\newblock The merged-staircase property: a necessary and nearly sufficient
  condition for sgd learning of sparse functions on two-layer neural networks.
\newblock In {\em Conference on Learning Theory}, pages 4782--4887. PMLR, 2022.

\bibitem[AAM23]{abbe2023sgd}
E.~Abbe, E.~B. Adsera, and T.~Misiakiewicz.
\newblock {SGD} learning on neural networks: leap complexity and
  saddle-to-saddle dynamics.
\newblock In {\em The Thirty Sixth Annual Conference on Learning Theory}, pages
  2552--2623. PMLR, 2023.

\bibitem[ABAB{\etalchar{+}}21]{abbe2021staircase}
E.~Abbe, E.~Boix-Adsera, M.~S. Brennan, G.~Bresler, and D.~Nagaraj.
\newblock The staircase property: How hierarchical structure can guide deep
  learning.
\newblock {\em Advances in Neural Information Processing Systems},
  34:26989--27002, 2021.

\bibitem[ABL17]{ABL17}
P.~Awasthi, M.~F. Balcan, and P.~M. Long.
\newblock The power of localization for efficiently learning linear separators
  with noise.
\newblock {\em J. {ACM}}, 63(6):50:1--50:27, 2017.

\bibitem[AGJ21]{arous2021online}
G.~B. Arous, R.~Gheissari, and A.~Jagannath.
\newblock Online stochastic gradient descent on non-convex losses from
  high-dimensional inference.
\newblock {\em Journal of Machine Learning Research}, 22(106):1--51, 2021.

\bibitem[AL88]{AL88}
D.~Angluin and P.~Laird.
\newblock Learning from noisy examples.
\newblock {\em Mach. Learn.}, 2(4):343--370, 1988.

\bibitem[AV99]{ArriagaVempala:99}
R.~Arriaga and S.~Vempala.
\newblock An algorithmic theory of learning: Robust concepts and random
  projection.
\newblock In {\em Proceedings of the 40th Annual Symposium on Foundations of
  Computer Science (FOCS)}, pages 616--623, 1999.

\bibitem[BBSS22]{BBSS22}
A.~Bietti, J.~Bruna, C.~Sanford, and M.~J. Song.
\newblock Learning single-index models with shallow neural networks.
\newblock {\em Advances in Neural Information Processing Systems},
  35:9768--9783, 2022.

\bibitem[BJW19]{BJW18}
A.~Bakshi, R.~Jayaram, and D.~P. Woodruff.
\newblock Learning two layer rectified neural networks in polynomial time.
\newblock In {\em Conference on Learning Theory}, pages 195--268. PMLR, 2019.

\bibitem[Bog98]{Bog:98}
V.~Bogachev.
\newblock {\em Gaussian measures}.
\newblock Mathematical surveys and monographs, vol. 62, 1998.

\bibitem[BPS{\etalchar{+}}19]{BPSBDC19}
A.~Beygelzimer, D.~Pal, B.~Szorenyi, D.~Thiruvenkatachari, C.-Y. Wei, and
  C.~Zhang.
\newblock Bandit multiclass linear classification: Efficient algorithms for the
  separable case.
\newblock In {\em International Conference on Machine Learning}, pages
  624--633. PMLR, 2019.

\bibitem[CCK17]{chernozhukov2017detailed}
V.~Chernozhukov, D.~Chetverikov, and K.~Kato.
\newblock Detailed proof of nazarov's inequality.
\newblock {\em arXiv preprint arXiv:1711.10696}, 2017.

\bibitem[CDG{\etalchar{+}}23]{CDGJM23}
S.~Chen, Z.~Dou, S.~Goel, A.~Klivans, and R.~Meka.
\newblock Learning narrow one-hidden-layer relu networks.
\newblock In {\em The Thirty Sixth Annual Conference on Learning Theory}, pages
  5580--5614. PMLR, 2023.

\bibitem[CST11]{cour2011partiallabel}
T.~Cour, B.~Sapp, and B.~Taskar.
\newblock Learning from partial labels.
\newblock {\em The Journal of Machine Learning Research}, 12:1501--1536, 2011.

\bibitem[DGT19]{DGT19}
I.~Diakonikolas, T.~Gouleakis, and C.~Tzamos.
\newblock Distribution-independent {PAC} learning of halfspaces with massart
  noise.
\newblock In H.~M. Wallach, H.~Larochelle, A.~Beygelzimer,
  F.~d'Alch{\'{e}}{-}Buc, E.~B. Fox, and R.~Garnett, editors, {\em Advances in
  Neural Information Processing Systems 32: Annual Conference on Neural
  Information Processing Systems 2019, NeurIPS 2019}, pages 4751--4762, 2019.

\bibitem[DH18]{DH18}
R.~Dudeja and D.~Hsu.
\newblock Learning single-index models in gaussian space.
\newblock In {\em Conference on Learning Theory}, pages 1887--1930. PMLR, 2018.

\bibitem[DK23]{diakonikolas2023algorithmic}
I.~Diakonikolas and D.~M. Kane.
\newblock {\em Algorithmic high-dimensional robust statistics}.
\newblock Cambridge university press, 2023.

\bibitem[DKK{\etalchar{+}}21a]{DKKTZ21}
I.~Diakonikolas, D.~M. Kane, V.~Kontonis, C.~Tzamos, and N.~Zarifis.
\newblock Agnostic proper learning of halfspaces under gaussian marginals.
\newblock In {\em Conference on Learning Theory}, pages 1522--1551. PMLR, 2021.

\bibitem[DKK{\etalchar{+}}21b]{DiakonikolasKKT21}
I.~Diakonikolas, D.~M. Kane, V.~Kontonis, C.~Tzamos, and N.~Zarifis.
\newblock Efficiently learning halfspaces with {T}sybakov noise.
\newblock In S.~Khuller and V.~V. Williams, editors, {\em {STOC} '21: 53rd
  Annual {ACM} {SIGACT} Symposium on Theory of Computing, 2021}, pages 88--101.
  {ACM}, 2021.

\bibitem[DKK{\etalchar{+}}22]{DKKTZ22}
I.~Diakonikolas, D.~M. Kane, V.~Kontonis, C.~Tzamos, and N.~Zarifis.
\newblock Learning general halfspaces with general massart noise under the
  gaussian distribution.
\newblock In {\em {STOC} '22: 54th Annual {ACM} {SIGACT} Symposium on Theory of
  Computing}, pages 874--885. {ACM}, 2022.

\bibitem[DKK{\etalchar{+}}23]{DKKTZfocs24}
I.~Diakonikolas, D.~M. Kane, V.~Kontonis, C.~Tzamos, and N.~Zarifis.
\newblock Agnostically learning multi-index models with queries.
\newblock {\em CoRR}, abs/2312.16616, 2023.
\newblock Conference version in Proceedings of FOCS'24.

\bibitem[DKKZ20]{DKKZ20}
I.~Diakonikolas, D.~M. Kane, V.~Kontonis, and N.~Zarifis.
\newblock Algorithms and sq lower bounds for pac learning one-hidden-layer relu
  networks.
\newblock In {\em Conference on Learning Theory}, pages 1514--1539. PMLR, 2020.

\bibitem[DKMR22]{DKMR22b}
I.~Diakonikolas, D.~Kane, P.~Manurangsi, and L.~Ren.
\newblock Cryptographic hardness of learning halfspaces with massart noise.
\newblock {\em Advances in Neural Information Processing Systems},
  35:3624--3636, 2022.

\bibitem[DKPZ21]{DiakonikolasKPZ21}
I.~Diakonikolas, D.~M. Kane, T.~Pittas, and N.~Zarifis.
\newblock The optimality of polynomial regression for agnostic learning under
  gaussian marginals in the {SQ} model.
\newblock In {\em Conference on Learning Theory, {COLT} 2021}, volume 134 of
  {\em Proceedings of Machine Learning Research}, pages 1552--1584. {PMLR},
  2021.

\bibitem[DKR23]{DKR23}
I.~Diakonikolas, D.~M. Kane, and L.~Ren.
\newblock Near-optimal cryptographic hardness of agnostically learning
  halfspaces and relu regression under gaussian marginals.
\newblock In {\em ICML}, 2023.

\bibitem[DKS17]{DKS17-sq}
I.~Diakonikolas, D.~M. Kane, and A.~Stewart.
\newblock Statistical query lower bounds for robust estimation of
  high-dimensional gaussians and gaussian mixtures.
\newblock In {\em 58th {IEEE} Annual Symposium on Foundations of Computer
  Science, {FOCS} 2017}, pages 73--84, 2017.
\newblock Full version at http://arxiv.org/abs/1611.03473.

\bibitem[DKS18]{DKS18-nasty}
I.~Diakonikolas, D.~M. Kane, and A.~Stewart.
\newblock Learning geometric concepts with nasty noise.
\newblock In {\em Proceedings of the 50th Annual {ACM} {SIGACT} Symposium on
  Theory of Computing, {STOC} 2018}, pages 1061--1073, 2018.

\bibitem[DKTZ20a]{DKTZ20}
I.~Diakonikolas, V.~Kontonis, C.~Tzamos, and N.~Zarifis.
\newblock Learning halfspaces with massart noise under structured
  distributions.
\newblock In J.~D. Abernethy and S.~Agarwal, editors, {\em Conference on
  Learning Theory, {COLT} 2020}, volume 125 of {\em Proceedings of Machine
  Learning Research}, pages 1486--1513. {PMLR}, 2020.

\bibitem[DKTZ20b]{DKTZ20c}
I.~Diakonikolas, V.~Kontonis, C.~Tzamos, and N.~Zarifis.
\newblock Non-convex {SGD} learns halfspaces with adversarial label noise.
\newblock In {\em Advances in Neural Information Processing Systems,
  {NeurIPS}}, 2020.

\bibitem[DKZ20]{DKZ20}
I.~Diakonikolas, D.~Kane, and N.~Zarifis.
\newblock Near-optimal {SQ} lower bounds for agnostically learning halfspaces
  and relus under gaussian marginals.
\newblock In H.~Larochelle, M.~Ranzato, R.~Hadsell, M.~Balcan, and H.~Lin,
  editors, {\em Advances in Neural Information Processing Systems 33: Annual
  Conference on Neural Information Processing Systems 2020, NeurIPS 2020},
  2020.

\bibitem[DLS22]{DLS22}
A.~Damian, J.~Lee, and M.~Soltanolkotabi.
\newblock Neural networks can learn representations with gradient descent.
\newblock In {\em Conference on Learning Theory}, pages 5413--5452. PMLR, 2022.

\bibitem[DMRT25]{DMRT25}
I.~Diakonikolas, M.~Ma, L.~Ren, and C.~Tzamos.
\newblock Statistical query hardness of multiclass linear classification with
  random classification noise.
\newblock {\em CoRR}, abs/2502.11413, 2025.

\bibitem[DPVLB24]{damian2024generativeexponent}
A.~Damian, L.~Pillaud-Vivien, J.~Lee, and J.~Bruna.
\newblock Computational-statistical gaps in gaussian single-index models.
\newblock In {\em The Thirty Seventh Annual Conference on Learning Theory},
  pages 1262--1262. PMLR, 2024.

\bibitem[Fel16]{Feldman16b}
V.~Feldman.
\newblock Statistical query learning.
\newblock In {\em Encyclopedia of Algorithms}, pages 2090--2095. 2016.

\bibitem[FGR{\etalchar{+}}13]{FGR+13}
V.~Feldman, E.~Grigorescu, L.~Reyzin, S.~Vempala, and Y.~Xiao.
\newblock Statistical algorithms and a lower bound for detecting planted
  cliques.
\newblock In {\em Proceedings of STOC'13}, pages 655--664, 2013.
\newblock Full version in Journal of the ACM, 2017.

\bibitem[FGV17]{FeldmanGV17}
V.~Feldman, C.~Guzman, and S.~S. Vempala.
\newblock Statistical query algorithms for mean vector estimation and
  stochastic convex optimization.
\newblock In P.~N. Klein, editor, {\em Proceedings of the Twenty-Eighth Annual
  {ACM-SIAM} Symposium on Discrete Algorithms, {SODA} 2017}, pages 1265--1277.
  {SIAM}, 2017.

\bibitem[FJS81]{Friedman:1980tu}
J.~H. Friedman, M.~Jacobson, and W.~Stuetzle.
\newblock {Projection Pursuit Regression}.
\newblock {\em J. Am. Statist. Assoc.}, 76:817, 1981.

\bibitem[FKKT21]{Fotakis2021coarselabellearning}
D.~Fotakis, A.~Kalavasis, V.~Kontonis, and C.~Tzamos.
\newblock Efficient algorithms for learning from coarse labels.
\newblock In {\em Conference on Learning Theory}, pages 2060--2079. PMLR, 2021.

\bibitem[GKLW19]{GeKLW19}
R.~Ge, R.~Kuditipudi, Z.~Li, and X.~Wang.
\newblock Learning two-layer neural networks with symmetric inputs.
\newblock In {\em 7th International Conference on Learning Representations,
  {ICLR} 2019}, 2019.

\bibitem[GLM18]{GLM18}
R.~Ge, J.~D. Lee, and T.~Ma.
\newblock Learning one-hidden-layer neural networks with landscape design.
\newblock In {\em 6th International Conference on Learning Representations,
  {ICLR} 2018, Vancouver, BC, Canada, April 30 - May 3, 2018, Conference Track
  Proceedings}. OpenReview.net, 2018.

\bibitem[GWBL20]{GWBL20}
S.~Garg, Y.~Wu, S.~Balakrishnan, and Z.~Lipton.
\newblock A unified view of label shift estimation.
\newblock In {\em Advances in Neural Information Processing Systems},
  volume~33, pages 3290--3300. Curran Associates, Inc., 2020.

\bibitem[HL93]{HL93}
P.~Hall and K.-C. Li.
\newblock {On almost Linearity of Low Dimensional Projections from High
  Dimensional Data}.
\newblock {\em The Annals of Statistics}, 21(2):867 -- 889, 1993.

\bibitem[HSSV22]{HsuSSV22}
D.~J. Hsu, C.~H. Sanford, R.~A. Servedio, and E.~Vlatakis{-}Gkaragkounis.
\newblock Near-optimal statistical query lower bounds for agnostically learning
  intersections of halfspaces with gaussian marginals.
\newblock In P.~Loh and M.~Raginsky, editors, {\em Conference on Learning
  Theory, 2-5 July 2022, London, {UK}}, volume 178 of {\em Proceedings of
  Machine Learning Research}, pages 283--312. {PMLR}, 2022.

\bibitem[Hub85]{Huber85-pp}
P.~J. Huber.
\newblock {Projection Pursuit}.
\newblock {\em The Annals of Statistics}, 13(2):435 -- 475, 1985.

\bibitem[INHS17]{Ishida2017Complementarylabels}
T.~Ishida, G.~Niu, W.~Hu, and M.~Sugiyama.
\newblock Learning from complementary labels.
\newblock In I.~Guyon, U.~V. Luxburg, S.~Bengio, H.~Wallach, R.~Fergus,
  S.~Vishwanathan, and R.~Garnett, editors, {\em Advances in Neural Information
  Processing Systems}, volume~30. Curran Associates, Inc., 2017.

\bibitem[JSA15]{JSA15}
M.~Janzamin, H.~Sedghi, and A.~Anandkumar.
\newblock Beating the perils of non-convexity: Guaranteed training of neural
  networks using tensor methods.
\newblock {\em arXiv preprint arXiv:1506.08473}, 2015.

\bibitem[Kea98]{Kearns:98}
M.~J. Kearns.
\newblock Efficient noise-tolerant learning from statistical queries.
\newblock {\em Journal of the ACM}, 45(6):983--1006, 1998.

\bibitem[KKMS05]{KKMS:05}
A.~Kalai, A.~Klivans, Y.~Mansour, and R.~Servedio.
\newblock Agnostically learning halfspaces.
\newblock In {\em Proceedings of the 46th IEEE Symposium on Foundations of
  Computer Science (FOCS)}, pages 11--20, 2005.

\bibitem[KLT09]{KlivansLT09}
A.~R. Klivans, P.~M. Long, and A.~K. Tang.
\newblock Baum's algorithm learns intersections of halfspaces with respect to
  log-concave distributions.
\newblock In {\em 13th International Workshop, {RANDOM} 2009}, pages 588--600,
  2009.

\bibitem[KOS04]{Kos:04}
A.~Klivans, R.~O'Donnell, and R.~Servedio.
\newblock Learning intersections and thresholds of halfspaces.
\newblock {\em Journal of Computer \& System Sciences}, 68(4):808--840, 2004.

\bibitem[KOS08]{KOS:08}
A.~Klivans, R.~O'Donnell, and R.~Servedio.
\newblock Learning geometric concepts via {G}aussian surface area.
\newblock In {\em Proc.\ 49th IEEE Symposium on Foundations of Computer Science
  (FOCS)}, pages 541--550, 2008.

\bibitem[KSST08]{banditron}
S.~M. Kakade, S.~Shalev-Shwartz, and A.~Tewari.
\newblock Efficient bandit algorithms for online multiclass prediction.
\newblock In {\em Proceedings of the 25th International Conference on Machine
  Learning}, ICML '08, page 440–447, New York, NY, USA, 2008. Association for
  Computing Machinery.

\bibitem[KTZ19]{KTZ19}
V.~Kontonis, C.~Tzamos, and M.~Zampetakis.
\newblock { Efficient Truncated Statistics with Unknown Truncation }.
\newblock In {\em 2019 IEEE 60th Annual Symposium on Foundations of Computer
  Science (FOCS)}, pages 1578--1595, Los Alamitos, CA, USA, November 2019. IEEE
  Computer Society.

\bibitem[LD14]{Liu2014supersetlearning}
L.~Liu and T.~Dietterich.
\newblock Learnability of the superset label learning problem.
\newblock In E.~P. Xing and T.~Jebara, editors, {\em Proceedings of the 31st
  International Conference on Machine Learning}, volume~32 of {\em Proceedings
  of Machine Learning Research}, pages 1629--1637, Bejing, China, 22--24 Jun
  2014. PMLR.

\bibitem[Led94]{Ledoux:94}
M.~Ledoux.
\newblock Semigroup proofs of the isoperimetric inequality in {E}uclidean and
  {G}auss space.
\newblock {\em Bull. Sci. Math.}, 118:485--510, 1994.

\bibitem[Li91]{Li91}
K.-C. Li.
\newblock Sliced inverse regression for dimension reduction.
\newblock {\em Journal of the American Statistical Association},
  86(414):316--327, 1991.

\bibitem[LWS18]{Lipton18a}
Z.~Lipton, Y.-X. Wang, and A.~Smola.
\newblock Detecting and correcting for label shift with black box predictors.
\newblock In {\em Proceedings of the 35th International Conference on Machine
  Learning}, volume~80 of {\em Proceedings of Machine Learning Research}, pages
  3122--3130. PMLR, 10--15 Jul 2018.

\bibitem[NDRT13]{NDRT13}
N.~Natarajan, I.~S. Dhillon, P.~K. Ravikumar, and A.~Tewari.
\newblock Learning with noisy labels.
\newblock In C.~Burges, L.~Bottou, M.~Welling, Z.~Ghahramani, and
  K.~Weinberger, editors, {\em Advances in Neural Information Processing
  Systems}, volume~26. Curran Associates, Inc., 2013.

\bibitem[Pis86]{Pisier:86}
G.~Pisier.
\newblock Probabilistic methods in the geometry of {B}anach spaces.
\newblock In {\em Lecture notes in Math.}, pages 167--241. Springer, 1986.

\bibitem[PRKM{\etalchar{+}}17]{PRMNQ17}
G.~Patrini, A.~Rozza, A.~Krishna~Menon, R.~Nock, and L.~Qu.
\newblock Making deep neural networks robust to label noise: A loss correction
  approach.
\newblock In {\em Proceedings of the IEEE conference on computer vision and
  pattern recognition}, pages 1944--1952, 2017.

\bibitem[SB14]{SB-book}
S.~Shalev{-}Shwartz and S.~Ben{-}David.
\newblock {\em Understanding Machine Learning - From Theory to Algorithms}.
\newblock Cambridge University Press, 2014.

\bibitem[SKP{\etalchar{+}}20]{Song-etal-survey20}
H.~Song, M.~Kim, D.~Park, Y.~Shin, and J.-G. Lee.
\newblock Learning from noisy labels with deep neural networks: A survey.
\newblock {\em IEEE Transactions on Neural Networks and Learning Systems},
  pages 1--29, 2020.

\bibitem[Sze67]{Sze67}
G.~Szeg{\"o}.
\newblock {\em Orthogonal Polynomials}.
\newblock Number $\tau$. 23 in American Mathematical Society colloquium
  publications. American Mathematical Society, 1967.

\bibitem[Tie23]{Tieg23}
S.~Tiegel.
\newblock Hardness of agnostically learning halfspaces from worst-case lattice
  problems.
\newblock In {\em The Thirty Sixth Annual Conference on Learning Theory}, pages
  3029--3064. PMLR, 2023.

\bibitem[Val84]{Valiant:84}
L.~Valiant.
\newblock A theory of the learnable.
\newblock {\em Communications of the ACM}, 27(11):1134--1142, 1984.

\bibitem[Vem10a]{Vempala10a}
S.~Vempala.
\newblock Learning convex concepts from gaussian distributions with {PCA}.
\newblock In {\em 51th Annual {IEEE} Symposium on Foundations of Computer
  Science, {FOCS}}, pages 124--130, 2010.

\bibitem[Vem10b]{Vempala10}
S.~Vempala.
\newblock A random-sampling-based algorithm for learning intersections of
  halfspaces.
\newblock {\em J. {ACM}}, 57(6):32:1--32:14, 2010.

\bibitem[VRW17]{VanRooyencorrupted}
B.~Van~Rooyen and R.~C. Williamson.
\newblock A theory of learning with corrupted labels.
\newblock {\em J. Mach. Learn. Res.}, 18(1):8501–8550, jan 2017.

\bibitem[WLT18]{WLT18}
R.~Wang, T.~Liu, and D.~Tao.
\newblock Multiclass learning with partially corrupted labels.
\newblock {\em IEEE Transactions on Neural Networks and Learning Systems},
  29(6):2568--2580, 2018.

\bibitem[Xia08]{Xia08}
Y.~Xia.
\newblock A multiple-index model and dimension reduction.
\newblock {\em Journal of the American Statistical Association},
  103(484):1631--1640, 2008.

\bibitem[XTLZ02]{xia2002adaptive}
Y.~Xia, H.~Tong, W.~K. Li, and L.~Zhu.
\newblock An adaptive estimation of dimension reduction space.
\newblock {\em Journal of the Royal Statistical Society Series B: Statistical
  Methodology}, 64(3):363--410, 2002.

\bibitem[ZLA21]{zhang21k}
M.~Zhang, J.~Lee, and S.~Agarwal.
\newblock Learning from noisy labels with no change to the training process.
\newblock In M.~Meila and T.~Zhang, editors, {\em Proceedings of the 38th
  International Conference on Machine Learning}, volume 139 of {\em Proceedings
  of Machine Learning Research}, pages 12468--12478. PMLR, 18--24 Jul 2021.

\bibitem[ZLF{\etalchar{+}}20]{Zang2020Complementary}
Y.~Zhang, F.~Liu, Z.~Fang, B.~Yuan, G.~Zhang, and J.~Lu.
\newblock Learning from a complementary-label source domain: Theory and
  algorithms.
\newblock {\em CoRR}, abs/2008.01454, 2020.

\end{thebibliography}

\newpage

\appendix

\section*{Appendix}

\section{Additional Preliminaries}
\label{sec:Addprelims}
\paragraph{Gaussian Space}
We denote by $L^2(\normal)$ the vector space of all functions $f:\R^d
\to \R$ such that $\E_{\vec x \sim \cN(\vec 0,\vec I)}[f^2(x)] < \infty$.
We define the standard $L^p$ norms with respect to the Gaussian measure, i.e., $\|g\|_{L^p} = ( \E_{\x \sim \cN(\vec 0,\vec I)} [ |g(\x)|^p)^{1/p}$.
The usual
inner product for this space is
$\E_{\vec x \sim \cN(\vec 0,\vec I)}[f(\vec x) g(\vec x)]$.
While usually one considers the probabilists' or physicists' Hermite polynomials,
in this work we define the \emph{normalized} Hermite polynomial of degree $i$ to be
\(
H_0(x) = 1, H_1(x) = x, H_2(x) = \frac{x^2 - 1}{\sqrt{2}},\ldots,
H_i(x) = \frac{He_i(x)}{\sqrt{i!}}, \ldots
\)
where by $He_i(x)$ we denote the probabilists' Hermite polynomial of degree $i$.
These normalized Hermite polynomials form a complete orthonormal basis for the
single-dimensional version of the inner product space defined above. To get an
orthonormal basis for $L^2(\normal)$, we use a multi-index $V\in \N^d$
to define the $d$-variate normalized Hermite polynomial as
$H_V(\vec x) = \prod_{i=1}^d H_{v_i}(x_i)$.  
The total degree of $H_V$ is
$|V| = \sum_{v_i \in V} v_i$.

\begin{fact}[Gaussian Density Properties]
\label{fact:gaussianfacts}
    Let $\mathcal{N}$ be the standard one-dimensional normal distribution. Then, the following properties hold:
\begin{enumerate}
    \item For any $t > 0$, it holds $e^{-t^2/2}/4 \leq \pr_{z \sim \mathcal{N}}[z > t] \leq e^{-t^2/2}/2$.
    \item For any $a, b \in \mathbb{R}$ with $a \leq b$, it holds $\pr_{z \sim \mathcal{N}}[a \leq z \leq b] \leq (b-a)/\sqrt{2\pi}$.
\end{enumerate}
\end{fact}

\begin{fact}[see, e.g., Lemma 6 in \cite{KTZ19}]\label{fact:gradientNorm}
 Let  $f \in L^2(\mathbb{R}^d, \mathcal{N}(0,I))$ with its $k$-degree Hermite expansion $f(\x)= \sum_{\alpha\in \N^d,\|\alpha\|_1\leq k}\widehat{f}(\alpha)H_{\alpha}(\x)$. It holds that $\E_{\x \sim \mathcal{N}(\vec 0,\vec I)} \left[ (\nabla f (\x) \cdot \vec e_i)^2 \right] = \sum_{\alpha \in \mathbb{N}^d
 \|\alpha\|_1\leq k} \alpha_i (\widehat{f}(\alpha))^2$.

\end{fact}

\begin{definition}[Ornstein-Uhlenbeck Noise Operator]\label{def:gaussian-noise}
Let \( k \in \mathbb{N} \) and \( \rho \in [0, 1] \). We define the Ornstein-Uhlenbeck operator \( T_\rho : \{ \mathbb{R}^d \to \mathbb{R} \} \to \{ \mathbb{R}^d \to \mathbb{R} \} \) that maps \( f : \mathbb{R}^d \to \mathbb{R} \) to the function \( T_\rho f : \mathbb{R}^d \to \mathbb{R} \) with
\[
T_\rho f(x) = \E_{z \sim \mathcal{N}} \left[ f \left( \sqrt{1 - \rho^2} \cdot x + \rho \cdot z \right) \right].
\]
\end{definition}
\begin{fact}[see, e.g.,~\cite{Bog:98}]\label{fct:semi-group}
 The following hold:
\begin{enumerate}
    \item For any $f,g\in L_1$ and $\rho\in(0,1)$, it holds that
    $
   \E_{\x\sim\calN}[ (T_\rho f) g]=\E_{\x\sim\calN}[ (T_{\rho}g(\x)) f(\x)]\;.
    $
    \item For any $g\in L_1$, it holds:
     \begin{enumerate}
        \item For any $\rho\in(0,1)$, $T_\rho g(\x)$ is differentiable at every point $\x$.
        \item For any $\rho\in(0,1)$ the $T_{\rho} g(\x)$ is $\|g\|_\infty/\rho$-Lipschitz, i.e., $\|\nabla T_{\rho} g(\x)\|\leq \|g\|_\infty/\rho$ for all $\x\in \R^d$.
        \item For any $p\geq 1$, $T_\rho$ is a contraction with respect to the norm $\|\cdot\|_{L_p}$, i.e., it holds $\|T_{\rho} g\|_{L^p}\leq \|g\|_{L^p}$.
    \end{enumerate}
\end{enumerate}
\end{fact}

\begin{fact}[Ledoux-Pisier {\cite{Pisier:86,Ledoux:94}}]\label{fact:ledouxpissier}
Let \( f : \mathbb{R}^d \mapsto \{\pm 1\} \) be a Boolean function. It holds
\[
\E_{x \sim \mathcal{N}}[f(x) T_\rho f(x)] \geq 1 - 2 \sqrt{\pi} \, \Gamma(f) \, \rho.
\]
\end{fact}

\begin{fact}\label{fact:circulant}
    Let \(\vec{c} \in \mathbb{R}^n \). Then the matrix

\[
\vec{M} = \begin{bmatrix}
c_0 & c_1 & c_2 & \cdots & c_{n-1} \\
c_{n-1} & c_0 & c_1 & \cdots & c_{n-2} \\
c_{n-2} & c_{n-1} & c_0 & \cdots & c_{n-3} \\
\vdots & \vdots & \vdots & \ddots & \vdots \\
c_1 & c_2 & c_3 & \cdots & c_0
\end{bmatrix}
\]

given by \( \vec{M}_{j,k} = c_{(k-j) \mod n} \) has eigenvalues

\[
\lambda_j = \vec{c} \cdot \boldsymbol{\omega}^{(j)}  \quad j = 0, 1, \ldots, n-1,
\]

where $\boldsymbol{\omega}^{(j)}\in \mathbb{C}^n$ a complex vector such that \( \boldsymbol{\omega}^{(j)}_k = e^{\frac{2\pi i}{n} jk}, k = 0, 1, \ldots, n-1 \).

\end{fact}

{The following claim bounds the Gaussian noise sensitivity of a function when we re-randomize its input over a subspace rather than over the entire space \(\mathbb{R}^d\). This claim is used in the proof of \Cref{cl:unif_claim} to show that a function with bounded Gaussian surface area remains nearly invariant when its input is re-randomized over sufficiently small cubes.} 
\begin{claim}[Gaussian Noise Sensitivity over a Subspace]\label{cl:gsaV}
Let $f:\mathbb{R}^d\to [K]$ be a a $K$-wise function and $V$ be a subspace of $\mathbb{R}^d$
\begin{align*}
    \pr[f(\x) \neq f(\x^{\perp V}+\x'^{V} )]\lesssim \sqrt{\eps}\Gamma(f)\;,
\end{align*}
where $\x^{V}$ and $\x'^{V}$ are $(1-\eps)$-correlated standard Gaussians.
\end{claim}
\begin{proof}
Without loss of generality, we assume that $f$ is a Boolean function, as we can use the union bound over the level sets  to reduce to this case.
Let the Hermite expansion of $f$ be  $f(\x)= \sum_{S\in \N^d } \widehat{f}(S)H_S(\x)$. Then we can express $f'(\x)=f(\x^{\perp V}+\x'^{V})$ in the same basis as 
\begin{align*}
 f'(\x)=   \sum_{S\in \N^d } \widehat{f}(S)H_S(\x^{\perp V}+\x'^{V})=&\sum_{S\in \N^d } \widehat{f}(S)H_S(\x^{\perp V}) T_\eps [H_S(\x'^{V})]= \sum_{S\in \N^d } \eps^{\abs{S_V}} \widehat{f}(S)H_S(\x)\;,
\end{align*}
where $S_V$ denotes subset of $S$ that corresponds to the coordinates in $V$.
Hence the sensitivity over the subspace $V$ can  be written as 
\begin{align*}
    \pr[f(\x) \neq f(\x^{\perp V}+\x'^{V} )]&= \frac{1}{2} -\frac{1}{2} \E[f(\x) f'(\x)]= \frac{1}{2} -\frac{1}{2} \sum_{S\in \N^d } \eps^{\abs{S_V}} (\widehat{f}(S))^2\\
&\le \frac{1}{2} -\frac{1}{2} \sum_{S\in \N^d } \eps^{\abs{S}} (\widehat{f}(S))^2= \frac{1}{2}- \frac{1}{2}  \E[f(\x) T_\eps [f(\x)]] \;.
\end{align*}
Noting that the final expression is the Gaussian sensitivity, 
the result follows.
\end{proof}
\paragraph{Legendre Polynomials}
Legendre polynomials are orthogonal on the interval $[-1,1]$. Their key properties are summarized in the following fact.

\begin{fact}[\cite{Sze67}] The Legendre polynomials $P_k$ for $k\in\Z$, satisfy the following properties:\label{fct:legendre}
\begin{enumerate}
    \item $P_k$ is a $k$-degree polynomial and $P_0(x)=1$ and $P_1(x)=x$.
    \item $\int_{-1}^1 P_i(x)P_j(x) \d x=2/(2i+1)\1(i=j)$, for all $i,j\in \Z$.
    \item $|P_k(x)|\leq 1$ for all $|x|\leq 1$.
    \item $P_k(x)=(-1)^k P_k(-x)$.
    \item $P_k(x)=2^{-k}\sum_{i=1}^{\lceil k/2 \rceil}\binom{k}{i}\binom{2k-2i}{k}x^{k-2i}$.
\end{enumerate}
    
\end{fact}

\section{Omitted Content from \Cref{sec:MulticlassAlgorithm,sec:generalAlgorithm,sec:applications}}

{The following claim shows that if a matrix has a small Frobenius norm but exhibits a large quadratic form in some direction, then by selecting those eigenvectors whose eigenvalues exceed a given threshold, one obtains a short list of vectors among which at least one correlates well with that direction. This claim is used for proving \Cref{prop:alg2,prop:MetaAlg2}.}
\begin{claim}\label{cl:coorelationofeigenvalues}
 Let $\vec M\in \mathbb{R}^{d\times d}$ a symmetric positive semi-definite (PSD) matrix and let \(\vec v\in \mathbb{R}^d \) with $\norm{\vec v}\le 1$ such that  \( \vec v^\top  \vec M \vec v \ge \alpha \). 
Then there exists a unit eigenvector $\vec u$ of $\vec M$ with eigenvalue at least $\alpha/2$ such that
\(
\abs{\vec u\cdot \vec v} \gtrsim ({\alpha}/{\|\vec M\|_F})^{3/2}\;.
\)
Moreover, the number of eigenvectors of $\vec M$ with eigenvalue greater than $\alpha/2$ is at most
\(
{4\|\vec M\|_F}/{\alpha^2}\;.
\)
\end{claim}
\begin{proof}
From the spectral theorem, let  \( \lambda_1 \ge \lambda_2 \ge \cdots \ge \lambda_d \ge 0 \) be the eigenvalues of $\vec M$ with their corresponding orthonormal eigenvectors \( \vec x^{(1)},  \vec x^{(2)}, \ldots, \vec  x^{(d)}\in \R^d \). Since \(\vec M \) is symmetric and psd, it can be decomposed as
\(
\vec M = \sum_{i=1}^d \lambda_i \vec x^{(i)} \vec (\vec x^{(i)})^\top.
\)
Moreover, we can write \( \vec v \) in the eigenbasis of \( \vec M \), i.e.,
\(
\vec v = \sum_{i=1}^d (\vec v\cdot \vec  x^{(i)}) \vec x^{(i)}.
\)
Therefore, we have that
\(
\vec v^{\top} \vec M \vec v  = \sum_{i=1}^d \lambda_i (\vec v\cdot \vec  x^{(i)})^2
\), and using that \( \vec v^{\top} \vec M \vec v \ge \alpha \), we get that
\(
\sum_{i=1}^n \lambda_i (\vec v\cdot \vec  x^{(i)})^2 \ge \alpha
\). 

Let \( S = \{ i : \lambda_i \ge \frac{\alpha}{2} \} \). Using the fact that $\sum_{i=1}^d( \vec v\cdot \vec  x^{(i)})^2\le 1$, we have that
\[
\sum_{i=1}^d \lambda_i (\vec v\cdot \vec  x^{(i)})^2 = \sum_{i\not\in S} \lambda_i (\vec v\cdot \vec  x^{(i)})^2 + \sum_{i\in S}\lambda_i (\vec v\cdot \vec  x^{(i)})^2\le \frac{\alpha}{2}+ \sum_{i\in S} \lambda_i(\vec v\cdot \vec  x^{(i)})^2\;.
\]
Hence, we have that
\(
\sum_{i \in S} \lambda_i (\vec v\cdot \vec  x^{(i)})^2 \ge \frac{a}{2}.
\) Furthermore, note that there must be at least one index \( j \in S \) such that 
\(
\lambda_j (\vec v\cdot \vec  x^{(j)})^2 \ge \frac{a}{2|S|}. 
\)
Also, from the fact that \( \lambda_j \le \|\vec M\|_F \), we have that
\(
(\vec v\cdot \vec  x^{(j)})^2 \ge \frac{a}{2|S| \|\vec M\|_F}.
\)
Taking the sum of the squares of the eigenvalues in $S$, we get that
\[
|S| \cdot \left( \frac{\alpha}{2} \right)^2 \le \sum_{i \in S} \lambda_i^2\leq \|\vec M\|_F^2 \Rightarrow|S| \le \frac{4\|\vec M\|_F^2}{\alpha^2}\;.
\]
Combining the above, we have that  $\vec v\cdot \vec  x^{(j)} \gtrsim (\alpha/\|\vec M\|_F)^{3/2}$, for some $j\in S$, and  $\abs{S}\le 4\|\vec M\|_F/\alpha^2$, which completes the proof of \Cref{cl:coorelationofeigenvalues}.
\end{proof}
{The following two facts are important properties of polynomial regression which are be used in \Cref{prop:MetaAlg2} for the analysis of our general algorithm.
\Cref{fact:regressionAlg} shows that by using a sufficiently small number of samples in polynomial time, one can efficiently obtain a low-degree polynomial whose \(L_2\) approximation error is within an additive \(\epsilon\) of the best achievable by any polynomial of the same degree.
\Cref{fact:regressionSubspaceBound} shows that the subspace spanned by the eigenvectors of a low-degree regression polynomial contains only a few eigenvectors corresponding to large eigenvalues.
 }
\begin{fact}[see, e.g., Lemma 3.3 in \cite{DKKTZ21}]\label{fact:regressionAlg}
     Let $\mathcal{D}$ be a distribution on $\mathbb{R}^d \times \{\pm 1\}$ whose x-marginal is $\mathcal{N}(\vec 0, \vec I)$. Let $k \in \mathbb{Z}_+$ and $\epsilon, \delta > 0$. There is an algorithm that draws $N = (dk)^{O(k)} \log(1/\delta)/\epsilon^2$ samples from $\mathcal{D}$, runs in time $\text{poly}(N, d)$, and outputs a polynomial $P(\x)$ of degree at most $k$ such that
\[
\E_{(\x,y) \sim \mathcal{D}} [(y - P(\x))^2] \leq \min_{P' \in \mathcal{P}_k} \E_{(\x,y) \sim \mathcal{D}} [(y - P'(\x))^2] + \epsilon,
\]
with probability $1 - \delta$.
\end{fact}
\begin{fact}[see, e.g., Lemma 3.3 in \cite{DKKTZ21}]
\label{fact:regressionSubspaceBound}
    Fix $\epsilon \in (0,1)$ and let $P(\x)$ be a degree-$k$ polynomial, such that
\[
\E_{(\x,y) \sim \mathcal{D}}[(y - P(\x))^2] \leq \min_{P' \in \mathcal{P}_k} \E_{(\x,y) \sim \mathcal{D}}[(y - P'(\x))^2] + O(\epsilon).
\]
Let $\vec M = \E_{\x \sim \mathcal{D}_\x} [\nabla P(\x) \nabla P(\x)^\top]$ and $V$ be the subspace spanned by the eigenvectors of $\vec M$ with eigenvalues larger than $\eta$. Then the dimension of the subspace $V$ is $\dim(V) = O(k/\eta)$ and moreover $\tr(\vec M)=O(k)$.
\end{fact}
{\Cref{fact:VarianceMonotonicity} shows that for an intersection of halfspaces, the variance in any defining direction is reduced from one by an amount that depends on how compressed the function is along that direction. This fact is the key ingredient in our structural result on intersections of halfspaces (see \Cref{lem:progressIntersections}).
}
\begin{fact}[Gaussian Monotonicity,  \cite{Vempala10a}]\label{fact:VarianceMonotonicity}
Let $\mathcal{K} \subseteq \mathbb{R}^d$ be an intersection of halfspaces and let $\mathcal{N}_d|_{\mathcal{K}}$ be the truncation of the standard Gaussian distribution in $d$ dimensions $\mathcal{N}_d$ to $\mathcal{K}$. For any $u \in \mathbb{S}^{d-1}$, we have $\operatorname{Var}_{\mathbf{x} \sim \mathcal{N}_d|_{\mathcal{K}}}(\mathbf{u} \cdot \mathbf{x}) \leq 1$. Moreover, if for some $T \in \mathbb{R}$ the halfspace $\{\mathbf{x} : \mathbf{u} \cdot \mathbf{x} + T \geq 0\}$ is one of the defining halfspaces of the intersection then, we have \textit{variance reduction} along $u$, i.e., $\operatorname{Var}_{\mathbf{x} \sim \mathcal{N}_d|_{\mathcal{K}}}(\mathbf{u} \cdot \mathbf{x}) \leq 1 - \frac{1}{2} e^{-\frac{1}{2} \max\{0,T\}^2}$ for a sufficiently large universal constant $C > 0$. 
\end{fact}

\section{{Additional Background on the SQ Model}} 
\label{app:SQ-basics}
{ In the following definition, we state the notion of pairwise correlation between two distributions with respect to a third distribution, and then we introduce the concept of a $(\gamma,\beta)$-correlated family.}
\begin{definition}\label{def:bccorrelated}
    The \textit{pairwise correlation} of two distributions with probability density functions (pdfs) $D_1, D_2 : \mathcal{X} \rightarrow \mathbb{R}_+$ with respect to a distribution with pdf $D : \mathcal{X} \rightarrow \mathbb{R}_+$, where the support of $D$ contains the supports of $D_1$ and $D_2$, is defined as $\chi_D(D_1, D_2) + 1 := \int_{x \in \mathcal{X}} D_1(x)D_2(x)/D(x)dx$. We say that a collection of $s$ distributions $\mathcal{D} = \{D_1, \ldots, D_s\}$ over $\mathcal{X}$ is $(\gamma, \beta)$\textit{-correlated relative to a distribution} $D$ if $|\chi_D(D_i, D_j)| \leq \gamma$ for all $i \neq j$, and $|\chi_D(D_i, D_j)| \leq \beta$ for $i = j$.

\end{definition}
{The following notion of statistical dimension, introduced in \cite{FGR+13}, 
effectively characterizes the difficulty 
of the decision problem}.
\begin{definition}
    For $\gamma, \beta > 0$, a decision problem $\mathcal{B}(\mathcal{D}, \mathcal{D})$, where $D$ is fixed and $\mathcal{D}$ is a family of distributions over $\mathcal{X}$, let $s$ be the maximum integer such that there exists $\mathcal{D} \subseteq \mathcal{D}$ such that $\mathcal{D}$ is $(\gamma, \beta)$\textit{-correlated relative to} $D$ and $|\mathcal{D}| \geq s$. We define the Statistical Query dimension with pairwise correlations $(\gamma, \beta)$ of $\mathcal{B}$ to be $s$ and denote it by $SD(\mathcal{B}, \gamma, \beta)$.

\end{definition}{
The connection between SQ dimension and lower bounds is captured by the following lemma.}
\begin{fact}[\cite{FGR+13}]
\label{lem:SQdim}
    Let $\mathcal{B}(\mathcal{D}, D)$ be a decision problem, where $D$ is the reference distribution and $\mathcal{D}$ is a class of distributions over $\mathcal{X}$. For $\gamma, \beta > 0$, let $s = SD(\mathcal{B}, \gamma, \beta)$. Any SQ algorithm that solves $\mathcal{B}$ with probability at least $2/3$ requires at least $s \cdot \gamma / \beta$ queries to the VSTAT(1/$\gamma$) oracle.
\end{fact}

{In order to construct a large set of nearly uncorrelated hypotheses, we need the following fact:}
\begin{fact}[see, e.g., Lemma 2.5 in \cite{DiakonikolasKPZ21}]\label{fact:NearOrthogonality} 
Let $0 < a, c < \frac{1}{2}$ and $m,n \in \mathbb{Z}_+$ such that $m \leq n^a$. There exists a set $S$ of $2^{\Omega(n^c)}$ matrices in $\mathbb{R}^{m \times n}$ such that every $\mathbf{U} \in S$ satisfies $\mathbf{U}\mathbf{U}^\top = \mathbf{I}_m$ and every pair $\mathbf{U}, \mathbf{V} \in S$ with $\mathbf{U} \neq \mathbf{V}$ satisfies $\|\mathbf{UV}^\top\|_F \leq O(n^{2c-1+2a})$.
\end{fact}
{Furthermore, to bound the correlation between two rotated versions of the same function, we will use the following fact, which relates this correlation to the Hermite coefficients of the function and the angle between the rotations. }

\begin{fact}[Lemma 2.3 in \cite{DiakonikolasKPZ21}]\label{fact:Correlation}
Let $g : \mathbb{R}^m \to \mathbb{R}$ and $\mathbf{U}, \mathbf{V} \in \mathbb{R}^{m \times n}$ be linear maps such that $\mathbf{U}\mathbf{U}^\top = \mathbf{V}\mathbf{V}^\top = \mathbf{I}_m$. Then, we have that
\[
\E_{\mathbf{x} \sim \mathcal{N}_n} \left[g(\mathbf{U}\mathbf{x})g(\mathbf{V}\mathbf{x})\right] \leq \sum_{t=0}^{\infty} \|\mathbf{UV}^\top\|_2^{t} \E_{\mathbf{x} \sim \mathcal{N}_m} \left[ (g^{[t]}(\mathbf{x}))^2 \right],
\]
where $g^{[t]}$ denotes the degree-$t$ Hermite part of $g$.
\end{fact}

\section{Relation to Other Complexity Measures}
\label{app:compl-measures}

\subsection{Related Complexity Measures} \label{ssec:complexity-meas-prose}
{

Several complexity measures have been proposed for learning 
real-valued concepts in the realizable setting. One measure closely 
related to our work is the \emph{generative exponent} $k^*$, 
introduced in \cite{damian2024generativeexponent}. The generative 
exponent $k^*$ is defined as the smallest positive integer $k$ for 
which $\E_{y \sim D_y} [\zeta_k^2(y)] \neq 0$, where 
$\zeta_k(y) \coloneqq \E_{x \sim \mathcal{N}(0,1)} [ He_k(x) \mid f(x) = y ]$ and 
$f:\R\to \R$ is a single-index link function.
This measure is used for the task of approximating 
the hidden direction $\w$ of a single-index model 
\( f(\w \cdot \x) \) from samples \((\x, f(\w \cdot \x))\) with \( \x \sim \mathcal{N}(\vec{0}, \vec{I}) \).
Using the  measure \( k^* \),  \cite{damian2024generativeexponent} 
derived nearly tight SQ bounds, with 
complexity of order \( d^{k^*/2} \). 
(Technically speaking, their upper bound 
requires knowledge of \( \zeta_{k^*}(y) \) 
or the existence of an orthonormal polynomial basis 
with respect to the marginal distribution 
over the labels \( D_y \), where \( \zeta_{k^*} \) has a large low-degree component.)

Our definition of well-behaved MIMs 
(\Cref{def:SimpleGoodCondition}) implicitly introduces 
a generalization of the generative exponent 
for PAC learning discrete-valued MIMs in the presence of noise. 
As we demonstrate in \Cref{ssec:gen}, 
for the special case of realizable learning of 
Single-Index models, our complexity measure
is essentially equivalent to the generative exponent. 
We establish this under two additional regularity assumptions: 
(i) a bound on the GSA of the ground truth concept, 
and (b) the existence of non-trivial distinguishing moments 
with respect to an already discovered subspace. 
(These regularity assumptions hold 
for several well-studied function classes such as multiclass linear classifiers and intersections of halfspaces.)

Note that our learning algorithm does not require prior knowledge of the moments conditioned on each level-set for two reasons: 
(i) when the number of labels is finite, 
we can leverage the moments of all level-sets efficiently;
and (ii) the bounded GSA assumption allows us to partition 
an already discovered subspace into finite-width regions 
to uncover additional directions using moments.
Moreover, by considering the existence of distinguishing moments 
only when our approximation to the target ground truth subspace 
is insufficient, our notion allows for relaxed approximation 
guarantees that yield more efficient algorithms.

Other relevant complexity measures that have gained 
attention recently are the \emph{information exponent} \cite{arous2021online} and the \emph{leap complexity} \cite{abbe2021staircase,abbe2022merged,abbe2023sgd}.
The information exponent is defined as the smallest integer \(k\) for 
which a single-index link function \(f:\R\to \R\) has a non-zero 
Hermite coefficient, i.e., the minimum \(k\) such that 
\(\E_{x\sim \cN(0,1)}[f(x)He_k(x)]\neq 0\).
The leap complexity generalizes this notion to real-valued MIMs: 
we say that a function $f:\R^K\to\R $ has leap complexity 
\(\mathrm{Leap}(f)=k\) if every Hermite decomposition of $f$ 
can be ordered such that each consecutive term introduces 
new directions via a polynomial of degree at most \(k\).
Since these measures rely solely on the non-negativity 
of the Hermite coefficients and do not incorporate label-based 
information (e.g., via level-sets), they are 
closely related to the CSQ complexity, which 
can be arbitrarily larger than the SQ complexity.
}

\subsection{Near-equivalence to Generative Exponent for Sindle-Index Case} \label{ssec:gen}

In this section, we compare our definition of well-behaved MIMs (\Cref{def:SimpleGoodCondition}) to the generative exponent (\Cref{def:generative-exponent}) introduced in \cite{damian2024generativeexponent}. We first give a formal definition of the generative exponent:
\begin{definition}[Generative exponent]\label{def:generative-exponent}
Let $f:\R\to \R$ be a single index link function. 
Let $D$ be the joint distribution of $(x,f(x))$, 
where $x\sim \cN(0,1)$.
The generative exponent $k^*$ is defined 
as the smallest positive integer $k$ for which 
$\E_{y \sim D_y} [\zeta_k^2(y)] \neq 0$, 
where $\zeta_k(y) \coloneqq \E_{x \sim \mathcal{N}(0,1)} [ He_k(x) \mid f(x) = y ]$.
\end{definition}
As we have mentioned in the previous subsection, 
the generative exponent focuses on subspace recovery 
(in the realizable setting) rather than accurate label prediction. 
Hence, in order for the two notions to be comparable, 
we first slightly modify our definition 
of \Cref{def:SimpleGoodCondition} to this setting.

\begin{definition}[Well-Behaved MIMs]\label{def:modified-well-behaved-MIMs}
Fix $m\in \mathbb{Z}_+,\sigma\in (0,1)$ and 
$\mathcal{Y}\subseteq \Z$ of finite cardinality.
Let $f: \R^d \to \mathcal{Y}$ be a $K$-MIM function  with hidden subspace $W$.
We say that a $K$-MIM function $f: \R^d \to \mathcal{Y}$ 
is  $m$-well-behaved MIM
if for any subspace $V\subseteq \R^d$ the following conditions holds:
\begin{enumerate}
    \item[a)] $V$ contains $W$, i.e., $W\subseteq V$, or
    \item[b)]   
there exists a point $\vec x_0\in \R^d$, a zero mean, unit-variance polynomial 
    $p:U\to \R$  of degree at most $m$ and $z\in \mathcal{Y}$ 
    such that 
    $$\E_{\x\sim \mathcal N(\vec 0,\vec I)}[p(\x^U)\Ind(f(\x)=z)\mid \x^V=\x_0^V] \neq 0\;,$$
 where $U=(V+W)\cap V^\perp$.
\end{enumerate}
\end{definition}
Compared to \Cref{def:SimpleGoodCondition}, 
our modified definition of well-behaved MIMs 
in \Cref{def:modified-well-behaved-MIMs} introduces the following changes:
\begin{enumerate}
    \item[(i)] We omitted the technical assumption of bounded GSA to direct our attention on the existence of distinguishing moments.

    \item[(ii)]  Instead of requiring that the function be approximately equal to one that depends solely on the projection onto \(V\) (as specified in condition (a) of \Cref{def:SimpleGoodCondition}), we now require that \(V\) contain the hidden subspace \(W\). This change emphasizes subspace recovery in the realizable setting and aligns our definition with that of the generative exponent.
    \item[(iii)] Since we are focusing on the realizable setting, we consider only the moments of $f$ rather than noisy version of $f$.
\end{enumerate}
In the following lemma, we show that for SIMs, 
the smallest degree $m$ for which \Cref{def:modified-well-behaved-MIMs} 
is satisfied is equal to the generative exponent.
\begin{lemma}
Let $ \cY$ be a finite cardinality set. 
Let $f:\R^d\to \cY$ be a SIM, that is, 
there exists a unit vector $\w$ such that 
$f(\x)=f(\x^\w)$ for all $\x\in\R^d$.
The smallest positive integer $m$ for which $f$ is an $m$-well-behaved MIM 
(as per \Cref{def:modified-well-behaved-MIMs}) is equal 
to the generative exponent $k^*$ of $f$ 
(see \Cref{def:generative-exponent}).
\end{lemma}
\begin{proof} We first show that $f$ is not an $m$-well-behaved MIM for all positive integers  $m<k^*$. Note that  
since $\cY$ is discrete for all $k\in \Z_+$ we have that 
 \begin{align*}
     \E_{y \sim D_y} [\zeta_k^2(y)]=\sum_{y\in \mathcal{Y}} \zeta_k^2(y)\pr_{x\sim \cN(0,1)}[f(x\w)=y]\;,
 \end{align*}
 where $\zeta_k(y)\eqdef\E_{x \sim \mathcal{N}(0,1)} [ He_k(x) \mid f(x\w) = y ]$ (see \Cref{def:generative-exponent}).
 Hence for each positive integer $m<k^*$ by the definition of the generative exponent for all $y\in \cY$ it holds that 
\begin{align*}
&\E_{x\sim \cN(0,1)}[He_m(x)\mid f(x\w)=y]^2\pr_{x\sim \cN(0,1)}[f(x\w)=y]\\&=\E_{\x\sim \cN(\vec 0,\vec I)}[He_m(\w \cdot \x)\mid f(\x)=y]^2\pr_{\x\sim \cN(\vec 0,\vec I)}[f(\x)=y]=0\;.\end{align*}
Otherwise, the generative exponent would be less than $k^*$. Thus, for every $m\in [k^*-1]$ and $y\in \cY$ either $\E_{\x\sim \cN(\vec 0,\vec I)}[He_m(\w \cdot \x)\mid f(\x)=y]$ or $\pr_{\x\sim \cN(\vec 0,\vec I)}[f(\x)=y]=0$. In either case, this implies that $\E_{\x\sim \cN(\vec 0,\vec I)}[He_m(\w \cdot \x)\Ind(f(\x)=y)]=0$, for all $m\in [k^*-1]$ and $y\in \cY$. 

Since any zero-mean polynomial of degree less $k^*$ can be decomposed to a linear combination of non-constant Hermite polynomials of degree less than $k^*$, it follows that $\E_{\x\sim \cN(\vec 0,\vec I)}[p(\w\cdot \x)\Ind(f(\x)=y)]=0$, for any zero-mean polynomial $p$ of degree less than $k^*$. Furthermore, because $f$ is a SIM, for any polynomial $p:\R^d\to \R$ we have
$$\E_{\x\sim \cN(\vec 0,\vec I)}[p(\x)\Ind(f(\x)=y)]=\E_{\x\sim \cN(\vec 0,\vec I)}[p((\w \cdot \x) \w)\Ind(f(\x)=y)]\;.$$
Therefore, for any zero-mean polynomial $p$ of degree less $k^*$ it holds that $\E_{\x\sim \cN(\vec 0,\vec I)}[p(\x)\Ind(f(\x)=y)]=0$, so $f$ does not satisfy the condition (b) of \Cref{def:modified-well-behaved-MIMs} even for $V=\{0\}$. This concludes the first part of the proof.

Now  to conclude the proof it suffices to show that $f$ is a $k^*$-well-behaved MIM. Note that since $\E_{y \sim D_y} [\zeta_{k^*}^2(y)]\neq 0$ and $\cY$ is discrete, we have that there exists a label $y\in \cY$ such that $\E_{\x\sim \cN(\vec 0,\vec I)}[He_{k^*}(\w \cdot \x)\Ind(f(\x)=y)]\neq 0$. Consider a subspace $V$ of $\R^d$ that does not contain $\w$ and denote by $\vec u$ the projection of $\w$ onto $V^{\perp}$. 
Note that since $V$ does not contain $\w$ it holds that $\norm{\vec u}=\norm{\w^{\perp V}}=\w \cdot \vec u\neq 0$.

Since $f$ is a SIM ,we have that for any $\lambda\in \R$
\begin{align*}
\E_{\x\sim \cN(\vec 0,\vec I)}[He_{k^*}(\lambda\vec u\cdot \x)\Ind(f(\x)=y)]=
    \E_{\x\sim \cN(\vec 0,\vec I)}[He_{k^*}(\lambda (\w \cdot \vec u)(\w \cdot \x))\Ind(f(\x)=y)]\;.
\end{align*}
Thus, setting $\lambda=1/(\w \cdot \vec u)$ gives us that 
\begin{align*}
\E_{\x\sim \cN(\vec 0,\vec I)}\left[He_{k^*}\left(\frac{\vec u\cdot \x}{\w \cdot \vec u}\right)\Ind(f(\x)=y)\right]\neq 0\;.
\end{align*}
Note that $\E_{\x\sim \cN(\vec 0,\vec I)}[He_{k^*}\left({\vec u\cdot \x}/{(\w \cdot \vec u)}\right)]=0$ by a change of variables, since $\vec u/(\w \cdot \vec u)$ is a unit vector. Thus, normalizing the above polynomial implies the existence of  a mean-zero, variance-one polynomial 
    $p:\R\to \R$  of degree $k^*$ 
    such that 
    $$\E_{\x\sim \mathcal N(\vec 0,\vec I)}[p(\vec u\cdot \x )\Ind(f(\x)=y)] \neq 0\;.$$
Therefore, there has to exist an $\x_0^V$ such that $\E_{\x\sim \mathcal N(\vec 0,\vec I)}[p(\vec u\cdot \x )\Ind(f(\x)=y)\mid \x^V=\x_0^V] \neq 0\;$, otherwise averaging over $\x_0^V$ would result to the above moment being zero.
Consequently, we have that $f$ is a $k^*$-well-behaved MIM, which completes the proof.
\end{proof}

\section{SQ Lower Bound for Multiclass Linear Classification from Partial Labels} \label{app:SQ-partial}
In this section, we will leverage the moment-matching 
construction developed in the previous section 
(for learning with RCN) in order to prove SQ hardness for learning multiclass linear classifiers with partial labels.

First we give a formal definition of the model of corruption and then we present our construction for the hard instances along with our main theorem.
\begin{definition}[Partial Label Distribution]\label{def:partial-label-distribution}
Let $D$ be a distribution over $\R^d\times 2^\mathcal{Y}$. We say that $D$ is a partial label distribution for the class $\mathcal{F}\subseteq\{f:\R^d\to \mathcal{Y}\}$ if there exists a function $f\in \mathcal{F}$ and a row-stochastic matrix $\vec H\in \R^{\abs{\cY}\times 2^{\abs{\cY}}}$ such that $\pr_{(\x,S)\sim D}[S=T\mid \x ]=\vec H_{f(\x),T}$ for all $\x\in \R^d, T\subseteq S$. Moreover, $\vec H_{i,S}=0$ if $i\not \in S$ and  $\pr_{(\x,S)\sim D}[y\not\in S\mid \x]=\sum_{S\not\ni y} \vec H_{f(\x),S}\geq\gamma$, for all $y\neq f(\x)$ for all $\x \in \R^d$, for some $\gamma>0$.
We refer to $\vec H$ as the confusion matrix of the distribution.
\end{definition}

\paragraph{SQ-Hard Instances}
To construct our family of hard instances, we consider planted versions of a two-dimensional model. 
This part of the construction is the same as the one used in \Cref{sec:RCNSQ} and it is defined in \Cref{eq: 2Dmodel,eq:family}.

As for the confusion matrix, we consider a subset of learning with partial labels problems called learning with contrastive labels \cite{Ishida2017Complementarylabels,Zang2020Complementary}. 
In this setting, we observe examples of the form $\{(\x^{(i)},\bar{y}_i)\}_1^m$, where $\bar{y}_i$ is almost surely different to the ground truth label $y_i$. Without loss of generality we will refer to the example distribution, $D$, as a distribution over $\R^d\times \cY$ rather than a distribution over $\R^d\times 2^\mathcal{Y}$ as there are $\abs{\cY}$ subsets $S\subseteq \cY$ where $\abs{S}=\abs{\cY}-1$.
 Moreover, by \Cref{def:partial-label-distribution}, we have that  there exists a  $\abs{\cY}\times \abs{\cY}$ stochastic matrix $\vec{H}$ such that $\pr_{(\x,\bar{y})\sim D}[\bar{y}= j\mid \x,f(\x)=i]=\vec H_{i,j}$ for all $\x\in \R^d,i,j\in \cY$ and  the identifiability constraint corresponds to the existence of a constant $\gamma\in (0,1]$ satisfying  $\vec H_{i,j}=0,\text{ if } i=j \text{ and } \vec H_{i,j}\ge \gamma,\text{ if } i\neq j$, for all $i,j\in \cY$.

 Using this construction we prove the main result of this section:
 \begin{theorem}[SQ Lower Bound for Learning with Partial Labels]\label{thm:SQ-partial-labels}
Any SQ algorithm that learns $K$-class Linear Classifiers given  partial labels, with $\gamma= \poly(1/K)$ (see \Cref{def:partial-label-distribution}), under the $d$-dimensional standard normal to error $1/(2K)$, requires either  $2^{d^{\Omega(1)}}$ queries or at least one query to  $\text{VSTAT}\left(d^{\Omega(K)}\right)$.
 \end{theorem}
Before proving \Cref{thm:SQ-partial-labels}, we first need to show that we can match the appropriate number of moments with the standard Gaussian. The proof of this lemma is similar to the proof of \Cref{lem:existance}.

 \begin{lemma}\label{lem:existanceCoarse}
Let $K\in \Z_+ $ with $K\ge 8$ such that $K$ is divisible by $4$.
There exists a $K \times K$ doubly stochastic matrix $\vec{H}$  such that the following hold: 
\begin{enumerate}
        \item[i)]$\vec H_{i,i}=0$, for all $ i\in [K]$ and $\min_{i\neq j} \vec H_{i,j}=\Omega(1/K^3).$
        \item[ii)] For every polynomial $p:\mathbb{R}^2\to \mathbb{R}$ of degree at most $K/2-2$ we have \\ $\E_{\x\sim \cN^2}\left[\left(\pr_{(\x,\bar{y})\sim D}\left[\bar{y}=i\mid \x\right]-\frac{1}{K}\right)p(\x)\right]=0$ for all $i\in [K]$.
    \end{enumerate}
  where $D$ is a contrastive label distribution over $\R^2\times [K]$ whose $\x$-marginal is $\cN(\vec 0,\vec I)$ that satisfies \Cref{def:partial-label-distribution} for the confusion matrix $\vec H$ and the function $f$ as defined in \Cref{eq: 2Dmodel}.
\end{lemma}
\begin{proof}
Let $\vec H\in \R^{K\times K}$ be a doubly stochastic matrix and  define the functions  $\tilde{f}_i$ for $i\in [K]$ as  $\tilde{f}_i(\x)=\sum_{j=1}^K\vec H_{j,i}\Ind(f(\x)=j)$. 
Also denote by $\boldsymbol{\omega}^{(k)}$ the root of unity vector of rate $k$, i.e. ${\boldsymbol{\omega}}^{(k)}\in \mathbb{C}^K$ such that $\boldsymbol{\omega}_{j}^{(k)}= e^{\frac{2\pi kj i}{K}}$.

Exactly as in the proof of \Cref{lem:existance} we have that by \Cref{cl:polynomial2kerequiv,cl:fourier} 
in order to have $0$ correlation with all polynomials of degree at most $m$, it suffices to have that $\boldsymbol{\omega}^{(k)} \in \ker(\vec{H}^\top)$ for all $k\in \{1,\dots , m,K-m,\dots, K-1\}$, since the roots of unity are symmetric, i.e., $\boldsymbol{\omega}^{(-k)}=\boldsymbol{\omega}^{(K-k)}$.

We first construct a stochastic vector $\vec h\in \mathbb{R}^K$ such that $\vec h_{0}=0$ and $\vec h_{j}=\Omega(1/K^3), \forall j\in \{1,\dots, K-1 \}$ that is orthogonal to the aforementioned root of unity vectors for $m=\Omega(K)$. Next, we extent this vector to a matrix. 

Define the vectors $\vec h^{(1)}=\boldsymbol{\omega}^{(0)}/K $, $ \vec h^{(2)}=(\boldsymbol{\omega}^{(K/2-1)}+\boldsymbol{\omega}^{(K/2+1)})/K$ and $\vec h'=2\vec h^{(1)}-\vec h^{(2)}$. 
It is easy to see that $\vec h'$ is a vector with non-negative entries and $\vec h_0'=0$. 
Moreover, it is true that $$\min_{j\in[K-1]} \vec h_j'\ge 2/K-\max_{j\in[K-1]} \vec h_j^{(2)}=2/K-2/K\cos\left(\frac{2\pi}{K}\right) \gtrsim \frac{1}{K^3}\;,$$
where the last equality follows from the Taylor series  $\cos(x)$ which is an alternating decreasing series when $x=\frac{2\pi}{K}\le 1$.

Let $\vec h$ be  the normalization of $\vec h'$ such that the sum of its coordinates equals $1$.
Since each element of $\vec h'$ is less than $4/K$ we have that the normalization factor is at most $4$ and hence $\min_{j\in[K-1]} \vec h_j\gtrsim \frac{1}{K^3}$. 
Moreover, from the well-known fact that $\boldsymbol{\omega}^{(0)},\dots,\boldsymbol{\omega}^{(K-1)}$ form an orthogonal basis of $\mathbb{C}^K$, we have that 
$\vec h$ is orthogonal to $\boldsymbol{\omega}^{(k)}$ for all $k\in \{1,\dots,K/2-2, K/2, K/2+2,\dots, K-1\}$, since it is a linear combination of $\boldsymbol{\omega}^{(0)}, \boldsymbol{\omega}^{(K/2-1)}, \boldsymbol{\omega}^{(K/2+1)}$.

Now consider $\vec{H}$, the circulant matrix generated by row shifts of the vector $h$. 
 By construction, it is true that $\vec{H}$ is row stochastic with $ \vec H_{i,i}=0 $ for all $ i\in [K] $ such that $ \min_{i\neq j} \vec H_{i,j}=\Omega(1/K^3)$.
Moreover, $\vec{H}$ is symmetric, because it holds that $\vec H_{i,j}=\vec h_{j-i\mod K}=\vec h_{i-j\mod K}=\vec H_{j,i}$ and $\vec h_j= \vec h_{K-j}, j\in [K-1]$. {This equality is justified by the identity $(-1)^j\cos\left({2\pi j}/{K}\right)=(-1)^{K-j}\cos\left({2\pi(K-j)}/{K}\right)$, which holds since $K$ is an even integer.} Hence $\vec{H}$ is also column stochastic.
 
 Hence, $\boldsymbol{\omega}^{(k)} \in \ker(\vec{H})=\ker(\vec{H}^\top)$, for all $k\in \{1,\dots,K/2-2, K/2, K/2+2,\dots, K-1\}$. 
 Therefore, applying \Cref{cl:polynomial2kerequiv} and  \Cref{fact:circulant} concludes the proof of \Cref{lem:existanceCoarse}.
\end{proof}

\subsection{Proof of \Cref{thm:SQ-partial-labels}}
Now, given \Cref{lem:existanceCoarse}, we proceed to the proof of \Cref{thm:SQ-partial-labels}. The proof is essentially the same to that of \Cref{thm:RCNLowerBoundK}, except that it employs the construction from \Cref{lem:existanceCoarse} rather than \Cref{lem:existance}. Despite that we include the complete proof here for completeness.
\begin{proof}[Proof of \Cref{thm:SQ-partial-labels}]
Consider the case where $K$ is divisible by $4$ and let  $\vec H\subseteq \R^{K\times K}$ be a matrix that satisfies the statement of \Cref{lem:existanceCoarse}
Let $f:\R^2\to [K]$ be the linear classifier defined in \Cref{eq: 2Dmodel}, and let $\mathcal{F}_{\mathcal{Q}}$ be the class of function defined in \Cref{eq:family}, where $\mathcal{Q}\subseteq \R^{d\times 2}$ the family that satisfies the statement of \Cref{fact:NearOrthogonality}.
We define the family of distributions $\mathcal{D}$ such that $D_{\vec Q}\in \mathcal{D}$ if and only if $D_{\vec Q}$ is a joint distribution $(\x,y_{\vec Q})$, where $\x \sim \cN(\vec 0,\vec I)$ and $y_{\vec Q}$ is a random variable supported on $[K]$ defined as $\pr[y_{\vec Q}=i\mid \x]=\vec H_{f(\vec Q\x), i}$ for some $\vec Q\in \vec \mathcal{Q}$. Note that all distributions in $\mathcal{D}$ satisfy \Cref{def:partial-label-distribution} for confusion matrix $\vec H$ and some function in $\mathcal{F}_{\cQ}$.

Denote by $D$ the product distribution $\cN^d\times \text{Unif}([K])$.
For $D_{\vec Q}, D_{\vec P}\in \mathcal{D}$ with $\vec Q\neq \vec P$, we have that 
\begin{align*}
    \chi_{D}(D_{\vec Q}, D_{\vec P})&= K\sum_{i=1}^K\E_{\x\sim \cN^d}\left[(\pr[y_{\vec Q}=i\mid \x]-\frac{1}{K})(\pr[y_{\vec P}=i\mid \x]-\frac{1}{K})\right]\\
    &= K\sum_{i=1}^K\E_{\x\sim \cN^d}\left[(\vec H_{f(\vec Q\x),i}-\frac{1}{K})(\vec H_{f(\vec P\x),i}-\frac{1}{K})\right]\\
    &= K\sum_{i=1}^K\sum_{j=1}^{\infty}\norm{\vec P\vec Q^{\top}}_2^j \E_{\x\sim \cN^2}\left[((\vec H_{f(\x),i}-\frac{1}{K})^{[j]})^2\right]\\
    &= \norm{\vec P\vec Q^{\top}}_2^{K/2-2}K\sum_{i=1}^K \E_{\x\sim \cN^2}\left[(\vec H_{f(\x),i} -\frac{1}{K})^2\right]\le  d^{\Omega(K)(c-1/2)}K\;,
\end{align*}
where we used \Cref{fact:Correlation}, Parseval's identity and the fact that $\vec H$ satisfies the statement of \Cref{lem:existanceCoarse}. Similarly, we have that $\chi_{D}(D_{\vec Q},D_{\vec Q})\le K$ for all $D_{\vec Q}\in \mathcal{D}$.

Consequently, we have that there exists a family of distributions of size $2^{\Omega(d^c)}$ that is $(d^{\Omega(K)(c-1/2)}K, K)$ correlated. Hence, by \Cref{lem:SQdim} we have that any SQ algorithm that solves the decision problem $\cB(\cD, D)$ requires $2^{\Omega(d^c)}d^{\Omega(K)(c-1/2)}$ queries to the  $\text{VSTAT}(1/(d^{\Omega(K)(c-1/2)}))$.

In order to complete the statement of the theorem, we need to show the following {standard} reduction from the testing problem to the learning problem. \begin{claim}
Let $\eps\leq 1/K$.
    Any SQ algorithm that learns multiclass linear classifiers with partial labels under the standard normal with error better than $\eps$ can be used to solve the decision problem  $\mathcal{B}(\mathcal{D},D)$ using one more query of accuracy $1/K-\eps$.
\end{claim}
\begin{proof}
Assume that there exists an SQ algorithm $\mathcal{A}$ that if given access to a partial label distribution $D'$ (see \Cref{def:partial-label-distribution}) over $\R^d\times 2^\mathcal{Y}$ whose $\x$-marginal is standard normal returns a hypothesis $h$ such that $\pr_{\x\sim \cN(\vec 0,\vec I)}[h(\x)\neq f(\x)]\le \eps$. Using $\mathcal A$, we can reduce the testing problem to the learning problem.

Note that by running  $\mathcal{A}$ to distribution  $D_{\vec P}\in \mathcal{D}$ that corresponds to the concept $f(\vec P\x)\in \mathcal{F}_{\mathcal{Q}}$, $\mathcal A$ returns a hypothesis $h$ such that $\pr_{\x\sim \cN(\vec 0,\vec I)}[h(\x)\neq f(\vec P\x)]\le \eps$. 
As a result, by the reverse triangle inequality, we have that testing the returned hypothesis disagreement with the contrastive labels gives us $\pr_{(\x,\bar{y})\sim D}[h(\x)\neq \bar{y}]\geq 1- \eps$.
However, running  $\mathcal{A}$ with SQ oracle access to  $D$ would result in a hypothesis $h$ that has  $\pr_{(\x,\bar{y})\sim D}[h(\x)\neq \bar{y}]=1-1/K,$ since $D$ is a product distribution with the uniform distribution over the labels.
Since, $\eps\leq 1/K$ we have that we can solve  $\mathcal{B}(\mathcal{D},D)$ with an additional query of accuracy less than $1/K-\eps$.
\end{proof}

For $K$ not divisible by $4$, we add $r = 4 \mod K$ additional classes whose weight vectors are identical to $\vec{w}^{(1)}$.
We construct the new confusion matrix as follows: first let $\vec H$ be the matrix that satisfies the conditions of \Cref{lem:existanceCoarse} for $i,j \in [K-r]$.
For the added classes $j\in\{K-r+1,\dots,K\}$, we set $\vec H_{j,j} = 0$, $\vec H_{j,i} = 1/K$ for all $j \neq i$ and $\vec H_{i,j}=1/K$ for $i\in [K-r]$.
Finally we normalize the elements of the rows such that the matrix is row stochastic.
Note that since these additional classes have zero probability (since ties are broken by the smallest index) and for each of the classes $i\in[K-r]$ the added labels are observed with the same rate, our {SQ hardness result} extends to the general case of any $K$.
Finally, note since the normalization factor is at most $1+4/K$ we have that $\min_{i\neq j} \vec H_{i,j}\gtrsim (1/(1+4/K))(1/K^3)\geq \poly(1/K)$, which completes the proof of \Cref{thm:SQ-partial-labels}.
\end{proof}

\section{Lower Bound For Single-Shot Optimization}
\label{app:ss}

In this section, we present an example in which a single-shot approach fails to learn the target function to sufficiently small error even if it is given access to exactly the first $O(1)$-moments. On the other hand, we show that if we take an iterative approach then we can learn the function in polynomial time using just $O_{\epsilon}(d^2)$ complexity.
Our results build on moment-matching constructions and properties of Legendre polynomials.

\begin{theorem}\label{thm:lower-bound}
Let $d\in\mathbb{Z}_+$ and $C_1,C_2>0$ be sufficiently large universal constants. There exists a family of functions
\[
\mathcal{F}=\{f(\vec U\x) : \vec U\in\R^{2\times d},\; \vec U\vec U^\top=\vec I_2\},
\]
such that:
\begin{itemize}[leftmargin=*]
    \item No algorithm given the first $C_1$ exact moments of $f(\vec U\x)\in \mathcal F$ can output another function $f'\in \mathcal F$ such that $ \E_{\x\sim\mathcal N(\vec 0,\vec I)}[(f(\vec U\x)-f'(\x))^2]\leq 1/C_2$.
    \item There exists an SQ algorithm that for any $\eps>0$, given SQ oracle access to a distribution $D$ over $\R^d\times \{0,1\}$ (with $\x\sim\cN(\vec 0,\vec I)$ and $y=f(\vec{U}\x)$ for some $f\in\mathcal{F}$), outputs an estimate $\widehat{\vec{U}}$ satisfying $\|\widehat{\vec U}-\vec U\|_2\leq \eps$ using $O(d^2)$ queries to $\mathrm{VSTAT}(C_2d^2/\eps^2)$.
\end{itemize}
\end{theorem}

\begin{proof}
The proof is based on a careful construction of a two-dimensional distribution and a corresponding function family. Assume that the algorithm is given $C_1=O(1)$ moments. Let $k=2 C_1$. Define the set $A$ (we use $A$ as the indicator function of this set) on $\R^2$ as in \Cref{prop:moment-matching-k-stuff} and let
\[
B(x,y) = \1(|x+y|\ge 2b)\;,
\]
where $b=\Theta(k)$ and is the smallest value such that $A(x,y)=0$ whenever $|x|\ge b$ or $|y|\ge b$. We then define the boolean function
\[
f(x,y)=A(x,y)+B(x,y) \;.
\]
Let $\mathcal{Q}\subseteq\R^{d\times2}$ be a set of matrices that satisfy \Cref{fact:NearOrthogonality}, i.e.,
\[
|\mathcal{Q}|=2^{d^{\Omega(1)}},\quad \vec Q\vec Q^{\top}=\vec I_{2},\quad \text{and}\quad \|\vec Q\vec P^{\top}\|_F\le O\Big(d^{-\Omega(1)}\Big)\;.
\]
For each $[\vec v,\vec u]=\vec Q\in\mathcal{Q}$, we define
$
f_{\vec Q}(\x)$ as follows: Let $\w^{(1)}=\vec v/\sqrt{2}+\vec u/\sqrt{2} $ and $\w^{(2)}=\vec v/\sqrt{2}-\vec u/\sqrt{2} $ and we let $
f_{\vec Q}(\x)=f(\w^{(1)}\cdot\x,\w^{(2)}\cdot \x)$. 
Let $
\mathcal{F}_{\mathcal{Q}}=\{f_{\vec Q}(\x):\vec Q\in\mathcal{Q}\}\subseteq\mathcal{F}$.
For each $\vec Q\in\mathcal{Q}$, let $D_{\vec Q}$ denote the joint distribution of $(\x, f_{\vec Q}(\x))$, where $\x\sim\mathcal{N}(\mathbf{0},\vec I)$. Define the class of distributions $
\mathcal{D}=\{D_{\vec Q}:\vec Q\in\mathcal{Q}\}$.

Now, suppose we have the first $k/2$-moments of a distribution $D$ drawn from $\mathcal{D}$. 
Note that because $\E[p(\x)A(\w^{(1)}\cdot\x,\w^{(2)}\cdot \x)]=0$ for any zero mean mean polynomial $p$ of degree at most $k$, we have that $\E[p(\x)f(\x)]=\E[p(\x)B(\x)]=\E[p(\x)B(\w^{(1)}\cdot\x+\w^{(2)}\cdot \x)]=\E[p(\x)B(\vec v\cdot\x\sqrt{2})]$. Therefore, we can exactly extract from the moments only $\w^{(1)}+\w^{(2)}=\vec v\cdot\x\sqrt{2}$. Therefore, with access to $O(1)$-moments, we can only extract the direction $\vec v$. We show that for any vector $\vec w$ so that $\vec w\cdot\vec u\leq c_1$ where $c_1>0$ is a sufficiently small absolute constant, then the hypothesis $A(\w^{(3)}\cdot\x,\w^{(4)}\cdot \x)$ achieves large $L_2^2$ error where $\w^{(3)}=\vec v/\sqrt{2}+\vec w/\sqrt{2} $ and $\w^{(4)}=\vec v/\sqrt{2}-\vec w/\sqrt{2}$. First, we choose $\w$ so that $\w\cdot \vec u=0$. Note that from Jensen's inequality it holds that
\begin{align}
 \E_{\x\sim\mathcal N(\vec 0,\vec I)}[&(A(\w^{(1)}\cdot\x,\w^{(2)}\cdot \x)-A(\w^{(3)}\cdot\x,\w^{(4)}\cdot \x))^2]\nonumber
 \\&\geq  \E_{\x\sim\mathcal N(\vec 0,\vec I)}\left[\left(\E_{\vec v\cdot \x\sim\mathcal N( 0,1)}[A(\w^{(1)}\cdot\x,\w^{(2)}\cdot \x)-A(\w^{(3)}\cdot\x,\w^{(4)}\cdot \x)]\right)^2\right]\;,\label{eq:refrence2}
\end{align}
where the last term corresponds to the variance over the direction $\vec u$. Note that $A$ is a  $2$-dimensional instance that depends only on matching the first $C_1$ moments where $C_1$ is an absolute constant therefore \Cref{eq:refrence2} (as by construction is non zero) is bounded from below by an absolute constant. Therefore, there exists a sufficiently small absolute constant $c_2>0$ such that 
\[
 \E_{\x\sim\mathcal N(\vec 0,\vec I)}\left[\left(\E_{\vec v\cdot \x\sim\mathcal N( 0,1)}[A(\w^{(1)}\cdot\x,\w^{(2)}\cdot \x)-A(\w^{(3)}\cdot\x,\w^{(4)}\cdot \x)]\right)^2\right]\geq c_2\;.
\]
Furthermore, note that the $L_2^2$ of $A(\w^{(1)}\cdot\x,\w^{(2)}\cdot \x)$ and $A(\w^{(3)}\cdot\x,\w^{(4)}\cdot \x)$ distance is a decreasing function with respect to $\theta$, where $\theta=\theta(\w,\vec u)$. Using the fact that the set $\theta\in[0,\pi/2]$ is compact, there exists a sufficiently small absolute constant $c_1>0$ such that if $\theta(\vec u,\w)\geq \pi/2-c_1$, then the $L_2^2$ of $A(\w^{(1)}\cdot\x,\w^{(2)}\cdot \x)$ and $A(\w^{(3)}\cdot\x,\w^{(4)}\cdot \x)$ is at least $c_2/2$. Using the fact that $d$ is sufficiently large the probability of finding $\w$ so that $\theta(\vec u,\w)\leq \pi/2-c_1$ is at least $\exp(-\Omega(d))$.

On the other hand, there exists an SQ Algorithm with only $O(d^2)$ queries can learn the two directions of $f$. With $O(d)$ queries (with accuracy $O(2^{-k}\eps)=O(\eps)$) recovers an $\eps$ approximation to the direction $\w=(\vec v+\vec u)/\sqrt{2}$. Next if we conditioning on the event $\{\w\cdot\x\geq \delta\}$ where $\delta>0$ is chosen so that $A(\w^(1)\cdot\x,\w^{(2)}\cdot\x)$ does not match moments with respect $\vec u$, which such exists and it is an absolute constant as it reweighs the probability mass of the intervals, therefore
\[
\E\big[y\,\x\x^\top\mid\mathbf{w}\cdot\x\geq \delta\big]
\]
reveals an $\eps$-approximation to the direction $\vec u$. This completes the proof.
\end{proof}

\subsection{Moment-Matching Construction}

We now describe the constructions used in the proof. These results build piecewise constant functions and moment-matching distributions that are central to our lower bound.

\begin{proposition}\label{prop:moment-matching-k-stuff}
Let $c>1$ be a universal constant and $k\in \Z_+$.
There exists a distribution $A$ on $\R^2$,  for which the following are true: 
\begin{enumerate}[leftmargin=*]
    \item\label{it:momentmatch} $A$ matches its first $k$ moments with the 2-dimensional standard Gaussian. 
    \item  for all $\x\in \supp(A)$ it holds that $\abs{\x_1},\abs{\x_2}\leq ck$ and
    $A$ is a union of $k+3$ intervals (in each orthogonal direction).
\end{enumerate}
\end{proposition}

\begin{proof}
    Let $x,y$ be two orthogonal directions. Define $A(x,y)=U(x)U(y)/\pr[x\in U]^2$ where $U$ is the set from \Cref{prop:matching1}. Then, we have that $A$ satisfies Items 1 and 2.
\end{proof}

\begin{proposition}
\label{prop:matching1}
For any $k\in\Z_+$ there exists a set $U\subseteq\R$ with $\pr_{x\sim \mathcal{N}(0,1)}[x\in U]=1/2$ such that $U$ is a union of $k$ intervals with $|\max(U)|=O(k)$ and
\[
\E_{y\sim\mathcal{N}(0,1)}\big[y^t\mid y\in U\big]=\E_{y\sim\mathcal{N}(0,1)}[y^t]
\]
for all integers $t=0,\ldots,k$.
\end{proposition}
\begin{proof}
 We first show that there exists a piecewise constant function $f:\R\to\{0,1\}$ that matches the moments of the standard normal distribution and vanishes for $|x|\ge O(k)$. This is formalized in the following lemma.
\begin{lemma}\label{lem:piecewise-small}
     Fix  $k\in Z_+$ and $\eta\in(0,1)$. There exists a $O(k)$-piecewise-constant function $f : \R \to \{0,2\}$ such that $f(x)=0$ for $|x|\geq \Theta(k)$ and for all $t\leq k$ it holds that
    \begin{align*}
          \left|\E_{y \sim \cN(0,1)}[y^t f(y)] - \E_{y \sim \cN(0,1)}[y^t]\right|=0\;.
    \end{align*}
\end{lemma}
Note that because $f(x)$ takes values in $\{0,2\}$, it holds that $\pr[f(x)=2]=1/2$ (which holds from \Cref{lem:piecewise-small} for $t=0$). It also holds that 
\[
 \left|\E_{y \sim \cN(0,1)}[y^t \mid f(y)=2] - \E_{y \sim \cN(0,1)}[y^t]\right|= \left|\E_{y \sim \cN(0,1)}[y^t f(y)]\frac{1}{2 \pr[f(y)=2]} - \E_{y \sim \cN(0,1)}[y^t]\right|=0\;.
\]
Taking $U=\1( f(x)=2)$, we complete the proof of \Cref{prop:matching1}.
\end{proof}
\subsection{Proof of \Cref{lem:piecewise-small}}
To prove \Cref{lem:piecewise-small}, we first show that there is a function that satisfies our properties up to an error and has exponential pieces.
\begin{lemma}\label{lem:piecewise}
     Fix  $k\in Z_+$ and $\delta \gtrsim k$ and $\eta\in(0,1)$. There exists a $(1/\eta+1)^{\poly(k)}$-piecewise-constant function $f_\eta : \R \to \{0,2\}$ such that $f(x)=0$ for $|x|\geq \delta$ and for all $t\leq k$ it holds that
    \begin{align*}
          \left|\E_{y \sim \cN(0,1)}[y^t f(y)] - \E_{y \sim \cN(0,1)}[y^t]\right|\leq \eta\;.
    \end{align*}
\end{lemma}
\begin{proof}
We first show that there exists a continue function that satisfies are properties.
 \begin{claim} \label{cl:explicit_calc}
    Fix  $k\in Z_+$ and $\delta \gtrsim k$. There exists an $f : \R \to [0,2]$ such that $f(x)=0$ for $|x|\geq \delta$ and for all $t\leq k$ it holds that
    \begin{align*}
          \E_{y \sim \cN(0,1)}[y^t f(y) ] = \E_{y \sim \cN(0,1)}[y^t]\;.
    \end{align*}
\end{claim}
We fix the following parameters throughout the proof (where $C$ denotes a sufficiently large absolute constant) $s =  (\eta/(\delta k))^{\poly(k)}$ and we choose $s$ so that there exists $i'\in \Z$ so that $-\delta+i' s=0$.
We partition the subset $[-\delta,\delta]$ in pieces $[is,(i+1)s)$ for $i \in \Z$. We define $\tilde{f}$ to be the following random piecewise-constant function: For each $i \in \{ -1/\delta,\ldots,1/\delta\}$ we let $\tilde{f}(x)$ be constant in the interval $x \in [is,(i+1)s)$, taking the following value:
\begin{align}\label{eq:def_f}
    \tilde{f}(x) = 
    \begin{cases}
        2 , &\text{with probability $p_i:= \int_{is}^{(i+1)s} \phi(x)f(x)\d x/(2\int_{is}^{(i+1)s} \phi(x)\d x$} ) \\
        0, &\text{with probability $1-p_i$} 
    \end{cases}
\end{align}
 and  we define $\tilde{f}(x)=0$ with probability 1 in the entire $(-\infty,-\delta] \cup [\delta,+\infty)$.
Our goal is to show that $\left|\E_{y \sim \cN(0,1)}[y^t \tilde f(y)] - \E_{y \sim \cN(0,1)}[y^t]\right|\leq \eta$.

First note that for any $i\in [-\delta,-\delta+1,\ldots,\delta]$ and $t\leq k$, we have on expectation that 
\begin{align}
   \E_{\tilde f}\left[\E_{y \sim \cN(0,1)}[y^t \tilde f(y)\Ind(y\in(is,(i+1)s))]\right]=\int_{is}^{(i+1)s} \phi(y)f(y)\d y \frac{\int_{is}^{(i+1)s} y^t\phi(y)\d y}{\int_{is}^{(i+1)s} \phi(y)\d y}\;.\label{eq:truncation-step}
\end{align}
Because we have choose $s$ so that no interval $(is,(i+1)s)$ contains $0$, that means that $y^t$ has constant sign inside the integration. Therefore, from mean valued theorem for integrals, we have that that there exists $\xi_1,\xi_2\in (is,(i+1)s)$ so that the following are true: 
\begin{equation*}
    \int_{is}^{(i+1)s} y^t\phi(y)\d y=\xi_1^t\int_{is}^{(i+1)s} \phi(y)\d y \quad \mathrm{and}\quad \int_{is}^{(i+1)s} \phi(y)f(y)y^t\d y=\xi_2^t\int_{is}^{(i+1)s} \phi(y)f(y)\d y\;.
\end{equation*}
Combining above with \Cref{eq:truncation-step}, we have that 
\begin{align*}
   \left|\E_{y \sim \cN(0,1)}[y^t  f(y)\{y\in(is,(i+1)s)\}]-  \E_{\tilde f}\left[\E_{y \sim \cN(0,1)}[y^t \tilde f(y)\{y\in(is,(i+1)s)\}]\right]\right|&=|\xi_1^t-\xi_2^t|\int_{is}^{(i+1)s} \phi(y)f(y)\d y 
   \\&\leq 2s|\xi_1^t-\xi_2^t|
   \\&\leq 2s^2 t\delta^{t-1}\;,
\end{align*}
where in the first inequality we used that $|f(y)\phi(y)|\leq 2$ and in the second we use that $|\xi_1^t-\xi_2^t|=\int_{\xi_1}^{\xi_2}tx^{t-1}\d x$ and $|\xi_{1}|,|\xi_{2}|\leq \delta$.
We start using a probabilistic argument, by taking the expectation and using Fubini's theorem, we have that 
\begin{align*}
   \left| \E_{\tilde f}\left[\E_{y \sim \cN(0,1)}[y^t \tilde f(y)-y^t]\right]\right|&= \left|\E_{\tilde f}\left[\sum_{i=-\delta}^{\delta-1}\E_{y \sim \cN(0,1)}[(y^t \tilde f(y)-y^t)\{y\in(is,(i+1)s)\}]\right]-\E_{y \sim \cN(0,1)}[y^t \{y\not\in(-\delta,\delta\}]\right|
    \\&\leq \left| \E_{y \sim \cN(0,1)}[y^t f(y)-y^t]\right|+\sum_{i=-\delta}^{\delta-1} 2s^2 t\delta^{t-1}=4s^2 t\delta^{t}=O(\eta)\;,
\end{align*}
where we used that by construction $ \E_{y \sim \cN(0,1)}[y^t f(y)-y^t]=0$.
    
\end{proof}

To complete the proof note that from \Cref{lemma:compact}, we have that there exists $f:\R\mapsto\{0,2\}$ that satisfies the properties of $\Cref{lem:piecewise-small}$.

\subsection{Proof of \Cref{cl:explicit_calc}}
\begin{proof}[Proof of \Cref{cl:explicit_calc}]
    Let \[
    f(x)=\left(1+\frac{p(x)\1(x\in[-1,1])}{\phi(x)}\right)\1(|x|\leq \delta)\;,
    \]
    for some appropriate chosen $k$-degree polynomial $p$.
    We have that
    \[
     \E_{y \sim \cN(0,1)}[y^t f(y)]= \E_{y \sim \cN(0,1)}[y^t \1(|y|\leq \delta) ]+\int_{-1}^{1}p(y)y^t\d y\;.
    \]
    We want to chose $p$ so that $\E_{y \sim \cN(0,1)}[y^t f(y) ] = \E_{y \sim \cN(0,1)}[y^t]$, therefore, we want
    \[
    \int_{-1}^{1}p(y)y^t\d y=\E_{y \sim \cN(0,1)}[y^t \1(|y|\geq \delta) ]\;.
    \]
    Equivalently because Legendre polynomials form a complete basis, it suffices to find a polynomial $p$ so that 
        \[
    \int_{-1}^{1}p(y)P_t(y)\d y=\E_{y \sim \cN(0,1)}[P_t(y) \1(|y|\geq \delta) ]\;,
    \]
    for all $t\leq k$. Therefore, we can write $p(y)=\sum_{i=0}^k a_i P_i(y)$ and $a_i = \frac{2i+1}{2} \int_{-1}^1 P_i(y) p(y) \d y$. Hence, we have that 
    \[
    a_i=\frac{2}{2i+1}\E_{y \sim \cN(0,1)}[P_t(y) \1(|y|\geq \delta) ]\;.
    \]
    Because for the Legendre polynomials it holds that  $|P_t(y)|\leq 1$ for all $|y|\leq 1$ in order to have $|p(x)|\leq \phi(1)$ it suffices to show that $\sum_{i=0}^k|a_i|\leq \phi(1)$.
    Using the property that $P_k(x)=2^{-k}\sum_{i=1}^{\lceil k/2 \rceil}\binom{k}{i}\binom{2k-2i}{k}x^{k-2i}$, we have that
    \begin{align*}
        a_t&=\frac{2}{2t+1}\E_{y \sim \cN(0,1)}[P_t(y) \1(|y|\geq \delta) ]=\frac{2}{2t+1}2^{-t}\sum_{i=1}^{\lceil t/2 \rceil}\binom{t}{i}\binom{2t-2i}{k}\E_{y \sim \cN(0,1)}[y^{t-2i}\1(|y|\geq \delta) ]
        \\&\lesssim \frac{\exp(-\delta^2/2)}{2t+1}2^{-t}\sum_{i=1}^{\lceil t/2 \rceil}\binom{t}{i}\binom{2t-2i}{k}\delta^{t-2i}\lesssim\frac{\exp(-\delta^2/2)}{2t+1} P_{t}(\delta)\;.
    \end{align*}
    Furthermore, using the Corollary 5.4 in \cite{DKS17-sq}, we have that $|P_{t}(\delta)|\leq O(\delta^t)$\;. Hence, we have that $|a_i|\lesssim \frac{\exp(-\delta^2/2)\delta^k}{2i+1}$. Choosing $\delta\gtrsim k$, we have that $\sum_{i=0}^{k} |a_i|\leq \phi(1)$ and that completes the proof.

\end{proof}

\subsection{Reducing the Number of Intervals}
The proof is very similar to the proof of \cite{DKZ20} but with the extra condition that we bound the maximum breakpoint of the piecewise constant function. For completeness, we state and add the extra details of the statements.
\begin{proposition}
    \label{prop:main_stuct-reducing}
Let $k,\ell$ be  positive integers with $\ell\geq k+3$ and $a,b\in \R$ with $b>a$. Let $D$ be a continuous distribution over $\R$ and let $\nu_0,\ldots,\nu_{k-1}\in\R$. Fix $b>0$. If for any $\eta>0$ there exists an at most $\ell$-piecewise constant function $g_{\eta}: \R \to \{a,b\}$ with $\{b_1^\eta,b_2^\eta,\ldots,b_{\ell-1}^\eta\}$ breakpoints
such that $|\E_{z \sim D}[g_{\eta}(z)z^t]-\nu_t|\leq\eta$ for every non-negative integer $t<k$ and $|b_1^\eta|,|b_{\ell-1}^\eta|\leq b$, then there exists an at most $(k+3)$-piecewise constant function $f: \R \to \{a,b\}$ with $\{b_1,b_2,\ldots,b_{k+2}\}$ breakpoints
such that $\E_{z \sim D}[f(z)z^t]=\nu_t$, for every non-negative integer $t<k$ and $|b_1|,|b_{k+2}|\leq b$.
\end{proposition}

\begin{proof}
Note that, we can always transform the function $g_\eta:\R\mapsto\{a,b\}$ to a $g_{\eta}':\R\mapsto \{\pm 1\}$ that satisfies similar properties. We define $g_{\eta}'(z)\eqdef (2g_\eta(z)-a-b)/(b-a)$ and let $\nu_t'=2\nu_t/(b-a)+(a+b)/(b-a)\E_{z \sim D}[z^t]$ and $\eta'=\eta(2/(b-a))$. Hence, we have that for any $\eta'>0$, there exists an at most $\ell$-piecewise constant function $g'_{\eta'}: \R \to \{\pm 1\}$
such that $|\E_{z \sim D}[g'_{\eta}(z)z^t]-\nu_t'|\leq\eta'$ for every non-negative integer $t<k$. 
By applying \Cref{lem:main_diff} and \Cref{lemma:compact}, we obtain that there exists an at most $(k+1)$-piecewise constant function $f': \R \to \{\pm 1\}$
such that $\E_{z \sim D}[f'(z)z^t]=\nu'_t$ which keeps the first and last breakpoints the same as the original function, for every non-negative integer $t<k$. By setting $f(z)=(f'(z)(b-a) +a+b)/2$, we complete the proof of \Cref{prop:main_stuct-reducing}.
\end{proof}

\begin{lemma}\label{lemma:compact}
Let $k$ be a positive integer and $b>0$. Let $D$ be a continuous distribution over $\R$ and let $\nu_0,\ldots,\nu_{k-1}\in\R$. If for any $\eta>0$ there exists an at most $(k+3)$-piecewise constant function $g_{\eta}: \R \to \{\pm 1\}$  with $\{b_1^\eta,b_2^\eta,\ldots,b_{k+2}^\eta\}$ breakpoints
such that $|\E_{z \sim D}[g_{\eta}(z)z^t]-\nu_t|\leq\eta$ and $|b_1^\eta|,|b_{k+2}^\eta|\leq b$, for every non-negative integer $t<k$, then there exists an at most $(k+3)$-piecewise constant function $f: \R \to \{\pm 1\}$ with $\{b_1,b_2,\ldots,b_{k+2}\}$ breakpoints
such that $\E_{z \sim D}[f(z)z^t]=\nu_t$, for every non-negative integer $t<k$ and $|b_1|,|b_{k+2}|\leq b$.
\end{lemma}
\Cref{lemma:compact} follows from the above using a compactness argument.

\begin{proof}
Let $p(z)$ be the pdf of $D$.
	For every $\eta>0$, we have that there exists a function $g_{\eta}$ such that $|\E_{z
		\sim D}[f_\eta(z)z^t]-\nu_t|\leq\eta$, for every non-negative integer $t<k$
	and the function $g_{\eta}$ is at most $(k+1)$-piecewise constant.
	Let $\vec M: \mathbb{\overline R}^{k} \mapsto \R^{k}$, where $M_i(\vec
	b)=\sum_{n=0}^{k}(-1)^{n+1}\int_{b_n}^{b_{n+1}} z^i p(z) \d z$ and
	$b_1\leq b_2\leq \ldots \leq b_{k}$, $b_0=-\infty$ and $b_{k+1}=\infty$. Here we assume
	without loss of generality that before the first breakpoint the function is
	negative because  we can always set the first breakpoint to be $-\infty$. It is
	clear that the function $\vec M$ is a continuous map and $\mathbb{\overline R}^{k+1}$ is a compact set,
	thus $\vec M\left(\mathbb{\overline R}^{k+1}\right)$ is a compact set.
	We also have that for every $\eta>0$ there is a point $\vec b\in \mathbb{\overline R}^{k+1}$
	such that $|\vec M(\vec b)\cdot\vec e_i -\nu_i|\leq \eta$ for all $i<k$. Thus, from compactness, we have that there
	exists a point $\vec b^*\in \mathbb{\overline R}^{k+1}$ such that $\vec M(\vec b^*)=\vec 0$.
	This completes the proof.
\end{proof}
The following lemma is analogous to Lemma~3.8 in \cite{DKZ20}. We include its proof for completeness, as our work considers more general distributions and requires specifying exact moment values.
\begin{lemma}\label{lem:main_diff}
Let $m$ and $k$ be positive integers such that $m>k+3$ and $\eta>0$. Let $D$ be a continuous distribution over $\R$ and let $\nu_0,\ldots,\nu_{k-1}\in\R$.
If there exists an $m$-piecewise constant $f:\R \mapsto \{\pm 1\}$
such that $|\E_{z \sim D}[f(z)z^t]-\nu_t|<\eta$ for all non-negative integers $t<k$,
then there exists an at most $(m-1)$-piecewise constant $g :\R \mapsto \{\pm 1\}$
such that $|\E_{z \sim D}[g(z)z^t]-\nu_t|<\eta$ for all non-negative integers $t<k$ and if $\{b_1,b_2,\ldots,b_{m-1}\}$ are the breakpoints of $f$ and $\{b_1',b_2',\ldots,b_{m-2}'\}$ are the breakpoints of $g$, it holds that $b_{1}'=b_{1}$ and $b_{m-2}'=b_{m-1}$.
\end{lemma}

\begin{proof}
Let $p(z)$ be the pdf of $D$.
Let $\{b_1,b_2,\ldots,b_{m-1}\}$ be the breakpoints of $f$, i.e., the points where the function $f$ changes value.
	Then let $F(z_1, z_2, \ldots, z_{m-1},z):\mathbb{\overline R}^m \mapsto \R$ be an $m$-piecewise constant function
	with breakpoints on $z_1, \ldots, z_{m-1}$, where $z_1<z_2< \ldots <z_{m-1}$
	and $F(b_1, b_2, \ldots, b_{m-1},z)=f(z)$.  For simplicity, let $\vec z=(z_1, \ldots, z_{m-1})$
	and  define $M_i(\vec z)= \E_{z \sim D}[F(\vec z, z)z^i]$ and let
	$\vec M(\vec z)=[M_0(\vec z), M_1(\vec z), \ldots M_{k-1}(\vec z)]^T$. It is
	clear from the definition that
	$M_i(\vec z)=\sum_{n=0}^{m-1}\int_{z_n}^{z_{n+1}} F(\vec z, z) z^i p(z) \d z =
	\sum_{n=0}^{m-1}a_n\int_{z_n}^{z_{n+1}} z^ip(z) \d z$,
	where $z_0= -\infty$ and $z_m=\infty$ and $a_n$ is the sign of $F(\vec z,z)$ in the interval $(z_n,z_{n+1})$.
	Note that $a_n=-a_{n+1}$ for every $0\leq n<m$.
	By taking the derivative of $M_i$ in $z_j$, for $0<j<m$, we get that
	\[\frac{\partial}{\partial z_j} M_i(\vec z)= 2a_{j-1} z_j^i p(z_j) \quad \text{and}\quad
	\frac{\partial}{\partial z_j} \vec M(\vec z)= 2a_{j-1} p(z_j) [1, z_j^1, \ldots ,z _j^{k-1}]^T\;.\]
	We now argue that for any $\vec z$ with distinct coordinates that there exists a vector $\vec u\in \R^{m-1}$ such that
	$\vec u=(0,\vec u_2,\ldots,\vec u_{k+1},0,0,\ldots,1,0)$ and the directional derivative of $\vec M$ in the $\vec u$ direction
	is zero. To prove this, we construct a system of linear equations such that
	$\nabla_{\vec u} M_i(\vec z)=0$, for all $0< i<k$. Indeed, we have
$\sum_{j=2}^{k} \frac{\partial}{\partial z_j}  M_i(\vec z) \vec u_j
	= - \frac{\partial}{\partial z_{m-2}}  M_i(\vec z) $ or $\sum_{j=2}^{k} a_{j-1} z_j^i p(z_j)\vec u_j=- a_{m-3} z_{m-2}^i p(z_{m-2})$,
	 which is linear in the variables $\vec u_j$. Let $\widehat{\vec u}$ be the vector with the first $k$ variables 
	 and let $\vec w$ be the vector of the right hand side of the system, i.e., $\vec w_i=- a_{m-3} z_{m-2}^i p(z_{m-2})$. Then
	 this system can be written in matrix form as $\vec V \vec D\widehat{ \vec u}=\vec w$, where $\vec V$ is the Vandermonde matrix,
	 i.e., the matrix that is $\vec V_{i,j}=\alpha_i^{j-1}$, for some values $\alpha_i$ and $\vec D$ is a diagonal matrix.
	 In our case, $\vec V_{i,j}=z_i^{j-1}$ and $\vec D_{j,j}= 2 a_{j-1}p(z_j)$.
	 It is known that the Vandermonde matrix has full rank iff for all $i\neq j$ we have $\alpha_i\neq \alpha_j$,
	 which holds in our setting. Thus, the matrix $\vec V \vec D$ is nonsingular and there exists a solution to the equation.
Thus, there exists a vector $\vec u$ with our desired properties and, moreover,
	any vector in this direction is a solution of this system of linear equations.
	Note that the vector $\vec u$ depends on the value of $\vec z$,
	thus we consider $\vec u(\vec z)$ be the (continuous) function that returns a vector $\vec u$ given $\vec z$.
	
We define a differential equation for the function $\vec v:\mathbb{\overline R}\mapsto\mathbb{\overline R}^{m-1}$, as follows: $\vec v(0)= \vec b$, where $\vec b=(b_1, \ldots, b_{m-1})$, and
$ \vec v'(T)=\vec u(\vec v(T))$ for all $T \in \mathbb{\overline R}$.
If $\vec v$ is a solution to this differential equation, then we have:
\[\frac{\d}{\d T} \vec M(\vec v(T))=\frac{\d}{\d \vec v(T)} \vec M(\vec v(T)) \frac{\d}{\d T} \vec v(T)
=\frac{\d}{\d \vec v(T)} \vec M(\vec v(T)) \vec u(\vec v(T)) =\vec 0\;,
\]
where we used the chain rule and that the directional derivative in $\vec u(\vec v(T))$ direction is zero.
This means that the function $\vec M(\vec v(t))$ is constant, and for all $0\leq j<k$, we have $|M_j-\nu_j|< \eta$, because we have that $|\E_{z \sim D}[F(z_1,\ldots, z_{m-1},z)z^t]-\nu_t|<\eta$. Furthermore, since $\vec u(\vec v(T))$ is continuous in $\vec v(T)$, this differential equation will be well founded and have a solution up until the point where either two of the $z_i$ approach each other or one of the $z_i$ approaches plus or minus infinity (the solution cannot oscillate, since $\vec v_{m-2}'(T)=1$ for all $T$).
	
Running the differential equation until we reach such a limit, we find a limiting value $\vec v^\ast$ of $\vec v(T)$ so that:
 There is an $i$ such that $\vec v_i^\ast=\vec v_{i+1}^\ast$, which
gives us a function that is at most $(m-2)$-piecewise constant, i.e., taking $F(\vec v^\ast,z)$. Which is guaranteed to exist since $\vec v_{m-2}'(T)=1$ for all $T$, thus there exists $T'$ such  that $\vec v_{m-2}^\ast=\vec v_{m-1}^\ast$. 
Thus, we have a function with at most $m-1$ breakpoints and the same moments. Furthermore, because $\vec u_0=\vec u_{m}=0$ the differential equation keeps the first and the last breakpoint constant we have that the first and last breakpoints of the new function are the same as in the original function.
This completes the proof.
\end{proof}

\end{document}